\documentclass[11pt]{article}
\usepackage[margin=1in]{geometry}
\newenvironment{proof}[1][Proof]{\begin{trivlist}
\item[\hskip \labelsep {\bfseries #1}]}{\end{trivlist}}

\usepackage[colorlinks]{hyperref}
\hypersetup{
   bookmarksnumbered,
   pdfstartview={FitH},
   citecolor={black},
   linkcolor={black},
   urlcolor={black},
   pdfpagemode={UseOutlines}
}

\usepackage{url}
\usepackage{exscale}
\usepackage{subfigure}
\usepackage{array}
\usepackage{epsfig}
\usepackage[round]{natbib}
\usepackage{xspace}
\usepackage[usenames]{color}
\usepackage{bm}
\usepackage{amsmath}
\usepackage{amssymb}
\usepackage{stmaryrd}
\usepackage{comment}
\usepackage{fancybox}
\usepackage{multirow}
\usepackage[vlined,linesnumbered,boxruled]{algorithm2e}
\usepackage{relsize}
\usepackage{rotating}

\usepackage{pifont}
\usepackage{hhline}
\newcommand{\ycolor}{\makebox[1.5cm][c]{\color{darkgreen} Yes }}
\newcommand{\ncolor}{\makebox[1.5cm][c]{\color{fullred} No }}

\newcommand{\compN}{\makebox[1.5cm][c]{ $O(\N^2)$ }}
\newcommand{\compLogN}{\makebox[1.5cm][c]{ $O(\N\log \N )$}}
\newcommand{\compLogNN}{\makebox[1.5cm][c]{ $O(\N\log \N)$}}
\newcommand{\compConst}{\makebox[1.5cm][c]{ $O(\N)$ }}

\usepackage{theorem}
\newtheorem{theorem}{Theorem}
\newtheorem{lemma}[theorem]{Lemma}

\newtheorem{corollary}[theorem]{Corollary}
\newtheorem{definition}[theorem]{Definition}
\newtheorem{problem}[theorem]{Problem}
\newtheorem{assumption}[theorem]{Assumption}

\newcommand{\True}{{\tt True}}
\newcommand{\False}{{\tt False}}

\newcommand{\card}[1]{\ensuremath{\operatorname{card}\left( #1\right)}}
\newcommand{\qed}{\hfill$\square$}
\newcommand{\Poisson}{\ensuremath{\mathrm{Poisson}}}  
\newcommand{\Binomial}{\ensuremath{\mathrm{Binomial}}}  
\newcommand{\given}{\,\vert\,} 
\newcommand{\biggiven}{\,\big\vert\,}
\newcommand{\naturals}{\mathbb{N}} 
\newcommand{\reals}{\mathbb{R}}  
\newcommand{\VolumeDBall}[1]{\zeta_{#1}} 
\newcommand{\BetaFunction}[2]{{{\tt Beta} (#1, #2)}}
\newcommand{\GammaFunction}[1]{{{\tt Gamma}(#1)}}
\newcommand{\PP}{\mathbb{P}} 
\newcommand{\EE}{\mathbb{E}} 
\newcommand{\TV}{\ensuremath{\mathrm{TV}}}
\newcommand{\NearNodes}{Near} 
\newcommand{\X}{{\cal X}} 
\newcommand{\Z}{X} 

\newcommand{\LargestComponent}{\ensuremath{N_\mathrm{max}}} 
\newcommand{\RVlength}{\ensuremath{{L}}} 

\newcommand{\prefix}{\pi} 

\newcommand{\N}{n}  
\newcommand{\I}{i}     

\newcommand{\MM}{M}  
\newcommand{\M}{m}     

\newcommand{\BVnorm}{\ensuremath{\mathrm{BV}}}  

\newcommand{\NumObs}{M}

\newcommand{\NumSteps}{N}

\newcommand{\Alg}{\ensuremath{{\mathrm{ALG}}}\xspace}
\newcommand{\AlgPRM}{\ensuremath{{\mathrm{PRM}}}}
\newcommand{\AlgsPRM}{\ensuremath{{\mathrm{sPRM}}}}
\newcommand{\AlgksPRM}{\ensuremath{{k\mbox{-}\mathrm{sPRM}}}}
\newcommand{\AlgRRT}{\ensuremath{{\mathrm{RRT}}}}
\newcommand{\AlgPRMstar}{\ensuremath{{\mathrm{PRM}^*}}}
\newcommand{\AlgkPRMstar}{\ensuremath{{k\mbox{-}\mathrm{PRM}^*}}}
\newcommand{\AlgRRG}{\ensuremath{{\mathrm{RRG}}}}
\newcommand{\AlgkRRG}{\ensuremath{{k\mbox{-}\mathrm{RRG}}}}
\newcommand{\AlgRRTstar}{\ensuremath{{\mathrm{RRT}^*}}}
\newcommand{\AlgkRRTstar}{\ensuremath{{k\mbox{-}\mathrm{RRT}^*}}}

\newcommand{\PoissonApproximation}{\ensuremath{{\cal P}}}

\DeclareMathOperator{\Cl}{cl}

\newcommand{\hideMaterial}[1]{}
\newcommand{\hideoldproofofnegresult}[1]{ }

\definecolor{darkgreen}{rgb}{0,0.5,0} 
\definecolor{fullred}{rgb}{0.85,.0,.1} 
\definecolor{brown}{rgb}{0.65,0.16,0.16} 
\title{Sampling-based Algorithms for  Optimal Motion Planning}
\date{}

\author{Sertac Karaman
  \quad\quad\quad\and \quad\quad\quad Emilio Frazzoli\thanks{The authors are with the Laboratory for
    Information and Decision Systems, Massachusetts Institute of Technology, Cambridge,
    MA.}
}

\begin{document}

\maketitle

\begin{abstract}
During the last decade, sampling-based path planning algorithms, such as Probabilistic RoadMaps (PRM) and Rapidly-exploring Random Trees (RRT), have been shown to work well in practice and possess theoretical guarantees such as probabilistic completeness. However, little effort has been devoted to the formal analysis of the quality of the solution returned by such algorithms, e.g., as a function of the number of samples. The purpose of this paper is to fill this gap, by rigorously analyzing the asymptotic behavior of the cost of the solution returned by stochastic sampling-based algorithms as the number of samples increases. A number of negative results are provided, characterizing existing algorithms, e.g., showing that, under mild technical conditions, the cost of the solution returned by broadly used sampling-based algorithms converges almost surely to a non-optimal value. The main contribution of the paper is the introduction of new algorithms, namely, PRM$^*$ and RRT$^*$, which are provably asymptotically optimal, i.e., such that the cost of the returned solution converges almost surely to the optimum. Moreover, it is shown that the computational complexity of the new algorithms is within a constant factor of that of their probabilistically complete (but not asymptotically optimal) counterparts. The analysis in this paper hinges on novel connections between stochastic sampling-based path planning algorithms and the theory of random geometric graphs.
\end{abstract}


\section{Introduction} \label{section:introduction}

The robotic motion planning problem has received a considerable amount of attention, especially over the last decade, as robots started becoming a vital part of modern industry as well as our daily life~\citep{Latombe:91,lavalle.book06,Choset.Lynch.ea:05}. Even though modern robots may possess significant differences in sensing, actuation, size,
workspace, application, etc., the problem of navigating through a complex environment is embedded
and essential in almost all robotics applications. Moreover, this problem is relevant to
other disciplines such as verification, computational biology, and computer
animation~\citep{latombe.ijrr99, bhatia.frazzoli.hscc04, branicky.curtis.ea.ieeeproc06,
  cortes.jailet.ea.icra07, liu.badler.comp_anim_conf03, finn.kavraki.algorithmica99}.

Informally speaking, given a robot with a description of its dynamics, a description of the
environment, an initial state, and a set of goal states, the motion planning problem is to
find a sequence of control inputs so as the drive the robot from its initial state to one of
the goal states while obeying the rules of the environment, e.g., not colliding with the
surrounding obstacles. An algorithm to address this problem is said to be {\em complete} 
if it terminates in finite time, returning a valid solution if one exists,  and failure otherwise.

Unfortunately, the problem is known to be very hard from the computational point of view. 
For example, a basic version of the motion planning problem, called the generalized piano movers problem, is PSPACE-hard~\citep{reif.sym_foun_com_sci79}. In fact, while complete planning algorithms exist~\citep[see, e.g.,][]{lozanoperez.wesley.comm_acm79,schwartz.sharir.adv_app_math83,canny.book88}, their complexity makes them unsuitable for practical applications. 

Practical planners came around with the development of cell decomposition methods~\citep{brooks.lozanoperez.icai83} and potential fields~\citep{khatib.ijrr86}. These approaches, if properly implemented, relaxed the completeness requirement to, for instance, {\em resolution completeness}, i.e., the ability to return a valid solution, if one exists, if the resolution parameter of the algorithm is set fine enough. These planners demonstrated remarkable performance in accomplishing various tasks in complex environments within reasonable time bounds~\citep{ge.cui.autorobo02}. However, their practical applications were mostly limited to state spaces with up to five dimensions, since decomposition-based methods suffered from large number of cells, and potential field methods from local minima~\citep{koren.borenstein.icra91}. Important contributions towards broader applicability of these methods include navigation functions~\citep{Rimon.Koditschek:92} and  randomization~\citep{barraquand.latombe.ijrr93}. 

The above methods rely on an explicit representation of the obstacles in the configuration space, which is used directly to construct a solution. This may result in an excessive computational burden in high dimensions, and in environments described by a large number of obstacles. Avoiding such a representation is the main underlying idea leading to the development of sampling-based algorithms~\citep{kavraki.latombe.icra94, kavraki.svetska.ea.tro96,lavalle.kuffner.ijrr01}. See ~\citet{lindemann.lavalle.symp_rr05} for a historical perspective. These algorithms proved to be very effective for motion planning in high-dimensional spaces, and  attracted significant attention over the last decade, including very recent work ~\citep[see, e.g.,][]{prentice.roy.ijrr09, tedrake.manchester.ea.ijrr, luders.karaman.ea.acc10, berenson.kuffner.ea.icra08, yershova.lavalle.rep08, stilman.schamburek.ea.icra07, koyuncu.ure.ea.j_intell_robot_syst10}. Instead of using an explicit representation of the environment, sampling-based algorithms rely on a collision checking module, providing information about feasibility of candidate trajectories, and connect a set of points sampled  from the obstacle-free space in order to build a graph (roadmap) of feasible trajectories. The roadmap is then used to construct the solution to the original motion-planning problem. 

Informally speaking, sampling-based methods provide large amounts of computational savings by avoiding explicit construction of obstacles in the state space, as opposed to most complete motion planning algorithms. Even though these algorithms are not complete, they provide {\em probabilistic completeness} guarantees in the sense that the probability that the planner fails to return a solution, if one exists, decays to zero as the number of samples approaches infinity~\citep{barraquand.kavraki.ea.ijrr97} ~\cite[see also][]{hsu.latombe.ea.icra97, kavraki.kolountzakis.ea.tro98, ladd.kavraki.tro04}. Moreover, the rate of decay of the probability of failure is exponential, under the assumption that the environment has good ``visibility'' properties~\citep{barraquand.kavraki.ea.ijrr97}. More recently, the empirical success of sampling-based algorithms was argued to be strongly tied to the hypothesis that most practical robotic applications, even though involving robots with many degrees of freedom, feature environments with such good visibility properties~\citep{hsu.latombe.ea.ijrr06}.

\subsection{Sampling-Based Algorithms}
Arguably, the most influential sampling-based motion planning algorithms to date include Probabilistic RoadMaps (PRMs)~\citep{kavraki.svetska.ea.tro96, kavraki.kolountzakis.ea.tro98}  and Rapidly-exploring Random Trees (RRTs)~\citep{kuffner.lavalle.icra00, lavalle.kuffner.ijrr01,lavalle.book06}.  Even though the idea of connecting points sampled randomly from the state space is essential in both approaches, these two algorithms differ in the way that they construct a graph connecting these points.

The PRM algorithm and its variants are multiple-query methods that first construct a graph (the roadmap), which represents a rich set of collision-free trajectories, and then answer queries by computing a shortest path that connects the initial state with a final state through the roadmap. The PRM algorithm  has been reported to perform well in high-dimensional state spaces~\citep{kavraki.svetska.ea.tro96}. Furthermore,  the PRM algorithm is probabilistically complete, and such that the probability of failure decays to zero exponentially with the number of samples used in the construction of the roadmap~\citep{kavraki.kolountzakis.ea.tro98}.
During the last two decades, the PRM algorithm has been a focus of robotics research: several improvements were suggested by many authors and the reasons to why it performs well in many practical cases were better understood~\citep[see, e.g.,][for some examples]{branicky.lavalle.ea.icra01, hsu.latombe.ea.ijrr06, ladd.kavraki.tro04}.

Even though multiple-query methods are valuable in highly structured environments, such as factory floors, most online planning problems do not require multiple queries, since, for instance, the robot moves from one environment to another, or the environment is not known a priori. Moreover, in some applications, computing a roadmap a priori may be computationally challenging or even infeasible. Tailored mainly for these applications, incremental sampling-based planning algorithms such as RRTs have emerged as an online, single-query counterpart to PRMs~\citep[see, e.g.,][]{kuffner.lavalle.icra00, hsu.kindel.ea.ijrr02}. The incremental nature of these algorithms avoids the necessity to set the number of samples a priori, and returns a solution as soon as the set of trajectories built by the algorithm is rich enough, enabling on-line implementations. Moreover, tree-based planners do not require connecting two states exactly and more easily handle systems with differential constraints.
The RRT algorithm has been shown to be probabilistically complete~\citep{kuffner.lavalle.icra00}, with an exponential rate of decay for the probability of
failure~\citep{frazzoli.dahleh.ea.jgcd02}.
The basic version of the RRT algorithm has been extended in several directions, and found many applications in the robotics domain and elsewhere~\citep[see, for instance,][]{frazzoli.dahleh.ea.jgcd02, bhatia.frazzoli.hscc04, cortes.jailet.ea.icra07, branicky.curtis.ea.ieeeproc06, branicky.curtis.ea.cdc03, zucker.kuffner.ea.icra07}. In particular, RRTs have been shown to work effectively for systems with differential constraints and nonlinear dynamics~\citep{lavalle.kuffner.ijrr01, frazzoli.dahleh.ea.jgcd02} as well as purely discrete or hybrid systems~\citep{branicky.curtis.ea.cdc03}. Moreover, the RRT algorithm was demonstrated in major robotics events on various experimental robotic platforms~\citep{bruce.veloso.lncs02, kuwata.teo.ea.cst09, teller.walter.ea.icra10, shkolnik.levashov.ea.unpub09, kuffner.kagami.ea.autorobo02}.

Other sampling-based planners of note include Expansive Space Trees (EST)~\citep{hsu.latombe.ea.icra97,Hsu.Latombe.ea:IJCGA99} and Sampling-based Roadmap of Trees (SRT)~\citep{Plaku.Bekris.ea:05}. The latter combines the main features of multiple-query algorithms such as PRM with those of single-query algorithms such as RRT and EST.

\subsection{Optimal Motion Planning}

In most applications, the quality of the solution returned by a motion planning algorithm is important. For example, one may be interested in solution paths of minimum cost, with respect to a given cost functional, such as the length of a path, or the time  required to execute it. The problem of computing optimal motion plans has been proven in~\citet{Canny.Reif:87} to be very challenging even in basic cases. 

In the context of sampling-based motion planning algorithms, the importance of computing 
optimal solutions has been pointed out in early seminal papers~\citep{lavalle.kuffner.ijrr01}.
However, optimality properties of sampling-based motion planning algorithms have not been systematically investigated, and  most of the relevant work relies on heuristics. 
For example, in many field implementations of sampling-based  planning algorithms~\citep[see, e.g.,][]{kuwata.teo.ea.cst09}, it is often the case that since a feasible path is found quickly, additional available computation time is devoted to improving the solution with heuristics until the solution is executed. \citet{urmson.simmons.iros03} proposed heuristics to bias the tree growth in RRT towards those regions that result in low-cost solutions. They have also shown experimental results evaluating the performance of different heuristics in terms of the quality of the solution returned. \citet{ferguson.stentz.iros06} considered running the RRT algorithm multiple times in order to progressively improve the quality of the solution. They showed that each run of the algorithm results in a path with smaller cost, even though the procedure is not guaranteed to converge to an optimal solution. Criteria for restarting multiple RRT runs, in a different context, were also proposed in~\citet{wedge.branicky.aaai_ai_conf08}.
A more recent approach is the transition-based RRT (T-RRT) designed to combine rapid exploration properties of the RRT with stochastic global optimization methods~\citep{Jaillet.Cortes.ea:TRO10,Berenson.Simeon.ea:ICRA11}.

A different approach that also offers optimality guarantees is based on graph search algorithms, such as A$^*$, applied over a finite discretization (based, e.g., on a grid, or a cell decomposition of the configuration space)   that is generated offline.
Recently, these algorithms received a large amount of attention. In particular, they were extended to run in an anytime fashion~\citep{likhachev.gordon.ea.nips04, likhachev.ferguson.ea.aij08}, deal with dynamic environments~\citep{stentz.ijcai95, likhachev.ferguson.ea.aij08}, and handle systems with differential constraints~\citep{likhachev.ferguson.ijrr09}.
These have also been successfully demonstrated on various robotic platforms~\citep{likhachev.ferguson.ijrr09, dolgov.thrun.ea.exp_robotics09}.
However, optimality guarantees of these algorithms are only ensured up to the grid resolution. Moreover, since the number of grid points grows exponentially with the dimensionality of the state space, so does the (worst-case) running time of these algorithms.

\subsection{Statement of Contributions}
To the best of the author's knowledge, this paper provides the first systematic and thorough analysis of optimality and complexity properties of the major paradigms for sampling-based path planning algorithms, for multiple- or single-query applications, and introduces 
the first  algorithms that are both asymptotically optimal and 
computationally efficient, with respect to other algorithms in this class. 
A summary of the contributions can be found below, and is shown in Table \ref{table:comparison}.

As a first set of results, it is proven that  the standard PRM and RRT algorithms are not asymptotically optimal, and that  the ``simplified'' PRM algorithm is asymptotically optimal, but computationally expensive. Moreover, it is shown that the $k$-nearest variant of the (simplified) PRM algorithm is not necessarily probabilistically complete (e.g., it is not probabilistically complete for $k=1$), and is not asymptotically optimal for any fixed $k$.

In order to address the limitations of sampling-based path planning algorithms available in the literature, new algorithms are proposed, i.e., PRM$^*$, RRG, and RRT$^*$, 
and proven to be probabilistically complete, asymptotically optimal, and computationally efficient. Of these, PRM$^*$ is a batch variable-radius PRM, applicable to multiple-query problems, in which the radius is scaled with the number of samples in a way that provably ensures both asymptotic optimality and computational efficiency. RRG is an incremental algorithm that builds a connected roadmap, providing similar performance to  PRM$^*$ in a single-query setting, and in an anytime fashion (i.e., a first solution is provided quickly, 
and monotonically improved if more computation time is available). The RRT$^*$ algorithm is a variant of RRG that incrementally builds a tree, providing anytime solutions, provably converging to an optimal solution, with minimal computational and memory requirements.

\begin{table}[htb]
\caption{Summary of results. Time and space complexity are expressed as a function of the number of samples $n$, for a fixed environment.}
\begin{center}
\begin{tabular}{| p{0.6cm} | c || p{1.99cm} | p{1.7cm} | p{1.84cm} | p{1.7cm} | p{1.7cm}| p{1.7cm}| }
\hline 
& \multirow{2}{*}{\makebox[1.5cm][l]{\bf \footnotesize Algorithm}} & \multirow{2}{*}{\parbox[t]{1\linewidth}{\centering \footnotesize Probabilistic Completeness}}& \multirow{2}{*}{\parbox[t]{1\linewidth}{\centering \footnotesize Asymptotic Optimality}} & \multirow{2}{*}{\parbox[t]{1\linewidth}{\centering \footnotesize Monotone Convergence}} & \multicolumn{2}{c|}{\makebox[2cm][l]{\centering \footnotesize Time Complexity}} & \multirow{2}{*}{\parbox[t]{1\linewidth}{\centering \footnotesize Space Complexity}} \\
\hhline{|~|~|~|~|~|-|-|~|}
& & & & & \parbox[t]{1\linewidth}{\centering \footnotesize Processing} & \parbox[t]{1\linewidth}{\centering \footnotesize Query} & \\
\hline \hline 
\multirow{3}{*}{\hspace{-0.05in}\rotatebox{90}{\parbox[t]{2cm}{\centering \bf \footnotesize Existing \\Algorithms}}}
& PRM & \ycolor & \ncolor & \ycolor  &  \compLogN & \compLogN & \compConst \\
\hhline{|~|-|-|-|-|-|-|-|}
& sPRM & \ycolor & \ycolor & \ycolor  &  \compN & \compN & \compN \\
\hhline{|~|-|-|-|-|-|-|-|}
& $k$-sPRM & Conditional   & \ncolor  & \ncolor & \compLogN & \compLogNN & \compConst\\
\hhline{|~|-|-|-|-|-|-|-|}
& RRT & \ycolor  & \ncolor &\ycolor & \compLogN & \compConst & \compConst \\
\hline
\hline
\multirow{6}{*}{\hspace{-0.05in}\rotatebox{90}{\parbox[t]{2.2cm}{\centering \bf \footnotesize Proposed \\ Algorithms}}} & PRM$^*$ & \multirow{2}{*}{\ycolor} & \multirow{2}{*}{\ycolor} & \multirow{2}{*}{\ncolor}   & \multirow{2}{*}{\compLogN} & \multirow{2}{*}{\compLogNN} & \multirow{2}{*}{\compLogNN} \\
\hhline{|~|-|~|~|~|~|~|~|}
& $k$-PRM$^*$ & & & & & & \\
\hhline{|~|-|-|-|-|-|-|-|}
& RRG & \multirow{2}{*}{\ycolor} & \multirow{2}{*}{\ycolor} & \multirow{2}{*}{\ycolor} & \multirow{2}{*}{\compLogN} & \multirow{2}{*}{ \compLogNN} & \multirow{2}{*}{\compLogNN} \\
\hhline{|~|-|~|~|~|~|~|~|}
& $k$-RRG & & & & & & \\
\hhline{|~|-|-|-|-|-|-|-|}
& RRT$^*$ & \multirow{2}{*}{\ycolor} & \multirow{2}{*}{\ycolor} & \multirow{2}{*}{\ycolor} & \multirow{2}{*}{\compLogN} & \multirow{2}{*}{\compConst} & \multirow{2}{*}{\compConst} \\
\hhline{|~|-|~|~|~|~|~|~|}
& $k$-RRT$^*$ & & & & & & \\
\hline
\end{tabular}
\end{center}
\label{table:comparison}
\end{table}%

In this paper, the problem of planning a path through a connected bounded subset of a $d$-dimensional Euclidean space is considered. As in the early seminal papers on incremental sampling-based motion planning algorithms such as~\citet{kuffner.lavalle.icra00}, no differential constraints are considered (i.e., the focus of the paper is on path planning problems), but our methods can be easily extended to planning in configuration spaces and applied to several practical problems of interest. The extension to systems with differential constraints is deferred to future work  (see~\citet{Karaman.Frazzoli:CDC10} for preliminary results).

Finally, the results presented in this article, and the  techniques used in the analysis of the algorithms, hinge on novel connections established between sampling-based path planning algorithms in robotics and the theory of random geometric graphs, which may be of independent interest.

A preliminary version of this article has appeared in~\citet{Karaman.Frazzoli:RSS10}. Since then a variety of new algorithms based on the the ideas behind PRM$^*$, RRG, and RRT$^*$ have been proposed in the literature. For instance, a probabilistically complete and probabilistically sound algorithm for solving a class of differential games has appeared in~\citet{Karaman.Frazzoli:WAFR10}. Algorithms based on the RRG were used to solve belief-space planning problems in~\citet{Bry.Roy:ICRA11}. The RRT$^*$ algorithm was used for anytime motion planning in~\citet{Karaman.Walter.ea:ICRA11}, where it was also demonstrated experimentally on a full-size robotic fork truck. In~\citet{Alterovitz.Patil.ea:ICRA11}, the analysis given in~\citet{Karaman.Frazzoli:RSS10} was used to guarantee computational efficiency and asymptotic optimality of a new algorithm that can trade off between exploration and optimality during planning.

A software library implementing the new algorithms introduced in this paper has been released as open-source software by the authors, and is currently available at 
\url{http://ares.lids.mit.edu/software/}

\subsection{Paper Organization}
This paper is organized as follows. Section~\ref{section:notation} lays the ground in terms of
notation and problem formulation. Section~\ref{section:algorithms} is devoted to the discussion of the algorithms that are considered in the paper: first, the main paradigms for sampling-based motion planning algorithms available in the literature are presented, together with their main variants. Then, the new proposed algorithms are presented and motivated. 
In Section~\ref{section:analysis} the properties of these algorithms are rigorously analyzed, formally establishing their probabilistic completeness and asymptotically optimality (or lack thereof), as well as their computational complexity as a function of the number of samples and of the number of obstacles in the environment. Experimental results are presented in Section~\ref{section:experiments}, to illustrate and validate the theoretical findings. 
Finally, Section~\ref{section:conclusion} contains conclusions and perspectives for future work.  In order not to excessively disrupt the flow of the presentation, a summary of notation used throughout the paper, as well as lengthy proofs of important results are presented in the Appendix.

\section{Preliminary Material} \label{section:notation}
This section contains some preliminary material that will be necessary for the discussion in the remainder of the paper. Namely, the problems of feasible and optimal motion planning is introduced, and some important results from the theory of random geometric graphs are summarized. The notation used in the paper is summarized in Appendix~\ref{appendix:notation}.

\subsection{Problem Formulation} \label{section:problem}

In this section, the feasible and optimal path planning problems are formalized.

Let $\X=(0,1)^d$  be the \emph{configuration space}, where $d \in \naturals$, $d \ge 2$. 
Let $\X_\mathrm{obs}$ be the {\em obstacle region}, such that $\X\setminus \X_\mathrm{obs}$ is an open set, and denote the {\em obstacle-free space} as $\X_\mathrm{free}=\mathrm{cl}(\X\setminus \X_\mathrm{obs})$, where $\mathrm{cl}(\cdot)$ denotes the closure of a set. The {\em initial condition} $x_\mathrm{init}$ is an element of $\X_\mathrm{free}$, and the {\em goal region} $\X_\mathrm{goal}$ is an open subset of $\X_\mathrm{free}$. 
A path planning problem is defined by a triplet $(\X_\mathrm{free},x_\mathrm{init}, \X_\mathrm{goal})$.

Let $\sigma : [0,1] \to \reals^d$; the {\em total variation} of $\sigma$ is defined as
$$
\mathrm{TV}(\sigma) = \sup_{\left\{n \in \naturals, 0 = \tau_0 < \tau_1 < \dots < \tau_n = s \right\}} \sum_{i = 1}^n \vert \sigma(\tau_{i}) - \sigma(\tau_{i-1}) \vert.
$$
A function $\sigma$ with $\mathrm{TV}(\sigma) < \infty$ is said to have {\em bounded variation}.

\begin{definition}[Path] \label{definition:path}
A function $\sigma : [0,1] \to \reals^d$ of bounded variation is called a
\begin{itemize}
\item {\em Path}, if it is continuous;
\item \emph{Collision-free path}, if it is a path, and $\sigma(\tau) \in \X_\mathrm{free}$, for all $\tau \in [0,1]$;
\item \emph{Feasible path}, if it is a collision-free path, $\sigma(0) = x_\mathrm{init}$, and $\sigma(1)  \in \mathrm{cl}(\X_\mathrm{goal})$.
\end{itemize} 
\end{definition}
The total variation of a path is essentially its length, i.e., the Euclidean distance traversed by the path in $\reals^d$. 
The {\em feasibility problem} of path planning is to find a feasible path, if one exists, and report
failure otherwise:
\begin{problem}[Feasible path planning] \label{problem:feasibility} %
Given a path planning problem $(\X_\mathrm{free}, x_\mathrm{init}, \X_\mathrm{goal})$, find a feasible path $\sigma: [0, 1] \to \X_\mathrm{free}$ such that $\sigma(0) = x_\mathrm{init}$ and $\sigma(1) \in \mathrm{cl}(\X_\mathrm{goal})$, if one exists. If no such path exists, report failure.
\end{problem}

Let $\Sigma$ denote the set of all paths, and $\Sigma_\mathrm{free}$ the set of all collision-free paths. Given two paths $\sigma_1, \sigma_2 \in \Sigma$, such that $\sigma_1(1)=\sigma_2(0)$, let $\sigma_1 \vert \sigma_2 \in \Sigma$ denote their concatenation, i.e., $(\sigma_1 \vert \sigma_2)(\tau) := \sigma_1 (2 \, \tau)$ for all $\tau \in [0,1/2]$ and $(\sigma_1\vert\sigma_2)(\tau) := \sigma_2 (2 \, \tau - 1)$ for all $\tau\in (1/2,1]$. Both $\Sigma$ and $\Sigma_\mathrm{free}$ are closed under concatenation.
Let $c : \Sigma \to \mathrm{R}_{\ge 0}$ be a function, called the {\em cost function}, which assigns a strictly positive cost to all non-trivial collision-free paths (i.e., $c(\sigma)=0$ if and only if $\sigma(\tau)=\sigma(0), \forall \tau \in [0,1]$). The cost function is assumed to be {\em monotonic}, in the sense that 
for all $\sigma_1, \sigma_2 \in \Sigma$, $c(\sigma_1) \le c(\sigma_1 | \sigma_2)$, and {\em bounded}, in the sense that there exists $k_c$ such that 
$c(\sigma) \le k_c \mathrm{TV}(\sigma)$, $\forall \sigma \in \Sigma$.

The {\em optimality problem} of path planning asks for finding a feasible path with minimum cost:

\begin{problem}[Optimal path planning] \label{problem:optimality} %
Given a path planning problem $(\X_\mathrm{free}, x_\mathrm{init}, \X_\mathrm{goal})$ and a cost function $c: \Sigma \to \reals_{\ge 0}$, find a feasible path $\sigma^*$ such that  $c(\sigma^*) = \min \{c(\sigma): \sigma \mbox{ is feasible}\}$. If no such path exists, report failure.
\end{problem}

\subsection{Random Geometric Graphs} \label{section:rgg}
The objective of this section is to summarize some of the results on random geometric graphs that are available in the literature, and are relevant to the analysis of sampling-based path planning algorithms. In the remainder of this article, several connections are made between the theory of random geometric graphs and path-planning algorithms in robotics, providing insight on a number of issues, including, e.g., probabilistic completeness and asymptotic optimality, as well as technical tools to analyze the algorithms and establish their properties. In fact, it turns out that the data structures constructed by most sampling-based motion planning algorithms in the literature coincide, in the absence of obstacles, with standard models of random geometric graphs.

Random geometric graphs are in general defined as stochastic collections of points in a metric space, connected pairwise by edges  if certain conditions (e.g., on the distance between the points) are satisfied. Such objects have been studied since their introduction by \citet{Gilbert:61}; see, e.g.,~\citet{penrose.book03} and ~\citet{Balister.Bollobas.ea:09} for an overview of recent results. From the theoretical point of view, the study of random geometric graphs makes a connection between random graphs~\citep{Bollobas:01} and percolation theory~\citep{Bollobas.Riordan:06}.  On the application side, in recent years, random geometric graphs have attracted significant attention  as models of {\em ad hoc} wireless networks~\citep{Gupta.Kumar:98,Gupta.Kumar:00}. 

Much of the literature on random geometric graphs deals with infinite graphs defined on unbounded domains, with vertices generated as a homogeneous Poisson point process. Recall that  a Poisson random variable of parameter $\lambda \in \reals_{> 0}$ is an integer-valued random variable $\Poisson(\lambda): \Omega \to \naturals_0$ such that $\PP(\Poisson(\lambda) = k) = e^{-\lambda} \lambda^k/k!$. A homogeneous Poisson point process of intensity $\lambda$ on $\mathbb{R}^d$ is a random countable set of points $\mathcal{P}^d_\lambda \subset \mathbb{R}^d$ such that, for any disjoint measurable sets $\mathcal{S}_1,\mathcal{S}_2 \subset \mathbb{R}^d$, $\mathcal{S}_1 \cap \mathcal{S}_2 = \emptyset$, the numbers of points of $\mathcal{P}^d_\lambda$ in each set are independent Poisson variables, i.e., $\card{\mathcal{P}^d_\lambda \cap \mathcal{S}_1} =\Poisson(\mu(\mathcal{S}_1) \lambda)$ and $\card{\mathcal{P}^d_\lambda \cap \mathcal{S}_2} = \Poisson(\mu(\mathcal{S}_2) \lambda)$.
In particular, the intensity of a homogeneous Poisson point process can be interpreted as the expected number of points generated in the unit cube, i.e., $\EE(\card{\mathcal{P}^d_\lambda \cap (0,1)^d}) = \EE(\Poisson(\lambda))= \lambda$.

Perhaps the most studied model of infinite random geometric graph is the following, introduced in~\citet{Gilbert:61}, and often called Gilbert's disc model, or Boolean model: 
\begin{definition}[Infinite random $r$-disc graph] Let $\lambda, r \in \reals_{>0}$, and $d \in \naturals$. An infinite random $r$-disc graph $G_\infty^\mathrm{disc}(\lambda, r)$ in $d$ dimensions is an infinite graph with vertices $\{\Z_i\}_{i \in \naturals} = \mathcal{P}^d_\lambda$, and such that $(\Z_i, \Z_j)$, $i,j \in \naturals$, is an edge  if and only if $\|\Z_i-\Z_j\| < r$. 
\end{definition}

A fundamental issue in infinite random graphs is whether the graph contains an infinite connected component, with non-zero probability. If it does, the random graph is said to {\em percolate}. Percolation is an important paradigm in statistical physics, with many applications in disparate fields such as  material science, epidemiology, and microchip manufacturing, just to name a few~\citep[see, e.g.,][]{sahimi.book94}. 

Consider the infinite random $r$-disc graph, for $r=1$, i.e., $G_\infty^\mathrm{disc}(\lambda, 1)$, and assume, without loss of generality, that the origin is one of the vertices of this graph. Let $p_k(\lambda)$ denote the probability that the connected component of $G_\infty^\mathrm{disc}(\lambda, 1)$ containing the origin contains $k$ vertices, and define $p_\infty(\lambda)$ as  $p_\infty (\lambda) = 1 - \sum_{k = 1}^{\infty} p_k (\lambda)$. The function $p_\infty: \lambda \to p_\infty(\lambda)$ is monotone, and  $p_\infty (0) = 0$ and $\lim_{\lambda \to \infty} p_\infty (\lambda) = 1$~\citep{penrose.book03}. A key result in percolation theory is that there exists a non-zero  {\em critical intensity}  $\lambda_\mathrm{c}$ defined as $\lambda_\mathrm{c} := \sup\{\lambda: p_\infty(\lambda) = 0\}$. In other words, for all $\lambda > \lambda_\mathrm{c}$,  there is a non-zero probability that the origin is in an infinite connected component of $G_\infty^\mathrm{disc}(\lambda, 1)$; moreover, under these conditions, the graph has precisely one infinite connected component, almost surely~\citep{meester.roy.book96}.   The function $p_\infty$ is continuous for all $\lambda \neq \lambda_\mathrm{c}$: in other words, the graph undergoes a phase transition at the critical density $\lambda_\mathrm{c}$, often also called the continuum percolation threshold~\citep{penrose.book03}. The exact value of $\lambda_\mathrm{c}$ is not known; Meester and Roy provide $0.696 < \lambda_c <  3.372$ for $d=2$~\citep{meester.roy.book96}, and simulations suggest that $\lambda_c \approx 1.44$~\citep{Quintanilla.Torquato.ea:00}.

For many applications, including the ones in this article,  models of finite graphs on a bounded domain are more relevant. Penrose introduced the following model~\citep{penrose.book03}:
\begin{definition}[Random $r$-disc graph]
Let $r \in \reals_{>0}$, and $\N, d \in \naturals$. A random $r$-disc graph $G^\mathrm{disc}(\N, r)$ in $d$ dimensions is a graph whose $\N$ vertices, $\{\Z_1, \Z_2, \ldots, \Z_\N\}$, are independent, uniformly distributed random variables in $(0,1)^d$, and such that $(\Z_i, \Z_j)$, $i,j \in \{1, \ldots, \N\}$, $i \neq j$, is an edge  if and only if $\|\Z_i-\Z_j\| < r$. 
\end{definition}
For finite random geometric graph models,  one is typically interested in whether a random geometric graph possesses certain properties asymptotically as $\N$ increases. 
Since the number of vertices is finite in random graphs, percolation can not be defined easily. In this case, percolation is studied in terms of the scaling of the number of vertices in the largest connected component with respect to the total number of vertices; in particular, a finite random geometric graph is said to percolate if it contains a ``giant''
 connected component containing at least a constant fraction of all the nodes. As in the infinite case, percolation in finite random geometric graphs is often a phase transition phenomenon. In the case of random $r$-disc graphs, 
\begin{theorem}[Percolation of random $r$-disc graphs~\citep{penrose.book03}]
Let $G^\mathrm{disc}(\N,r)$ be a random $r$-disc graph in $d \ge 2$ dimensions, and let $\LargestComponent(G^\mathrm{disc}(\N,r))$ be the number of vertices in its largest connected component.  Then,  almost surely,
$$\lim_{\N \to \infty} \frac{\LargestComponent(G^\mathrm{disc}(\N,r_n))}{\N} = 0, \qquad \mbox{ if } r_n < \left(\lambda_\mathrm{c}/{\N}\right)^{1/d}, $$
and 
$$\lim_{\N \to \infty} \frac{\LargestComponent(G^\mathrm{disc}(\N,r))}{\N} > 0, \qquad \mbox{ if } r_n > \left(\lambda_\mathrm{c}/{\N}\right)^{1/d},$$
where $\lambda_\mathrm{c}$ is the continuum percolation threshold. 
\end{theorem}
A random $r$-disc graph with $\lim_{n \to \infty} n r_n^d = \lambda \in (0,\infty)$ is said to operate in the {\em thermodynamic limit}. It is said to be in {\em subcritical} regime when $\lambda < \lambda_c$ and {\em supercritical} regime when $\lambda > \lambda_c$.

Another property of interest is connectivity. Clearly, connectivity implies percolation. Interestingly, emergence of connectivity in random geometric graphs is a phase transition phenomenon, as percolation. The following result is available in the literature:

\begin{theorem}[Connectivity of random $r$-disc graphs~\citep{penrose.book03}]\label{theorem:penrose}
Let $G^\mathrm{disc}(\N,r)$ be a random $r$-disc graph in $d$ dimensions.  Then,
$$\lim_{\N \to \infty} \PP\left( \{G^\mathrm{disc}(\N,r) \mbox{ is connected } \} \right) = 
\left\{ \begin{array}{ll} 
1, & \mbox{ if } \zeta_d r^d> \log(\N)/\N,\\[1ex]
0, & \mbox{ if } \zeta_d r^d< \log(\N)/\N,\\
 \end{array}\right.$$
 where $\zeta_d$ is the volume of the unit ball in $d$ dimensions.
\end{theorem}

Another model of random geometric graphs considers edges between $k$ nearest neighbors. (Note that there are no ties, almost surely.) Both infinite and finite models are considered, as follows. 
\begin{definition}[Infinite random $k$-nearest neighbor graph]
Let $\lambda \in \reals_{>0}$, and  $d, k \in \naturals$. 
An infinite random $k$-nearest neighbor graph $G_\infty^\mathrm{near}(\lambda, k)$ in $d$ dimensions is an infinite graph with vertices $\{\Z_i\}_{i \in \naturals} = \mathcal{P}^d_\lambda$, and such that $(\Z_i, \Z_j)$, $i,j \in \naturals$, is an edge  if $\Z_j$ is among the $k$ nearest neighbors of $\Z_i$, or if $\Z_i$ is among the $k$ nearest neighbors of $\Z_j$. 
\end{definition}
\begin{definition}[Random $k$-nearest neighbor graph]
Let $d, k, \N \in \naturals$. 
A random $k$-nearest neighbor graph $G^\mathrm{near}(\N,k)$ in $d$ dimensions is a graph whose $\N$ vertices, $\{\Z_1, \Z_2, \ldots, \Z_\N\}$, are independent, uniformly distributed random variables in $(0,1)^d$, and such that $(\Z_i, \Z_j)$, $i,j \in \{1, \ldots, \N\}$, $i \neq j$,  is an edge  if $\Z_j$ is among the $k$ nearest neighbors of $\Z_i$, or if $\Z_i$ is among the $k$ nearest neighbors of $\Z_j$. 
\end{definition}

Percolation and connectivity for random $k$-nearest neighbor graphs exhibit phase transition phenomena, as in the random $r$-disc case. However, the results available in the literature are more limited. Results on percolation are only available for infinite graphs: 
\begin{theorem}[Percolation in infinite random $k$-nearest graphs~\citep{Balister.Bollobas.ea:09}]
Let $G^\mathrm{near}_\infty(\lambda, k)$ be an infinite random $k$-nearest neighbor graph in $d \ge 2$ dimensions. Then, there exists a constant $k^\mathrm{p}_d> 0$ such that 
$$\PP \left( \left\{ G^\mathrm{near}_\infty(1, k) \mbox{ has an infinite component } \right\} \right) = \left\{ \begin{array}{ll}
1, & \mbox{ if } k \ge k^\mathrm{p}_d,\\[1ex]
0, & \mbox{ if } k < k^\mathrm{p}_d.
\end{array} \right.$$
\end{theorem}
The value of $k^\mathrm{p}_d$ is not known. However, it is believed that $k^\mathrm{p}_2 = 3$, and $k^\mathrm{p}_d =2$ for all $d \ge 3$~\citep{Balister.Bollobas.ea:09}. It is known that percolation does not occur for $k = 1$~\citep{Balister.Bollobas.ea:09}.

Regarding connectivity of random $k$-nearest neighbor graphs, the only available results in the literature are not stated in terms of a given number of vertices: rather, the results are stated in terms of the restriction of a homogeneous Poisson point process to the unit cube. 
In other words, the vertices of the graph are obtained as $\{\Z_1, \Z_2, \ldots\} = \mathcal{P}^d_\lambda  \cap (0,1)^d$. This is equivalent to setting the number of vertices as a Poisson random variable of parameter $\N$,  and then sampling the $\Poisson(\N)$ vertices independently and uniformly in $(0,1)^d$:

\begin{lemma}[\citet{stoyan.kendall.ea.book95}] \label{lemma:poissonization}
Let $\{\Z_i \}_{i \in \naturals}$ be a sequence of points drawn independently and uniformly from $\mathcal{S} \subseteq \X$. Let $\Poisson(\N)$ be a Poisson random variable with parameter $\N$. Then, $\{\Z_1, \Z_2, \dots, \Z_{\Poisson(n)}\}$ is the restriction to $\mathcal{S}$ of a homogeneous Poisson point process with intensity $\N/\mu(\mathcal{S})$.
\end{lemma}

The main advantage in using such a model to generate the vertices of a random geometric graph is independence: in the Poisson  case, the numbers of points in any two disjoint measurable regions $\mathcal{S}_1, \mathcal{S}_2 \subset [0,1]^d$, $\mathcal{S}_1 \cap \mathcal{S}_2 = \emptyset$,  are independent Poisson random variables, with mean $\mu(\mathcal{S}_1)\lambda$ and $\mu(\mathcal{S}_2)\lambda$, respectively. These two random variables would not be independent if the total number of vertices were fixed a priori (also called a binomial point process). With some abuse of notation, such a random geometric graph model will be indicated as $G^\mathrm{near}(\Poisson(n), k)$. 

\begin{theorem}[Connectivity of random $k$-nearest graphs~\citep{Balister.Bollobas.ea:09b,Xue.Kumar:04}]\label{theorem:kumar}
Let $G^\mathrm{near}(\Poisson(n), k)$ indicate a $k$-nearest neighbor graph model in $d = 2$ dimensions, such that its vertices are generated using a Poisson point process of intensity $n$. Then, there exists a constant $k^\mathrm{c}_2> 0$ such that 
$$\lim_{\N \to \infty} \PP\left( \left\{ G^\mathrm{near}(\Poisson(\N), \lfloor k \log(\N) \rfloor) \mbox{ is connected } \right\} \right) = \left\{ \begin{array}{ll}
1, & \mbox{ if } k \ge k^\mathrm{c}_2,\\[1ex]
0, & \mbox{ if } k < k^\mathrm{c}_2.
\end{array} \right.$$
\end{theorem}
The value of $k^\mathrm{c}_2$ is not known; the current best estimate is 
$0.3043 \le k^\mathrm{c}_2 \le 0.5139$~\citep{Balister.Bollobas.ea:05}.

Finally, the last model of random geometric graph that will be relevant for the analysis of the algorithms in this paper is the following: 

\begin{definition}[Online nearest neighbor graph]
Let $d, \N \in \naturals$. 
An online nearest neighbor graph $G^\mathrm{ONN}(\N)$ in $d$ dimensions is a graph whose $\N$ vertices, $(\Z_1, \Z_2, \ldots, \Z_\N)$, are independent, uniformly distributed random variables in $(0,1)^d$, and such that $(\Z_i, \Z_j)$, $i,j \in \{1, \ldots, \N\}$, $j>1$, is an edge if and only if $\|\Z_i -\Z_j\|  = \min_{1\le k<j}  \|\Z_k-\Z_j\|$. 
\end{definition}
Clearly, the online nearest neighbor graph is connected by construction, and trivially percolates. Recent results for this random geometric graph model include estimates of the total power-weighted edge length and an analysis of the vertex degree distribution, see, e.g.,~\citet{Wade:09}.

\section{Algorithms} \label{section:algorithms} 

In this section, a number of sampling-based motion planning algorithms are introduced. 
First, some common primitive procedures are defined. Then, the PRM and the RRT  algorithms are outlined, as they are representative of the major paradigms  for sampling-based motion planning algorithms in the literature. Then, new algorithms, namely PRM$^*$ and RRT$^*$, are introduced, as asymptotically optimal and computationally efficient versions of their ``standard" counterparts.  

\subsection{Primitive Procedures} \label{section:algorithms:primitive_procedures}
Before discussing the algorithms, it is convenient to introduce the primitive procedures that they rely on.

\paragraph{Sampling:} Let  $\mathtt{Sample}: \omega \mapsto \{\mathtt{Sample}_i(\omega)\}_{i \in \naturals_0} \subset \X$ be a map from $\Omega$ to sequences of points in $\X$, such that the random variables $\mathtt{Sample}_i$, $i \in \naturals_0$, are independent and identically distributed (i.i.d.). For simplicity, the samples are assumed to be drawn from a uniform distribution, even though results extend naturally to any absolutely continuous distribution with density bounded away from zero on $\X$. It is convenient to consider another map, $\mathtt{SampleFree}: \omega \mapsto \{\mathtt{SampleFree}_i(\omega)\}_{i\in \naturals_0} \subset \X_\mathrm{free}$  that returns sequences of i.i.d. samples from $\X_\mathrm{free}$. For each $\omega \in \Omega$, the sequence $\{ \mathtt{SampleFree}_i(\omega)\}_{i \in \naturals_0}$ is the subsequence of $\{ \mathtt{Sample}_i (\omega)\}_{i \in \naturals_0}$ containing only the samples in $\X_\mathrm{free}$, i.e., $\{\mathtt{SampleFree}_i(\omega)\}_{i \in \naturals_0} = \{\mathtt{Sample}_i(\omega)\}_{i \in \naturals_0} \cap \X_\mathrm{free}$.

\paragraph{Nearest Neighbor:} Given a graph $G = (V,E)$, where $V \subset \X$, a point $x \in
\X$ , the function ${\tt Nearest} : (G, x) \mapsto v \in V$ returns the vertex in $V$ that is
``closest'' to $x$ in terms of a given distance function. In this paper, the Euclidean distance is used (see, e.g.,~\citet{lavalle.kuffner.ijrr01} for alternative choices), and 
hence
$${\tt Nearest} (G = (V,E), x) := \mathrm{argmin}_{v \in V} \Vert x-v \Vert.$$
A set-valued version of this function is also considered, ${\tt kNearest} : (G, x,k) \mapsto \{v_1, v_2, \ldots, v_k\}$, returning the $k$ vertices in $V$ that are nearest to $x$, according to the same distance function as above. (By convention, if the cardinality of $V$ is less than $k$, then the function returns $V$.)

\paragraph{Near Vertices:} Given a graph $G = (V, E)$, where $V \subset \X$, a point $x \in
\X$, and a positive real number  $r \in \reals_{>0}$, the function ${\tt \NearNodes}: (G, x, r) \mapsto
V'\subseteq V$ returns the vertices in $V$ that are contained in a ball of radius $r$ centered 
at $x$, i.e., 
$$\mathtt{Near}(G = (V,E), x, r) := \left\{v \in V: v \in \mathcal{B}_{x,r} \right\}.$$

\paragraph{Steering:} Given two points $x, y \in \X$, the function ${\tt Steer} : (x,y) \mapsto z$ returns
a point $z \in \X$ such that $z$ is ``closer'' to $y$ than $x$ is. Throughout the paper,
the point $z$ returned by the function ${\tt Steer}$ will be such that $z$ minimizes $\Vert z - y
\Vert$ while at the same time maintaining $\Vert z - x \Vert \le \eta$, for a prespecified $\eta >
0$,\footnote{This steering procedure is used widely in the robotics literature, since its
  introduction in~\citet{kuffner.lavalle.icra00}. Our results also extend to the Rapidly-exploring
  Random Dense Trees~\citep[see, e.g.,][]{lavalle.book06}, which are slightly modified versions of the
  RRTs that do not require tuning any prespecified parameters such as $\eta$ in this case.} i.e.,
$$
{\tt Steer} (x, y) := \displaystyle \mathrm{argmin}_{z \in \mathcal{B}_{x, \eta}} 
\Vert z - y \Vert.
$$

\paragraph{Collision Test:} Given two points $x,x' \in \X$, the Boolean function ${\tt
  CollisionFree} (x,x')$ returns $\True$ if the line segment between $x$ and $x'$ lies in
$\X_\mathrm{free}$, i.e., $[x, x'] \subset \X_\mathrm{free}$, and $\False$ otherwise.

\subsection{Existing Algorithms}\label{section:oldalgo}
Next, some of the sampling-based algorithms available in the literature are outlined.
For convenience, inputs and outputs of the algorithms are not shown explicitly, but are as follows. 
All algorithms take as input a path planning problem $(\X_\mathrm{free}, x_\mathrm{init}, \X_\mathrm{goal})$, an integer $\N \in \naturals$, and a cost function $c: \Sigma \to \reals_{\ge 0}$, if appropriate. These inputs are shared with functions and procedures called within the algorithms. All algorithms return a graph $G=(V,E)$, where $V \subset \X_\mathrm{free}$, $\card{V} \le \N+1$,  and $E \in \mathrm{V} \times\mathrm{V}$. The solution of the path planning problem can be easily computed from such a graph, e.g., using standard shortest-path algorithms. 

\paragraph{Probabilistic RoadMaps (PRM):}
The Probabilistic RoadMaps algorithm is primarily aimed at multi-query applications. In its basic version, it consists of a pre-processing phase, in which a roadmap is constructed by attempting connections among $\N$ randomly-sampled points in $\X_\mathrm{free}$, and a query phase, in which paths connecting initial and final conditions through the roadmap are sought. 
``Expansion'' heuristics for enhancing the roadmap's connectivity are available in the literature~\citep{kavraki.svetska.ea.tro96} but have no impact on the analysis in this paper, and will not be discussed.

The pre-processing phase, outlined in Algorithm \ref{algorithm:PRM}, begins with an empty graph. At each iteration, a point  $x_\mathrm{rand} \in \X_\mathrm{free}$ is sampled, and added to the vertex set $V$. Then, connections are attempted between $x_\mathrm{rand}$ and other vertices in $V$ within a ball of radius $r$ centered at $x_\mathrm{rand}$, in order of increasing distance from $x_\mathrm{rand}$, using a simple local planner (e.g., straight-line connection). Successful (i.e., collision-free) connections  result in the addition of a new edge to the edge set $E$.  To avoid unnecessary computations (since the focus of the algorithm is establishing connectivity), connections between $x_\mathrm{rand}$ and vertices in the same connected component are avoided. Hence, the roadmap constructed by PRM is a forest, i.e., a collection of trees.

\begin{algorithm}
$V \leftarrow \emptyset$; $E \leftarrow \emptyset$\;
\For {$i=0,\ldots, \N$} {
	$x_\mathrm{rand} \leftarrow {\tt SampleFree}_i$\; 
		$U \leftarrow \mathtt{Near}(G=(V,E), x_\mathrm{rand}, r)$ \label{line:PRMneighbors}\;
		$V \leftarrow V \cup \{x_\mathrm{rand}\}$\; 
		\ForEach{$u \in U$, in order of increasing $\|u-x_\mathrm{rand}\|$,}{
		\If {$x_\mathrm{rand}$ and $u$ are not in the same connected component of $G=(V,E)$ 
				\label{line:connected_component_check}}
		{\lIf{ $\mathtt{CollisionFree}(x_\mathrm{rand},u)$} {
		$E \leftarrow E \cup \{(x_\mathrm{rand}, u), (u, x_\mathrm{rand})\}$\;
		}}}}
\Return {$G=(V,E)$}\;
  \caption{PRM (preprocessing phase)}
  \label{algorithm:PRM}
\end{algorithm}

Analysis results in the literature are only available for a ``simplified'' version of the PRM algorithm~\citep{kavraki.kolountzakis.ea.tro98}, referred to as sPRM in this paper. The simplified algorithm initializes the vertex set with the initial condition, samples $\N$ points from $\X_\mathrm{free}$, 
and then attempts to connect points within a distance $r$, i.e., using a similar logic as PRM, with the difference that connections between vertices in the same connected component are allowed. Notice that in the absence of obstacles, i.e., if $\X_\mathrm{free} = \X$, the roadmap constructed in this way is a random $r$-disc graph.

\begin{algorithm}
$V \leftarrow \{x_\mathrm{init}\} \cup \{\mathtt{SampleFree}_i\}_{i=1, \ldots, \N}$; $E \leftarrow \emptyset$\;
\ForEach {$v \in V$}{
	$U \leftarrow \mathtt{Near}(G=(V,E), v, r) \setminus \{v\}$\label{line:sPRMneighbors}\;
	\ForEach{$u \in U$}{
		\lIf{$\mathtt{CollisionFree}(v,u)$}{
			$E \leftarrow E \cup \{(v,u), (u,v)\}$
			}}}
\Return {$G=(V,E)$}\;
  \caption{sPRM}
  \label{algorithm:sPRM}
\end{algorithm}

Practical implementation of the (s)PRM algorithm have often considered different choices for the 
set $U$ of vertices to which connections are attempted (i.e., line \ref{line:PRMneighbors} in Algorithm \ref{algorithm:PRM}, and line \ref{line:sPRMneighbors} in Algorithm \ref{algorithm:sPRM}). In particular, the following criteria are of particular interest: 
\begin{itemize}
\item {\bf $k$-Nearest  (s)PRM}: Choose the nearest $k$ neighbors to the vertex under consideration, for a given $k$ (a typical value is reported as $k=15$~\citep{lavalle.book06}). In other words,  $U \leftarrow \mathtt{kNearest}(G=(V,E),x_\mathrm{rand},k)$ in line \ref{line:PRMneighbors} of Algorithm \ref{algorithm:PRM} and 
	$U \leftarrow \mathtt{kNearest}(G=(V,E), v, k) \setminus \{v\}$ 
	in line \ref{line:sPRMneighbors} of Algorithm 	\ref{algorithm:sPRM}. The roadmap constructed in this way in an obstacle-free environment is a random $k$-nearest graph. 

\item {\bf Bounded-degree (s)PRM}: For any fixed $r$, the average number of connections attempted at each iteration is proportional to the number of vertices in $V$, and can result in an excessive computational burden for large $\N$. To address this,  an upper bound
$k$ can be imposed on the cardinality of the set $U$ (a typical value is reported as $k=20$~\citep{lavalle.book06}). In other words, 
$U \leftarrow \mathtt{Near}(G,x_\mathrm{rand},r) \cap \mathtt{kNearest}(G,x_\mathrm{rand},k)$ in line \ref{line:PRMneighbors} of Algorithm \ref{algorithm:PRM}, and 
$U \leftarrow (\mathtt{Near}(G,v,r) \cap \mathtt{kNearest}(G,v,k)) \setminus \{v\}$ in line \ref{line:sPRMneighbors} of Algorithm \ref{algorithm:sPRM}.
\item {\bf Variable-radius (s)PRM}: Another option to maintain the degree of the vertices in the roadmap small is to make the connection radius $r$ a function of $\N$, as opposed to a fixed parameter. However, there are no clear indications in the literature on the appropriate functional relationship between $r$ and  $\N$. 
\end{itemize}

\paragraph{Rapidly-exploring Random Trees (RRT):}

The Rapidly-exploring Random Tree algorithm is primarily aimed at single-query applications. 
In its basic version, the algorithm incrementally builds a tree of feasible trajectories, rooted at the initial condition. An outline of the algorithm is given in Algorithm \ref{algorithm:RRT}. 
The algorithm is initialized  with a graph that includes the initial state as its single vertex, and no edges. At each iteration, a point $x_\mathrm{rand} \in \X_\mathrm{free}$ is sampled. An attempt is made to connect the nearest vertex $v\in V$ in the tree to the new sample.  If such a connection is successful, $x_\mathrm{rand}$ is added to the vertex set, and $(v, x_\mathrm{rand})$ is added to the edge set. In the original version of this algorithm, the iteration is stopped as soon as the tree contains a node in the goal region. In this paper, for consistency with the other algorithms (e.g., PRM),  the iteration is performed $\N$ times. In the absence of obstacles, i.e., if $\X_\mathrm{free}=\X$, the tree constructed in this way is an online nearest neighbor graph.

\begin{algorithm}
  $V \leftarrow \{ x_\mathrm{init}\}$; $E \leftarrow \emptyset$\; 
  \For {$i =1, \ldots, \N$} { \label{line:iteration_start}
    $x_\mathrm{rand} \leftarrow {\tt SampleFree}_i$\;
    $x_\mathrm{nearest} \leftarrow \mathtt{Nearest}(G=(V,E),x_\mathrm{rand})$\; 
    $x_\mathrm{new} \leftarrow \mathtt{Steer}(x_\mathrm{nearest},x_\mathrm{rand})$ \;
    \If{$\mathtt{ObtacleFree}(x_\mathrm{nearest},x_\mathrm{new})$}{
    $V \leftarrow V \cup \{x_\mathrm{new}\}$;
    $E \leftarrow E \cup \{(x_\mathrm{nearest},x_\mathrm{new})\}$ \;
    }}
\Return {$G = (V,E)$}\;

  \caption{RRT}
  \label{algorithm:RRT}
\end{algorithm}

A variant of RRT consists of growing two trees, respectively rooted at the initial state, and at a state in the goal set. To highlight the fact that the sampling procedure must not necessarily be stochastic, the algorithm is also referred to as Rapidly-exploring Dense Trees (RDT)~\citep{lavalle.book06}.

\subsection{Proposed algorithms}\label{section:newalgo}
In this section, the new algorithms considered in this paper are presented. These algorithms are proposed as asymptotically optimal and computationally efficient versions of their ``standard'' counterparts, as will be made clear through the analysis in the next section. 
Input and output data are the same as in the algorithms introduced in Section \ref{section:oldalgo}. 

\paragraph{Optimal Probabilistic RoadMaps (PRM$^*$):} In the standard PRM algorithm, as well as in its simplified ``batch'' version considered in this paper, connections are attempted between roadmap vertices that are within a fixed radius $r$ from one another. The constant $r$ is thus a parameter of PRM. 
The proposed algorithm---shown in Algorithm \ref{algorithm:PRM*}---is similar to sPRM, with the only difference being that the connection radius $r$  is chosen as a function of $\N$, i.e., $r = r(\N) := \gamma_{\mathrm{PRM}} (\log(\N)/\N)^{1/d}$, where $\gamma_{\mathrm{PRM}} > \gamma^*_\mathrm{PRM} = 2 (1 + 1/d)^{1/d} \left( \mu(\X_\mathrm{free})/\VolumeDBall{d} \right)^{1/d}$, $d$ is the dimension of the space $\X$, $\mu(\X_\mathrm{free})$ denotes the Lebesgue measure (i.e., volume) of the obstacle-free space, and $\VolumeDBall{d}$ is the volume of the unit ball in the $d$-dimensional Euclidean space. Clearly, the connection radius decreases with the number of samples. The rate of decay is such that the average number of connections attempted from a roadmap vertex is proportional to $\log(\N)$.

Note that in the discussion of variable-radius PRM in~\citet{lavalle.book06}, it is suggested that the radius be chosen as a function of sample dispersion. (Recall that the dispersion of a point set contained in a bounded set $\mathcal{S} \subset \reals^d$ is the radius of the largest empty ball centered in $\mathcal{S}$.)  
Indeed, the dispersion of a set of $\N$ random points sampled uniformly and independently in a bounded set is $O( (\log(n)/n)^{1/d})$~\citep{niederreiter.book92}, which is precisely the rate at which the connection radius is scaled in the PRM$^*$ algorithm.

\begin{algorithm}
$V \leftarrow \{x_\mathrm{init}\} \cup \{ \mathtt{SampleFree}_i\}_{i=1,\ldots,\N}$; $E \leftarrow \emptyset$\;
\ForEach {$v \in V$}{
	$U \leftarrow \mathtt{Near}(G=(V,E), v, \gamma_{\mathrm{PRM}} (\log(\N)/\N)^{1/d}) \setminus \{v\}$\label{line:PRM*neighbors}\;
	\ForEach{$u \in U$}{
		\lIf{$\mathtt{CollisionFree}(v,u)$}{
			$E \leftarrow E \cup \{(v,u), (u,v)\}$
			}}}
\Return {$G=(V,E)$}\;
  \caption{PRM$^*$}
  \label{algorithm:PRM*}
\end{algorithm}

Another version of the algorithm, called $k$-nearest PRM$^*$, can be considered, motivated by the $k$-nearest PRM implementation previously mentioned, whereby the number $k$ of nearest neighbors to be considered is not a constant, but is chosen  as a function of the cardinality of the roadmap $\N$. More precisely, $k(\N) := k_\mathrm{PRM} \log (\N)$, where 
$k_\mathrm{PRM} > k^*_\mathrm{PRM}= e \, (1+ 1/d)$, and $ U \leftarrow \mathtt{kNearest}(G=(V,E), v, k_\mathrm{PRM}\log(\N)) \setminus \{v\}$ in line \ref{line:PRM*neighbors} of Algorithm \ref{algorithm:PRM*}. 

Note that $k^*_\mathrm{PRM}$ is a constant that only depends on $d$, and does not otherwise depend on the problem instance, unlike $\gamma^*_\mathrm{PRM}$. Moreover, $k_\mathrm{PRM} = 2e$ is a valid choice for all problem instances.

\paragraph{Rapidly-exploring Random Graph (RRG):}
The Rapidly-exploring Random Graph algorithm was introduced as an incremental (as opposed to batch) algorithm to build a {\em connected} roadmap, possibly containing cycles. 
The RRG algorithm is similar to RRT in that it first attempts to connect the nearest node to the new sample. If the connection attempt is successful, the new node is added to the vertex set.
However, RRG has the following difference. Every time a new point  $x_\mathrm{new}$ is added to the vertex set $V$, then connections are attempted from all other vertices in $V$ that are within a  ball of radius $r(\card{V})=\min\{\gamma_\mathrm{RRG} (\log(\card{V})/\card{V})^{1/d},\eta\}$, where $\eta$ is the constant appearing in the definition of the local steering function, and $\gamma_\mathrm{RRG} > \gamma_\mathrm{RRG}^* =2 \, (1 + 1/d)^{1/d} \left( \mu(\X_\mathrm{free})/\VolumeDBall{d} \right)^{1/d}$. For each successful connection, a new  edge is added to the edge set $E$. Hence, it is clear that, for the same sampling sequence, the RRT graph (a directed tree) is a subgraph of the RRG graph (an undirected graph, possibly containing cycles). In particular, the two graphs share the same vertex set, and the edge set of the RRT graph is a subset of that of the RRG graph. 

\begin{algorithm}
  $V \leftarrow \{ x_\mathrm{init}\}$; $E \leftarrow \emptyset$\; 
  \For {$i = 1, \ldots, \N$} { 
    $x_\mathrm{rand} \leftarrow {\tt SampleFree}_i$\;
    $x_\mathrm{nearest} \leftarrow \mathtt{Nearest}(G=(V,E),x_\mathrm{rand})$\; 
    $x_\mathrm{new} \leftarrow \mathtt{Steer}(x_\mathrm{nearest},x_\mathrm{rand})$ \;
    \If{$\mathtt{ObtacleFree}(x_\mathrm{nearest},x_\mathrm{new})$}{
    $X_\mathrm{near} \leftarrow \mathtt{Near}(G=(V,E), x_\mathrm{new},\min \{\gamma_\mathrm{RRG} (\log(\card{V})/\card{V})^{1/d}, \eta\})$ \label{line:RRGneighbors}\; 
        $V \leftarrow V \cup \{x_\mathrm{new}\}$;
    $E \leftarrow E \cup \{(x_\mathrm{nearest},x_\mathrm{new}), (x_\mathrm{new},x_\mathrm{nearest})\}$ \;
    \ForEach{$x_\mathrm{near} \in X_\mathrm{near}$}{
    \lIf{$\mathtt{CollisionFree}(x_\mathrm{near},x_\mathrm{new})$}{
    $E \leftarrow E \cup \{(x_\mathrm{near},x_\mathrm{new}), (x_\mathrm{new},x_\mathrm{near}) \} $}
    }
}}
\Return {$G = (V,E)$}\;

  \caption{RRG}
  \label{algorithm:RRG}
\end{algorithm}

Another version of the algorithm, called $k$-nearest RRG, can be considered, in which connections are sought to  $k$ nearest neighbors, with $k = k(\card{V}) := k_\mathrm{RRG} \log (\card{V})$, where $k_\mathrm{RRG} > k^*_\mathrm{RRG} = e \, (1 + 1/d)$, and $X_\mathrm{near} \leftarrow \mathtt{kNearest}(G=(V,E), x_\mathrm{new},k_\mathrm{RRG} \log(\card{V}))$, in line \ref{line:RRGneighbors} of Algorithm \ref{algorithm:RRG}.

Note that $k^*_\mathrm{RRG}$ is a constant that depends only on $d$, and does not depend otherwise on the problem instance, unlike $\gamma^*_\mathrm{RRG}$. Moreover, $k_\mathrm{RRG} = 2e$ is a valid choice for all problem instances.

\paragraph{Optimal RRT (RRT$^*$):} 
Maintaining a tree structure rather than a graph is not only  economical in terms of memory requirements,  but may also be advantageous in some applications, due to,
for instance, relatively easy extensions to motion planning problems with differential constraints,
or to cope with modeling errors. The RRT$^*$ algorithm is obtained by modifying RRG  in such a way that formation of cycles is avoided, by removing ``redundant'' edges, i.e., edges that are not part of a shortest path from the root of the tree (i.e., the initial state) to a vertex. Since the RRT and RRT$^*$ graphs are directed trees with the same root and  vertex set, and edge sets that are subsets of that of RRG, this amounts to a ``rewiring'' of the RRT tree, ensuring that vertices are reached through a minimum-cost path.

Before discussing the algorithm, it is necessary to introduce a few new functions.  Given two points $x_1, x_2 \in \reals^d$, let $ {\tt Line}(x_1,x_2) : [0,s] \to
\X$ denote the straight-line path from $x_1$ to $x_2$.  Given a tree $G = (V,E)$, let ${\tt Parent}: V \to V$ be a function that maps a vertex $v \in V$ to the unique vertex $u \in V$ such that $(u,v) \in E$. By convention, if $v_0 \in V $ is the root vertex of $G$, $\mathtt{Parent}(v_0) = v_0$. Finally, let $\mathtt{Cost}: V \to \reals_{\ge 0}$ be a function that maps a vertex $v \in V$ to the cost of the unique path from the root of the tree to $v$. For simplicity, in stating the algorithm we will assume an additive cost function, so that 
$\mathtt{Cost}(v) = \mathtt{Cost}(\mathtt{Parent}(v)) + c(\mathtt{Line}(\mathtt{Parent}(v), v))$, although this is not necessary for the analysis in the next section. 
By convention, if $v_0 \in V$ is the root vertex of $G$, then $\mathtt{Cost}(v_0) = 0$. 

\begin{algorithm}
  $V \leftarrow \{ x_\mathrm{init}\}$; $E \leftarrow \emptyset$\; 
  \For {$i=1,\ldots,\N$} { 
    $x_\mathrm{rand} \leftarrow {\tt SampleFree}_i$\;
    $x_\mathrm{nearest} \leftarrow \mathtt{Nearest}(G=(V,E),x_\mathrm{rand})$\; 
    $x_\mathrm{new} \leftarrow \mathtt{Steer}(x_\mathrm{nearest},x_\mathrm{rand})$ \;
    \If{$\mathtt{ObtacleFree}(x_\mathrm{nearest},x_\mathrm{new})$}{
    $X_\mathrm{near} \leftarrow \mathtt{Near}(G=(V,E), x_\mathrm{new},\min \{\gamma_\mathrm{RRT^*} (\log(\card{V})/\card{V})^{1/d}, \eta\})$ \label{line:RRT*neighbors}\; 
        $V \leftarrow V \cup \{x_\mathrm{new}\}$\;
        $x_\mathrm{min} \leftarrow x_\mathrm{nearest}$; $c_\mathrm{min} \leftarrow \mathtt{Cost}(x_\mathrm{nearest}) + c(\mathtt{Line}(x_\mathrm{nearest}, x_\mathrm{new}))$\; 
        \ForEach(\hfill // Connect along a minimum-cost path){$x_\mathrm{near} \in X_\mathrm{near}$}{
        \If{$\mathtt{CollisionFree}(x_\mathrm{near},x_\mathrm{new}) \wedge 
        \mathtt{Cost}(x_\mathrm{near}) + c(\mathtt{Line}(x_\mathrm{near},x_\mathrm{new})) < c_\mathrm{min}$}{
        $x_\mathrm{min} \leftarrow x_\mathrm{near}$; $c_\mathrm{min} \leftarrow \mathtt{Cost}(x_\mathrm{near}) + c(\mathtt{Line}(x_\mathrm{near},x_\mathrm{new}))$
        }
        }
    $E \leftarrow E \cup \{(x_\mathrm{min},x_\mathrm{new})\}$\;
    \ForEach(\hfill // Rewire the tree){$x_\mathrm{near} \in X_\mathrm{near}$}{
    \lIf{$\mathtt{CollisionFree}(x_\mathrm{new},x_\mathrm{near}) \wedge 
    \mathtt{Cost}(x_\mathrm{new}) + c(\mathtt{Line}(x_\mathrm{new},x_\mathrm{near})) < \mathtt{Cost}(x_\mathrm{near})$}{
    $x_\mathrm{parent} \leftarrow \mathtt{Parent}(x_\mathrm{near})$\;
    $E \leftarrow (E \setminus \{(x_\mathrm{parent}, x_\mathrm{near})\}) \cup \{(x_\mathrm{new},x_\mathrm{near})\} $}
    }
}}
\Return {$G = (V,E)$}\;
  \caption{RRT$^*$}
  \label{algorithm:RRT*}
\end{algorithm}

The RRT$^*$ algorithm, shown in Algorithm \ref{algorithm:RRT*}, adds points to the vertex set $V$ in the same way as RRT and RRG. It also considers connections from the new vertex $x_\mathrm{new}$ to vertices in $X_\mathrm{near}$, i.e., other vertices that are within distance $r(\card{V})=\min\{\gamma_\mathrm{RRT^*} (\log(\card{V})/\card{V})^{1/d},\eta\}$ from $x_\mathrm{new}$. However, not all feasible connections result in new edges being inserted in the edge set $E$. In particular, (i) an edge is created from the vertex in $X_\mathrm{near}$ that can be connected to $x_\mathrm{new}$ along a path with minimum cost, and (ii) new edges are created from $x_\mathrm{new}$ to vertices in $X_\mathrm{near}$, if the path through $x_\mathrm{new}$ has lower cost than the path through the current parent; in this case, the edge linking the vertex to its current parent is deleted, to maintain the tree structure. 

Another version of the algorithm, called $k$-nearest RRT$^*$, can be considered, in which connections are sought to  $k$ nearest neighbors, with $k(\card{V}) = k_\mathrm{RRG} \log (\card{V})$, and 
    $X_\mathrm{near} \leftarrow \mathtt{kNearest}(G=(V,E), x_\mathrm{new},k_\mathrm{RRG} \log(i))$, 
in line \ref{line:RRT*neighbors} of Algorithm \ref{algorithm:RRT*}.

\section{Analysis}
\label{section:analysis}
In this section, a number of results concerning the probabilistic completeness, asymptotic optimality, and complexity of the algorithms in Section \ref{section:algorithms} are presented. 

The return value of Algorithms~\ref{algorithm:PRM}-\ref{algorithm:RRT*} is a graph. Since the sampling procedure $\mathtt{SampleFree}$ is  stochastic, the returned graph is in fact a random variable.\footnote{We will not address the case in which the sampling procedure is deterministic, but refer the reader to~\citet{lavalle.branicky.ea.ijrr04}, which contains an in-depth discussion of the relative merits of randomness and determinism in sampling-based motion planning algorithms.} 
Since the sampling procedure is modeled as a map from the sample space $\Omega$ to infinite sequences in $\mathcal{X}$, sets of vertices and edges of the graphs maintained by the algorithms  can be defined as functions from the sample space $\Omega$ to appropriate sets.
More precisely, let \Alg  be a label indicating one of the algorithms in Section \ref{section:algorithms}, and let $\{{V}^\Alg_i(\omega)\}_{i \in \naturals}$ 
and 
$\{{E}^\Alg_i(\omega)\}_{i \in \naturals}$
be, respectively, the sets of vertices and edges in the graph returned by algorithm \Alg , indexed by the number of samples, for a particular realization of the sample sequence. (In other words, these are sequences of functions defined from $\Omega$ into finite subsets of $\X_\mathrm{free}$ or $\X_\mathrm{free} \times \X_\mathrm{free}$.) Similarly, let ${G}^\Alg_i = ({V}^\Alg_i,{E}^\Alg_i)$. (The label \Alg  will be at times omitted when the algorithm being used is clear from the context.)

All algorithms considered in the paper are sound, in the sense that they only return graphs with vertices and edges representing points and paths in $\X_\mathrm{free}$.This statement can be easily verified by inspection of the algorithms in Section \ref{section:algorithms}.

\subsection{Probabilistic Completeness} \label{section:feasibility} \label{section:completeness}
In this section, the feasibility problem is considered, and the (probabilistic) completeness properties of the algorithms in Section \ref{section:algorithms} are analyzed. 
First, some preliminary definitions are given, followed by a definition of probabilistic completeness. Then, completeness properties of various sampling-based motion planning algorithms are stated.

Let $\delta >0$ be a real number. A state $x \in \X_\mathrm{free}$ is said to be a {\em $\delta$-interior state} of $\X_\mathrm{free}$, if the closed ball of radius $\delta$ centered at $x$ lies entirely inside $\X_\mathrm{free}$. The {\em $\delta$-interior} of $\X_\mathrm{free}$, denoted as $\mathrm{int}_\delta (\X_\mathrm{free})$, is defined as the collection of all $\delta$-interior states, i.e., $\mathrm{int}_\delta (\X_\mathrm{free}) := \{x \in \X_\mathrm{free} \,\vert\, {\cal B}_{x,\delta} \subseteq \X_\mathrm{free}\}$. In other words, the $\delta$-interior of $\X_\mathrm{free}$ is the set of all states that are  at least a distance $\delta$ away from any point in the obstacle set (see Figure~\ref{figure:delta_interior}).
A collision-free path $\sigma : [0,1] \to \X_\mathrm{free}$ is said to have {\em strong $\delta$-clearance}, if $\sigma$ lies entirely inside the $\delta$-interior of $\X_\mathrm{free}$, i.e., $\sigma(\tau) \in \mathrm{int}_\delta(\X_\mathrm{free})$ for all $\tau \in [0, 1]$.
A path planning problem $(\X_\mathrm{free}, x_\mathrm{init}, \X_\mathrm{goal})$ is said to be {\em robustly feasible} if there exists a path with strong $\delta$-clearance, for some $\delta > 0$, that solves it. 
In terms of the notation used in this paper, the notion of probabilistic completeness can be stated as follows.

\begin{definition}[Probabilistic Completeness] \label{definition:probabilistic_completeness}
An algorithm ALG is probabilistically complete, if, for any robustly feasible path planning problem $(\X_\mathrm{free}, x_\mathrm{init}, \X_\mathrm{goal})$, 
$$
\liminf_{\N \to \infty} \PP \left(\{ \exists x_\mathrm{goal} \in V^\Alg_\N \cap \X_\mathrm{goal} \mbox{ such that } x_\mathrm{init} \mbox{ is connected to }x_\mathrm{goal} \mbox{ in } G^\Alg_\N\ \}\right) = 1.
$$
\end{definition}

If an algorithm is probabilistically complete, and the path planning problem is robustly feasible, the limit $$\lim \nolimits_{\N \to \infty} \PP \left(\{ \exists x_\mathrm{goal} \in V^\Alg_\N \cap \X_\mathrm{goal} \mbox{ such that } x_\mathrm{init} \mbox{ is connected to }x_\mathrm{goal} \mbox{ in } G^\Alg_\N\ \}\right)$$ exists and is equal to 1. On the other hand, the same limit is equal to zero for any sampling-based algorithm (including probabilistically complete ones) if the problem is not robustly feasible, unless the samples are drawn from a singular distribution adapted to the problem.

\begin{figure}
\centering
\includegraphics[height = 3cm]{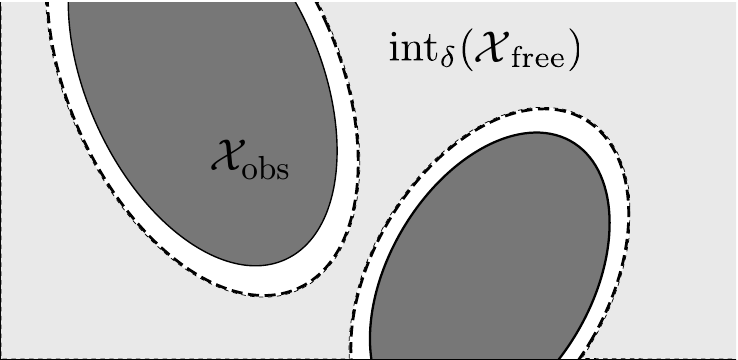}
\caption{An illustration of the $\delta$-interior of $\X_\mathrm{free}$. The obstacle region $\X_\mathrm{obs}$ is shown in dark grey and the $\delta$-interior of $\X_\mathrm{free}$ is shown in light grey. The distance between the dashed boundary of $\mathrm{int}_\delta(\X_\mathrm{free})$ and the solid boundary of $\X_\mathrm{free}$ is precisely $\delta$.}
\label{figure:delta_interior}
\end{figure}

It is known from the literature that the sPRM and RRT algorithms are probabilistically complete, and that the probability of finding a solution if one exists approaches one exponentially fast with the number of vertices in the graph returned by the algorithms. In other words,
\begin{theorem}[Probabilistic completeness of sPRM~\citep{kavraki.kolountzakis.ea.tro98}] 
Consider a robustly feasible path planning problem $(\X_\mathrm{free},x_\mathrm{init}, \X_\mathrm{goal})$. There exist constants $a > 0$ and $\N_0 \in \naturals$, dependent only on $\X_\mathrm{free}$ and $\X_\mathrm{goal}$, such that 
$$\PP\left( \left\{ \exists \, x_\mathrm{goal} \in V^\mathrm{sPRM}_\N \cap \X_\mathrm{goal}: x_\mathrm{goal} \mbox { is connected  to } x_\mathrm{init} \mbox{ in } G^\mathrm{sPRM}_\N \right\}\right) > 1-  e^{-a\,\N}, \quad \forall \N>\N_0.$$
\end{theorem}

\begin{theorem}[Probabilistic Completeness of RRT~\citep{lavalle.kuffner.ijrr01}] \label{theorem:RRT_completeness}
Consider a robustly feasible path planning problem $(\X_\mathrm{free}$, $x_\mathrm{init}, \X_\mathrm{goal})$. There exist constants $a > 0$ and $\N_0 \in \naturals$, both dependent only on $\X_\mathrm{free}$ and $\X_\mathrm{goal}$, such that 
$$\PP\left( \left\{ V^\mathrm{RRT}_\N \cap \X_\mathrm{goal} \neq \emptyset \right\} \right) > 1- e^{-a\,\N}, \quad \forall \N>\N_0.$$
\end{theorem}

On the other hand, the probabilistic completeness results do not necessarily extend to the heuristics used in practical implementations of the (s)PRM algorithm, as detailed in Section \ref{section:algorithms}. For example, consider the $k$-nearest sPRM algorithm, where $k=1$. That is, each vertex is connected to its nearest neighbor and the resulting undirected graph is returned as the output. This sPRM algorithm will be called the 1-nearest sPRM, and indicated with the label $\mathrm{1PRM}$. The RRT algorithm can be thought of as the incremental version of the 1-nearest sPRM algorithm: the RRT algorithm also connects each sample to its nearest neighbor, but forces connectivity of the graph by an incremental construction. The following theorem shows that the 1-nearest sPRM algorithm is not probabilistically complete, although the RRT is (see Theorem~\ref{theorem:RRT_completeness}).
Furthermore, the probability that it {\em fails} to find a path converges to one as the number of samples approaches infinity.

\begin{theorem}[Incompleteness of $k$-nearest sPRM for $k = 1$] \label{theorem:incompleteness_1PRM}
The $k$-nearest sPRM algorithm is not probabilistically complete for $k = 1$. Furthermore,
$$
\lim_{\N \to \infty} \PP \left(\{ \exists x_\mathrm{goal} \in V^\mathrm{1PRM}_\N \cap \X_\mathrm{goal} \mbox{ such that } x_\mathrm{init} \mbox{ is connected to }x_\mathrm{goal} \mbox{ in } G^\Alg_\N\ \}\right) = 0.
$$

\end{theorem}
The proof of this theorem requires two intermediate results that are provided below.
For simplicity of presentation, consider the case when $\X_\mathrm{free} = \X$.
Let $G^{1\mathrm{PRM}}_\N = (V^{1\mathrm{PRM}}_\N, E^{1\mathrm{PRM}}_\N)$ denote the graph returned by the 1-nearest sPRM algorithm, when the algorithm is run with $\N$ samples. Let $\RVlength_\N$ denote the total length of all the edges present in $G^{1\mathrm{PRM}}_\N$. Recall that $\VolumeDBall{d}$ denotes the volume of the unit ball in the $d$-dimensional Euclidean space. Let $\VolumeDBall{d}'$ denote the volume of the union of two unit balls whose centers are a unit distance apart.

\begin{lemma}[Total length of the 1-nearest neighbor graph~\citep{wade.aap07}] \label{lemma:incompleteness_1PRM:length}
For all $d \ge 2$, $\RVlength_\N/\N^{1 - 1/d}$ converges to a constant in mean square, i.e., 
$$
\lim_{\N \to \infty}  \EE \left[ \left( \frac{ \RVlength_\N }{\N^{1 - 1/d}} - \left(1 + \frac{1}{d}\right) \left( \frac{1}{\VolumeDBall{d}} - \frac{\VolumeDBall{d}}{2\, (\VolumeDBall{d}')^{1 + 1/d}}  \right) \right)^2 \right] = 0.
$$
\end{lemma}
\begin{proof}
This lemma is a direct consequence of Theorem 3 of~\citet{wade.aap07}. \qed
\end{proof}

Let $N_\N$ denote the number of connected components of $G^{1\mathrm{PRM}}_\N$. 
\begin{lemma}[Number of connected components of the 1-nearest neighbor graph] \label{lemma:incompleteness_1PRM:components}
For all $d \ge 2$, $N_\N/\N$ converges to a constant in mean square, i.e., 
$$
\lim_{\N \to \infty} \EE\left[ \left( \frac{N_\N}{\N} - \frac{\VolumeDBall{d}}{2\,\VolumeDBall{d}'} \right)^2 \right] = 0.
$$
\end{lemma}

\begin{proof}
A reciprocal pair is a pair of vertices each of which is the other one's nearest neighbor. In a graph formed by connecting each vertex to its nearest neighbor, any connected component includes exactly one reciprocal pair whenever the number of vertices is greater than 2 \citep[see, e.g.,][]{eppstein.paterson.ea.disc_comp_geo97}.
The number of reciprocal pairs in such a graph was shown to converge to $\VolumeDBall{d} / (2 \VolumeDBall{d}')$ in mean square in~\citet{henze.aap87} (see also Remark 2 in~\citet{wade.aap07}).
\qed
\end{proof}

\begin{proof}[Proof of Theorem~\ref{theorem:incompleteness_1PRM}]
Let $\widetilde{\RVlength}_\N$ denote the average length of a connected component in $G^{1\mathrm{PRM}}_\N$, i.e., $\widetilde{\RVlength}_\N = \RVlength_\N / N_\N$. 
Let $\RVlength_\N'$ denote the length of the connected component that includes $x_\mathrm{init}$. Since the samples are drawn independently and uniformly, the random variables $\widetilde{\RVlength}_\N$ and $\RVlength_\N'$ have the same distribution (although they are clearly dependent).
Let $\gamma_\RVlength$ denote the constant that $\RVlength_\N/\N^{1 - 1/d}$ converges to (see Lemma~\ref{lemma:incompleteness_1PRM:length}). Similarly, let $\gamma_N$ denote the constant that $N_\N/\N$ converges to (see Lemma~\ref{lemma:incompleteness_1PRM:components}).

Recall that convergence in mean square implies convergence in probability and hence convergence in distribution~\citep{grimmett.stirzaker.book01}. Since both $\RVlength_\N/\N^{1-1/d}$ and $N_\N/\N$ converge in mean square to constants and $\PP (\{ N_\N = 0 \}) = 0$ for all $\N \in \naturals$, by Slutsky's theorem~\citep{resnick.book99}, $\N^{1/d} \, \widetilde{\RVlength}_\N  = \frac{\RVlength_\N/\N^{1-1/d}}{N_\N/\N}$ converges to $\gamma := \gamma_\RVlength/\gamma_N$ in distribution. In this case, it also converges in probability, since $\gamma$ is a constant~\citep{grimmett.stirzaker.book01}. Then, $\N^{1/d}\, \RVlength_\N'$ also converges to $\gamma$ in probability, since $\widetilde{\RVlength}_\N$ and $\RVlength_\N'$ are identically distributed for all $\N \in \naturals$. Thus, $\RVlength'_\N$ converges to $0$ in probability, i.e., 
$\lim_{\N \to \infty} \PP \left(\left\{ \RVlength_\N' > \epsilon \right\}\right) = 0$, for all $\epsilon> 0$.

Let $\epsilon > 0$ be such that $\epsilon < \inf_{x \in \X_\mathrm{goal}} \Vert x - x_\mathrm{init} \Vert$.
Let $A_\N$ denote the event that the graph returned by the 1-nearest sPRM algorithm contains a feasible path, i.e., one that starts from $x_\mathrm{init}$ and reaches the goal region
Clearly, the event $\{ \RVlength'_\N > \epsilon \}$ occurs whenever $A_\N$ does, i.e., $A_\N \subseteq \{\RVlength'_\N > \epsilon \}$. Then, $\PP (A_\N) \le \PP (\{ \RVlength'_\N > \epsilon \})$. Taking the limit superior of both sides
$$
\liminf_{\N \to \infty} \PP (A_\N) 
\,\,\le\,\,
\limsup_{\N \to \infty} \PP (A_\N) 
\,\,\le\,\,
\limsup_{\N \to \infty} \PP (\{\RVlength'_\N > \epsilon\}) 
\,\,=\,\,
0.
$$
In other words, the limit $\lim_{\N \to \infty} \PP (A_\N)$ exists and is equal zero. 
\qed
\end{proof}

Consider the variable-radius sPRM algorithm. 
The following theorem asserts that variable-radius sPRM algorithm is not probabilistically complete in the subcritical regime.

\begin{theorem}[Incompleteness of variable-radius sPRM with $r(\N) = \gamma \N^{-1/d}$] \label{theorem:incompleteness_rPRM}
There exists a constant $\gamma > 0$ such that the variable radius sPRM with connection radius $r(\N) = \gamma\N^{-1/d}$ is not probabilistically complete.
\end{theorem}

The proof of this result requires some intermediate results from random geometric graph theory. Recall that $\lambda_c$ is the critical density, or continuum percolation threshold (see Section~\ref{section:rgg}). 
Given a Borel set $\Gamma \subseteq \reals^d$, let $G^\mathrm{disc}_\Gamma (n,r)$ denote the random $r$-disc graph formed with vertices independent and uniformly sampled from $\Gamma$ and edges connecting two vertices, $v$ and $v'$, whenever $\| v - v' \| < r_n$.

\begin{lemma}[\citet{penrose.book03}] \label{lemma:subcritical_percolation}
Let $\lambda \in (0, \lambda_c)$ and $\Gamma \subset \reals^d$ be a Borel set. 
Consider a sequence $\{r_\N\}_{\N \in \naturals}$ that satisfies $\N\, r_\N^d \le \lambda$, $\forall \N \in \naturals$. Let $\LargestComponent (G^\mathrm{disc}_\Gamma (\N,r_\N))$ denote the size of the largest component in $G^\mathrm{disc}_\Gamma (\N , r_\N)$.
Then, there exist constants $a, b > 0$ and $m_0 \in \naturals$ such that for all $m \ge m_0$,
$$
\PP \left(\left\{ \LargestComponent (G^\mathrm{disc}_\Gamma (\N,r_\N)) \ge m \right\}\right) \le \N \left( e^{-a\,m} + e^{-b \, \N}\right).
$$
\end{lemma}

\begin{proof}[Proof of Theorem~\ref{theorem:incompleteness_rPRM}]
Let $\epsilon > 0$ such that $\epsilon < \inf_{x \in X_\mathrm{goal}} \Vert x - x_\mathrm{init} \Vert$ and that the $2\,\epsilon$-ball centered at $x_\mathrm{init}$ lies entirely within the obstacle-free space. 
Let $G_\N^\mathrm{PRM} = (V_\N^\mathrm{PRM}, E_\N^\mathrm{PRM})$ denote the graph returned by this variable radius sPRM algorithm, when the algorithm is run with $\N$ samples. 
Let $G_\N = (V_\N, E_\N)$ denote the the restriction of $G_\N^\mathrm{PRM}$ to the $2\,\epsilon$-ball centered at $x_\mathrm{init}$ defined as $V_\N = V_\N^\mathrm{PRM} \cap {\cal B}_{x_\mathrm{init}, 2\, \epsilon}$ and $E_\N = (V_\N \times V_\N) \cap E_\N^\mathrm{PRM}$.

Clearly, $G_\N$ is equivalent to the random $r$-disc graph on $\Gamma = {\cal B}_{x_\mathrm{init}, 2\,\epsilon}$.
Let $\LargestComponent(G_\N)$ denote the number of vertices in the largest connected component of $G_\N$. By Lemma~\ref{lemma:subcritical_percolation}, there exists constants $a, b > 0$ and $m_0 \in \naturals$ such that 
$$
\PP (\{ \LargestComponent (G_\N) \ge m\}) \le \N \left(e^{-a\,m} + e^{-b\,\N}\right),
$$
for all $m \ge m_0$.
Then, for all $m = \lambda^{-1/d}\, (\epsilon/2) \, \N^{1/d} > m_0$,
$$
\PP \left(\left\{ \LargestComponent(G_\N)\ge \lambda^{-1/d} \frac{\epsilon}{2}\,\N^{1/d} \right\}\right) \le \N \left(e^{-a\, \lambda^{-1/d}\,(\epsilon/2)\,\N^{1/d}} + e^{-b\,\N}\right).
$$

Let $\RVlength_\N$ denote the total length of all the edges in the connected component that includes $x_\mathrm{init}$. Since $r_\N = \lambda^{1/d} \N^{-1/d}$,
$$
\PP \left(\left\{\RVlength_\N \ge \frac{\epsilon}{2}\right\}\right) \le \N \left(e^{-a\,\lambda^{-1/d}\,(\epsilon/2)\,\N^{1/d}} + e^{-b\,\N}\right).
$$
Since the right hand side is summable, by the Borel-Cantelli lemma the event
$
\left\{ \RVlength_\N \ge \epsilon/2 \right\}
$
occurs infinitely often with probability zero, i.e., $\PP (\limsup_{\N \to \infty}\{\RVlength_\N \ge \epsilon/2 \}) = 0$. 

Given a graph $G = (V,E)$ define the diameter of this graph as the distance between the farthest pair of vertices in $V$, i.e., $\max_{v,v' \in V} \Vert v - v' \Vert$.
Let $D_\N$ denote the diameter of the largest component in $G_\N$. Clearly, $D_\N \le \RVlength_\N$ holds surely. Thus, $\PP \left(\limsup_{\N \to \infty} \left\{ D_\N \ge \epsilon/2 \right\}\right) = 0$.

Let $I \in \naturals$ be the smallest number that satisfies $r_I \le \epsilon/2$. 
Notice that the edges connected to the vertices $V_\N^\mathrm{PRM} \cap {\cal B}_{x_\mathrm{init}, \epsilon}$ coincide with those connected to $V_\N \cap {\cal B}_{x_\mathrm{init}, \epsilon}$, for all $\N \ge I$.
Let $R_\N$ denote distance of the farthest vertex $v \in V_\N^\mathrm{PRM}$ to $x_\mathrm{init}$ in the component that contains $x_\mathrm{init}$ in $G_\N^\mathrm{PRM}$.
Notice also that $R_\N \ge \epsilon$ only if $D_\N \ge \epsilon/2$, for all $\N \ge I$. 
That is, for all $\N \ge I$, $\left\{ R_\N \ge \epsilon \right\} \subseteq \left\{ D_\N \ge \epsilon/2 \right\}$, which implies $\PP \left(\limsup_{\N \to \infty} \left\{ R_\N \ge \epsilon \right\}\right) = 0$.

Let $A_\N$ denote the event that the graph returned by this variable radius sPRM algorithm includes a path that reaches the goal region. Clearly, $\{ R_\N \ge \epsilon\}$ holds, whenever $A_\N$ holds. Hence, $\PP (A_\N) \le \PP (\{ R_\N \ge \epsilon \})$. Taking the limit superior of both sides yields
$$
\liminf_{\N \to \infty} \PP (A_\N) 
\,\,\le\,\, \limsup_{\N \to \infty} \PP(A_\N) 
\,\,\le\,\, \limsup_{\N \to \infty} \PP \left(\left\{ R_\N \ge \epsilon \right\}\right) 
\,\,\le\,\, \PP \Big( \limsup_{\N \to \infty}  \left\{ R_\N  \ge \epsilon \right\}\Big) = 0.
$$
Hence, $\lim_{\N \to \infty} \PP(A_\N) = 0$. \qed
\end{proof}

Finally, the probabilistic completeness of the new algorithms proposed in Section \ref{section:algorithms} is established. Probabilistic completeness of PRM$^*$ is implied by its asymptotic optimality, proved in Section~\ref{section:optimality}.

\begin{theorem}[Completeness of PRM$^*$]
The PRM$^*$ algorithm is probabilistically complete. 
\end{theorem}

Probabilistic completeness of RRG and RRT$^*$ is a straightforward consequence of the probabilistic completeness of RRT:
\begin{theorem}[Probabilistic completeness of RRG and RRT$^*$] \label{thoerem:completeness_rrg_rrtstar}
The RRG and RRT$^*$ algorithms are probabilistically complete.  Furthermore, for any robustly feasible path planning problem $(\X_\mathrm{free}, x_\mathrm{init}, \X_\mathrm{goal})$, 
there exist constants $a > 0$ and $\N_0 \in \naturals$, both dependent only on $\X_\mathrm{free}$ and $\X_\mathrm{goal}$, such that 
$$\PP\left( \left\{ V^\mathrm{RRG}_\N \cap \X_\mathrm{goal} \neq \emptyset \right\} \right) > 1- e^{-a\,\N}, \quad \forall \N>\N_0,$$
and
$$\PP\left( \left\{ V^{\mathrm{RRT}^*}_\N \cap \X_\mathrm{goal} \neq \emptyset \right\} \right) > 1- e^{-a\,\N}, \quad \forall \N>\N_0.$$
\end{theorem}
\begin{proof}
By construction, $V^\mathrm{RRG}_\N(\omega) = V^{\mathrm{RRT}^*}_\N(\omega) =V^\mathrm{RRT}_\N(\omega)$, for all $\omega \in \Omega$ and $\N \in \naturals$. Moreover, the RRG and RRT$^*$ algorithms return connected graphs. Hence the result follows directly from the probabilistic completeness of RRT.  \qed\end{proof}

In particular, note that if the RRT algorithm returns a feasible solution by iteration $n$, 
so will the RRG and RRT$^*$ algorithms, assuming the same sample sequence. 

\subsection{Asymptotic Optimality} \label{section:optimality}

In this section, the optimality problem of path planning is considered. The algorithms presented in Section~\ref{section:algorithms} are analyzed, in terms of their ability to return solutions whose cost converge to the global optimum. First, a definition of asymptotic optimality is provided as almost-sure convergence to optimal paths. Second, it is shown that the RRT algorithm lacks the asymptotic optimality property. Third, the PRM$^*$, RRG, and RRT$^*$ algorithms, as well as their $k$-nearest implementations, are shown to be asymptotically optimal.

Recall from Section \ref{section:completeness} that an algorithm is probabilistically complete if the algorithm finds with high probability a solution to path planning problems that are robustly feasible, i.e., for which feasible path exists with strong $\delta$-clearance. A similar approach is used to define asymptotic optimality, relying on a notion of weak $\delta$-clearance and on a continuity property for the cost of paths, which will be introduced below.

Let $\sigma_1, \sigma_2 \in \Sigma_\mathrm{free}$ be two collision-free paths with the same end points. A path $\sigma_1$ is said to be homotopic to $\sigma_2$, if there exists a continuous function $\psi : [0,1] \to \Sigma_\mathrm{free}$, called the {\em homotopy}, such that $\psi (0) = \sigma_1$, $\psi (1) = \sigma_2$, and $\psi(\tau)$ is a collision-free path in for all $\tau \in [0,1]$. Intuitively, a path that is homotopic to $\sigma$ can be continuously transformed to $\sigma$ through $\X_\mathrm{free}$ ~\citep[see][]{munkres.book00}. 
A collision-free path $\sigma:[0,s] \to \X_\mathrm{free}$ is said to have {\em weak $\delta$-clearance}, if there exists a path $\sigma'$ that has strong $\delta$-clearance and there exist a homotopy $\psi$, with $\psi(0)=\sigma$, $\psi(1)=\sigma'$, and for all $\alpha \in (0,1]$ there exists $\delta_\alpha > 0$ such that $\psi(\alpha)$ has strong $\delta_\alpha$-clearance. See Figure~\ref{figure:delta_clearance} for an illustration of the weak $\delta$-clearance property. 
A path that violates the weak $\delta$-clearance property is shown in Figure~\ref{figure:no_delta_clearance}.
Weak $\delta$-clearance does not require points along a path to be at least a distance $\delta$ away from the obstacles (see Figure~\ref{figure:no_delta_clearance_3d}). In fact, a collision-free path with uncountably many points lying on the boundary of an obstacle can still have weak $\delta$-clearance.

\begin{figure}[htb]
\centering
\includegraphics[height = 3cm]{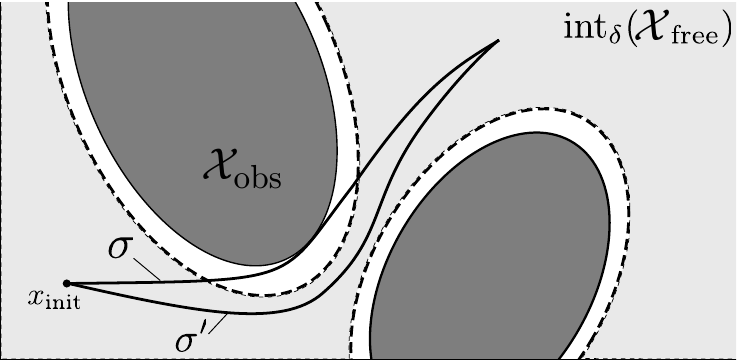}
\caption{An illustration of a path $\sigma$ with weak $\delta$-clearance. The path $\sigma'$ that lies inside $\mathrm{int}_\delta (\X_\mathrm{free})$ and is in the same homotopy class as $\sigma$ is also shown in the figure. Note that $\sigma$ does not have strong $\delta$-clearance.}
\label{figure:delta_clearance} 
\end{figure}

\begin{figure}[htb]
\centering
\includegraphics[height = 3cm]{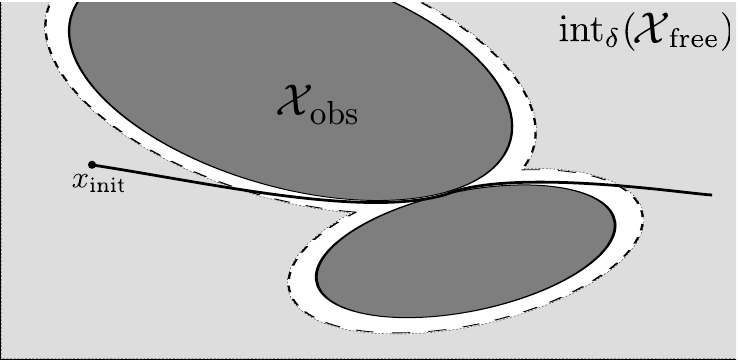}
\caption{An illustration of an example path $\sigma$ that does not have  weak $\delta$-clearance. For any positive value of $\delta$, there is no path in $\mathrm{int}_\delta(\X_\mathrm{free})$ that is in the same homotopy class as $\sigma$.}
\label{figure:no_delta_clearance} 
\end{figure}

\begin{figure}[ht]
\centering
\mbox{
\includegraphics[height = 3.5cm]{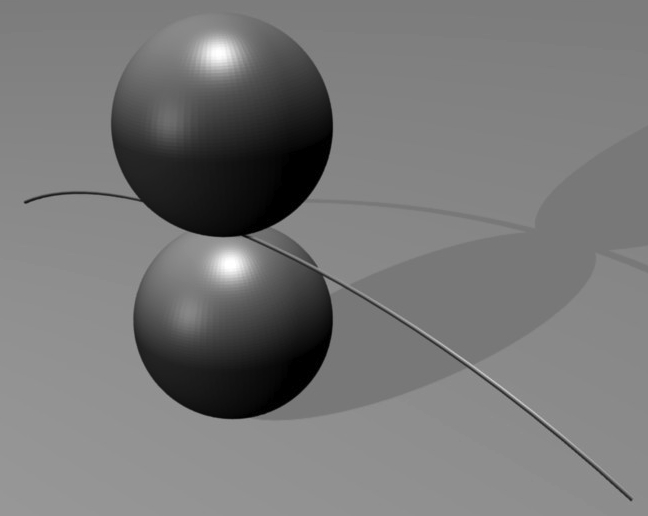}
\quad\quad\includegraphics[height = 3.5cm]{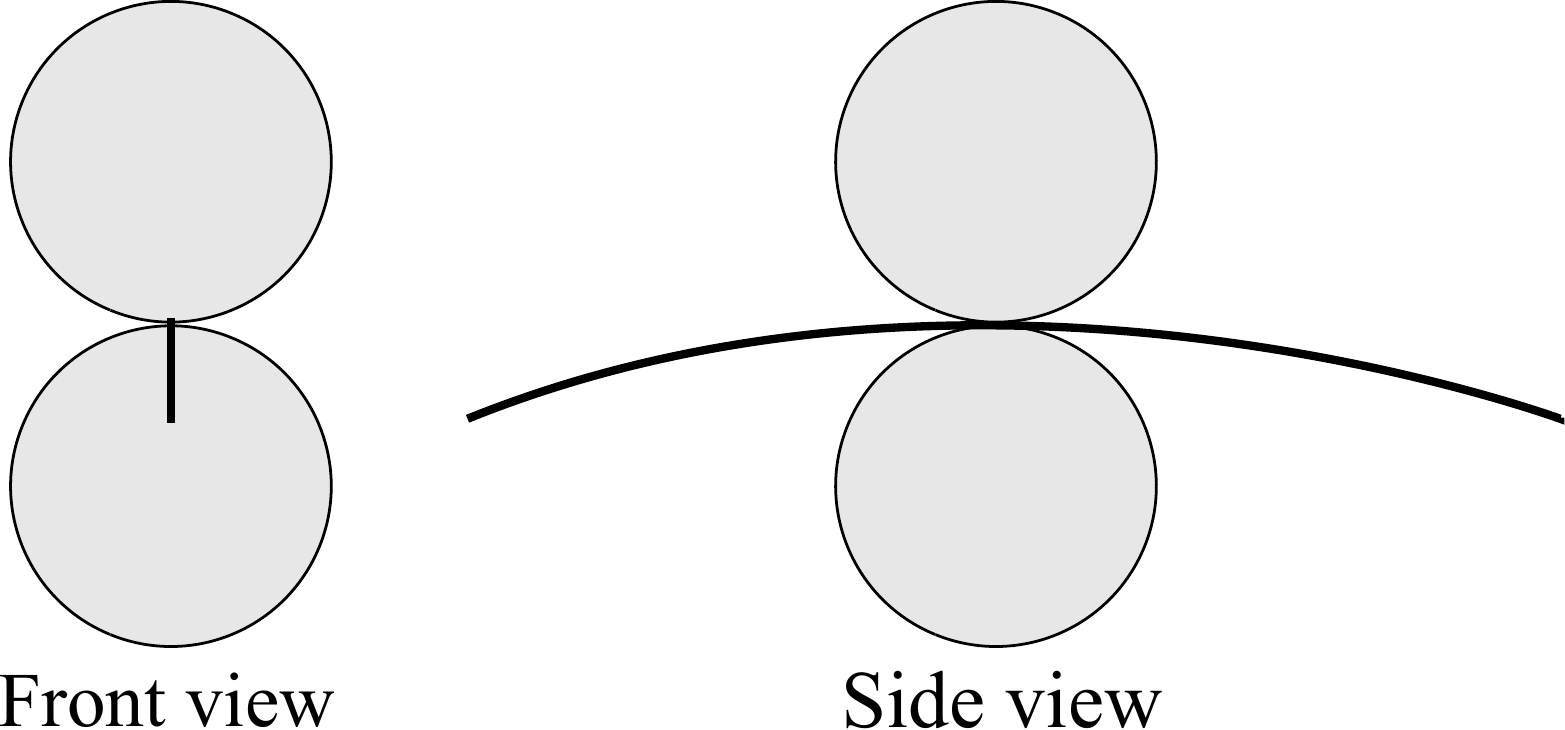}
}
\caption{An illustration of a path that has weak $\delta$-clearance. The path passes through a point where two spheres representing the obstacle region are in contact. Clearly, the path does not have strong $\delta$-clearance.
}
\label{figure:no_delta_clearance_3d} 
\end{figure}

Next, the set of all paths with bounded length is introduced as a normed space, which allows taking the limit of a sequence of paths.
Recall that $\Sigma$ is the set of all paths, and  $TV(\cdot)$ denotes the total variation, i.e., the length, of a path (see Section~\ref{section:problem}). Given $\sigma_1, \sigma_2 \in \Sigma$ with $\sigma_1 : [0,1] \to \X$ and $\sigma_2 : [0, 1] \to \X$,  the addition operation is defined as $(\sigma_1 + \sigma_2)(\tau) = \sigma_1 (\tau) + \sigma_2 (\tau)$ for all $\tau \in [0,1]$. 
The set of paths $\Sigma$ is closed under addition. Given a path $\sigma : [0,1] \to \X$ and a scalar $\alpha \in \reals$, the multiplication by a scalar operation is defined as $(\alpha \sigma)(\tau) := \alpha \, \sigma(\tau)$ for all $\tau \in [0,1]$. With these addition and multiplication by a scalar operations, the function space $\Sigma$ is, in fact, a vector space.
On the vector space $\Sigma$, define the norm $\Vert \sigma \Vert_\mathrm{BV} := \int_{0}^1 \vert  \sigma(\tau) \vert \; d \tau+ \mathrm{TV}(\sigma)$, and denote the function space $\Sigma$ endowed with the norm $\Vert \cdot \Vert_\mathrm{BV}$ by $\mathrm{BV}(\X)$.
The norm $\Vert \cdot \Vert_\mathrm{BV}$ induces the following distance function: 
$$\mathrm{dist} (\sigma_1, \sigma_2) = \Vert \sigma_1 - \sigma_2 \Vert_\mathrm{BV} =  \int_{0}^1 \big\Vert (\sigma_1 - \sigma_2)(\tau) \big\Vert d \tau + \mathrm{TV}(\sigma_1 - \sigma_2)$$ 
where $\Vert \cdot \Vert$ is the usual Euclidean norm.
A sequence $\{\sigma_\N\}_{\N \in \naturals}$ of paths is said to converge to  a path $\bar\sigma$, denoted as $\lim_{\N \to \infty} \sigma_\N = \bar\sigma$, if the norm of the difference between $\sigma_\N$ and $\bar\sigma$ converges to zero, i.e., $\lim_{\N \to \infty} \Vert \sigma_\N - \bar\sigma \Vert_\mathrm{BV} = 0$.

A feasible path $\sigma^*\in \X_\mathrm{free}$ that solves the optimality problem (Problem~\ref{problem:optimality}) is said to be a {\em robustly optimal solution} if it has weak $\delta$-clearance and, for any sequence of  collision-free paths $\{\sigma_\N\}_{\N \in \naturals}$, $\sigma_\N \in \X_\mathrm{free}$, $\forall \N \in \naturals$, such that $\lim_{\N \to \infty} \sigma_\N = \sigma^*$, $\lim_{\N \to \infty}c(\sigma_\N) = c(\sigma^*)$. 
Clearly, a path planning problem that has a robustly optimal solution is necessarily robustly feasible. Let  $c^*=c(\sigma^*)$ be the cost of an optimal path, and let ${Y}_\N^\Alg$ be the extended random variable corresponding to the cost of the minimum-cost solution included in the graph returned by \Alg  at the end of iteration $\N$.
\begin{definition}[Asymptotic Optimality] \label{definition:asymptotic_optimality}
An algorithm ALG is asymptotically optimal if, for any path planning problem   $(\X_\mathrm{free}, x_\mathrm{init}, \X_\mathrm{goal})$ and cost function $c : \Sigma \to \reals_{\ge 0}$ that admit a robustly optimal solution with finite cost $c^*$,
$$
\PP \left(\left\{ \limsup_{\N \to \infty}{Y}_\N^\Alg = c^* \right\}\right) = 1. 
$$
\end{definition}

Note that, since ${Y}_\N^\Alg\ge c^*$, $\forall \N \in \naturals$, asymptotic optimality of \Alg  implies that the limit $\lim\nolimits_{\N \to \infty} {Y}^\Alg_\N$ exists, and is equal to $c^*$. Clearly, probabilistic completeness is necessary for asymptotic optimality. Moreover, the probability that a sampling-based algorithm converges to an optimal solution almost surely has probability either zero or one. That is, a sampling-based algorithm either converges to the optimal solution in almost all runs, or the convergence does not occur in almost all runs.

\begin{lemma} \label{lemma:kolmogorov_zero_one}
Given that $\limsup_{\N \to \infty} Y_\N^\Alg < \infty$, i.e., $\Alg$ finds a feasible solution eventually, the probability that $\limsup \nolimits_{\N \to \infty} Y_\N^\Alg =c^*$  is either zero or one.
\end{lemma}
\begin{proof}
Conditioning on the event $\{\limsup_{\N \to \infty} Y_\N^\Alg < \infty\}$ ensures that $Y_\N^\Alg$ is finite, thus a random variable, for all large $\N$. 
Given a sequence $\{Y_n\}_{n \in \naturals}$ of random variables, let ${\cal F}_m'$ denote the $\sigma$-field generated by the sequence $\{Y_n\}_{n = m}^\infty$ of random variables.
The tail $\sigma$-field ${\cal T}$ is defined as ${\cal T} = \bigcap_{\N \in \naturals} {\cal F}_\N'$.
An event $A$ is said to be a {\em tail event} if $A \in {\cal T}$. Any tail event occurs with probability either zero or one by the Kolmogorov zero-one law~\citep{resnick.book99}.
Consider the sequence $\{ Y_\N^\Alg \}_{\N \in \naturals}$ of random variables. Let ${\cal F}_m'$ denote the $\sigma$-fields generated by $\{ Y_n^\Alg \}_{n = m}^{\infty}$. Then,
$
\left\{\limsup\nolimits_{\N\to\infty} Y_\N^\Alg = c^* \right\} = \left\{ \limsup\nolimits_{\N \to \infty, \,\N \ge m} Y_\N^\Alg = c^* \right\} \in {\cal F}_m' \mbox{ for all } \N \in \naturals.
$
Hence, $\left\{ Y_\N^\Alg = c^* \right\} \in \bigcap_{\N \in \naturals} {\cal F}_m'$ is a tail event. The result follows by the Kolmogorov zero-one law.\qed
\end{proof}

Among the first steps in assessing the asymptotic optimality properties of an algorithm \Alg  is determining whether  the limit $\lim\nolimits_{\N \to \infty} Y_\N^\Alg$ exists. It turns out that if the graphs returned by \Alg  satisfy a monotonicity property, then the limit exists, and is in general a random variable, indicated with $Y_\infty^\Alg$.  

\begin{lemma}\label{lemma:monotonicity} If $G_i^\Alg(\omega) \subseteq G_{i+1}^\Alg(\omega)$, $\forall \omega \in \Omega$ and $\forall i \in \naturals$, then 
$\lim_{\N\to \infty} Y_\N^\Alg(\omega) = Y^\Alg_\infty(\omega).$
\end{lemma}
\begin{proof}
Since $G_i^\Alg(\omega) \subseteq G_{i+1}^\Alg(\omega)$, then $Y_{i+1}^\Alg(\omega) \le 
Y_{i}^\Alg(\omega)$, for all $\omega \in \Omega$. Since $Y_i^\Alg \ge c^*$, then the sequence converges to some limiting value, dependent on $\omega$, i.e., $Y_\infty^\Alg(\omega)$.\qed
\end{proof}
Of the algorithms presented in Section \ref{section:algorithms}, it is easy to check that PRM, sPRM, RRT, RRG, and RRT$^*$ satisfy the monotonicity property in Lemma \ref{lemma:monotonicity}. 
On the other hand, $k$-nearest sPRM and PRM$^*$ do not: in these cases, the random variable $Y^\Alg_{i+1}$ is not necessarily dominated by $Y^\Alg_i$. This is evident in numerical experiments, 
e.g., see Figures \ref{figure:prm_vs_prmstar_2d} and \ref{figure:prm_to_5d} in Section \ref{section:experiments}.

In order to avoid trivial cases of asymptotic optimality, it is necessary to rule out problems in which optimal solutions can be computed after a finite number of samples.  
Let $\Sigma^*$ denote the set of all optimal paths, i.e., the set of all paths that solve the
optimal planning problem (Problem~\ref{problem:optimality}), and $\X_\mathrm{opt}$ denote the set
of states that an optimal path in $\Sigma^*$ passes through, i.e.,
$$
\X_\mathrm{opt} = \{ x \in X_\mathrm{free} \,\vert\, \exists \sigma^* \in \Sigma^*,  \tau \in [0,1] \mbox{ such that } x = \sigma^*(\tau) \}.
$$

\begin{assumption}[Zero-measure Optimal Paths] \label{assumption:zeromeasureoptimal} %
  The set of all points traversed by an optimal trajectory has measure
  zero, i.e., $ \mu \left( \X_\mathrm{opt} \right) = 0$.
\end{assumption}
Most cost functions and problem instances of interest satisfy this assumption, including, e.g., the Euclidean length of the path when the goal region is convex. This assumption does not imply that there is a single optimal path; indeed, there are problem instances with uncountably many optimal paths, for which Assumption~\ref{assumption:zeromeasureoptimal} holds. (A simple example is the motion planning problem in three dimensional Euclidean space where a ball shaped obstacle is placed between the initial state and the goal region.)
Assumption~\ref{assumption:zeromeasureoptimal} implies that no sampling-based planning algorithm can find a solution to the optimality problem in a finite number of iterations.
\begin{lemma} \label{lemma:nonequal_optimal} %
  If Assumption \ref{assumption:zeromeasureoptimal} holds, the probability that a sampling-based algorithm \Alg  returns a graph containing an optimal path at a finite iteration $\N \in \naturals$
  is zero, i.e.,
  $$
  \PP \left( \cup_{\N \in \naturals} \{ {Y}^\Alg_\N = c^*\} \right) = 0.
  $$
\end{lemma}
\begin{proof}
Let $B_\N$ denote the event that \Alg  constructs a graph containing a path with cost exactly equal to $c^*$ at the end of iteration $i$, i.e., $B_\N = \{{ Y}^\Alg_\N = c^*\}$. Let $B$ denote the event that 
\Alg   returns a graph containing a path that costs exactly $c^*$ at some finite iteration $i$. Then, $B$ can be written as $B = \cup_{\N \in \naturals} B_\N$. Since $B_\N \subseteq B_{\N+1}$, by monotonocity of measures, $\lim_{i \to \infty} \PP (B_\N) = \PP(B)$. By Assumption \ref{assumption:zeromeasureoptimal}  and the definition of the sampling procedure, $\PP (B_\N) = 0$ for all $\N \in \naturals$, since the probability that the set $\bigcup_{i = 1}^\N \{ {\tt SampleFree}(i)\}$ of points contains a point from a zero-measure set is zero. Hence,  $\PP (B) = 0$. 
\qed\end{proof}
In the remainder of the paper, it will be tacitly assumed that Assumption \ref{assumption:zeromeasureoptimal}, and hence Lemma \ref{lemma:nonequal_optimal}, hold. 

\subsubsection{Existing algorithms}

\label{section:nonoptimality}

The algorithms in Section \ref{section:oldalgo} were originally introduced to efficiently solve the feasibility problem, relaxing the completeness requirement to probabilistic completeness. Nevertheless, it is of interest to establish whether these algorithms are asymptotically optimal in addition to being probabilistically complete.  (The first two results in this section rely on results that will be proven in Section~\ref{sec:optimalitynew}, i.e., the fact that the RRT algorithm is not asymptotically optimal, and the PRM$^*$ algorithm is asymptotically optimal)

First, consider the PRM algorithm and its variants. The PRM algorithm, in its original form, is not asymptotically optimal. 

\begin{theorem}[Non-optimality of PRM] The PRM algorithm is not asymptotically optimal.
\label{thm:nonoptimalityofPRM}
\end{theorem}
{\begin{proof}
The proof is based on a counterexample,  establishing a form of equivalence between PRM and RRT, which in turn will be proven not to be asymptotically optimal in Theorem \ref{theorem:optimality_rrt}. 
Consider a convex obstacle-free environment, e.g., 
$\X_\mathrm{free}=\X$, and choose the connection radius for PRM and the steering parameter for RRT  such that $r, \eta > \mathrm{diam}(\X)$. At each iteration, exactly one vertex and one edge is added to the graph, since (i) all connection attempts using the local planner (e.g., straight line connections as considered in this paper) are collision-free, and (ii) at the end of each iteration, the graph is connected (i.e., it  contains only one connected component). In particular, the graph returned by the PRM algorithm in this case is a tree, and the arborescence obtained by choosing as the root the first sample point, i.e., $\mathtt{SampleFree}_0$, is an online nearest-neighbor graph (see Section~\ref{section:rgg}) coinciding with the graph returned by RRT with the random initial condition $x_\mathrm{init}=\mathtt{SampleFree}_0$. 

Recall that the PRM algorithm is applicable for multiple-query planning problems: in other words, the graph returned by the PRM algorithm is used to solve path planning problems from arbitrary $x_\mathrm{init}\in \X_\mathrm{free}$ and $\X_\mathrm{goal} \subset \X_\mathrm{free}$. (Note that all such problems admit robust optimal solutions.) In particular, for $x_\mathrm{init} = \mathtt{SampleFree}_0$, and any $X_\mathrm{goal}$, then $Y_\N^\mathrm{PRM}(\omega) = Y_\N^\mathrm{RRT}(\omega)$, for all $\omega \in \Omega$, $\N \in \naturals$. In particular, since both PRM and RRT satisfy the monotonicity condition in Lemma \ref{lemma:monotonicity}, Theorem~\ref{theorem:optimality_rrt} implies that 
$$
\PP \left( \left\{\limsup_{\N \to \infty} Y_\N^\mathrm{PRM} = c^* \right\}\right) =
\PP \left( \left\{\lim_{\N \to \infty} Y_\N^\mathrm{PRM} = c^* \right\}\right) = 
\PP \left( \left\{\lim_{\N \to \infty} Y_\N^\mathrm{RRT} = c^* \right\}\right) = 0.$$
\qed
\end{proof}

 The lack of asymptotic optimality of PRM is due to its incremental construction, coupled with the constraint eliminating edges making unnecessary connections within a connected component. 
Such a constraint is not present in the batch construction of the sPRM algorithm, which is indeed asymptotically optimal (at the expense of computational complexity, see Section \ref{section:complexity}).

\begin{theorem}[Asymptotic Optimality of sPRM]
The sPRM algorithm is asymptotically optimal.
\label{thm:optimalityofsPRM}
\end{theorem}

\begin{proof}
By construction, $V_\N^\mathrm{sPRM} (\omega) = V_\N^{\mathrm{PRM}^*}(\omega)$, and $E_\N^\mathrm{sPRM} (\omega) \supseteq E_\N^{\mathrm{PRM}^*} (\omega)$ for all $\omega \in \Omega$. Hence, the graph returned by sPRM includes all the paths that are present in the graph returned by PRM$^*$. Then, asymptotic optimality of sPRM follows from that of PRM$^*$, which will be proven in Theorem~\ref{theorem:optimality_prmstar}.
\qed
\end{proof}

On the other hand, as in the case of probabilistic completeness, the heuristics that are often used in the practical implementation of (s)PRM are not asymptotically optimal.

\begin{theorem}[Non-optimality of $k$-nearest sPRM] \label{theorem:nonoptimality_kPRM}
The $k$-nearest sPRM algorithm is not asymptotically optimal, for any constant $k \in \naturals$.
\end{theorem}

This theorem will be proven under the assumption that the underlying point process is Poisson. More precisely, the algorithm is analyzed when it is run with $\Poisson(\N)$ samples. That is, the realization of the random variable $\Poisson(\N)$ determines the  number of points sampled independently and uniformly in $\X_\mathrm{free}$. Hence, the expected number of samples is equal to $\N$, although its realization may slightly differ. However, since the Poisson random variable has exponentially-decaying tails, its large deviations from its mean is unlikely (see, e.g.,~\citet{grimmett.stirzaker.book01} for a more precise statement). With a slight abuse of notation, the cost of the best path in the graph returned by the $k$-nearest sPRM algorithm when the algorithm is run with $\Poisson(\N)$ number of samples is denoted by $Y_\N^{k\mathrm{PRM}}$, and 
it is shown that $\PP (\{ \limsup_{\N \to \infty} Y_\N^{k\mathrm{PRM}} = c^*  \}) = 0$. 

\begin{proof}[Proof of Theorem~\ref{theorem:nonoptimality_kPRM}]

Let $\sigma^*$ denote an optimal path and $s^*$ denote its length, i.e., $s^*=TV(\sigma^*)$. For each $\N$, consider a tiling of $\sigma^*$ with disjoint open hypercubes, each with edge length $2\, \N^{-1/d}$, such that the center of each cube is a point on $\sigma^*$. See Figure~\ref{figure:cube_tiling}. Let $M_\N$ denote the maximum number of tiles that can be generated in this manner and note
$
M_\N \ge \frac{s^*}{2} \, \N^{1/d}.
$
Partition each tile into several open cubes as follows: place an inner cube with edge length $\N^{-1/d}$ at the center of the tile and place several outer cubes each with edge length $\frac{1}{2} \, \N^{-1/d}$ around the cube at the center as shown in Figure~\ref{figure:cube_tiling}. 
Let $F_d$ denote the number of outer cubes. 
The volumes of the inner cube and each of the outer cubes are $\N^{-1}$ and $2^{-d}\, \N^{-1}$, respectively.

\begin{figure}[hb]
\centering
\includegraphics[height=2cm]{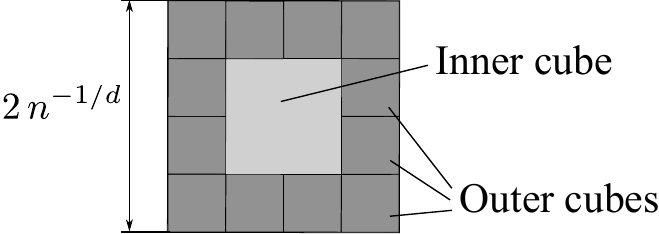} \quad\quad\quad 
\includegraphics[height=2cm]{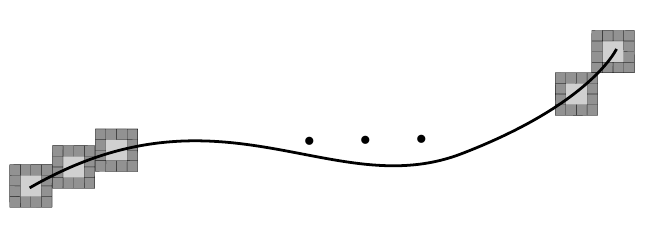}
\caption{An illustration of the tiles mention in the proof of Theorem~\ref{theorem:nonoptimality_kPRM}. A single tile is shown in the left; a tiling of the optimal trajectory $\sigma^*$ is shown on the right.}
\label{figure:cube_tiling}
\end{figure}

For $\N \in \naturals$ and $m \in \{1,2,\dots, M_\N \}$, consider the tile $m$ when the algorithm is run with $\Poisson(\N)$ samples. Let $I_{\N,m}$ denote the indicator random variable for the event that the center cube of this tile contains no samples, whereas every outer cube contains at least $k+1$ samples, in tile $m$. 

The probability that the inner cube contains no samples is $e^{-1/\mu(\X_\mathrm{free})}$. The probability that an outer cube contains at least $k+1$ samples is $1 - \PP \left( \{ \Poisson{(2^{-d}/\mu(\X_\mathrm{free}))} \ge k + 1 \} \right) = 1 - \PP (\{ \Poisson{(2^{-d}/\mu(\X_\mathrm{free}))} \le k \}) = 1 - \frac{\Gamma(k+1, 2^{-d}/\mu(\X_\mathrm{free}))}{k!}$, where $\Gamma(\cdot,\cdot)$ is the incomplete gamma function~\citep{abramowitz.stegun.book64}. 
Then, noting that the cubes in a given tile are disjoint and using the independence property of the Poisson process (see Lemma~\ref{lemma:poissonization}), 
$$
\EE\left[ I_{\N,m} \right] 
\,\,=\,\, 
e^{-1/\mu(\X_\mathrm{free})} \, \left(1 - \frac{\Gamma(k+1, 2^{-d}/\mu(\X_\mathrm{free}))}{k!}\right)^{F_d} 
\,\,>\,\, 0,
$$
which is a constant that is independent of $\N$; denote this constant by $\alpha$.

Let $G_\N = (V_\N , E_\N )$ denote the graph returned by the $k$-nearest PRM algorithm by the end of $\Poisson(\N)$ iterations. Observe that if $I_{\N,m} = 1$, then there is no edge of $G_\N$ crossing the cube of side length $\frac{1}{2} \, \N^{-1/d}$ that is centered at the center of the inner cube in tile $m$ (shown as the white cube in Figure~\ref{figure:cube_tiling_edge_crossing}). To prove this claim, note the following two facts. First, no point that is outside of the cubes can have an edge that crosses the inner cube. Second, no point in one of the outer cubes has an edge that has length greater than $\frac{\sqrt{d}}{2} \, i^{-1/d}$. Thus, no edge can cross the white cube illustrated in Figure~\ref{figure:cube_tiling_edge_crossing}.

\begin{figure}[ht]
\centering
\includegraphics[height=3cm]{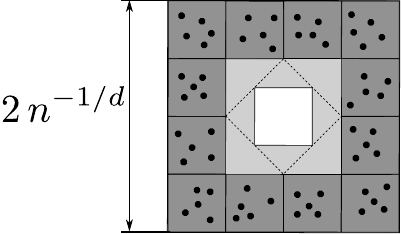}
\caption{The event that the inner cube contains no points and each outer cube contains at least $k$ points of the point process is illustrated. The cube of side length $\frac{1}{2}\,\N^{-1/d}$ is shown in white.}
\label{figure:cube_tiling_edge_crossing}
\end{figure}

Let $\sigma_\N$ denote the path in $G_\N$ that is closest to $\sigma^*$ in terms of the bounded variation norm. Let $U_\N := \Vert \sigma_\N - \sigma^*\Vert_\mathrm{BV}$. Notice that $U_\N \ge \frac{1}{2}\,\N^{-1/d} \, \sum_{m = 1}^{M_\N} I_{\N,m} = \frac{1}{2}\,\N^{-1/d} \, M_\N\, I_{\N,1} = \frac{s^*}{4} I_{\N,1}$. Then, 
$$
\EE\left[ \limsup_{\N \to \infty} U_\N \right] 
\,\,\ge\,\, 
\limsup_{\N \to \infty} \EE\left[U_\N \right] 
\,\,\ge\,\,
\limsup_{\N \to \infty} \frac{s^*}{4}\,\EE\left[I_{\N,m}\right] 
\,\,\ge\,\,
\frac{\alpha\, s^*}{4} 
\,\,>\,\, 
0,
$$
where the first inequality follows from Fatou's lemma~\citep{resnick.book99}. This implies $\PP (\{\limsup_{\N \to \infty} U_\N > 0\}) > 0$. Since $U_i > 0$ implies $Y_\N > c^*$ surely,
$$
\PP \left(\left\{ \limsup\nolimits_{\N \to \infty} Y_\N > c^*\right\}\right) \ge \PP \left( \left\{ \limsup\nolimits_{\N \to \infty}U_\N > 0 \right\}\right) > 0.
$$
That is, $\PP \left(\left\{ \limsup\nolimits_{\N \to \infty} Y_\N = c^*\right\}\right) < 1 $. In fact, by Lemma~\ref{lemma:kolmogorov_zero_one}, $\PP \left(\left\{ \limsup\nolimits_{\N \to \infty} Y_\N = c^*\right\}\right) = 0$.\qed
\end{proof}

Second, asymptotic optimality of a large class of variable radius sPRM algorithms is considered. 
Consider a variable radius sPRM in which connection radius satisfies $r(\N) \le \gamma \, \N^{-1/d}$ for some $\gamma > 0$ and for all $\N \in \naturals$. The next theorem shows that this algorithm lacks the asymptotic optimality property.

\begin{theorem}[Non-optimality of variable radius sPRM with $r(\N) = \gamma\,\N^{-1/d}$] \label{theorem:nonoptimality_vrPRM} %
Consider a variable radius sPRM algorithm with connection radius 
$
r(\N) = \gamma \, \N^{-1/d}.
$
This sPRM algorithm is not asymptotic optimal for any $\gamma \in \reals_{\ge 0}$.
\end{theorem}
\begin{proof}
Let $\sigma^*$ denote a path that is a robust solution to the optimality problem.
Let $\N$ denote the number of samples that the algorithm is run with.
For all $\N$, construct a set $B_\N = \{B_{\N,1}, B_{\N,2}, \dots, B_{\N,M_\N}\}$ of openly disjoint balls as follows. Each ball in $B_\N$ has radius $r_\N = \gamma\, \N^{-1/d}$, and lies entirely inside $\X_\mathrm{free}$. Furthermore, the balls in $B_\N$ ``tile'' $\sigma^*$ such that the center of each ball lies on $\sigma^*$ (see Figure~\ref{figure:tiling_optimal}). 
Let $M_\N$ denote the maximum number of balls, $\bar{s}$ denote the length of the portion of $\sigma^*$ that lies within the $\delta$-interior of $\X_\mathrm{free}$, and $\N_0 \in \naturals$ denote the number for which $r_\N \le \delta$ for all $\N \ge \N_0$.

Then, for all $\N \ge \N_0$,
$$
M_\N \ge \frac{\bar{s}}{2 \, \gamma\, \left(\frac{1}{\N}\right)^{1/d}} = \frac{\bar{s}}{2\,\gamma} \, \N^{1/d}.
$$

\begin{figure}[htb]
\centering
\includegraphics[height = 3cm]{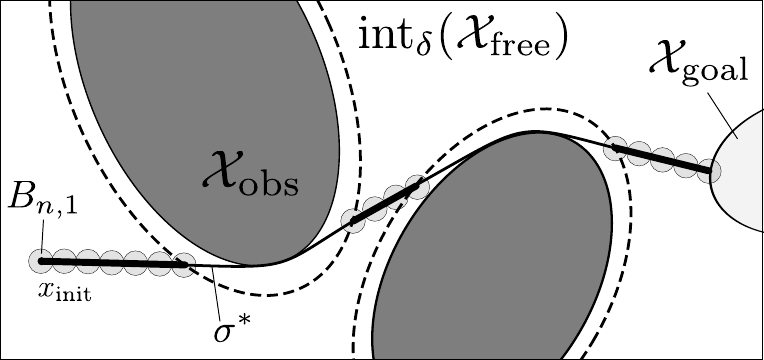}
\caption{An illustration of the covering of the optimal path, $\sigma^*$, with openly disjoint balls. The balls cover only a portion of $\sigma^*$ that lies within the $\delta$-interior of $\X_\mathrm{free}$.}
\label{figure:tiling_optimal}
\end{figure}

Indicate the graph returned by this sPRM algorithm as $G_\N=(V_\N, E_\N)$. Denote the event that the ball $B_{\N,m}$ contains no vertex in $V_\N$ by $A_{\N,m}$. Denote the indicator random variable for the event $A_{\N,m}$ by $I_{\N,m}$, i.e., $I_{\N,m} = 1$ when $A_{\N,m}$ holds and $I_{\N,m} = 0$ otherwise. Then, for all $\N \ge \N_0$,
$$
\EE[I_{\N,m}] = \PP(A_{\N,m}) = \left( 1 - \frac{\mu(B_{\N,m})}{\mu(\X_\mathrm{free})} \right)^\N
= \left(1 - \frac{\VolumeDBall{d}\, \gamma^d}{\mu(\X_\mathrm{free})} \, \frac{1}{\N} \right)^\N
$$

Let $N_\N$ be the random variable that denotes the total number of balls in $B_\N$ that contain no vertex in $V_\N$, i.e., $N_\N = \sum_{m = 1}^{M_\N} I_{\N,m}$. Then, for all $\N \ge \N_0$,
$$
\EE [N_\N] 
\,\,=\,\, 
\EE \left[\sum\nolimits_{m = 1}^{M_\N} I_{\N,m}\right] 
\,\,=\,\, 
\sum_{m = 1}^{M_\N} \EE [I_{\N,m}] 
\,\,=\,\, 
M_\N \,\, \EE[I_{\N,1}] 
\,\,\ge\,\, 
\frac{\bar{s}}{2 \, \gamma} \, \N^{1/d}\, \left(1 - \frac{\VolumeDBall{d} \, \gamma^d}{\mu(\X_\mathrm{free})}\frac{1}{\N} \right)^\N.
$$

Consider a ball $B_{\N,m}$ that contains no vertices of this sPRM algorithm. Then, no edges of the graph returned by this algorithm cross the ball of radius $\frac{\sqrt{3}}{2}r_\N$ centered at the center of $B_{\N,m}$. See Figure~\ref{figure:ball_prm_nonoptimality}. 
\begin{figure}[htb]
\centering
\includegraphics[height = 5cm]{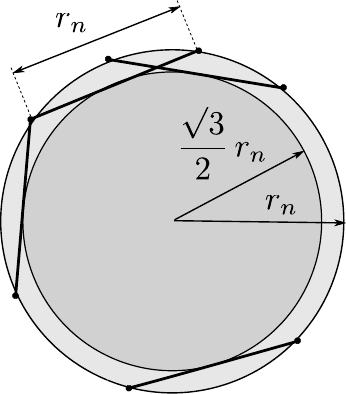}
\caption{If the outer ball does not contain vertices of the PRM graph, then no edge of the graph corresponds to a path crossing the inner ball.}
\label{figure:ball_prm_nonoptimality}
\end{figure}

Let $P_\N$ denote the (finite) set of all acyclic paths that reach the goal region in the graph returned by this sPRM algorithm when the algorithm is run with $\N$ samples.
Let $U_\N$ denote the total variation of the path that is closest to $\sigma^*$ among all paths in $P_\N$, i.e., $U_\N := \min_{\sigma_\N \in P_\N} \Vert \sigma_\N - \sigma^* \Vert_\mathrm{BV}$. Then, 
$$
\EE[ U_\N ] 
\,\,\ge\,\, 
\EE\left[\gamma \left(\frac{1}{\N}\right)^{1/d} \, N_\N \right] 
\,\,\ge \,\,
\frac{\bar{s}}{2}\, \left(1 - \frac{\VolumeDBall{d} \, \gamma^d}{\mu(\X_\mathrm{free})}\frac{1}{\N} \right)^\N.
$$
Taking the limit superior of both sides, the following inequality can be established:
$$
\EE\left[ \limsup_{\N \to \infty} U_\N\right]  
\,\,\ge\,\,
\limsup_{\N \to \infty} \EE\left[ U_\N \right] 
\,\,\ge\,\, 
\limsup_{\N \to \infty} \frac{\bar{s}}{2}\, \left(1 - \frac{\VolumeDBall{d} \, \gamma^d}{\mu(\X_\mathrm{free})}\frac{1}{\N} \right)^\N
\,\,= \,\, \frac{\bar{s}}{2} \, e^{-\frac{\VolumeDBall{d} \, \gamma^d}{\mu(\X_\mathrm{free})}} > 0,
$$
where the first inequality follows from Fatou's lemma~\citep{resnick.book99}. Hence, $\PP (\{ \limsup_{\N \to \infty} U_\N >0 \}) > 0$, which implies that 
$\PP \left( \left\{ \limsup_{\N \to \infty} Y_\N^\Alg > c^* \right\} \right) > 0$. That is, $\PP \left( \left\{ \limsup_{\N \to \infty} Y_\N^\Alg = c^* \right\} \right) < 1$. In fact, $\PP \left( \left\{ \limsup_{\N \to \infty} Y_\N^\Alg = c^* \right\} \right) = 0$ by the Kolmogorov zero-one law (see Lemma~\ref{lemma:kolmogorov_zero_one}). 
\qed
\end{proof}

\paragraph{Rapidly-exploring Random Trees}
In this section, it is shown that the minimum-cost path in the RRT algorithm converges to a certain random variable, however, under mild technical assumptions, this random variable is {\em not} equal to the optimal cost, with probability one. 

\begin{theorem}[Non-optimality of RRT] \label{theorem:optimality_rrt} %
The RRT algorithm is not asymptotically optimal.
\end{theorem}
The proof of this theorem can be found in Appendix~\ref{section:proof:theorem:optimality_rrt}. 
Note that, since at each iteration the RRT algorithm either adds a vertex and an edge, or leaves the graph unchanged, $G_i^\mathrm{RRT}(\omega)\subseteq G_{i+1}^\mathrm{RRT}(\omega)$, for all $i \in \naturals$ and all $\omega \in \Omega$, and hence the limit $\lim_{\N\to\infty} Y_\N^\mathrm{RRT}$ exists and is equal to the random variable  $Y_\infty^\mathrm{RRT}$.
In conjunction with Lemma~\ref{lemma:kolmogorov_zero_one}, Theorem~\ref{theorem:optimality_rrt} implies that this limit is strictly greater than $c^*$ almost surely, i.e., $\PP \left(\{\lim_{\N \to \infty} {Y}^\mathrm{RRT}_\N > c^* \}\right) = 1$. In other words,  the cost of the best solution returned by RRT converges to a suboptimal value, with probability one. In fact, it is possible to construct problem instances such that the probability that the first solution returned by the RRT algorithm has arbitrarily high cost is bounded away from zero~\citep{Nechushtan.Raveh.ea:10}. 

Since the cost of the best path returned by the RRT algorithm converges to a random variable, Theorem~\ref{theorem:optimality_rrt} provides new insight explaining the effectiveness of  approaches as in~\citet{ferguson.stentz.iros06}. In fact, running multiple instances of the RRT algorithm amounts to drawing multiple samples of ${Y}^\mathrm{RRT}_\infty$.

\subsubsection{Proposed algorithms }
\label{sec:optimalitynew}

In this section, the proposed algorithms are analyzed for asymptotic optimality, i.e., almost sure convergence to optimal solutions. It is shown that the PRM$^*$, RRG, and RRT$^*$ algorithms, as well as their $k$-nearest implementations, are all asymptotically optimal.
The proofs of the following theorems are quite lengthy, and will be provided in the appendix.

Recall that $d$ denotes the dimensionality of the configuration space, $\mu(\X_\mathrm{free})$ denotes the Lebesgue measure of the obstacle-free space, and $\VolumeDBall{d}$ denotes the volume of the unit ball in the $d$-dimensional Euclidean space. Proofs of the following theorems can be found in Appendices \ref{proof:optimality_prmstar}--\ref{proof:optimality_rrtstar}.

\begin{theorem}[Asymptotic optimality of PRM$^*$] \label{theorem:optimality_prmstar}
If $\gamma_\mathrm{PRM} > 2 \, (1 + 1/d)^{1/d} \, \left( \frac{\mu(X_\mathrm{free})}{\VolumeDBall{d}} \right)^{1/d}$, then the PRM$^*$ algorithm is asymptotically optimal.
\end{theorem}

\begin{theorem}[Asymptotic optimality of $k$-nearest PRM$^*$]\label{theorem:optimality_k_prmstar}
If $k_\mathrm{PRM} > e \, (1 + 1/d) $, then the $k$-nearest implementation of the PRM$^*$ algorithm is asymptotically optimal. 
\end{theorem}

\begin{theorem}[Asymptotic optimality of RRG] \label{theorem:optimality_rrg}
If $\gamma_\mathrm{PRM} > 2 \, (1 + 1/d)^{1/d} \, \left( \frac{\mu(X_\mathrm{free})}{\VolumeDBall{d}} \right)^{1/d}$, then the RRG algorithm is asymptotically optimal.
\end{theorem}

\begin{theorem}[Asymptotic optimality of $k$-nearest RRG] \label{theorem:optimality_k_rrg}
If $k_\mathrm{RRG} > e \, (1 + 1/d)$, then the $k$-nearest implementation of the RRG algorithm is asymptotically optimal.
\end{theorem}

\begin{theorem}[Asymptotic optimality of RRT$^*$] \label{theorem:optimality_rrtstar}
If $\gamma_{\mathrm{RRT}^*} > (2 \, (1 + 1/d))^{1/d} \,\left( \frac{\mu(X_\mathrm{free})}{\VolumeDBall{d}} \right)^{1/d}$, then the RRT$^*$ algorithm is asymptotically optimal.
\end{theorem}

\begin{theorem}[Asymptotic optimality of $k$-nearest RRT$^*$] \label{theorem:optimality_k_rrtstar}
If $k_{\mathrm{RRT}^*} >  2^{d+1} \, e \, (1 + 1/d)$, then the $k$-nearest implementation of the RRT$^*$ algorithm is asymptotically optimal.
\end{theorem}
The proof of the latter theorem follows from those of Theorems~\ref{theorem:optimality_k_rrg} and \ref{theorem:optimality_rrtstar}.

\subsection{Computational Complexity} \label{section:complexity} 

The objective of this section is to compare the computational complexity of the algorithms provided in Section~\ref{section:algorithms}. First, each algorithm is analyzed in terms of the number of calls to the ${\tt CollisionFree}$ procedure. Second, the computational complexity of certain primitive procedures such as ${\tt Nearest}$ and ${\tt Near}$ (see Section~\ref{section:algorithms:primitive_procedures}) are analyzed. Using these results, a thorough analysis of the computational complexity of the all the algorithms is given in terms of the number of simple operations, such as comparisons, additions, multiplications. An analysis of the computational complexity of the query phase, i.e., the complexity of extracting the optimal solution from the graph returned by these algorithms, is also provided.

The following notation for asymptotic computational complexity will be used throughout this section.
Let $W^\Alg_\N (P)$ be a function of the graph returned by algorithm \Alg  when \Alg  is run with inputs $P = (\X_\mathrm{free}, x_\mathrm{init}, \X_\mathrm{goal})$ and $\N$. Clearly, $W^\Alg_\N (P)$ is a random variable. Let $f: \naturals \to \naturals$ be an increasing function with $\lim_{\N \to \infty} f(\N) = \infty$. The random variable $W^\Alg_\N$ is said belong to $\Omega(f(\N))$, denoted as $W^\Alg_\N \in \Omega (f(\N))$, if there exists a problem instance $P = (\X_\mathrm{free}, x_\mathrm{init}, \X_\mathrm{goal})$ such that $\liminf_{\N \to \infty} \EE [W^\Alg_\N (P) / f(\N)] > 0$. Similarly, $W^\Alg_\N$ is said to belong to $O(f(\N))$ if $\limsup_{\N \to \infty} \EE[W^\Alg_\N (P) / f(\N)] < \infty$ for all problem instances $P = (\X_\mathrm{free}, x_\mathrm{init}, \X_\mathrm{goal})$.

\paragraph{Number of calls to the ${\tt CollisionFree}$ procedure}
Let $\NumObs_\N^\Alg$ denote the total number of calls to the ${\tt CollisionFree}$ procedure by algorithm \Alg  in iteration $\N$. 

First, lower-bounds are established for the PRM and sPRM algorithms.

\begin{lemma}[PRM] \label{lemma:complexity:num_obs:prm}
$\NumObs^\AlgPRM_\N \in \Omega(\N)$.
\end{lemma}
\begin{proof}
Consider the problem instance $(\X_\mathrm{free}, x_\mathrm{init}, \X_\mathrm{goal})$, where $\X_\mathrm{free}$ is composed of two openly-disjoint sets $\X_1$ and $\X_2$ (see Figure~\ref{figure:prm_collisionfree_complexity}). The set $\X_2$ is designed to be a hyperrectangle shaped set with one side equal to $r/2$, where $r$ is the connection radius. 

\begin{figure}[ht]
\begin{center}
\includegraphics[height = 3cm]{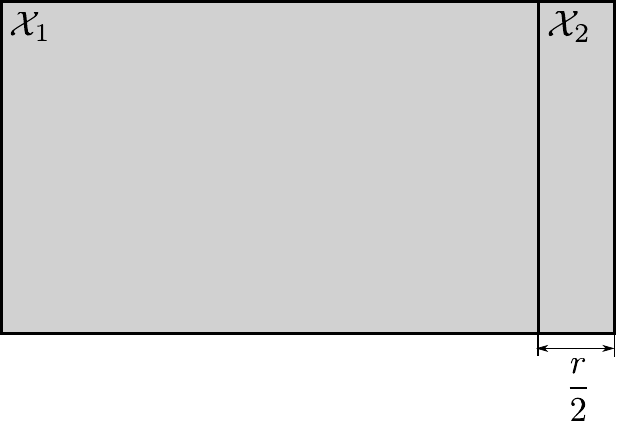}
\caption{An illustration of $\X_\mathrm{free} = \X_1 \cup \X_2$.}
\label{figure:prm_collisionfree_complexity}
\end{center}
\end{figure}

Any $r$-ball centered at a point in $\X_2$ will certainly contain a nonzero measure part of $\X_2$. Define $\bar{\mu}$ as the volume of the smallest region in $\X_2$ that can be intersected by an $r$-ball centered at $\X_2$, i.e., $\bar{\mu} := \inf_{x \in \X_2} \mu({\cal B}_{x,r} \cap \X_1)$. Clearly, $\bar{\mu} > 0$. 

Thus, for any sample $\Z_\N$ that falls into $\X_2$, the PRM algorithm will attempt to connect $\Z_\N$ to a certain number of vertices that lies in a subset $\X_1'$ of $\X_1$ such that $\mu(\X_1') \ge \bar{\mu}$. The expected number of vertices in $\X_1'$ is at least $\bar{\mu} \, \N$. Moreover, none of these vertices can be in the same connected component with $\Z_\N$. Thus, $\EE[\NumObs^\AlgPRM_\N / \N] > \bar{\mu}$. The result is obtained by taking the limit inferior of both sides.
\qed
\end{proof}

\begin{lemma}[sPRM] \label{lemma:complexity:num_obs:sprm}
$\NumObs^\AlgsPRM_\N \in \Omega (\N)$.
\end{lemma}
\begin{proof}
The proof of a stronger result is provided. It is shown that for all problem instances $P = (\X_\mathrm{free}, x_\mathrm{init}, \X_\mathrm{goal})$, $\liminf_{\N \to \infty} \EE[\NumObs^\AlgsPRM_\N/\N] > 0$, which implies the lemma.
Recall from Algorithm~\ref{algorithm:sPRM} that $r$ denotes the connection radius. Let $\bar{\mu}$ denote the volume of the smallest region that can be formed by intersecting $\X_\mathrm{free}$ with an $r$-ball centered at a point inside $\X_\mathrm{free}$, i.e., 
$
\bar{\mu} := \inf_{x \in \X_\mathrm{free}}\mu({\cal B}_{x,r} \cap \X_\mathrm{free}).
$
Recall that $\X_\mathrm{free}$ is the closure of an open set. Hence, $\bar{\mu} > 0$.

Clearly, $\NumObs_\N$, the number of calls to the ${\tt CollisionFree}$ procedure in iteration $\N$, is equal to the number of nodes inside the ball of radius $r$ centered at the last sample point $X_\N$. Moreover, the volume of the $\X_\mathrm{free}$ that lies inside this ball is at least $\bar{\mu}$. 
Then, the expected value of $\NumObs_\N$ is lower bounded by the expected value of a binomial random variable with parameters $\bar{\mu}/\mu(\X_\mathrm{free})$ and $\N$, since the underlying point process is binomial. Thus, 
$
\EE[\NumObs^\AlgsPRM_\N] \ge \frac{\bar{\mu}}{\mu(\X_\mathrm{free})} \, \N.
$
Then, $\EE[\NumObs_\N/\N] \ge \bar{\mu}/\X_\mathrm{free}$ for all $\N \in \naturals$. Taking the limit inferior of both sides gives the result.
\qed
\end{proof}

Clearly, for $k$-nearest PRM, $\NumObs^\AlgksPRM_\N = k$ for all $\N \in \naturals$ with $\N > k$. Similarly, for the RRT, $\NumObs^\AlgRRT_\N = 1$ for all $\N \in \naturals$.

The next lemma upper-bounds the number of calls to the ${\tt CollisionFree}$ procedure in the proposed algorithms.
\begin{lemma}[PRM$^*$, RRG, and RRT$^*$] \label{lemma:complexity:num_obs:proposed}
$\NumObs^\AlgPRMstar_\N, \, \NumObs^\AlgRRG_\N,\, \NumObs^\AlgRRTstar_\N \in O(\log \N)$.
\end{lemma}
\begin{proof}
First, consider PRM$^*$.
Recall that $r_\N$ denotes the connection radius of the PRM$^*$ algorithm.
Recall that the $r_\N$ interior of $\X_\mathrm{free}$, denoted by $\mathrm{int}_{r_\N} (\X_\mathrm{free})$, is defined as the set of all points $x$, for which the $r_\N$-ball centered at $x$ lies entirely inside $\X_\mathrm{free}$.
Let $A$ denote the event that the sample $X_\N$ drawn at the last iteration falls into the $r_\N$ interior of $\X_\mathrm{free}$. Then,
$$
\EE\big[\NumObs_\N^\AlgPRMstar\big] =  \EE \big[ \NumObs_\N^\AlgPRMstar \biggiven A \big] \,  \PP(A) + \EE\big[\NumObs_\N^\AlgPRMstar \biggiven A^c\big] \, \PP (A^c).
$$

Let $\N_0 \in \naturals$ be the smallest number such that $\mu(\mathrm{int}_{r_\N}(\X_\mathrm{free})) > 0$. Clearly, such $\N_0$ exists, since $\lim_{\N \to \infty} r_\N = 0$ and $\X_\mathrm{free}$ has non-empty interior.
Recall that $\VolumeDBall{d}$ is the volume of the unit ball in the $d$-dimensional Euclidean space and that the connection radius of the PRM$^*$ algorithm is $r_\N = \gamma_\mathrm{PRM} (\log \N / \N)^{1/d}$.
Then, for all $\N \ge \N_0$
$$
\EE\big[\NumObs_\N^\AlgPRMstar \biggiven A \big] = \frac{\VolumeDBall{d} \, \gamma_\mathrm{PRM}}{\mu(\mathrm{int}_{r_\N}(\X_\mathrm{free}))} \log \N. 
$$

On the other hand, given that $X_\N \notin \mathrm{int}_{r_\N} (\X_\mathrm{free})$, the $r_\N$-ball centered at $X_\N$ intersects a fragment of $\X_\mathrm{free}$ that has volume less than the volume of an $r_\N$-ball in the $d$-dimensional Euclidean space. Then, for all $\N > \N_0$,
$
\EE\big[\NumObs_\N^\AlgPRMstar \biggiven A^c \big] \le \EE\big[\NumObs_\N^\AlgPRMstar \biggiven A \big].
$

Hence, for all $\N \ge \N_0$, 
$$
\EE\left[\frac{\NumObs_\N^\AlgPRMstar}{\log \N} \right] \le \frac{\VolumeDBall{d} \, \gamma_\mathrm{PRM}}{\mu(\mathrm{int}_{r_\N}(\X_\mathrm{free}))} \le \frac{\VolumeDBall{d} \, \gamma_\mathrm{PRM}}{\mu(\mathrm{int}_{r_{\N_0}}(\X_\mathrm{free}))}.
$$

Next, consider the RRG. Recall that $\eta$ is the parameter provided in the ${\tt Steer}$ procedure (see Section~\ref{section:algorithms:primitive_procedures}). Let $D$ denote the diameter of the set $\X_\mathrm{free}$, i.e., $D := \sup_{x,x' \in \X_\mathrm{free}} \| x - x' \|$. Clearly, whenever $\eta \ge D$, $V^\AlgPRMstar = V^\AlgRRG = V^\AlgRRTstar$ surely, and the claim holds. 

To prove the claim when $\eta < D$, let $C_\N$ denote the event that for any point $x \in \X_\mathrm{free}$ the RRG algorithm has a vertex $x' \in V^\AlgRRG_\N$ such that $\| x - x' \| \le \eta$. As shown in the proof of Theorem~\ref{theorem:optimality_rrg} (see Lemma~\ref{lemma:bounding_c_i}), there exists $a,b > 0$ such that $\PP (C_\N^c) \le a\, e^{-b\,\N}$. Then, 
$$
\EE\left[ \NumObs_\N^\AlgRRG \right] = \EE\left[ \NumObs_\N^\AlgRRG \biggiven C_\N \right] \, \PP(C_\N) + \EE\left[ \NumObs_\N^\AlgRRG \biggiven C_\N^c \right] \, \PP(C_\N^c),
$$
Clearly, $\EE\left[ \NumObs_\N^\AlgRRG \biggiven C_\N^c \right] \le \N$. Hence, the second term of the sum on the right hand side converges to zero as $\N$ approaches infinity. On the other hand, given that $C_\N$ holds, the new vertex that will be added to the graph at iteration $\N$, if such a vertex is added at all, will be the same as the last sample, $\Z_\N$. To complete the argument, given any set of $\N$ points placed inside $\mu(X_\mathrm{free})$, let $N_\N$ denote the number of points that are inside a ball of radius $r_\N$ that is centered at a point $\Z_\N$ sampled uniformly at random from $\mu(X_\mathrm{free})$. The expected number of points inside this ball is no more than
$
\frac{\VolumeDBall{d}\,r_\N^d}{\mu(X_\mathrm{free})} \, \N.
$
Hence, $\EE[\NumObs_\N^\AlgRRG \given C_\N] < \frac{\VolumeDBall{d}\,\gamma_\mathrm{PRM}}{\mu(X_\mathrm{free})} \log \N$, which implies the existence of a constant $\phi_1 \in \reals_{\ge 0}$ such that $\limsup_{\N \to \infty} \EE[\NumObs^\AlgRRG_\N/(\log \N)] \le \phi_1$. 

Finally, since $\NumObs_\N^\AlgRRTstar = \NumObs_\N^\AlgRRG$ holds surely, $\limsup_{\N \to \infty} \EE[\NumObs^\AlgRRG_\N/(\log \N)] \le \phi_1$ also.
\qed
\end{proof}

Trivially, $\NumObs^\AlgkPRMstar_\N = \NumObs^\AlgkRRG_\N = \NumObs^\AlgkRRTstar_\N = k \, \log \N$ for all $\N$ with $\N / \log \N > k$.

\paragraph{Complexity of the ${\tt CollisionFree}$ procedure} In this section, complexity of the ${\tt CollisionFree}$ procedure in terms of the number of obstacles in the environment is analyzed, which is a widely-studied problem in the literature (see, e.g.,~\citet{Lin.Manocha:04} for a survey). The main result is based on~\citet{Six.Wood:82}, which shows that checking collision with $\M$ obstacles can be executed in $O(\log^d\M)$ time using data structures based on spatial trees~\citep[see also][]{Edelsbrunner.Maurer.inf_proc_lett81,Hopcroft.Schwartz.ea:83}.

\paragraph{Complexity of the ${\tt Nearest}$ procedure}
The nearest neighbor search problem has been widely studied in the literature, since it has many applications in, e.g., computer graphics, database systems, image processing, data mining, pattern recognition, etc.~\citep{ samet.book89b, samet.book89a}. 
Clearly, a brute-force algorithm that examines every vertex runs in $O(\N)$ time and requires $O(1)$ space. However, in many online real-time applications such as robotics, it is highly desirable to reduce the computation time of each iteration under sublinear bounds, e.g., in $O(\log \N)$ time, especially for anytime algorithms that provide better solutions as the number of iterations increase.

Fortunately, existing algorithms for computing an ``approximate'' nearest neighbor, if not an exact one, are computationally very efficient. In the sequel, a vertex $y$ is said to be an $\varepsilon$-approximate nearest neighbor of a point $x$ if $\Vert y - x \Vert \le (1 + \varepsilon) \, \Vert z - x \Vert $, where $z$ is the true nearest neighbor of $x$. An approximate nearest neighbor can be computed using balanced-box decomposition (BBD) trees, which achieves $O(c_{d,\varepsilon} \log n)$ query time using $O (d \, n)$ space~\citep{arya.mount.ea.jacm99}, where $c_{d, \varepsilon} \le d \lceil 1 + 6d/\varepsilon \rceil^d$. This algorithm is computationally optimal in fixed dimensions, since it closely matches a lower bound for algorithms that use a tree structure stored in roughly linear space~\citep{arya.mount.ea.jacm99}. Using approximate nearest neighbor computation in the context of both PRMs and RRTs was discussed very recently in~\citet{yershova.lavalle.tro07, plaku.kavraki.wafr08}.

Let $G = (V,E)$ be a graph with $V \subseteq \X$ and let $x \in \X$. The discussion above implies that the number of simple operations executed by the ${\tt Nearest}(G,x)$ procedure is $\Theta(\log \vert V \vert)$ in fixed dimensions, if the ${\tt Nearest}$ procedure is implemented using a tree structure that is stored in linear space. 

\paragraph{Complexity of the ${\tt Near}$ procedure}

Problems similar to that solved by the ${\tt \NearNodes}$ procedure are also widely-studied in the literature, generally under the name of {\em range search problems}, as they have many applications in, for instance, computer graphics and spatial database systems~\citep{samet.book89a}. In the worst case and in fixed dimensions, computing the exact set of vertices that reside in a ball of radius $r_n$ centered at a query point $x$ takes $O(n^{1 - 1/d} + m)$ time using $k$-d trees~\citep{lee.wong.acta_informatica77}, where $m$ is the number of vertices returned by the search (see also~\citet{chanzy.devroye.ea.acta_informatica01} for an analysis of the average case).

Similar to the nearest neighbor search, computing approximate solutions to the range search problem is computationally easier. A range search algorithm is said to be $\varepsilon$-approximate if it returns all vertices that reside in the ball of size $r_n$ and no vertices outside a ball of radius $(1 + \varepsilon)\,r_n$, but may or may not return the vertices that lie outside the former ball and inside the latter ball. Computing $\varepsilon$-approximate solutions using BBD-trees requires $O(2^d\log n + d^2(3\sqrt{d}/\varepsilon)^{d-1})$ time when using $O(d \, n)$ space, in the worst case~\citep{arya.mount.comp_geo00}. Thus, in fixed dimensions, the complexity of this algorithm is $O(\log n + (1/ \varepsilon)^{d-1})$, which is known to be optimal, closely matching a lower bound~\citep{arya.mount.comp_geo00}. More recently, algorithms that can provide trade-offs between time and space were also proposed~\citep{arya.malamatos.ea.symp_dis_alg05}.

Note that the ${\tt \NearNodes}$ procedure can be implemented as an approximate range search while maintaining the asymptotic optimality guarantee. Notice that the expected number of vertices returned by the ${\tt \NearNodes}$ procedure also does not change, except by a constant factor.
Hence, the ${\tt \NearNodes}$ procedure can be implemented to run in order $\log n$ expected time in the limit\ and linear space in fixed dimensions.

\paragraph{Time complexity of the processing phase} 
The following results characterize the asymptotic computational complexity of various sampling-based algorithms in terms of the number of simple operations such as comparisons, additions, and multiplications. 

Let $\N$ denote the total number of iterations (or, alternatively, the number of samples), and $\M$ denote the number of obstacles in the environment. 
Then, by Lemmas~\ref{lemma:complexity:num_obs:prm} and \ref{lemma:complexity:num_obs:sprm}, $\NumSteps^\AlgPRM_\N, \,\NumSteps^\AlgsPRM_\N \in \Omega(\N^2 \log^d \M)$. In the $k$-nearest sPRM and RRT algorithms, $\Omega(\log \N)$ time is spent on finding the ($k$-)nearest neighbor(s) and $\Omega(\log^d\M)$ time is spent on collision checking at each iteration. Hence, $\NumSteps^\AlgksPRM_\N,  \NumSteps^\AlgRRT_\N \in \Omega (\N \log\N + \N \log^d \M)$. 

In all the proposed algorithms, $O(\log \N)$ time is spent on finding the near neighbors, and $\log \N \log^d \M$ time is spent on collision checking. Thus, $\NumSteps^\Alg_\N \in O (\N \, \log\N \log^d\M)$ for $ALG \in \{ \AlgPRMstar, \AlgkPRMstar,$ $\AlgRRG, \AlgkRRG, \AlgRRTstar, \AlgkRRTstar\}$.

\paragraph{Time complexity of the query phase}
After algorithm \Alg  returns the graph $G^\Alg_\N$, the optimal path must be extracted from this graph using, e.g., Dijkstra's shortest path algorithm~\citep{schrijver.book03}. In this section, the complexity of this operation, called the query phase, is discussed.

The following lemma yields the asymptotic computational complexity of computing shortest paths. Let $G = (V,E)$ be a graph. A length function $l : E \to \reals_{>0}$ is a function that assigns each edge in $E$ a positive length.
Given a vertex $v \in V$, the shortest paths tree for $G$, $l$, and $v$ is a graph $G' = (V,E')$, where $E' \subseteq E$ such that for any $v' \in V \setminus \{v\}$, there exists a unique path in $G$ that starts from $v$ and reaches $v'$, moreover, this path is the optimal such path in $G$.
\begin{lemma}[Complexity of shortest paths~\citep{schrijver.book03}]
Given a graph $G = (V,E)$, a length function $l : E \to \reals_{>0}$, and a vertex $v \in V$, the shortest path tree for $G$, $l$, and $v$ can be found in time $O(\vert V \vert \log (\vert V \vert) + \vert E \vert)$.
\end{lemma}

It remains to determine the number of vertices and edges in $G^\Alg_\N = (V^\Alg_\N, E^\Alg_\N)$, for each algorithm \Alg . 

Trivially, $\vert E^\Alg_\N \vert \in \Omega(\N)$ holds for all the algorithms discussed in this paper, in particular, for $ALG \in \{\AlgPRM, \AlgksPRM, \AlgRRT\}$. 
For the sPRM algorithm, a stronger bound can be provided: $\vert E^\AlgsPRM_\N \vert \in \Omega(\N^2)$. To prove this claim, consider the problem instance $(\X_\mathrm{free}, x_\mathrm{init}, \X_\mathrm{goal})$, where $\X_\mathrm{free} = \X = (0,1)^d$. Then, the straight path between any two vertices will be collision-free. Thus, the number of edges is exactly equal to the number of calls to the {\tt CollisionFree} procedure. Then, the result follows from Lemma~\ref{lemma:complexity:num_obs:sprm}.

For the proposed algorithms, $\vert E^\AlgPRMstar_\N \vert, \vert E^\AlgRRG_\N \vert  \in O(\N \log \N)$. Since the number of edges is always less than or equal to the total number of calls to the ${\tt CollisionFree}$ procedure, this claim follows directly from Lemma~\ref{lemma:complexity:num_obs:proposed}.
Finally, $\vert E^\AlgkPRMstar_\N \vert, \vert E^\AlgkRRG_\N \vert \in O(\N\,\log \N)$ and $\vert E^\AlgRRTstar_\N \vert,\, \vert E^\AlgkRRTstar_\N \vert \in O(\N)$ all hold trivially.

\paragraph{Space complexity}

Space complexity of an algorithm \Alg  is defined as the amount of memory that is used by \Alg  to compute the graph $G^\Alg_\N = (V^\Alg_\N, E^\Alg_\N)$. Clearly, in all algorithms discussed in this paper, the space complexity is the size of $G^\Alg_\N$, i.e., $\vert V^\Alg_\N \vert + \vert E^\Alg_\N \vert$. Since the number of edges is at least as much as the number of vertices in $G^\Alg_\N$ for all algorithms discussed in this paper, the space complexity of an algorithm, in this context, is the number edges in the graph that it returns, which was determined in the previous section.

\section{Numerical Experiments} \label{section:experiments}\label{section:simulations}

This section is devoted to an experimental study of the algorithms considered in the paper. 
All algorithms were implemented in C and run on a computer with 2.66 GHz processor and 4GB RAM
running the Linux operating system. Unless otherwise noted, total variation of a path is its cost.

A first set of experiments were run to illustrate the different performance of $k$-nearest PRM and of PRM$^*$. The $k$-nearest PRM and the PRM$^*$ algorithms were run alongside in two dimensional configuration-space and the cost of the best path in both algorithms is plotted versus the number of iterations in Figure~\ref{figure:prm_vs_prmstar_2d}. The $k$-nearest PRM does not converge to optimal solutions, unlike PRM$^*$. 
The performance of the PRM$^*$ algorithm is also shown in configuration spaces of dimensions up to five in Figure~\ref{figure:prm_to_5d}.

The main bulk of the experiments were aimed at demonstrating the performance of the RRT$^*$ algorithm, especially in comparison with its ``standard'' counterpart, i.e., RRT. 
Three problem instances were considered. In the first two, the cost function is the Euclidean path
length. 

The first scenario includes no obstacles. Both algorithms are run in a square environment.
The trees maintained by the algorithms are shown in Figure~\ref{figure:sim1} at several stages. The
figure illustrates that, in this case, the RRT algorithm does not improve the feasible solution to
converge to an optimum solution. On the other hand, running the RRT$^*$ algorithm further improves
the paths in the tree to lower cost ones. The convergence properties of the two algorithms are also
investigated in Monte-Carlo runs. Both algorithms were run for 20,000 iterations 500 times and the
cost of the best path in the trees were averaged for each iteration. The results are shown in
Figure~\ref{figure:sim1cost}, which shows that in the limit the RRT algorithm has cost very close to
a $\sqrt{2}$ factor the optimal solution (see~\citet{lavalle.kuffner.tech_rep09} for a similar result
in a deterministic setting), whereas the RRT$^*$ converges to the optimal solution. Moreover, the
variance over different RRT runs approaches 2.5, while that of the RRT$^*$ approaches zero. Hence,
almost all RRT$^*$ runs have the property of convergence to an optimal solution, as expected.

In the second scenario, both algorithms are run in an environment in presence of obstacles. In
Figure~\ref{figure:sim2}, the trees maintained by the algorithms are shown after 20,000 iterations.
The tree maintained by the RRT$^*$ algorithm is also shown in Figure~\ref{figure:sim2optrrt} in
different stages. It can be observed that the RRT$^*$ first rapidly explores the state space just
like the RRT. Moreover, as the number of samples increase, the RRT$^*$ improves its tree to include
paths with smaller cost and eventually discovers a path in a different homotopy class, which reduces
the cost of reaching the target considerably. Results of a Monte-Carlo study for this scenario is
presented in Figure~\ref{figure:sim2cost}. Both algorithms were run alongside up until 20,000
iterations 500 times and cost of the best path in the trees were averaged for each iteration. The
figures illustrate that all runs of the RRT$^*$ algorithm converges to the optimum, whereas the RRT
algorithm is about 1.5 of the optimal solution on average. The high variance in solutions returned
by the RRT algorithm stems from the fact that there are two different homotopy classes of paths that
reach the goal. If the RRT luckily converges to a path of the homotopy class that contains an
optimum solution, then the resulting path is relatively closer to the optimum than it is on average.
If, on the other hand, the RRT first explores a path of the second homotopy class, which is often
the case for this particular scenario, then the solution that RRT converges to is generally around
twice the optimum.

Finally, in the third scenario, where no obstacles are present, the cost function is selected to be
the line integral of a function, which evaluates to 2 in the high cost region, 1/2 in the low cost
region, and 1 everywhere else. The tree maintained by the RRT$^*$ algorithm is shown after 20,000
iterations in Figure~\ref{figure:sim3}. Notice that the tree either avoids the high cost region or
crosses it quickly, and vice-versa for the low-cost region. (Incidentally, this behavior corresponds
to the well known Snell-Descartes law for refraction of light, see~\citet{Rowe.Alexander.00} for a
path-planning application.)

To compare the running time, both algorithms were run alongside in an environment with no obstacles for 
up to one million iterations. Figure~\ref{figure:sim0time}, shows the ratio of the running time
of RRT$^*$ and that of RRT versus the number of iterations averaged over 50 runs. As expected from
the complexity analysis of Section~\ref{section:complexity}, this ratio converges to a constant
value. A similar figure is produced for the second scenario and provided in
Figure~\ref{figure:sim1time}.

The RRT$^*$ algorithm was also run in a 5-dimensional state space. The number of iterations versus the cost of the best path averaged over 100 trials is shown in Figure~\ref{figure:rrtstar_5d}. A comparison with the RRT algorithm is provided in the same figure. The ratio of the running times of the RRT$^*$ and the RRT algorithms is provided in Figure~\ref{figure:rrtstar_5d_runtime}.
The same experiment is carried out for a 10-dimensional configuration space. The results are shown in Figure~\ref{figure:rrtstar_10d}. 

\begin{figure}[htp]
  \begin{center}
  \includegraphics[width=0.6\textwidth]{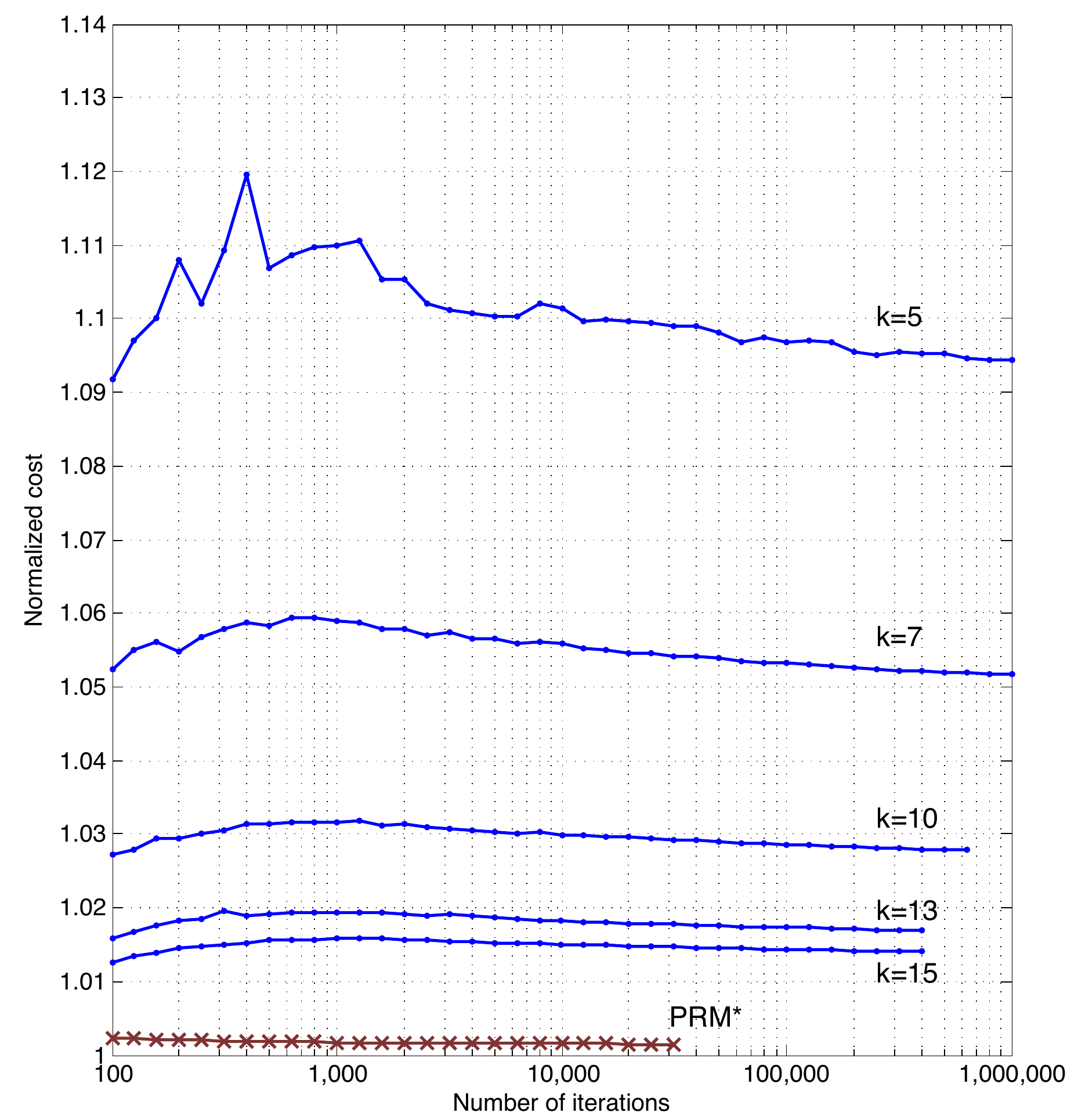} \label{sim_prm:prm_all}
    \caption{The cost of the best path in the $k$-nearest sPRM algorithm, and that in the PRM$^*$ algorithm are shown versus the number of iterations in simulation examples with no obstacles. The $k$-nearest sPRM algorithm was run for $k = 5, 7,10, 13, 15$, each of which is shown separately in blue, and the PRM$^*$ algorithm is shown in red. The values are normalized so that the cost of the optimal path is equal to one. The iterations were stopped when the query phase of the algorithms exceeded the memory limit (approximately 4GB).}
    \label{figure:prm_vs_prmstar_2d}
  \end{center}
\end{figure}

\begin{figure}[htp]
\begin{center}
\mbox{ \subfigure[]{\scalebox{0.4}{\includegraphics[trim = 1.3cm 7cm 1.8cm 7cm, clip = true]{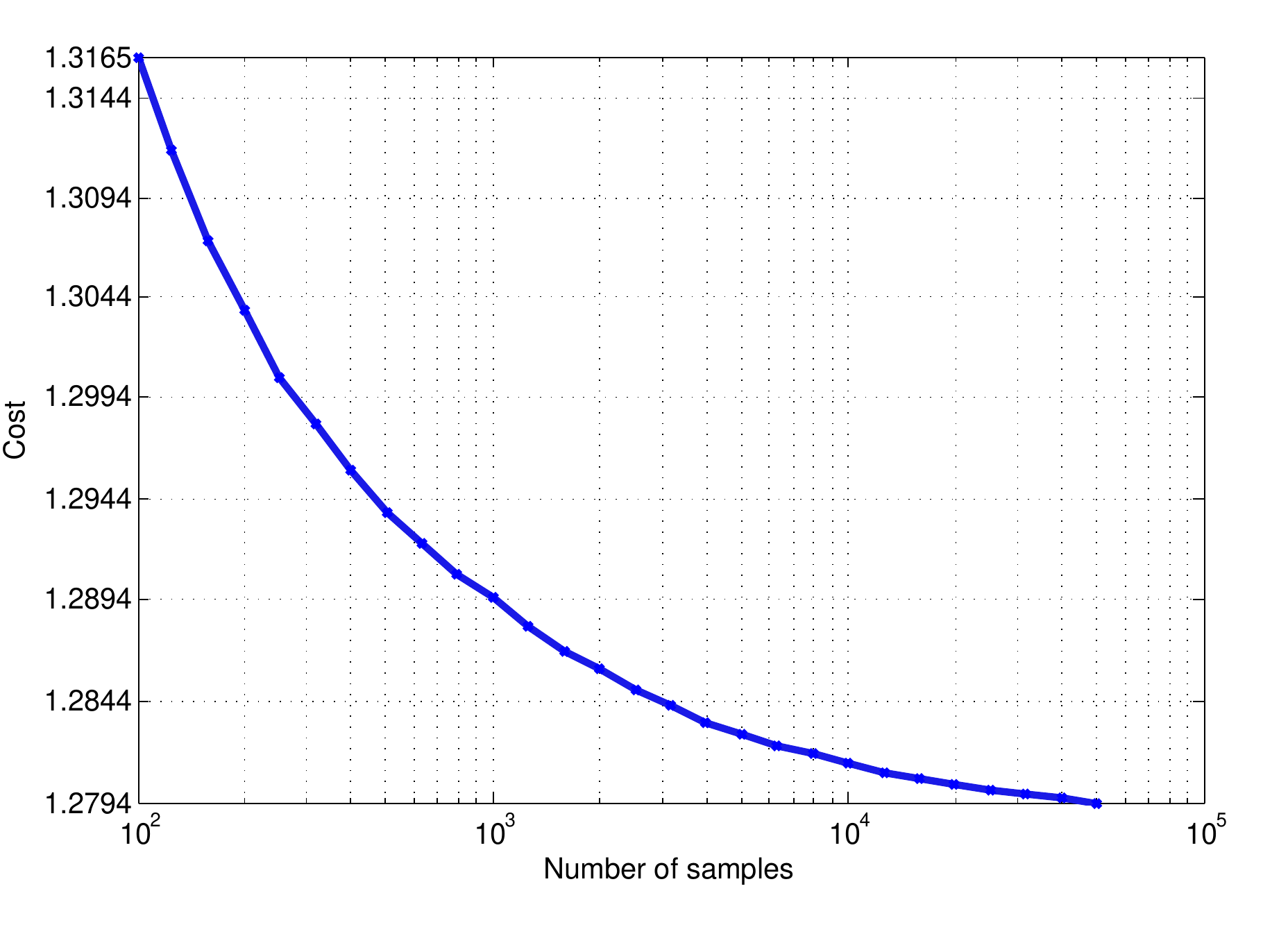}} }
\subfigure[]{\scalebox{0.4}{\includegraphics[trim = 1.3cm 7cm 1.8cm 7cm, clip = true]{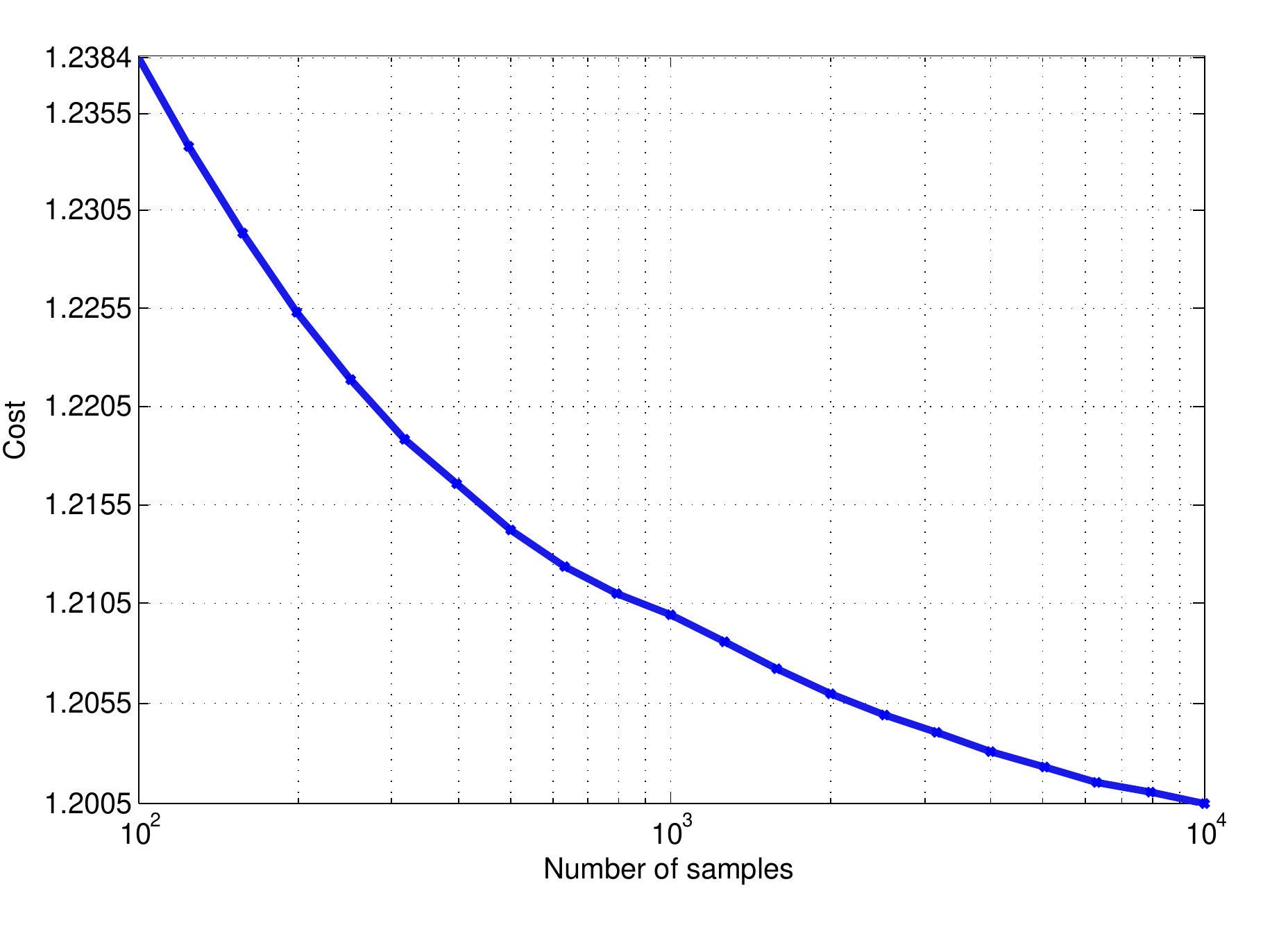}} } }
\mbox{ \subfigure[]{\scalebox{0.4}{\includegraphics[trim = 1.3cm 7cm 1.8cm 7cm, clip = true]{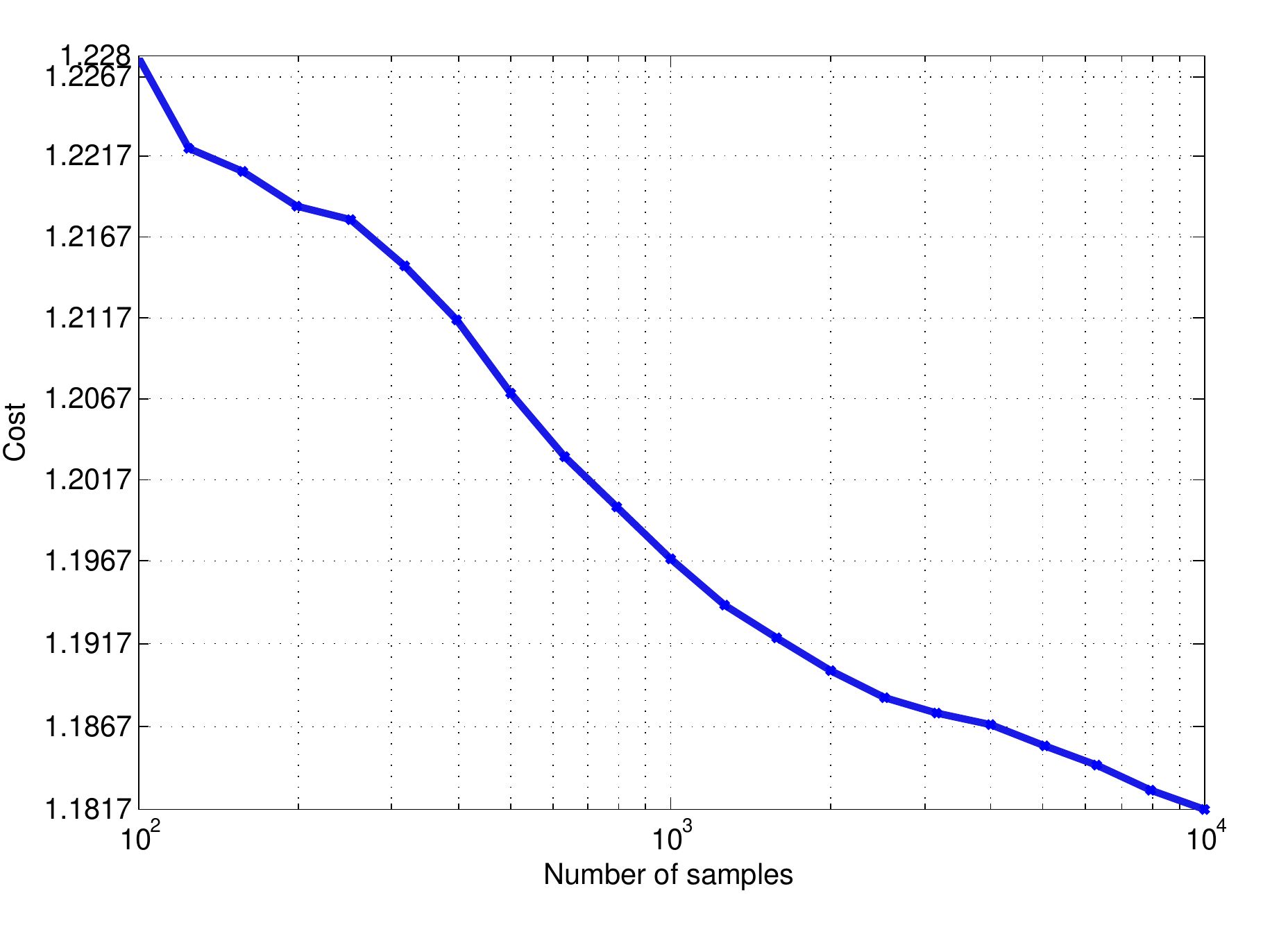}} }
\subfigure[]{\scalebox{0.4}{\includegraphics[trim = 1.3cm 7cm 1.8cm 7cm, clip = true]{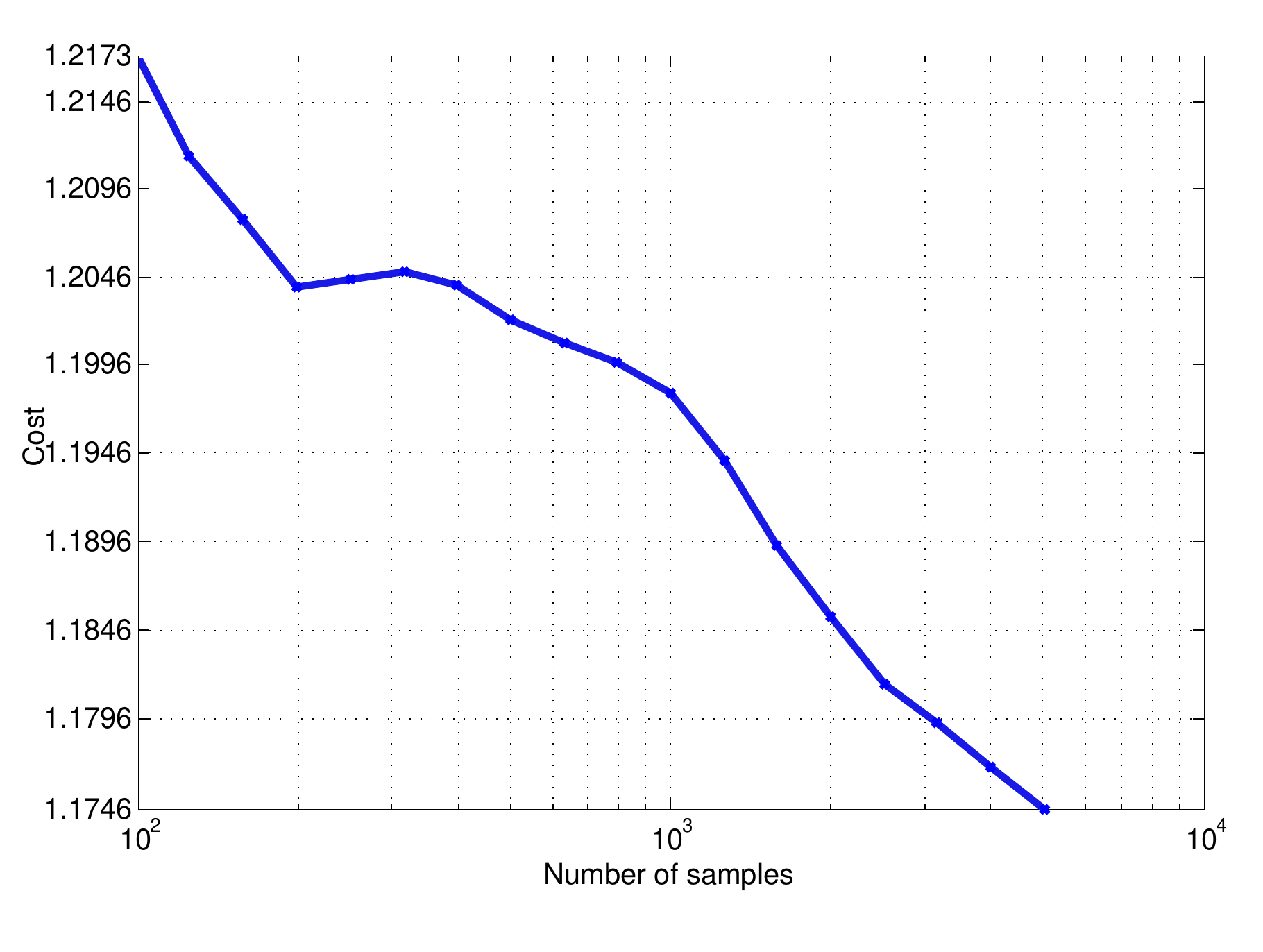}} } }
\caption{Cost of the best path in the PRM$^*$ algorithm is shown in up to 2, 3, 4, and 5 dimensional configuration spaces, in Figures (a), (b), (c), and (d), respectively. The initial condition and goal region are on opposite vertices of the unit cube $(0,1)^d$. The obstacle region is a cube centered at $(0.5, 0.5, \dots, 0.5)$ and has volume $0.5$ in all cases.}
\label{figure:prm_to_5d}
\end{center}
\end{figure}

\begin{figure*}[htp]
  \begin{center}
    \mbox{ \subfigure[]{\scalebox{0.25}{\includegraphics[trim = 3.1cm 1.8cm 2.5cm 1.4cm, clip =
          true]{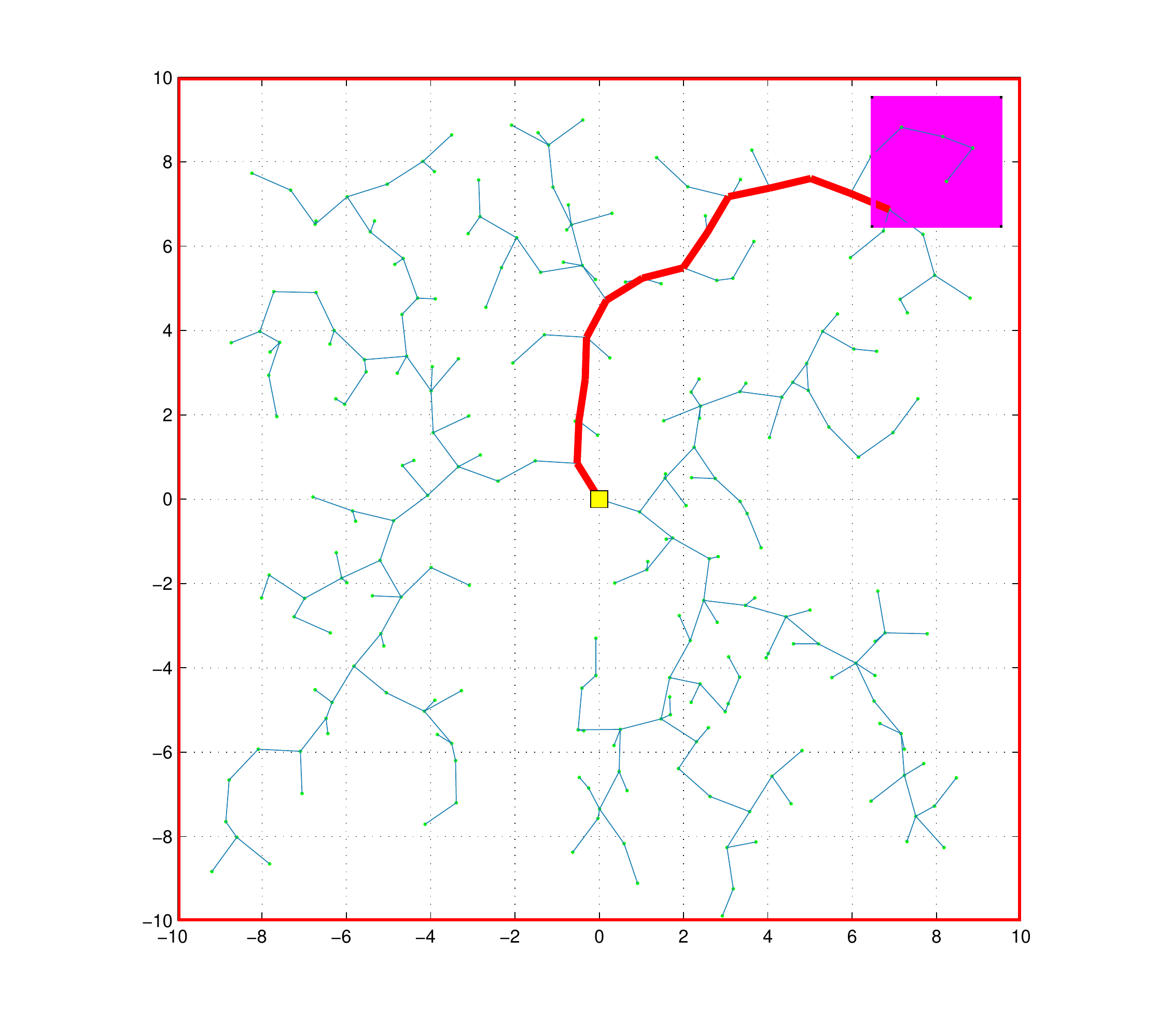}} 
        \label{sim1_rrt_10}}
      \subfigure[ ]{\scalebox{0.25}{\includegraphics[trim = 3.1cm 1.8cm 2.5cm 1.4cm, clip =
          true]{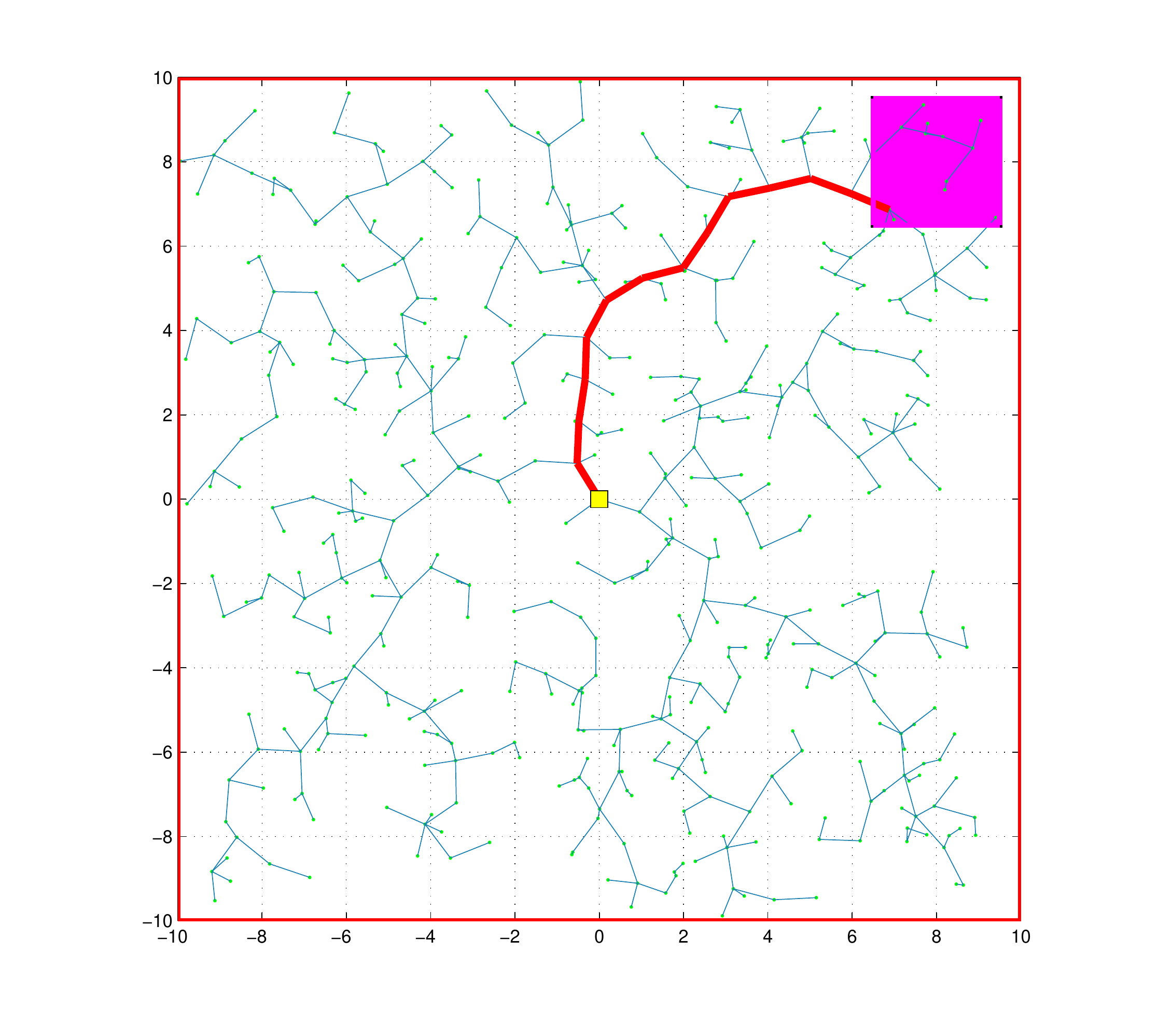}} 
        \label{sim1_optrrt_10}}
      \subfigure[ ]{\scalebox{0.25}{\includegraphics[trim = 3.1cm 1.8cm 2.5cm 1.4cm, clip =
          true]{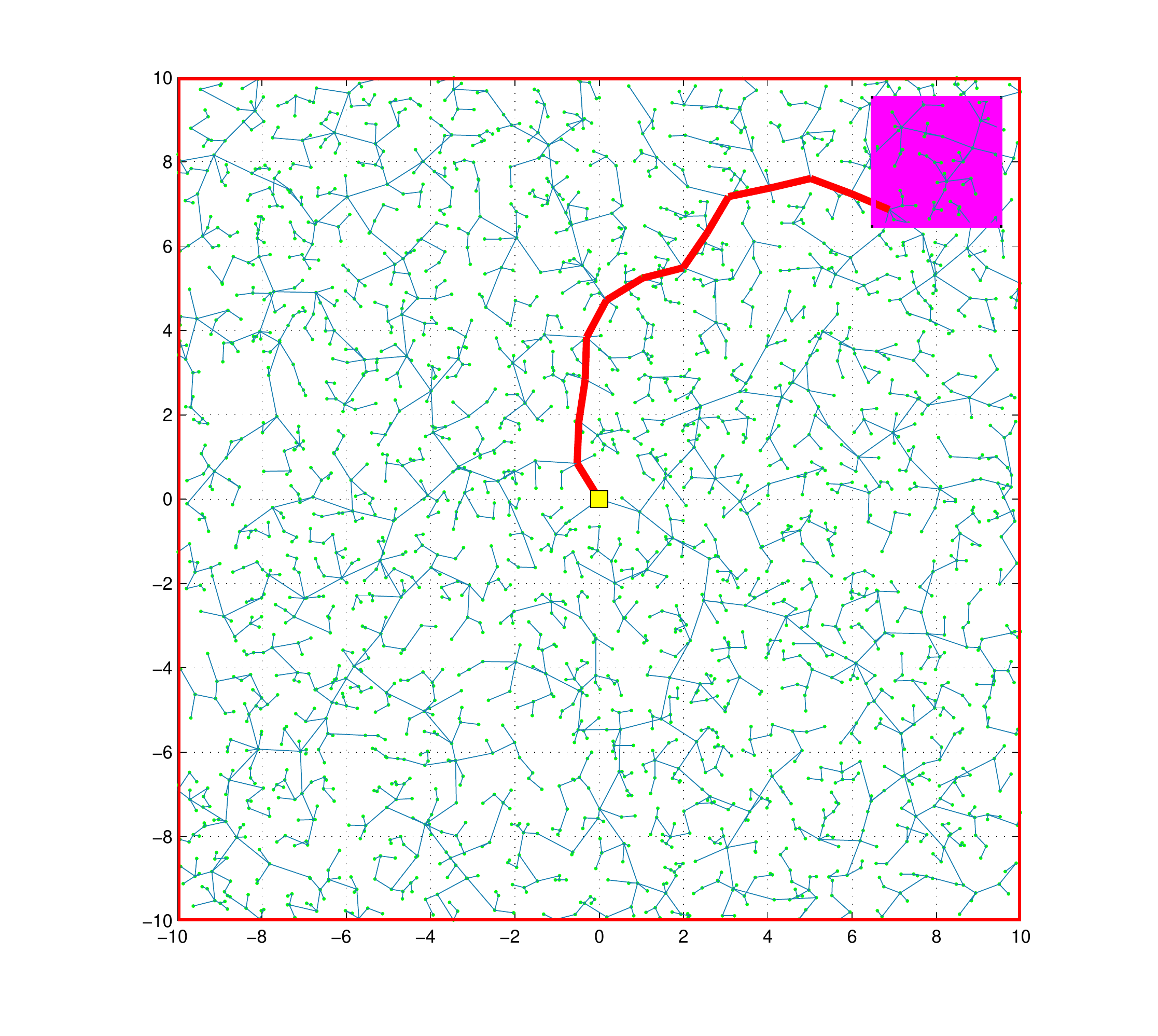}} 
        \label{sim1_optrrt_10}}
      \subfigure[ ]{\scalebox{0.25}{\includegraphics[trim = 3.1cm 1.8cm 2.5cm 1.4cm, clip =
          true]{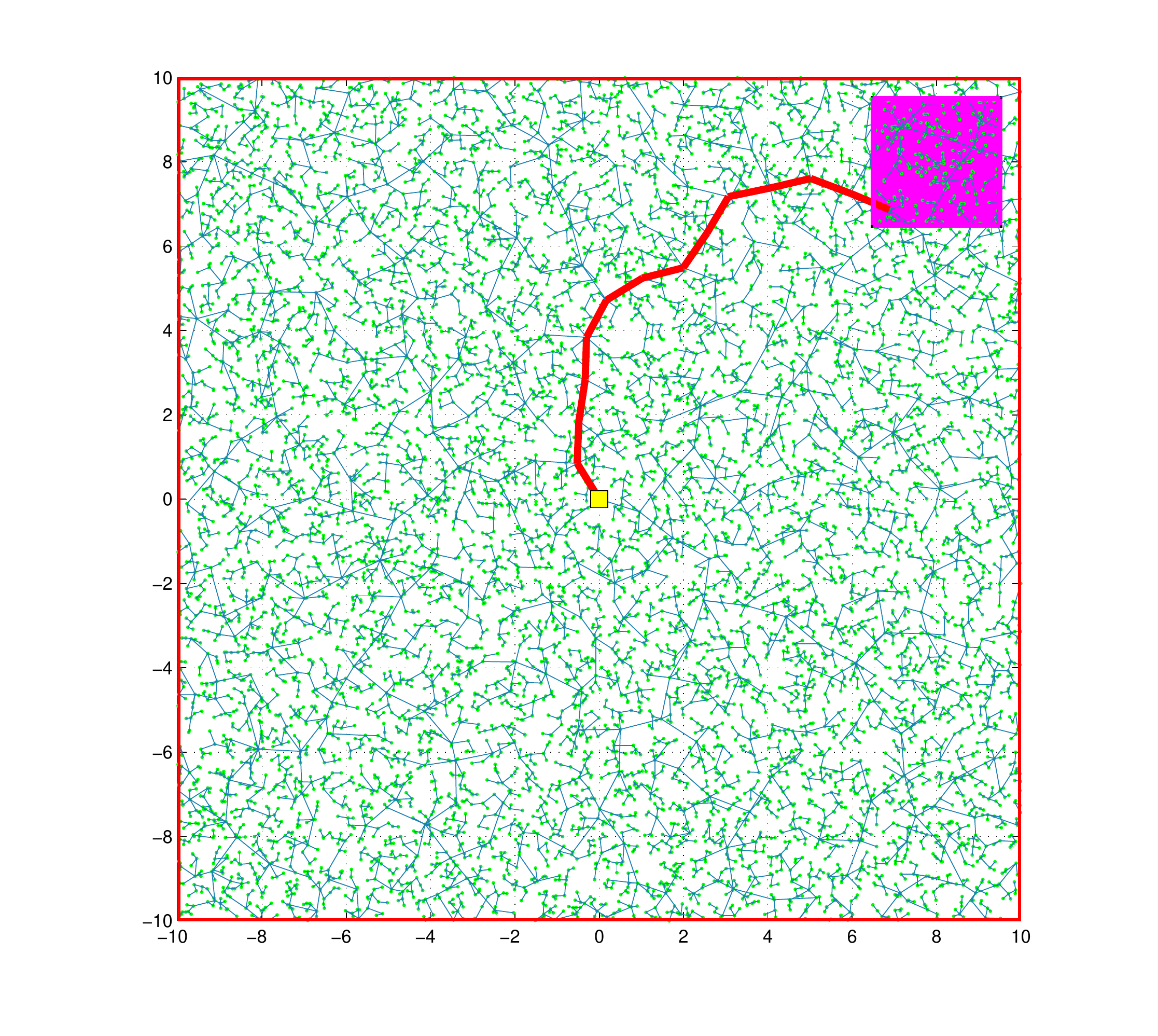}} 
        \label{sim1_optrrt_10}}}
    \mbox{ \subfigure[]{\scalebox{0.25}{\includegraphics[trim = 3.1cm 1.8cm 2.5cm 1.4cm, clip =
          true]{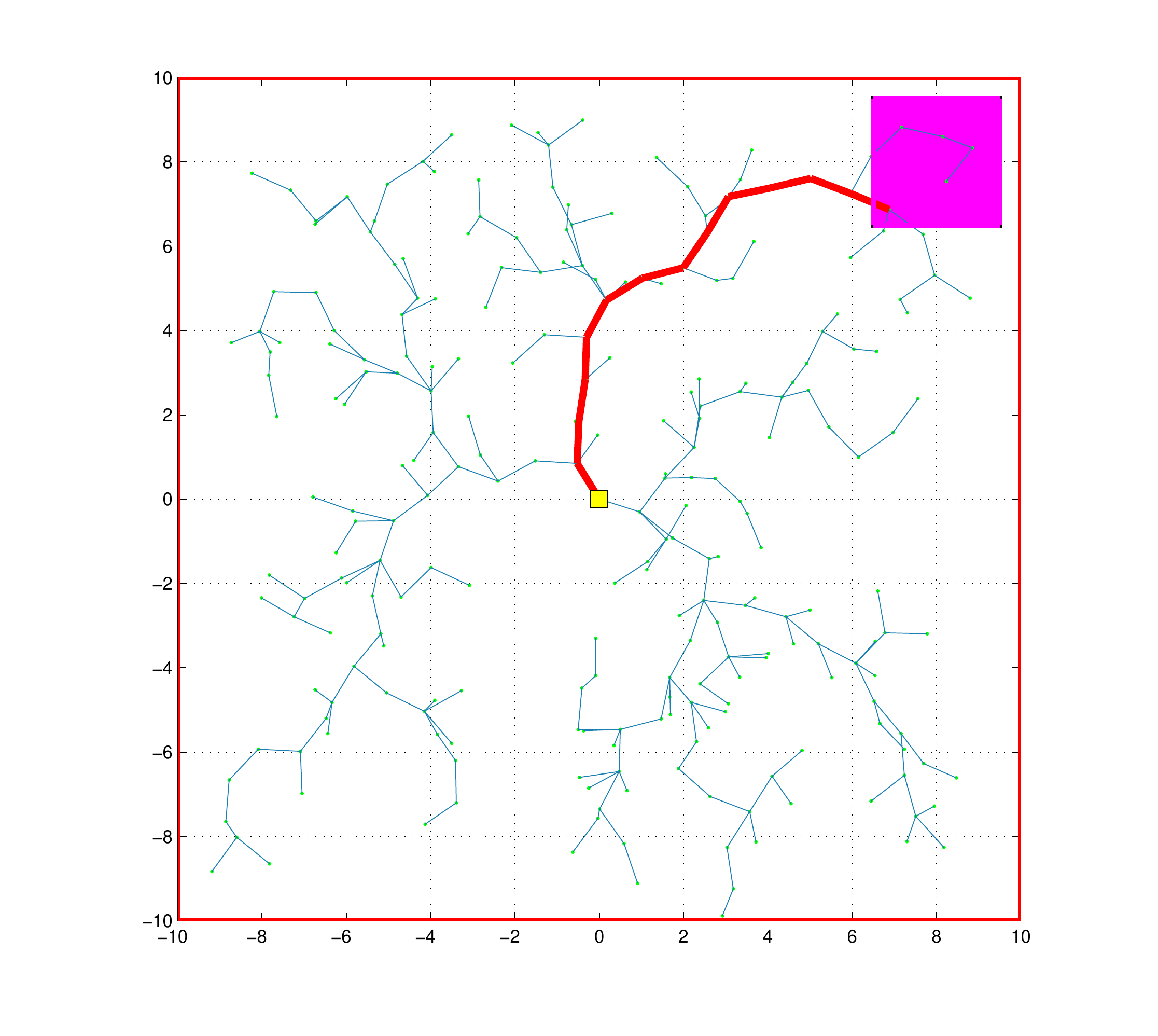}}
        \label{sim1_rrt_10}} \subfigure[ ]{\scalebox{0.25}{\includegraphics[trim = 3.1cm 1.8cm
          2.5cm 1.4cm, clip =
          true]{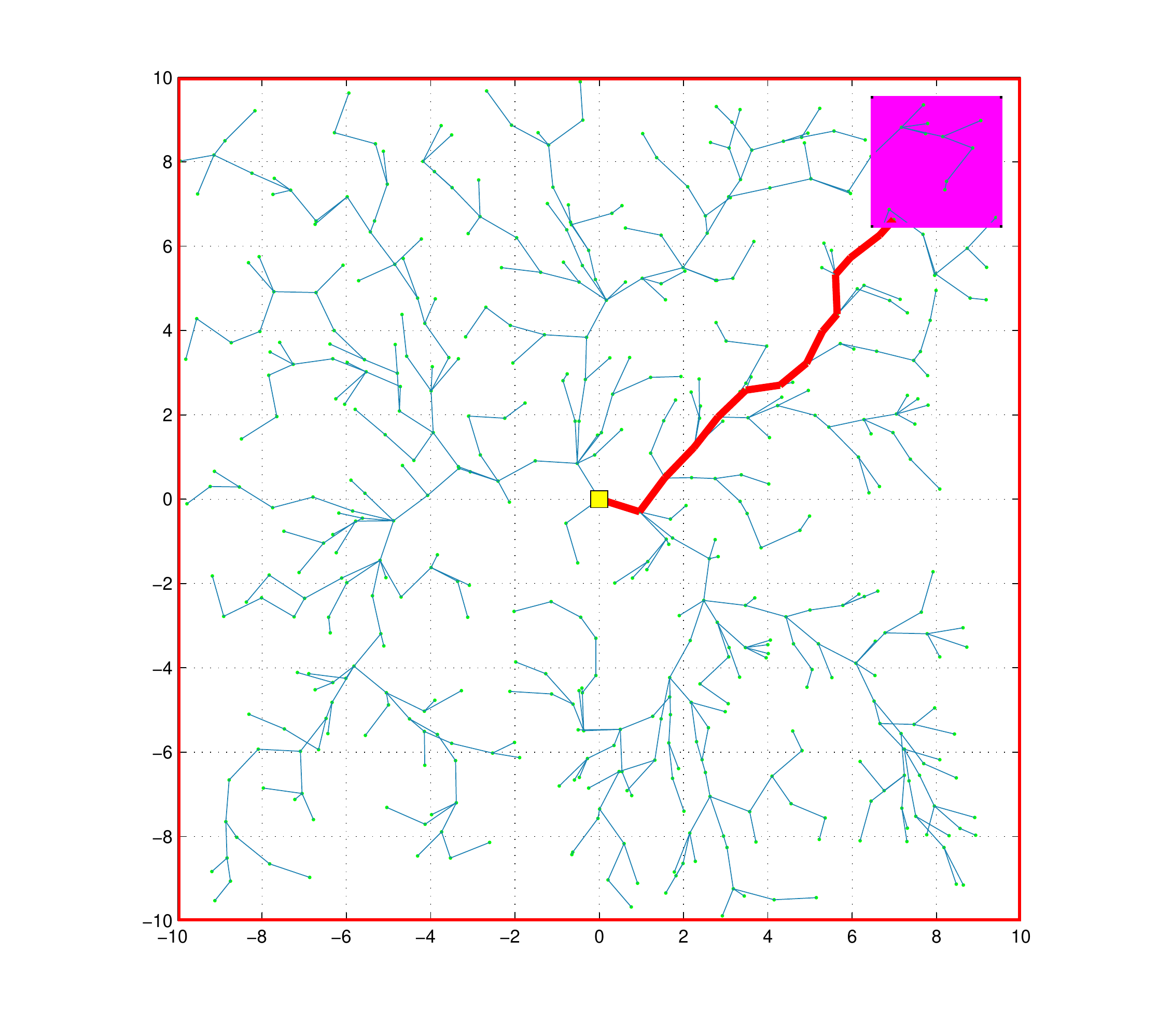}}
        \label{sim1_optrrt_10}} \subfigure[ ]{\scalebox{0.25}{\includegraphics[trim = 3.1cm 1.8cm
          2.5cm 1.4cm, clip =
          true]{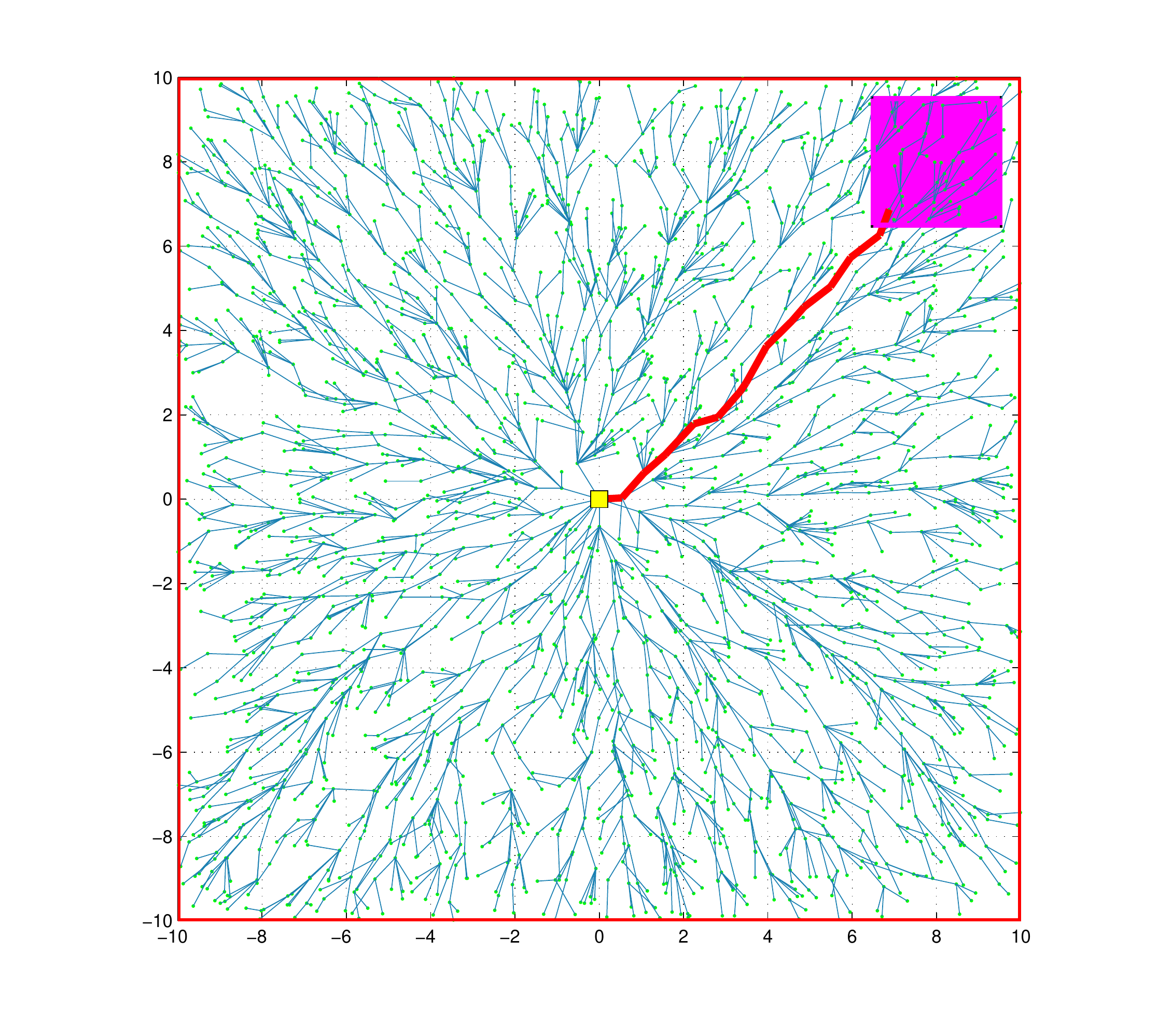}}
        \label{sim1_optrrt_10}} \subfigure[ ]{\scalebox{0.25}{\includegraphics[trim = 3.1cm 1.8cm
          2.5cm 1.4cm, clip =
          true]{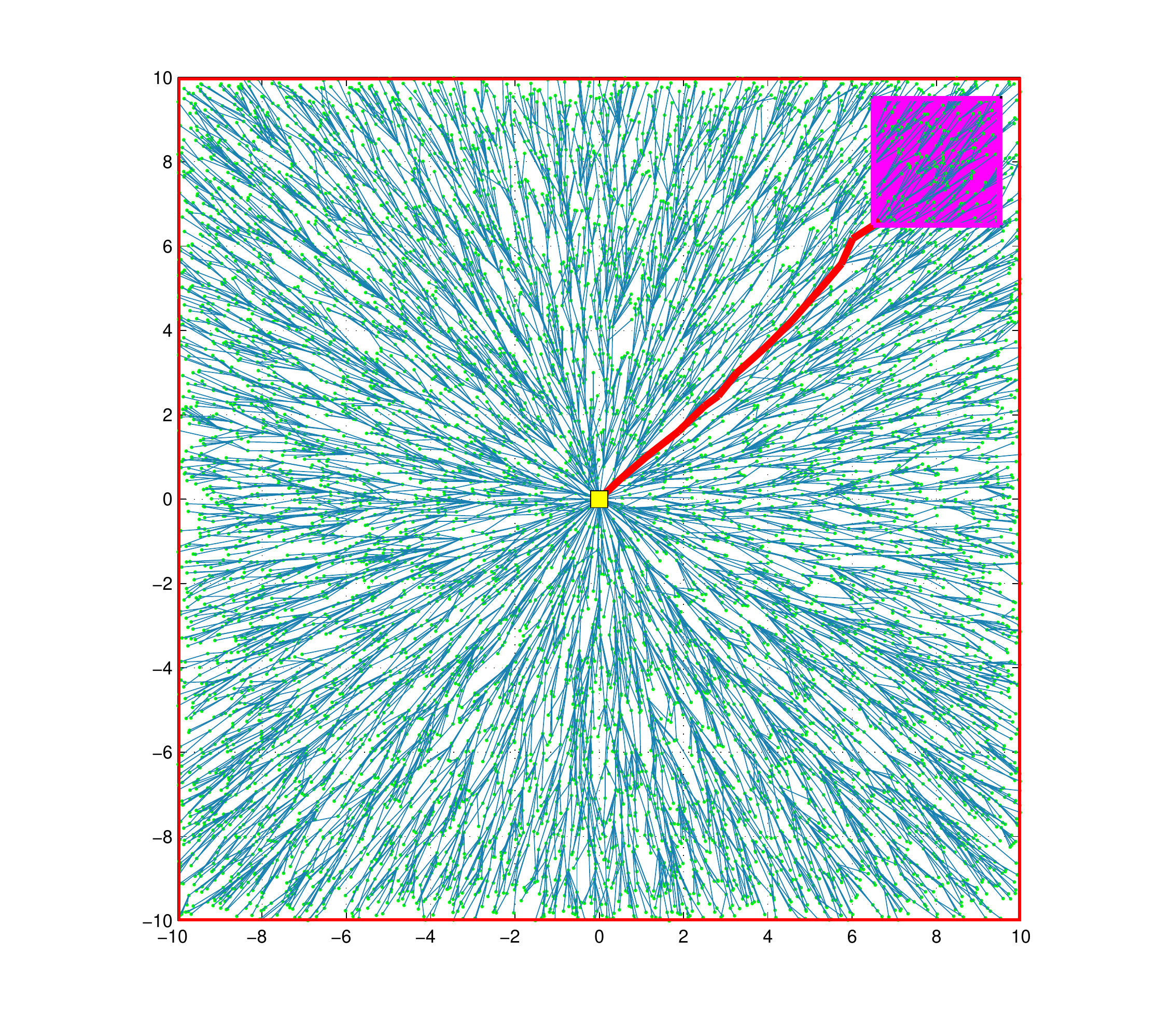}}
        \label{sim1_optrrt_10}}} \mbox{ \subfigure[]{\scalebox{0.51}{\includegraphics[trim = 3.1cm
          1.8cm 2.5cm 1.4cm, clip =
          true]{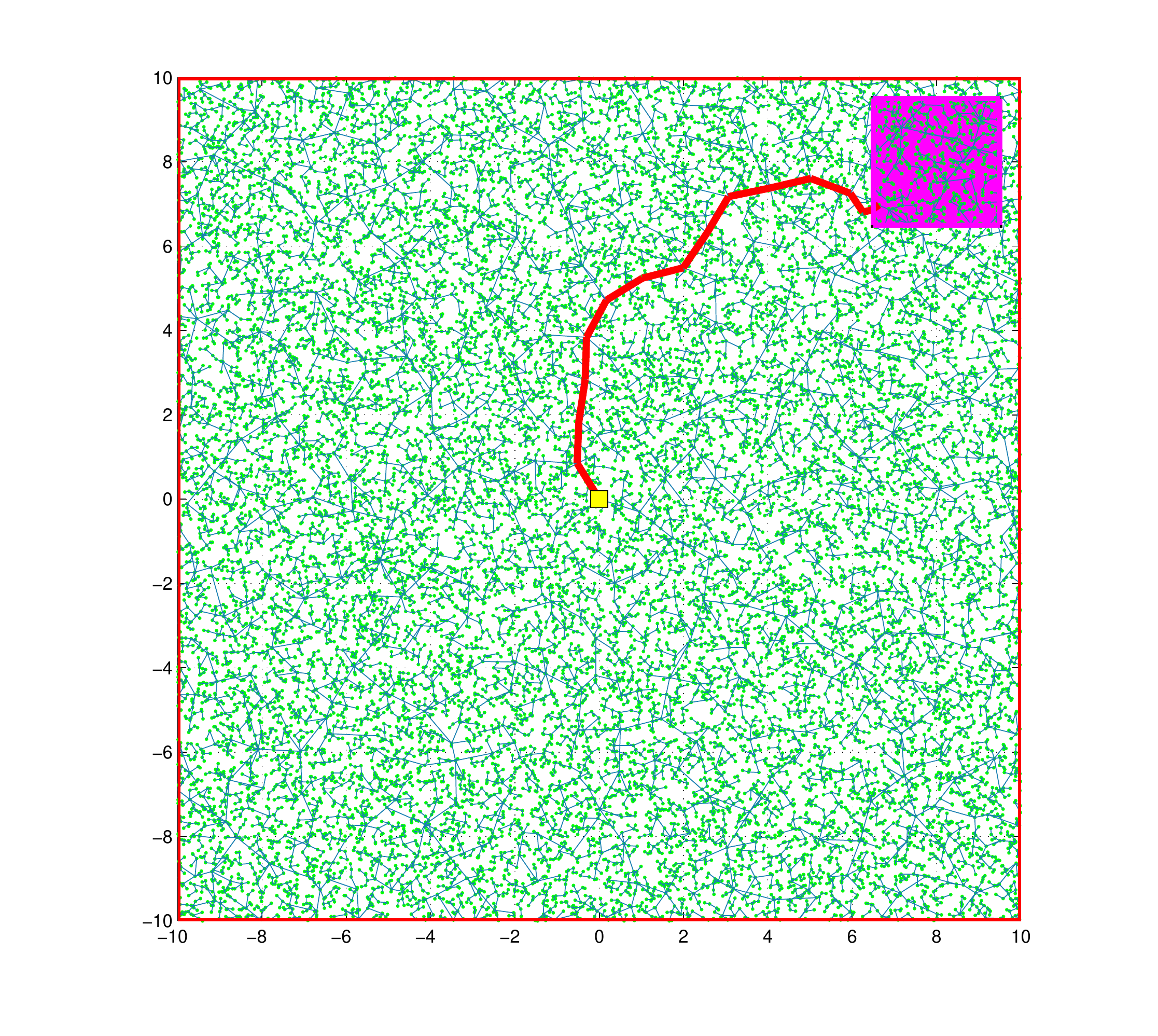}} \label{sim1_rrt_20}}
      \subfigure[ ]{\scalebox{0.51}{\includegraphics[trim = 3.1cm 1.8cm 2.5cm 1.4cm, clip =
          true]{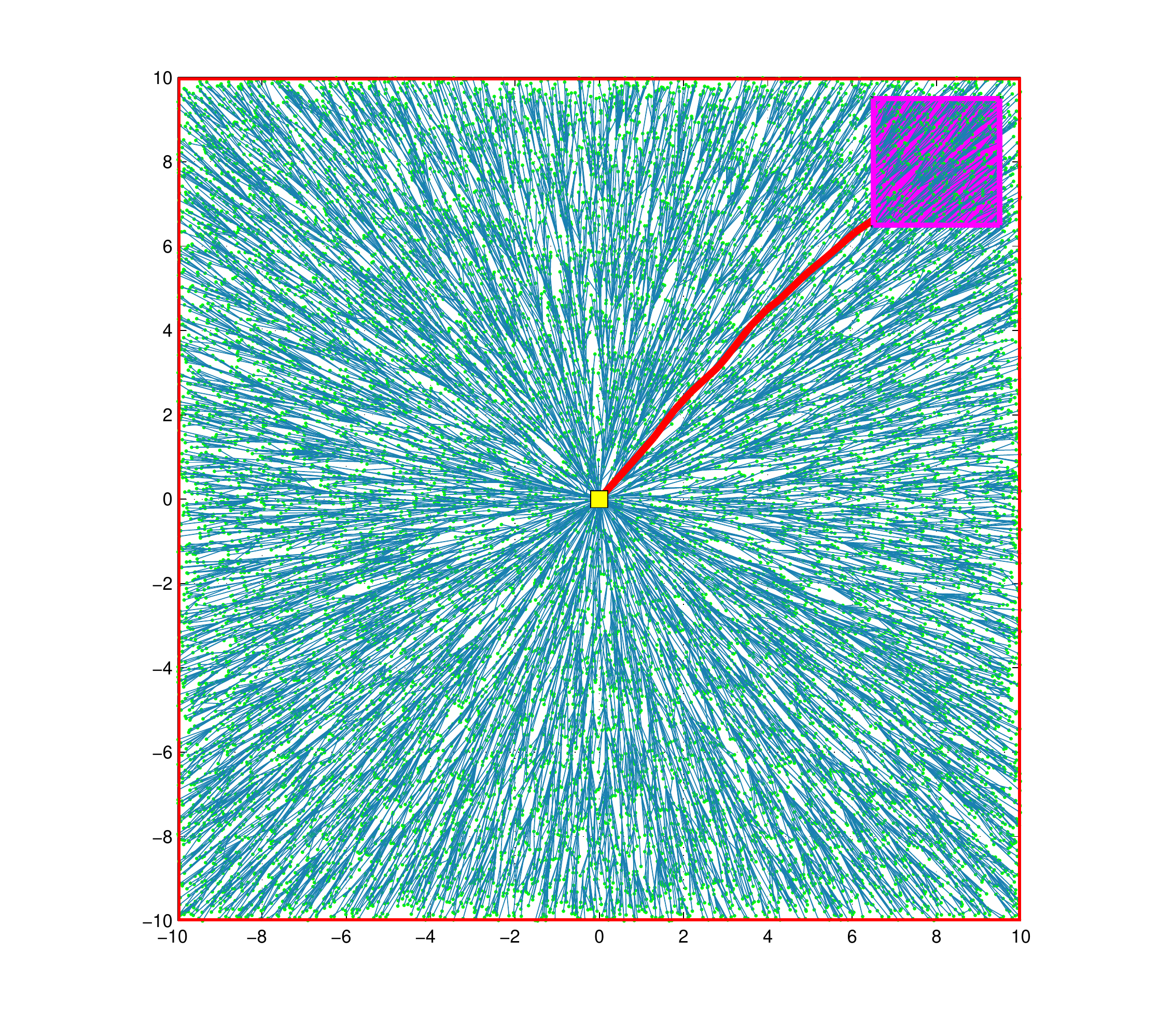}} 
        \label{sim1_optrrt_20}}
    }
    \caption{A Comparison of the RRT$^*$ and RRT algorithms on a simulation example with no
      obstacles. Both algorithms were run with the same sample sequence. Consequently, in this case,
      the vertices of the trees at a given iteration number are the same for both of the algorithms;
      only the edges differ. The edges formed by the RRT algorithm are shown in (a)-(d) and (i),
      whereas those formed by the RRT$^*$ algorithm are shown in (e)-(h) and (j). The tree snapshots
      (a), (e) contain 250 vertices, (b), (f) 500 vertices, (c), (g) 2500 vertices, (d), (h) 10,000 vertices and
      (i), (j) 20,000 vertices. The goal regions are shown in magenta (in upper right). The best paths
      that reach the target in all the trees are highlighted with red.}
    \label{figure:sim1}
  \end{center}
\end{figure*}

\begin{figure}[htb]
  \begin{center}
    \mbox{ \subfigure[]{\scalebox{0.5}{\includegraphics[trim = 0cm 9cm 0cm 9cm, clip =
          true]{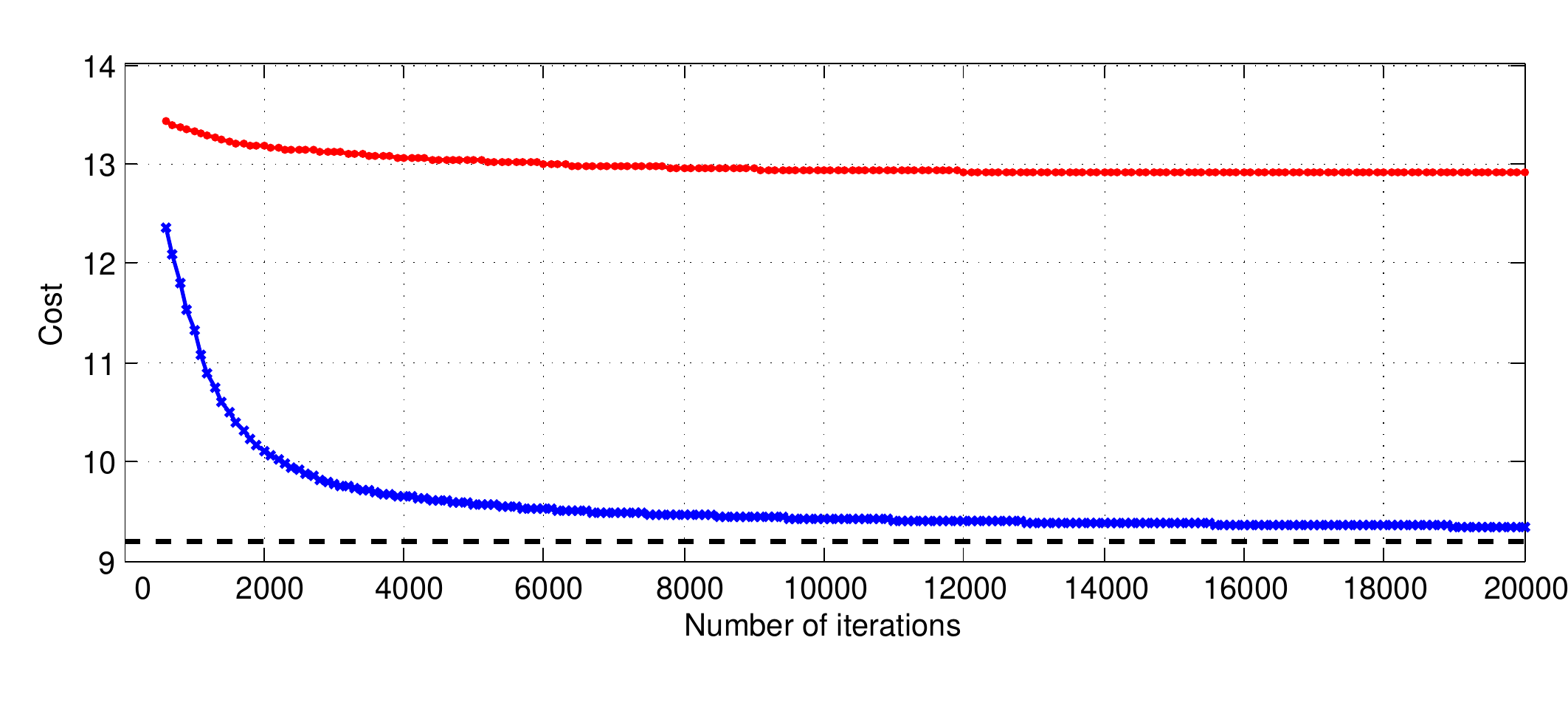}} \label{sim0_mc_cost}}}
    \mbox{ \subfigure[ ]{\scalebox{0.5}{\includegraphics[trim = 0cm 9cm 0cm 9cm, clip
          = true
          ]{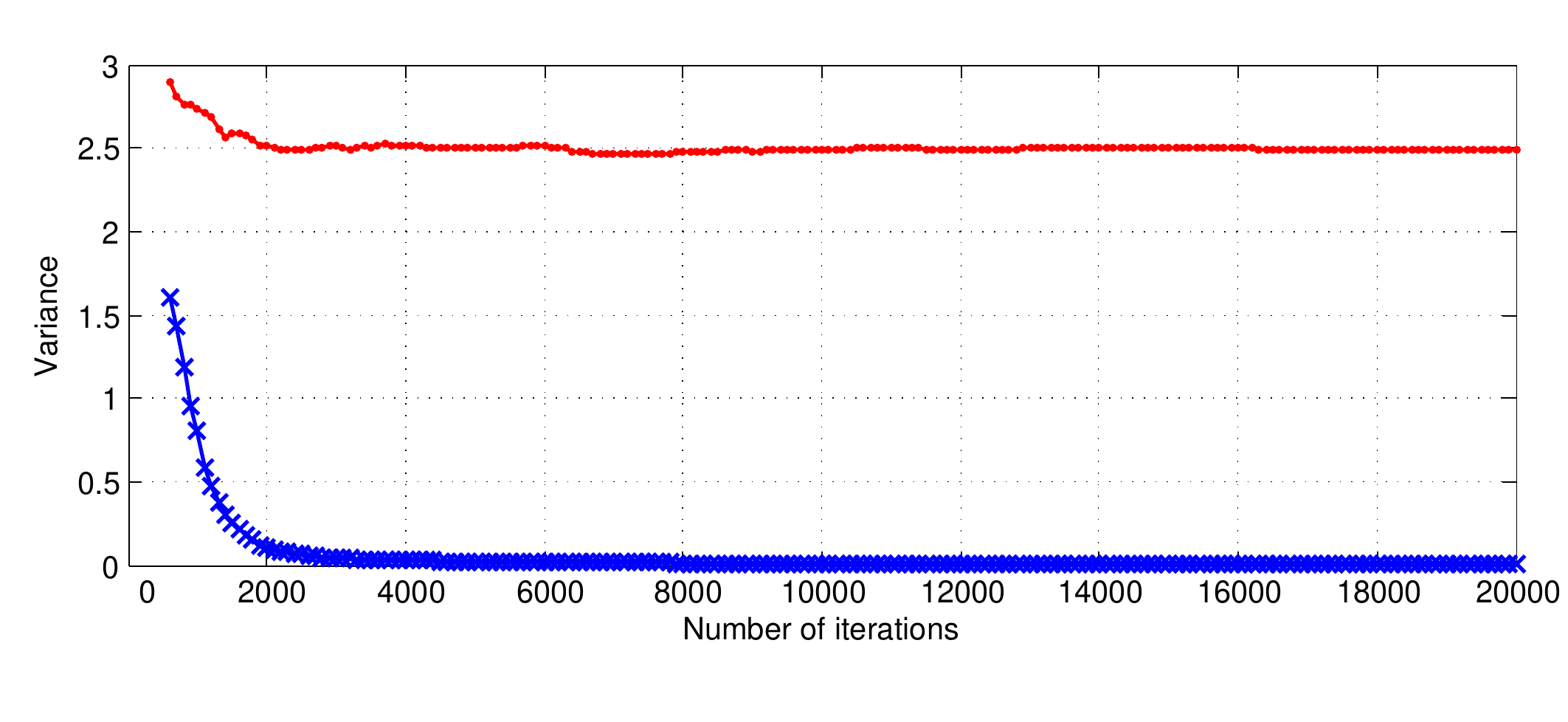}} \label{sim0_mc_cost_var}}}
    \caption{The cost of the best paths in the RRT (shown in red) and the RRT$^*$ (shown in blue)
      plotted against iterations averaged over 500 trials in (a). The optimal cost is shown in
      black. The variance of the trials is shown in (b).}
    \label{figure:sim1cost}
  \end{center}
\end{figure}

\begin{figure*}[htb]
  \begin{center}
    \mbox{ \subfigure[]{\scalebox{0.65}{\includegraphics[trim = 4.55cm 7.6cm 4.05cm 7.35cm, clip =
          true]{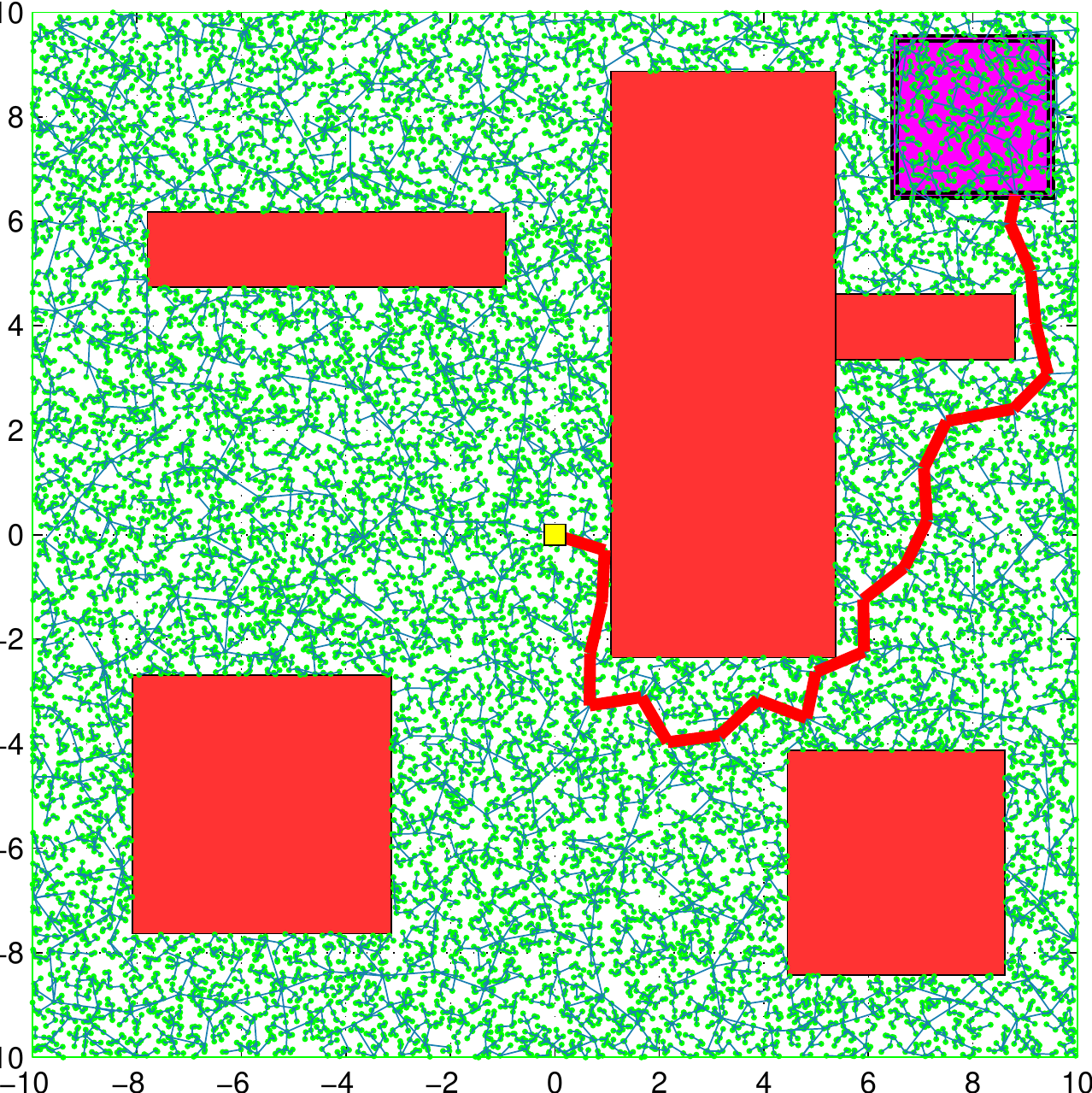}} \label{sim2_rrt_10}} \subfigure[
      ]{\scalebox{0.65}{\includegraphics[trim = 4.55cm 7.6cm 4.05cm 7.35cm, clip =
          true]{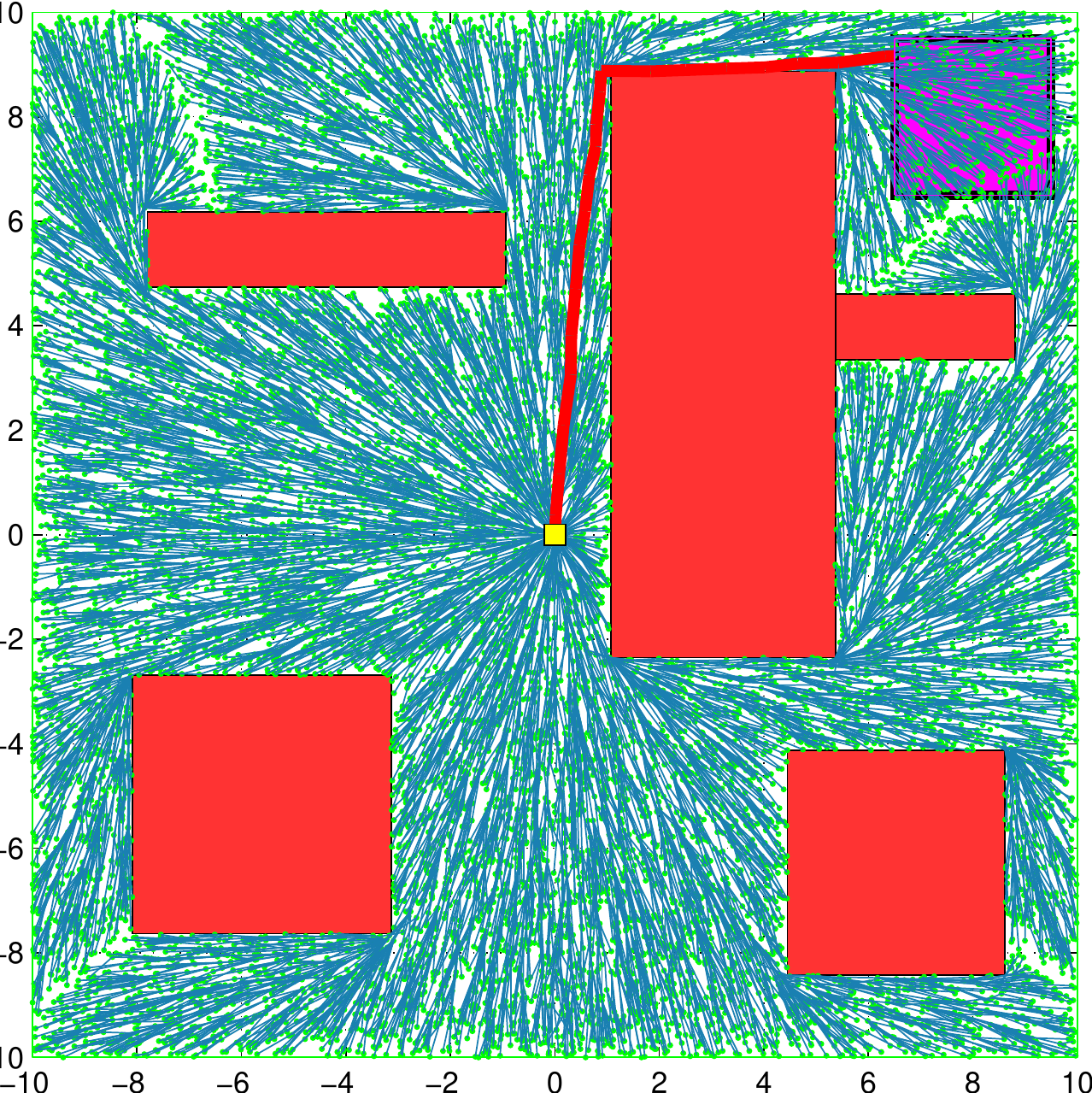}} \label{sim2_optrrt_10}} }
    \caption{A Comparison of the RRT (shown in (a)) and RRT$^*$ (shown in (b)) algorithms on a
      simulation example with obstacles. Both algorithms were run with the same sample sequence
      for 20,000 samples. The cost of best path in the RRT and the RRG were 21.02 and 14.51,
      respectively.}
    \label{figure:sim2}
  \end{center}
\end{figure*}

\begin{figure*}[ht]
  \begin{center}
    \mbox{ \subfigure[]{\scalebox{0.45}{\includegraphics[trim = 4.55cm 7.6cm 4.05cm 7.35cm, clip =
          true]{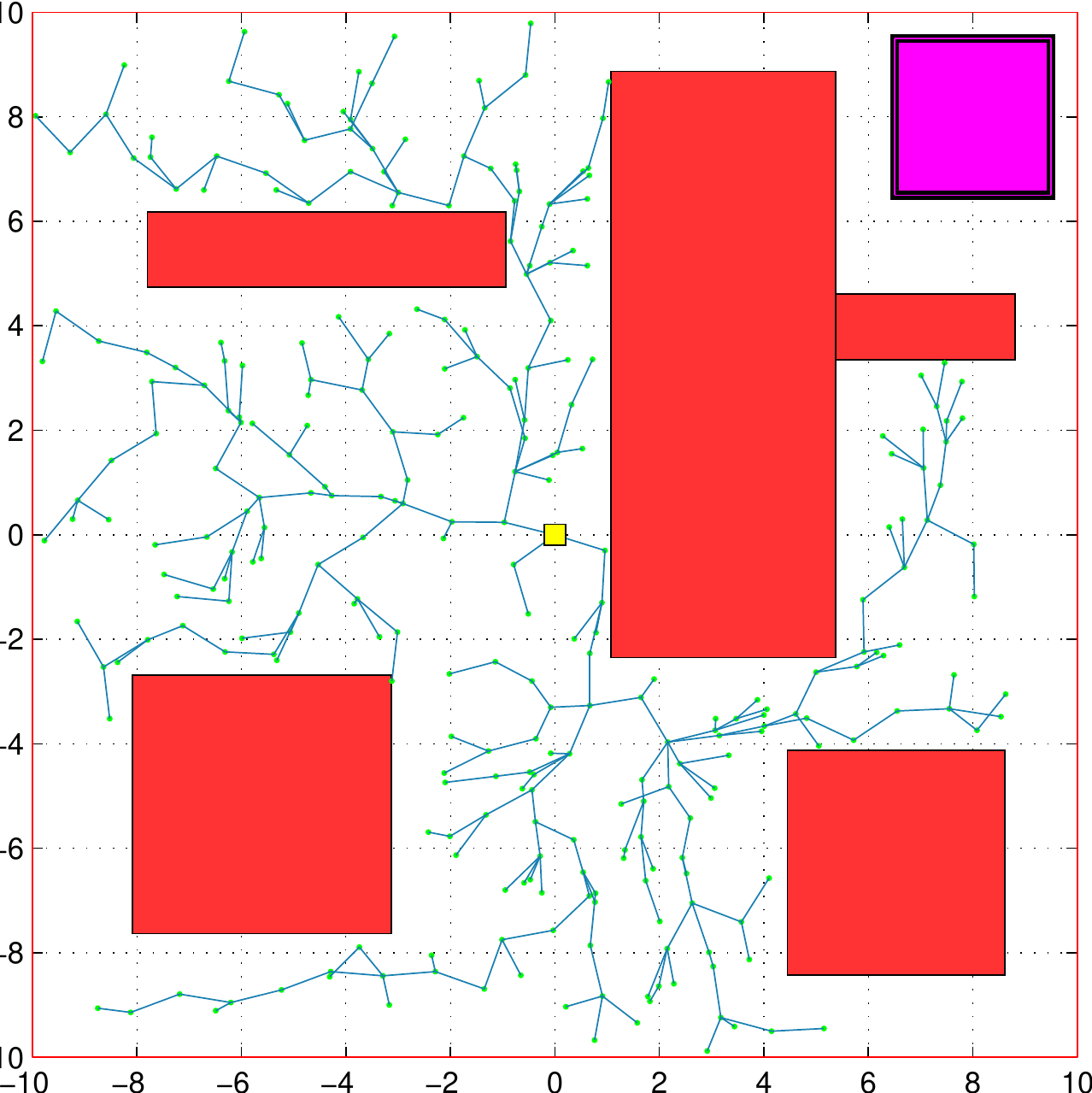}} \label{sim1_rrt_10}} \subfigure[
      ]{\scalebox{0.45}{\includegraphics[trim = 4.55cm 7.6cm 4.05cm 7.35cm, clip = true
          ]{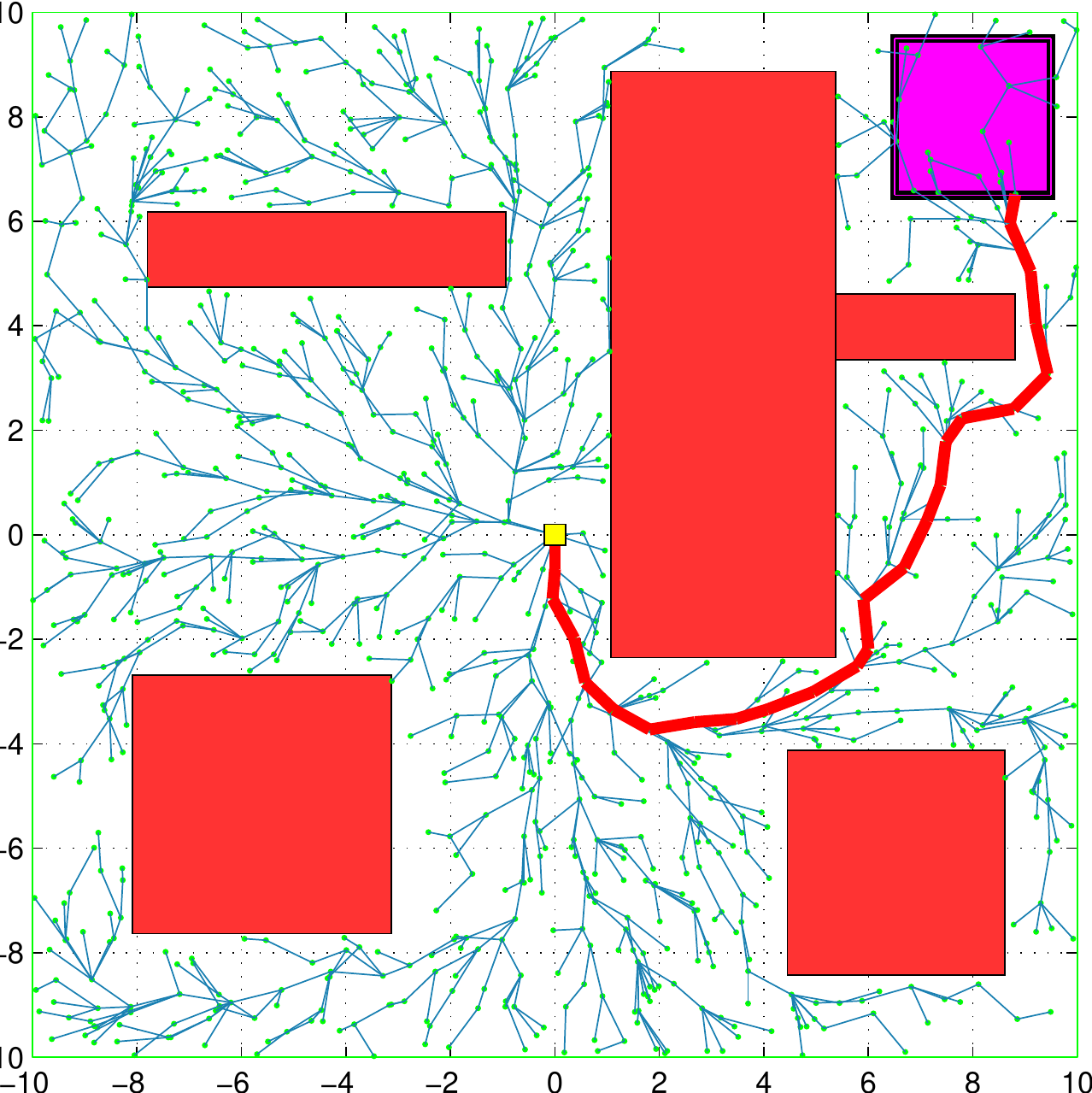}} \label{sim1_optrrt_10}}
      \subfigure[]{\scalebox{0.45}{\includegraphics[trim = 4.55cm 7.6cm 4.05cm 7.35cm, clip = true
          ]{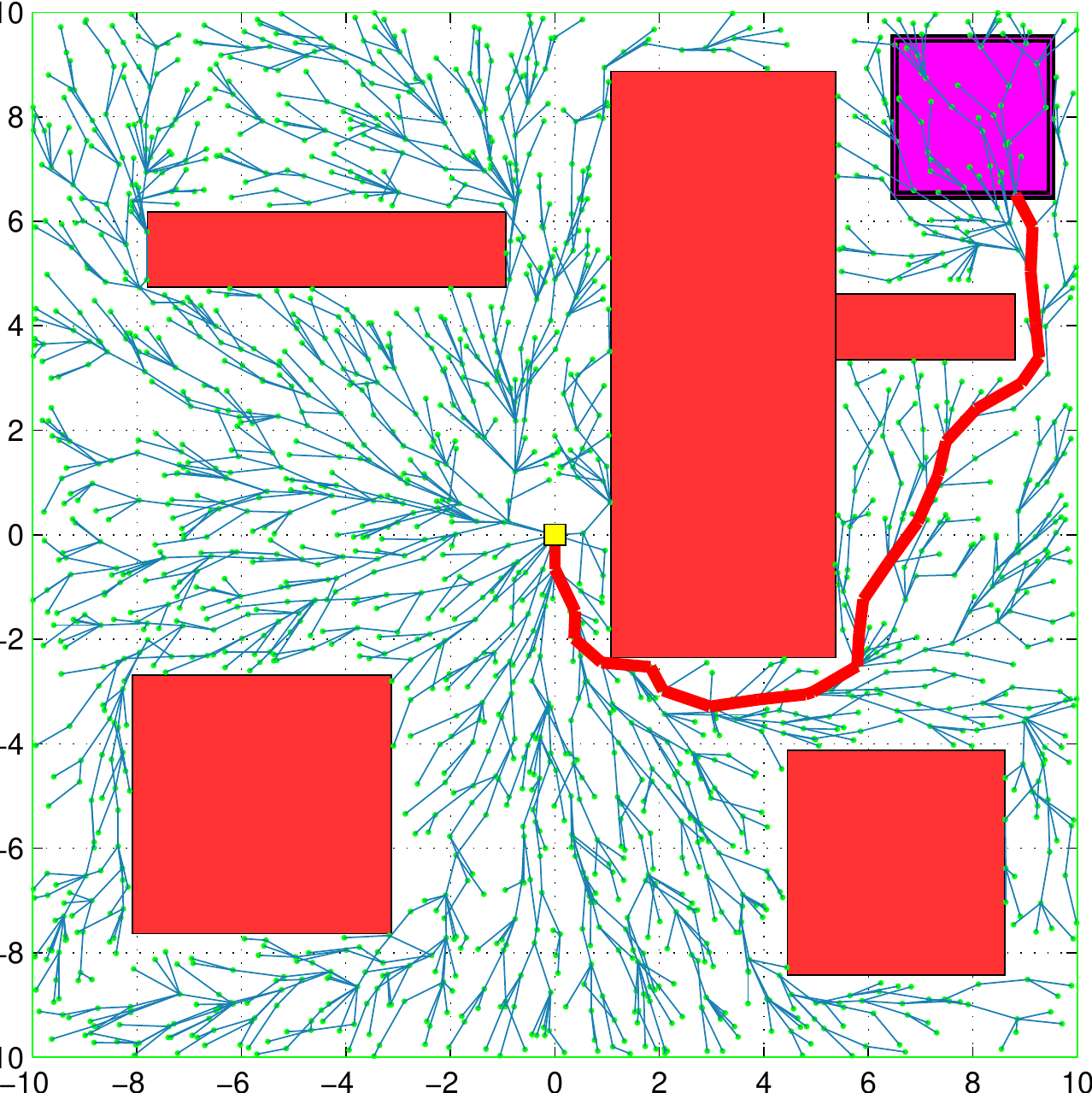}} \label{sim1_rrt_20}} } \mbox{ \subfigure[
      ]{\scalebox{0.45}{\includegraphics[trim = 4.55cm 7.6cm 4.05cm 7.35cm, clip = true
          ]{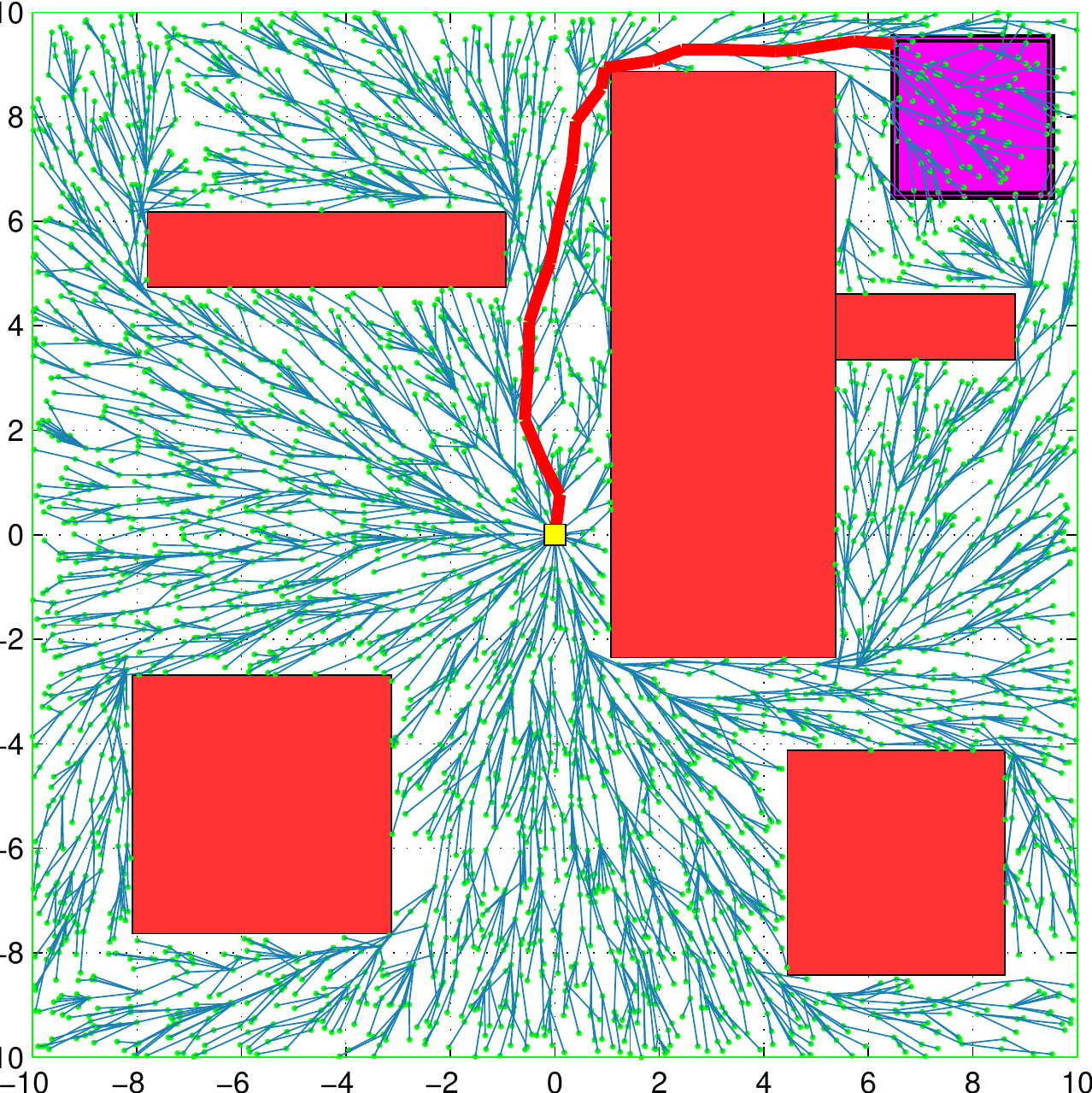}} \label{sim1_optrrt_20}}
      \subfigure[]{\scalebox{0.45}{\includegraphics[trim = 4.55cm 7.6cm 4.05cm 7.35cm, clip = true
          ]{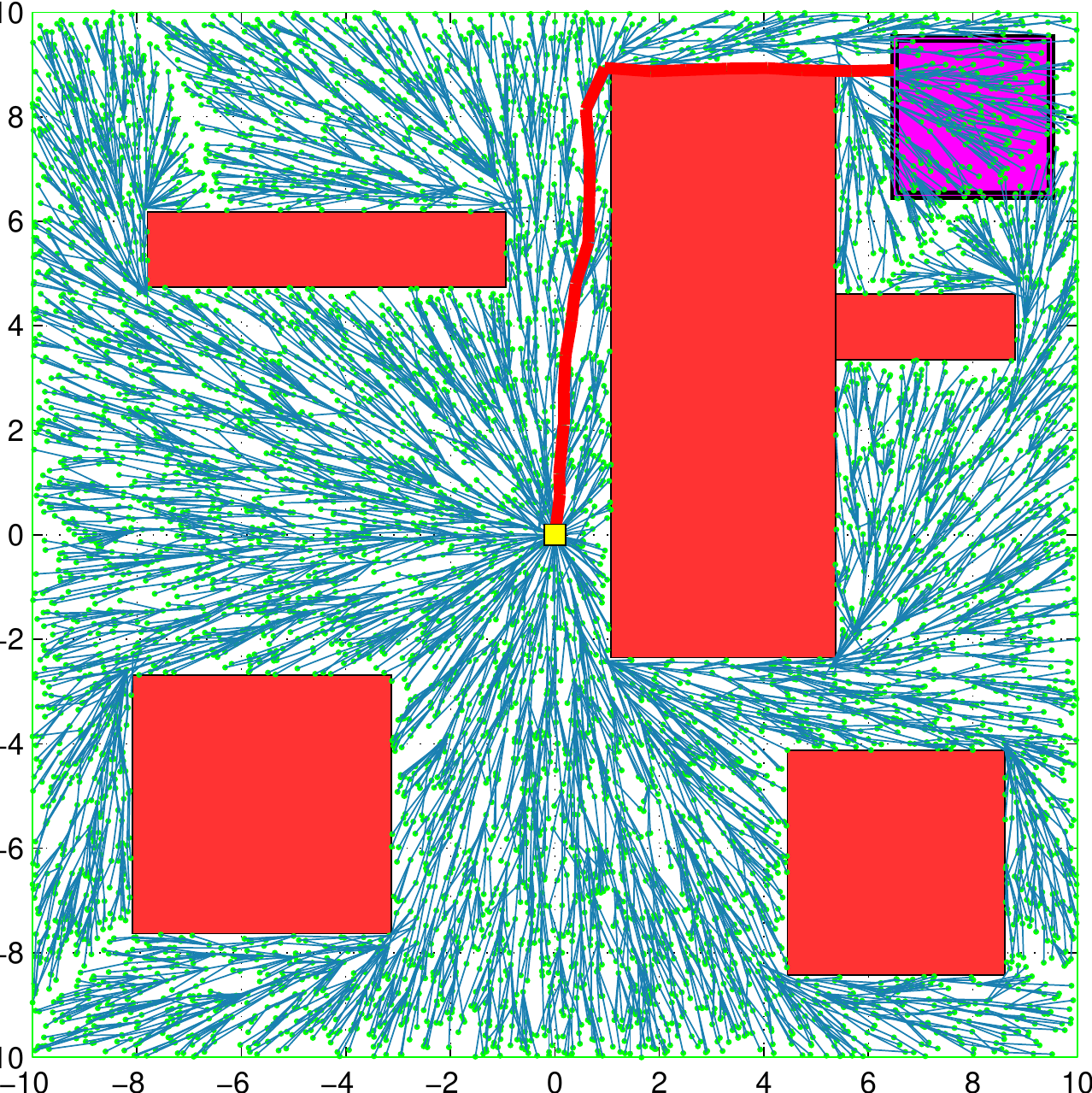}} \label{sim1_rrt_20}} \subfigure[
      ]{\scalebox{0.45}{\includegraphics[trim = 4.55cm 7.6cm 4.05cm 7.35cm, clip = true
          ]{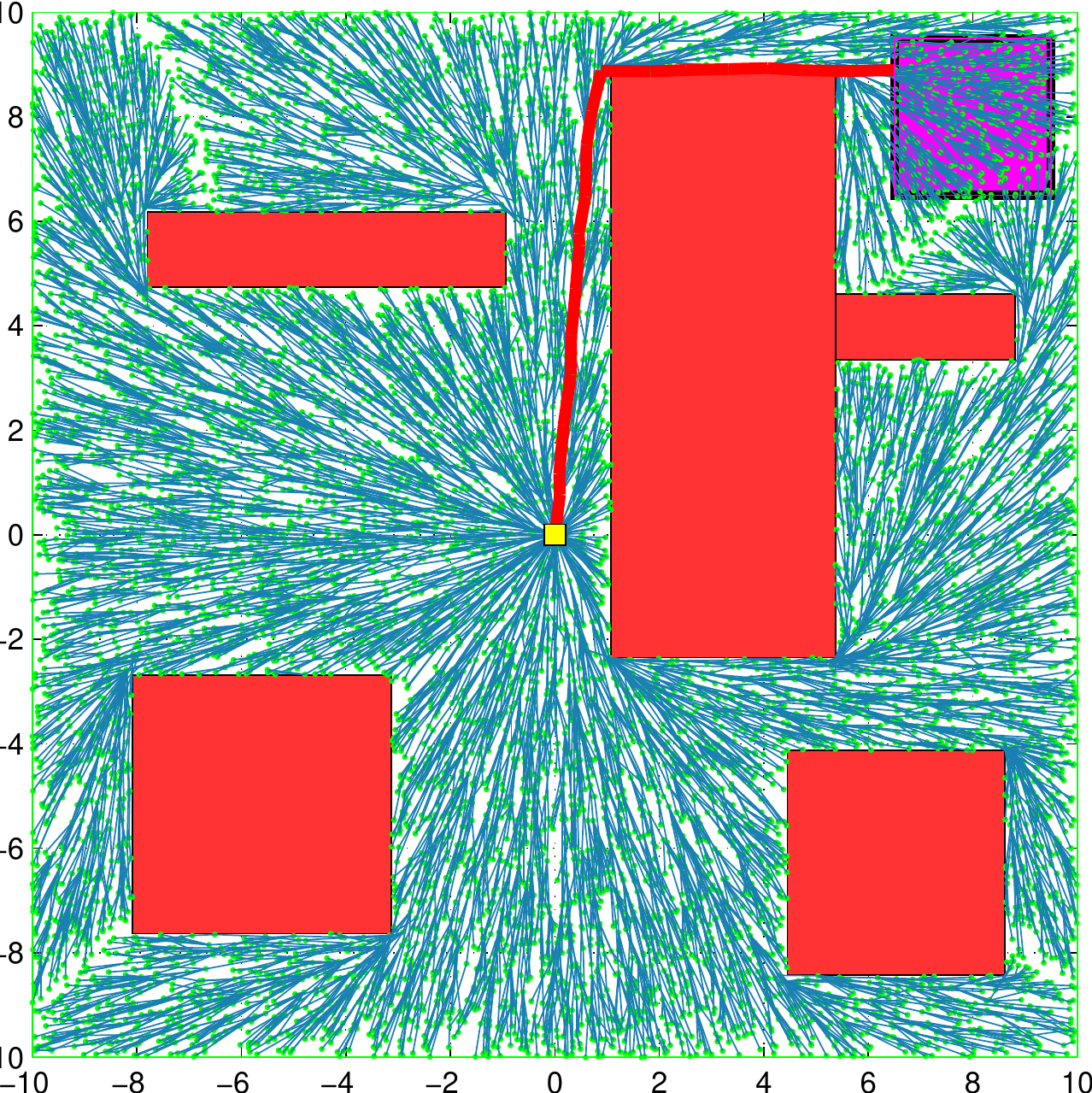}} \label{sim1_optrrt_20}} }
    \caption{RRT$^*$ algorithm shown after 500 (a), 1,500 (b), 2,500 (c), 5,000 (d), 10,000 (e), 15,000
      (f) iterations.}
    \label{figure:sim2optrrt}
  \end{center}
\end{figure*}

\begin{figure}[ht]
  \begin{center}
    \mbox{
      \subfigure[]{\scalebox{0.5}{\includegraphics[trim = 0cm 9cm 0cm 9cm, clip = true, height =
    8.8cm]{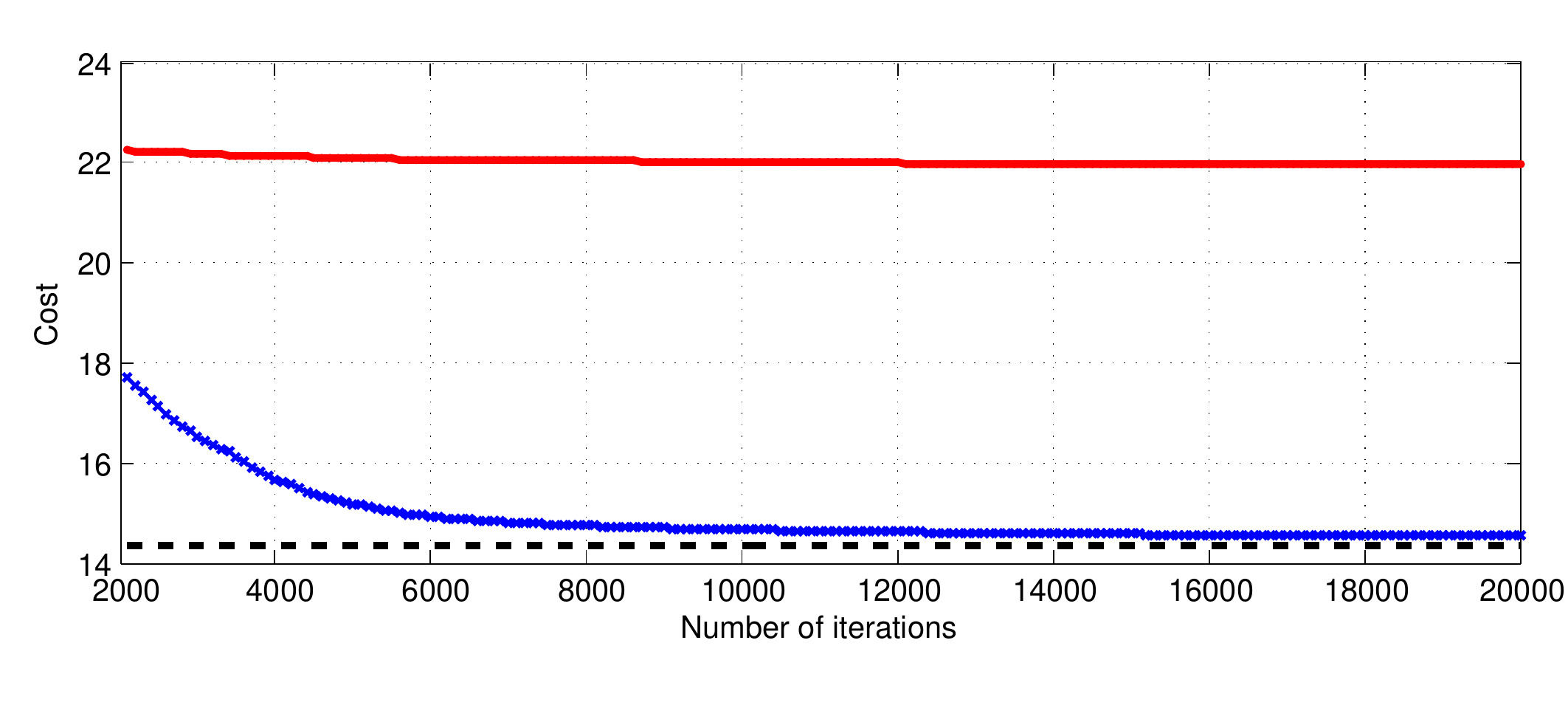}} 
        \label{sim0_mc_cost}}}
    \mbox{ 
      \subfigure[]{\scalebox{0.5}{\includegraphics[trim = 0cm 9cm 0cm 9cm, clip = true, height =
    8.8cm]{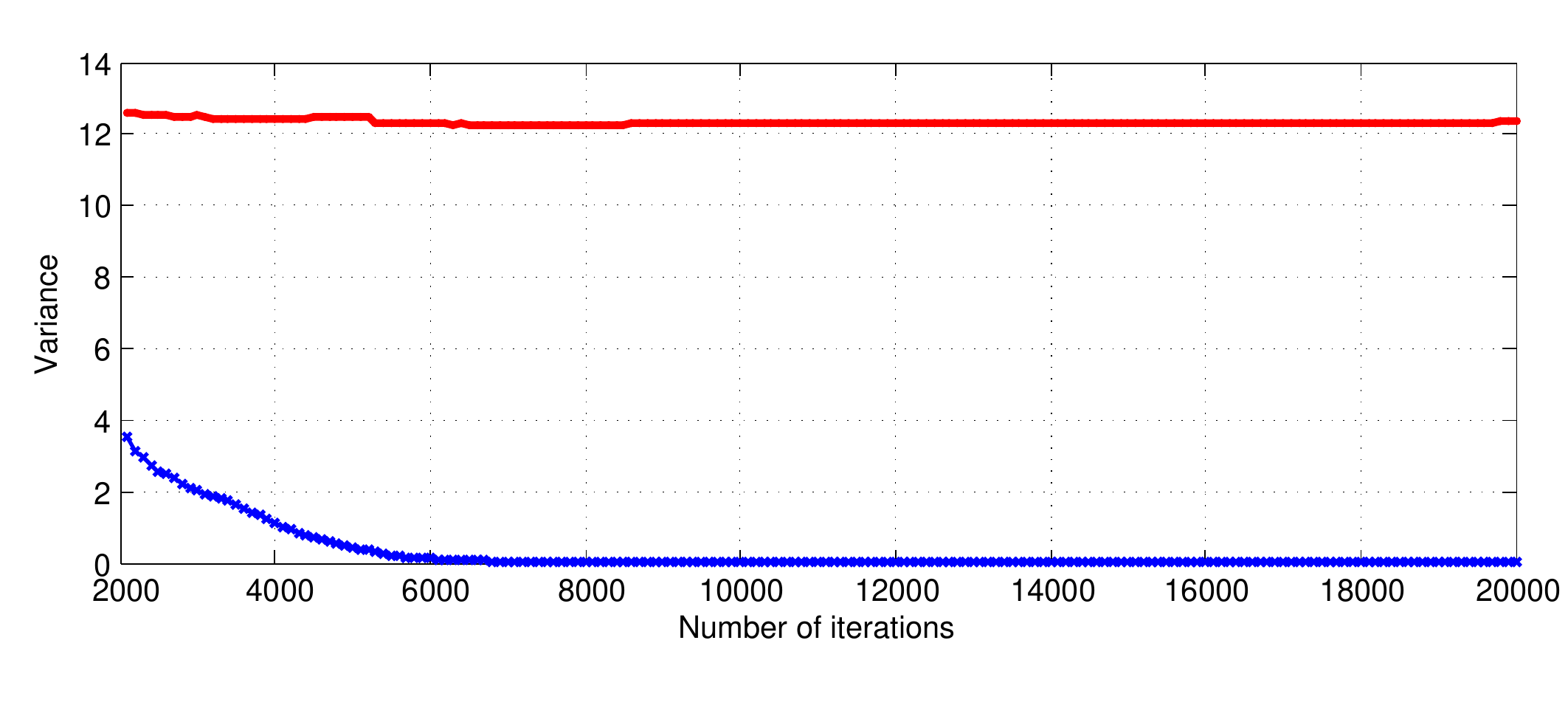}} 
        \label{sim0_mc_cost_var}}}
    \caption{An environment cluttered with obstacles is considered. The cost of the best paths in
      the RRT (shown in red) and the RRT$^*$ (shown in blue) plotted against iterations averaged
      over 500 trials in (a). The optimal cost is shown in black. The variance of the trials is
      shown in (b).}
    \label{figure:sim2cost}
  \end{center}
\end{figure}

\begin{figure}[htb]
  \begin{center}
    \includegraphics[trim = 2.9cm 1.8cm 2cm 1.4cm, clip = true, height =
    8.8cm]{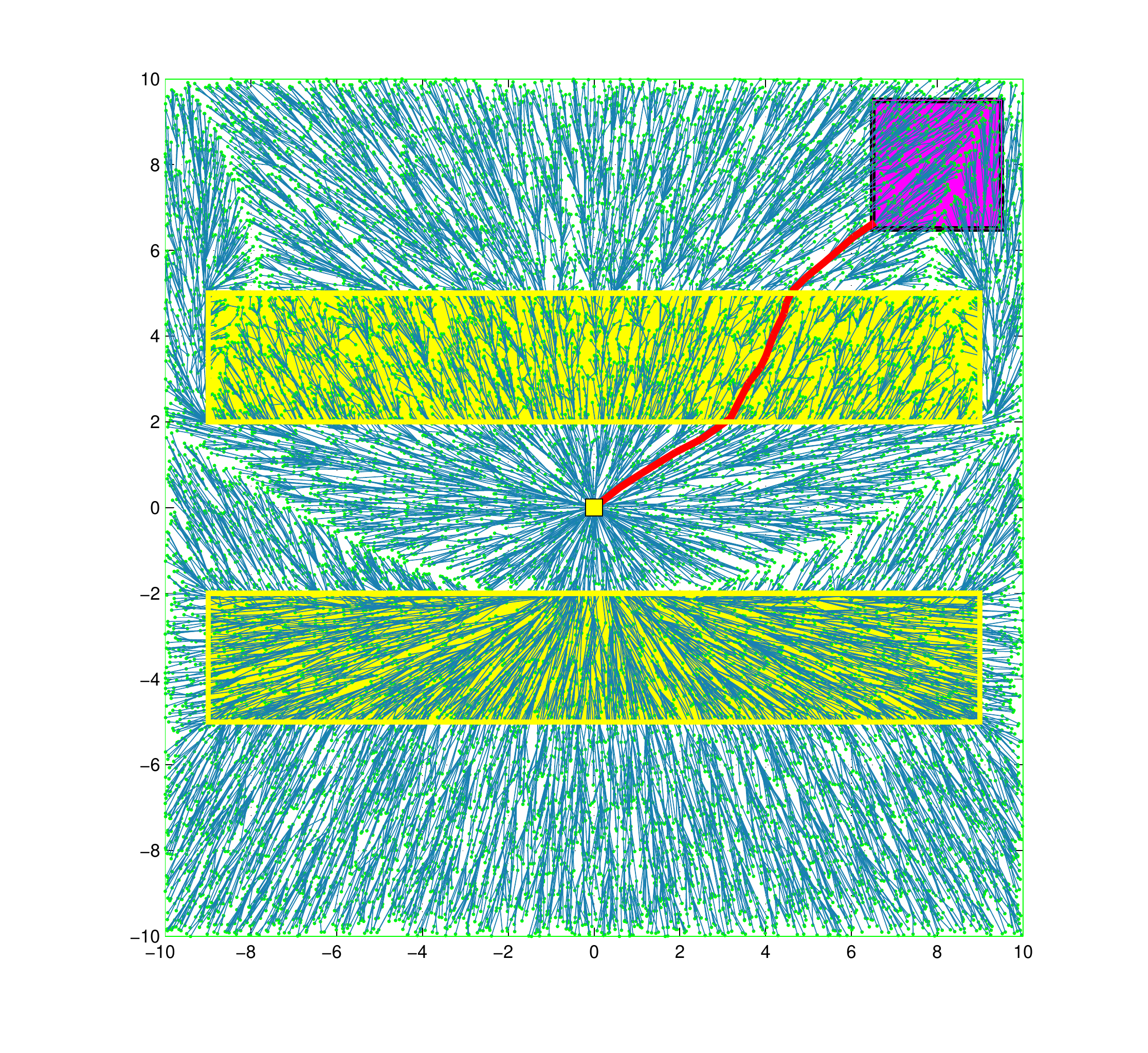}
    \caption{RRT$^*$ algorithm at the end of iteration 20,000 in an environment with no obstacles.
      The upper yellow region is the high-cost region, whereas the lower yellow region is low-cost.}
    \label{figure:sim3}
  \end{center}
\end{figure}

\begin{figure}[ht]
  \begin{center}
    \includegraphics[trim = 0cm 9cm 0cm 9cm, clip = true, height =
    5cm]{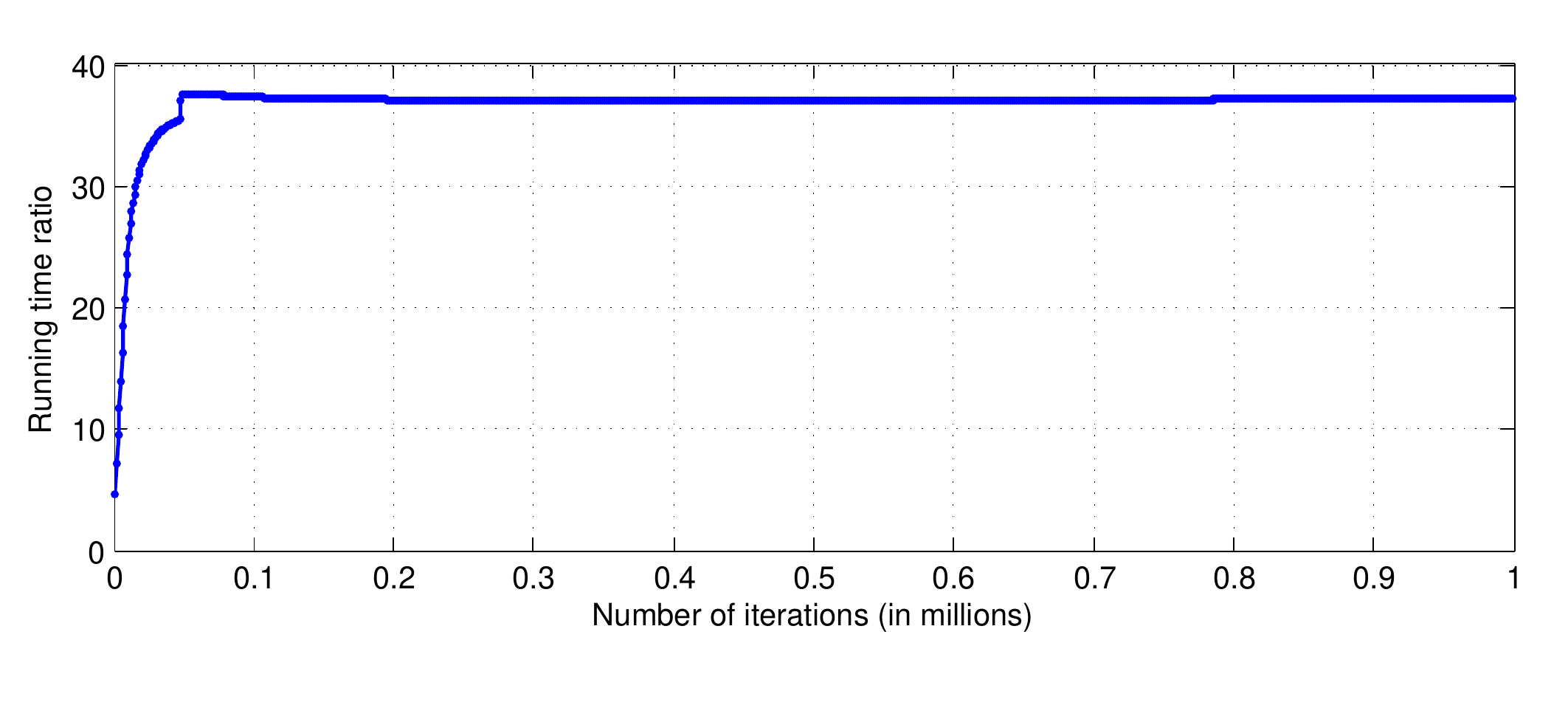}
    \caption{A comparison of the running time of the RRT$^*$ and the RRT algorithms. The ratio of
      the running time of the RRT$^*$ over that of the RRT up until each iteration is plotted versus
      the number of iterations.}
    \label{figure:sim0time}
  \end{center}
\end{figure}

\begin{figure}[ht]
  \begin{center}
    \includegraphics[trim = 0cm 9cm 0cm 9cm, clip = true, height =
    5cm]{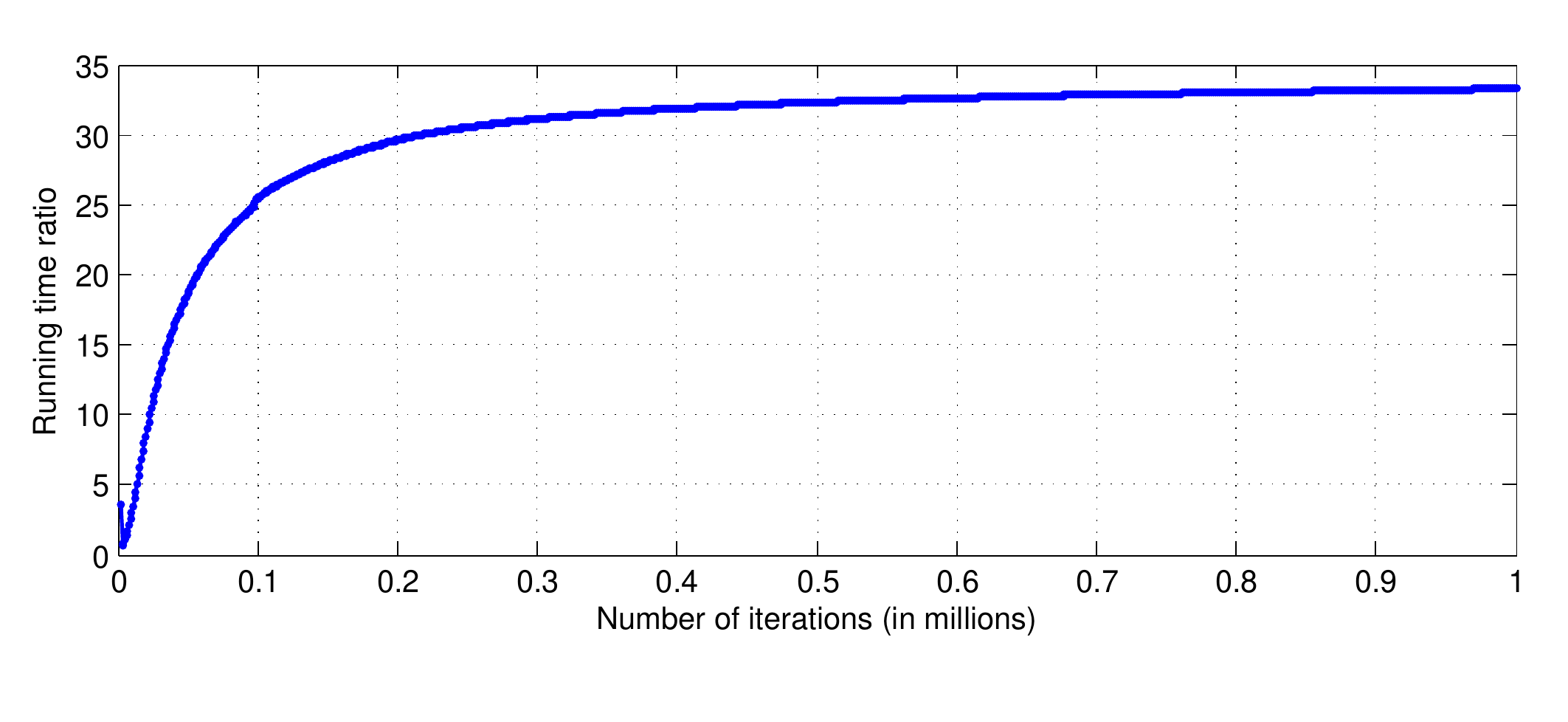}
    \caption{A comparison of the running time of the RRT$^*$ and the RRT algorithms in an
      environment with obstacles. The ratio of the running time of the RRT$^*$ over that of the RRT
      up until each iteration is plotted versus the number of iterations.}
    \label{figure:sim1time}
  \end{center}
\end{figure}

\begin{figure}[htb]
\begin{center}
\mbox{ \subfigure[]{\scalebox{0.5}{\includegraphics[trim = 0cm 9cm 0cm 9cm, clip = true]{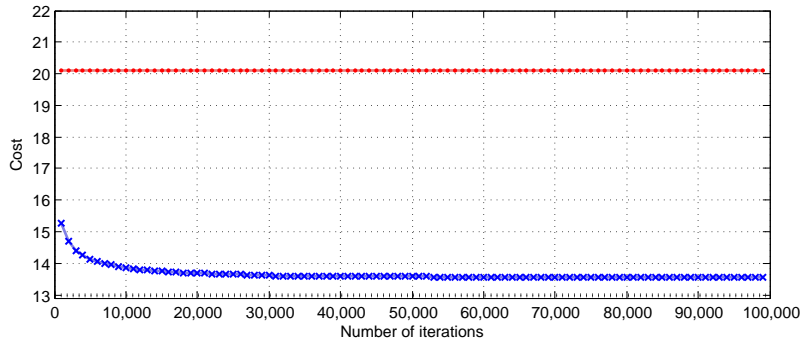}} \label{sim_rrtstar:d5_cost}}}
\mbox{ \subfigure[ ]{\scalebox{0.5}{\includegraphics[trim = 0cm 9cm 0cm 9cm, clip= true]{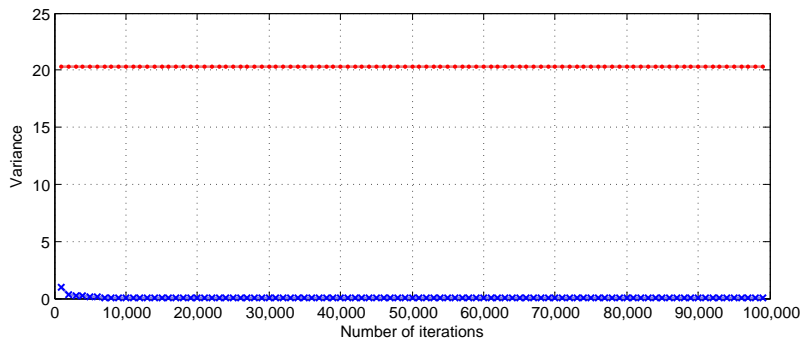}} \label{sim_rrtstar:d5_var}}}
\caption{The cost of the best paths in the RRT (shown in red) and the RRT$^*$ (shown in blue) run in a 5 dimensional obstacle-free configuration space plotted against iterations averaged over 100 trials in (a). The optimal cost is shown in black. The variance of the trials is shown in (b).}
\label{figure:rrtstar_5d}
\end{center}
\end{figure}

\begin{figure}[htb]
  \begin{center}
  \mbox{ \subfigure[ ]{\scalebox{0.5}{\includegraphics[trim = 0cm 9cm 0cm 9cm, clip= true]{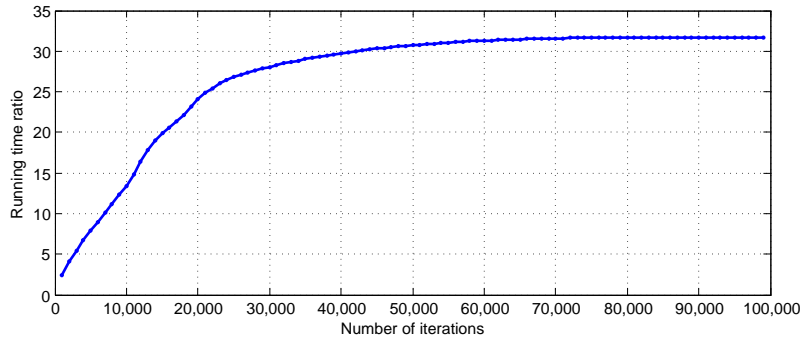}}}}
    \caption{The ratio of the running time of the RRT and the RRT$^*$ algorithms is shown versus the number of iterations.}
    \label{figure:rrtstar_5d_runtime}
  \end{center}
\end{figure}

\begin{figure}[ht]
\begin{center}
\mbox{ \subfigure[]{\scalebox{0.5}{\includegraphics[trim = 0cm 9cm 0cm 9cm, clip = true]{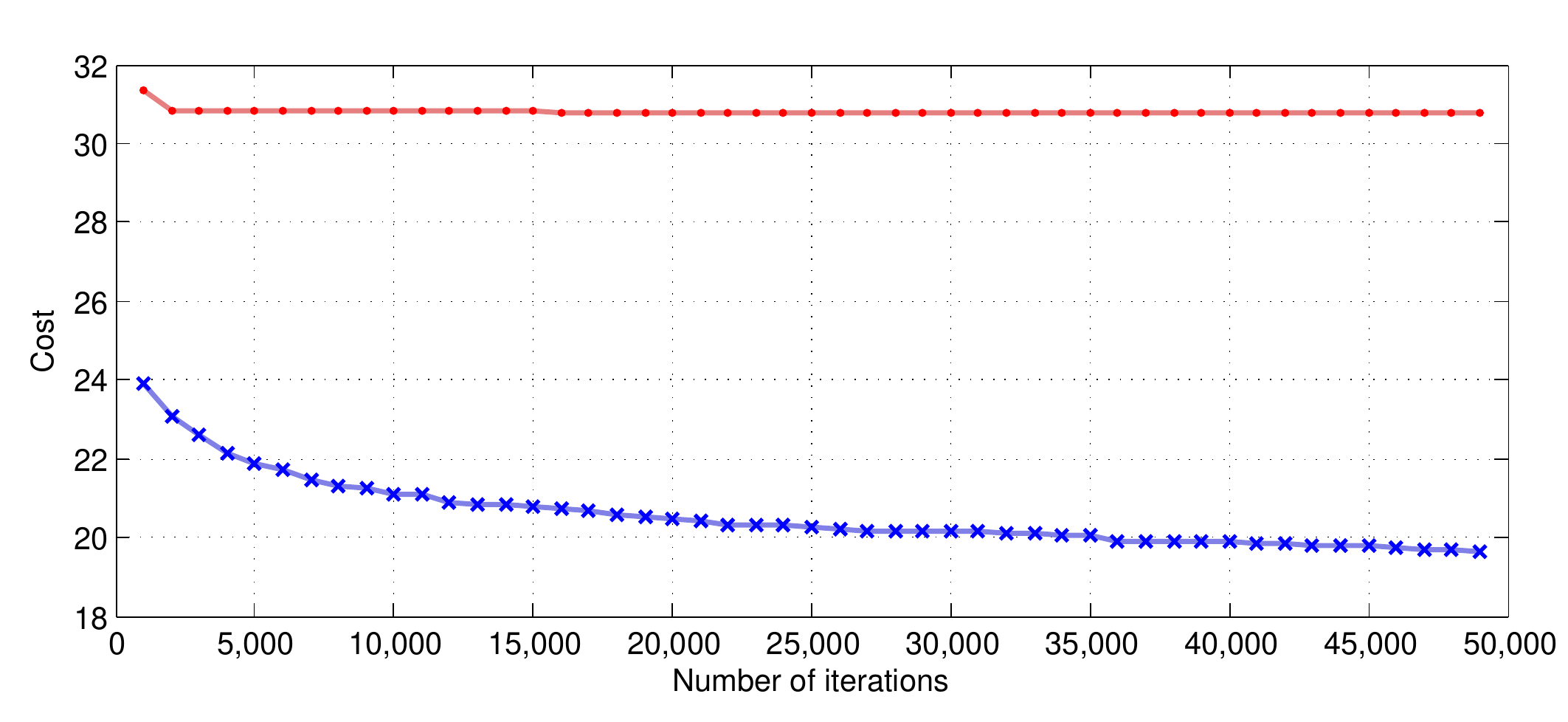}} \label{sim_rrtstar:d10_cost}}}
\mbox{ \subfigure[ ]{\scalebox{0.5}{\includegraphics[trim = 0cm 9cm 0cm 9cm, clip= true]{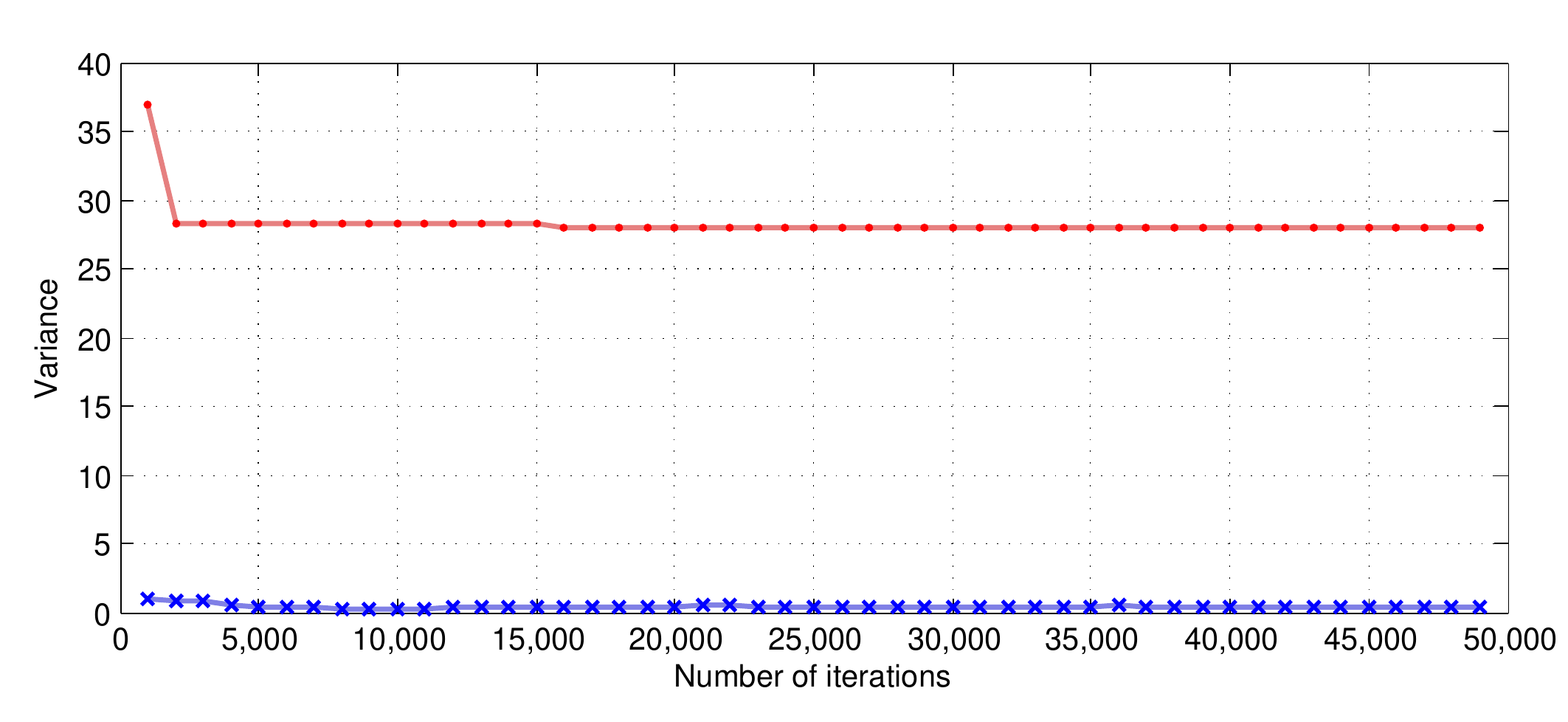}} \label{sim_rrtstar:d10_var}}}
\caption{The cost of the best paths in the RRT (shown in red) and the RRT$^*$ (shown in blue) run in a 10 dimensional configuration space involving obstacles plotted against iterations averaged over 25 trials in (a). The variance of the trials is shown in (b).}
\label{figure:rrtstar_10d}
\end{center}
\end{figure}

\section{Conclusion} \label{section:conclusion}
This paper presented the results of a thorough analysis of sampling-based algorithms for optimal path planning. It is shown that broadly used algorithms from the literature, while probabilistically complete, are not asymptotically optimal, i.e., they will return 
a solution to the path planning problem with high probability if one exists, but the cost of the solution returned by the algorithm will not converge to the optimal cost as the number of samples increases. In particular, it is proven that the PRM and RRT algorithms are not asymptotically optimal. A simplified version of PRM is asymptotically optimal, but is computationally expensive. In addition, it is shown that certain heuristic versions of PRM are not only not asymptotically complete, but also not necessarily complete. 

In order to address these limitations of existing algorithms, a number of new algorithms are introduced, and proven to be asymptotically optimal and computational efficient, with respect to probabilistically complete algorithms in this class. In other words, asymptotic optimality imposes only a constant factor increase in complexity with respect to probabilistic completeness. The first algorithm, called PRM$^*$, is a variant of PRM, with a variable connection radius that scales as $\log(n)/n$, where $n$ is the number of samples. In other words, the average number of connections made at each iteration is proportional to $\log(n)$. The second new algorithm, called RRG, incrementally builds a connected roadmap, augmenting the RRT algorithm with connections within a ball scaling as $\log(n)/n$. The third new algorithm, called RRT$^*$, is a version of RRG that incrementally builds a tree. Experimental evidence that demonstrate the effectiveness of the algorithms proposed and support the
theoretical claims were also provided.

A common theme in the paper is that, in order to ensure both asymptotic optimality and computational efficiency, connections between samples should be sought within balls of radius scaling as $\log(n)/n$. If these balls shrink faster as $n$ increases, the algorithms are not asymptotically optimal (but may still be probabilistically complete); on the other hand, if these balls shrink slower, the complexity of the algorithms will suffer.
On average, the proposed scaling laws will result in an average number of connections  per iteration that is proportional to $\log(n)$. Hence, it is natural to consider variants of these algorithms that make connections to $k \log(n)$ neighbors surely. Indeed, it is shown that these algorithms do share the same asymptotic optimality and computational efficiency properties of their counterparts, as long as $k$ is no smaller than a constant $k^*_\mathrm{RRG}$. It is remarkable that this constant only depends on the dimension of the space, and is otherwise independent from the problem instance.

The analysis of the results in the paper relies on techniques used to analyze random geometric graphs. Indeed, the algorithms considered in this paper build graphs that have many characteristics in common with well known classes of random geometric graphs. Interestingly, such geometric graphs exhibit phase transition phenomena, including percolation and connectivity, for thresholds matching those found for probabilistic completeness and asymptotic optimality of sampling-based algorithms. This leads to a natural conjecture that 
a sampling-based path planning algorithm is probabilistically complete if and only if the underlying random geometric graph percolates, and is asymptotically optimal if and only if the underlying random geometric graph is connected. 

The work presented in this paper can be extended in numerous directions. 
First of all, it would be of interest to establish broader connections between sampling-based path planning algorithms and random geometric graphs, e.g., by proving or disproving the conjecture above, and by possibly improving on current algorithms through a better understanding of the underlying mathematical objects. Similar analysis techniques can also be used to analyze other sampling-based path planning algorithms that were not analyzed in this paper, such as EST. In addition, it is of interest to investigate deterministic sampling-based algorithms, in which samples are generated using deterministic dense sequences of points with, e.g., low dispersion, as opposed to random sequences. 

Second, it is of great practical interest to  address motion planning problems subject to more complex constraints. For example, motion planning problems for mobile robots should consider the robot's dynamics, and hence differential constraints on the feasible trajectories (these are also called kino-dynamic planning problems). In addition, it is of interest to consider  optimal planning problems in the presence of
temporal/logic constraints on the trajectories, e.g., expressed using formal specification languages
such as Linear Temporal Logic, or the $\mu$-calculus. Such constraints correspond to, e.g., rules of
the road constraints for autonomous ground vehicles, mission specifications for autonomous robots,
and rules of engagement in military applications. Ultimately, incremental sampling-based algorithms
with asymptotic optimality properties may provide the basic elements for the on-line solution of
differential games, as those arising when planning in the presence of dynamic obstacles.

Finally, it is noted that the proposed algorithms may have applications outside of the robotic
motion planning domain. In fact, the class of sampling-based algorithm described in this paper
can be readily extended to deal with problems described by partial differential equations, such as the eikonal equation and the Hamilton-Jacobi-Bellman equation.

\section*{Acknowledgments}
The authors are grateful to Professors M.S. Branicky, G.J. Gordon, and S. LaValle, as well as the anonymous reviewers, for their insightful comments on draft versions of this paper. This research was supported in part by the Michigan/AFRL Collaborative Center on Control Sciences, AFOSR grant \#FA 8650-07-2-3744, and by the National Science Foundation, grant CNS-1016213.

\bibliography{Karaman.Frazzoli.IJRR10.final}

\bibliographystyle{plainnat}

\appendix 

\section*{Appendix}

\section{Notation} \label{appendix:notation}

Let $\naturals$ denote the set of positive integers and $\reals$ denote the set of reals. Let $\naturals_0=\naturals \cup \{0\}$, and $\reals_{>0}$, $\reals_{\ge 0}$ denote the sets of positive and non-negative reals, respectively. A sequence on a set $A$ is a mapping from $\naturals$ to $A$, denoted as $\{ a_i\}_{i \in \naturals}$, where $a_i \in A$ is the element that $i \in \naturals$ is mapped to. Given $a, b \in \reals$, closed and open intervals between $a$ and $b$ are denoted by $[a,b]$ and $(a,b)$, respectively. The Euclidean norm is denoted by $\Vert \cdot \Vert$. 
Given a set $\X \subset \reals^d$, the closure of $\X$ is denoted by $\Cl(\X)$. The closed ball of radius $r>0$ centered at $x \in \reals^d$, i.e., , i.e., $\{ y \in \reals^d \,\vert\,\, \Vert y - x \Vert \le r\}$, is denoted as ${\cal B}_{x, r}$; ${\cal B}_{x,r}$ is also called the $r$-ball centered at $x$. Given a set $\X \subseteq \reals^d$, the Lebesgue measure of $X$ is denoted by $\mu (\X)$. The Lebesgue measure of a set is also referred to as its volume. The volume of the unit ball in $\reals^d$, is denoted by $\VolumeDBall{d}$, i.e., $\VolumeDBall{d}=\mu ({\cal B}_{0,1})$.
The letter $e$ is used to denote the base of the natural logarithm, also called Euler's number.

Given a probability space $(\Omega, {\cal F}, \PP)$, where $\Omega$ is a sample space, ${\cal F} \subseteq 2^\Omega$ is a $\sigma-$algebra, and $\PP$ is a probability measure, an event $A$ is an element of ${\cal F}$. The complement of an event $A$ is denoted by $A^c$. Given a sequence of events $\{ A_\N \}_{\N \in \naturals}$, the event $\cap_{\N = 1}^\infty \cup_{i = \N}^\infty A_i$ is denoted by $\limsup_{\N \to \infty} A_\N$ (also called the event that $A_\N$ occurs infinitely often); the event $\cup_{\N = 1}^\infty \cap_{i = \N}^\infty A_i$ is denoted by $\liminf_{\N \to \infty} A_\N$. A (real) random variable is a measurable function that maps $\Omega$ into $\reals$. An extended (real) random variable can also take the values $\pm\infty$. The expected value of a random variable $Y$ is $\EE[Y]=\int_\Omega Y \; d\PP$. 
A sequence of random variables $\{ Y_\N\}_{\N \in \naturals}$ is said to converge surely to a random variable $Y$ if $\lim_{\N \to \infty}Y_\N (\omega) = Y (\omega)$ for all $\omega \in \Omega$; the sequence is said to converge almost surely if $\PP (\{ \lim_{\N \to \infty} Y_\N = Y\}) = 1$. Finally, if $\varphi (\omega)$ is a property that is either true or false for a given $\omega \in \Omega$, the event that denotes the set of all samples $\omega$ for which $\varphi(\omega)$ holds, i.e., $\{ \omega \in \Omega \,\vert\, \varphi(\omega) \mbox{ holds} \}$, is written as $\{\varphi\}$, e.g., $\{\omega \in \Omega \,\vert\, \lim_{\N \to \infty} Y_\N (\omega) = Y(\omega)\}$ is simply written as $\{\lim_{\N \to \infty}Y_\N = Y\}$. The Poisson random variable with parameter $\lambda$ is denoted by $\Poisson(\lambda)$. The binomial random variable with parameters $n$ and $p$ is denoted by $\Binomial(n, p)$.

Let $f(n)$ and $g(n)$ be two functions with domain and range $\naturals$ or $\reals$. The function $f(n)$ is said to be $O(g(n))$, denoted as $f(n) \in O(g(n))$, if there exists two constants $M$ and $n_0$ such that $f(n) \le M g(n)$ for all $n \ge n_0$. The function $f(n)$ is said to be $\Omega(g(n))$, denoted as $f(n) \in \Omega(g(n))$, if there exists constants $M$ and $n_0$ such that $f(n) \ge M g(n)$ for all $n \ge n_0$. The function $f(n)$ is said to be $\Theta(g(n))$, denoted as $f(n) \in \Theta(g(n))$, if $f(n) \in O(g(n))$ and $f(n) \in \Omega(g(n))$.

Let $\X$ be a subset of $\reals^d$. A (directed) graph $G = (V,E)$ on $\X$ is composed of a vertex set $V$ and an edge set $E$, such that $V$ is a finite subset of $\X$, and $E$ is a subset of $V \times V$. A directed path on $G$ is a sequence $(v_1, v_2, \dots, v_n)$ of vertices such that $(v_i,v_{i+1}) \in E$ for all $1 \le i\le n-1$. Given a vertex $v \in V$, the sets $\{ u \in V \,\vert\, (u,v) \in E\}$ and $\{u \in V \,\vert\, (v,u) \in E\}$ are said to be its incoming neighbors and outgoing neighbors, respectively. A (directed) tree is a directed graph, in which each vertex but one has a unique incoming neighbor; the vertex with no incoming neighbor is called the root vertex. Vertices of a tree are often also called nodes.

\section{Proof of Theorem~\ref{theorem:optimality_rrt} (Non-optimality of RRT)} \label{section:proof:theorem:optimality_rrt}

For simplicity, the theorem will be proven assuming that (i) the environment contains no obstacles, i.e., $\mathcal{X}_\mathrm{free} = [0,1]^{d}$, and (ii) the parameter $\eta$ of the steering procedure is set large enough, e.g., $\eta \ge \mathrm{diam}\left(\mathcal{X}_\mathrm{free}\right)=\sqrt{d}$.  One one hand, considering this case is enough to prove that the RRT algorithm is not asymptotically optimal, as it demonstrates a case for which the RRT algorithm fails to converge to an optimal solution, although the problem instance is clearly robustly optimal. On the other hand, these assumptions are not essential, and the claims extend to the more general case, but the technical details of the proof are considerably more complicated.

The proof can be outlined as follows. Order the vertices in the RRT according to the iteration at which they are added to the tree. The set of vertices that contains the $k$-th child of the root along with all its descendants in the tree is called the $k$-th branch of the tree. First, it is shown that a necessary condition for the asymptotic optimality of RRT is that infinitely many branches of the tree  contain vertices outside a small ball centered at the initial condition. Then, the RRT algorithm is shown to violate this  condition, with probability one.

\subsection{A necessary condition}

First, we provide a necessary condition for the RRT algorithm to be asymptotically optimal.

\begin{lemma} \label{lemma:rrt_optimality:necessary_condition}
Let $0<R<\inf_{y \in \mathcal{X}_\mathrm{goal}} \Vert y-x_\mathrm{init} \Vert$.
The event $\{ \lim_{N \to \infty} Y_n^\mathrm{RRT} = c^*\}$ occurs only if the $k$-th branch of the RRT contains vertices outside the $R$-ball centered at $x_\mathrm{init}$ for infinitely many $k$.
\end{lemma}

\begin{proof}
Let $\{x_1, x_2, \dots\}$ denote the set of children to the root vertex in the order they are added to the tree. Let $\Gamma(x_k)$ denote the optimal cost of a path starting from the root vertex, passing through $x_k$, and reaching the goal region. By our assumption that the measure of the set of all points that are on the optimal path is zero (see Assumption 27 and Lemma 28), the probability that $\Gamma(x_k) = c^*$ is zero for all $k \in \naturals$. Hence, 
$$
\PP\Big(\bigcup\nolimits_{k \in \naturals} \left\{ \Gamma(x_k) = c^* \right\} \Big) \,\, \le \,\, \sum_{k =1}^\infty \PP \big( \{\Gamma(x_k) = c^*\} \big) = 0.
$$

Let $A_k$ denote the event that at least one vertex in the $k$-th branch of the tree is outside the ball of radius $R$ centered at $x_\mathrm{init}$ in some iteration of the RRT algorithm.
Consider the case when the event $\{\limsup_{k \to \infty} A_k\}$ does not occur and the events $\{\Gamma(x_k) > c^*\}$ occur for all $k \in \naturals$. Then, $A_k$ occurs for only finitely many $k$. Let $K$ denote the largest number such that $A_K$ occurs. Then, the cost of the best path in the tree is at least $\sup\{\Gamma(x_k) \given k \in \{1,2,\dots, K\}\}$, which is strictly larger than $c^*$, since $\{\Gamma(x_k) > c^*\}$ for all finite $k$. Thus, $\lim_{n \to \infty} Y_n^\mathrm{RRT} > c^*$ must hold.
That is, we have argued that
$$
\Big(\limsup_{k \to \infty} A_k\Big)^c \cap \Big(\bigcap_{k \in \mathbb{N}} \{ \Gamma(x_k) > c^* \} \Big) \subseteq \Big\{ \lim_{n \to \infty} Y_n^\mathrm{RRT} > c^* \Big\}.
$$
Taking the complement of both sides and using monotonicity of probability measures,
\begin{eqnarray*}
\PP \left( \big\{ \lim_{n \to \infty} Y_n^\mathrm{RRT} = c^* \big\} \right) & \le & \PP\Big( \big(\limsup_{k \to \infty} A_k \big) \cup \big(\bigcup\nolimits_{k \in \naturals} \{\Gamma(x_k) = c^*\}\big) \Big), \\
& \le & \PP\Big( \limsup_{k \to \infty} A_k  \Big) + \PP\Big( \bigcup\nolimits_{k \in \naturals} \{\Gamma(x_k) = c^*\} \Big),
\end{eqnarray*}
where the last inequality follows from the union bound. The lemma follows from the fact that the last term in the right hand side is equal to zero as shown above.
\qed
\end{proof}

\subsection{Length of the first path in a branch}
The following result provides a useful characterization of the RRT structure.
\begin{lemma} \label{lemma:rrt_optimality:connection_statistics}
Let $U = \{X_1,X_2, \dots, X_n\}$ be a set of independently sampled and uniformly distributed points in the $d$-dimensional unit cube, $[0,1]^d$. Let $X_{n+1}$ be a point that is sampled independently from all the other points according to the uniform distribution on $[0,1]^d$. Then, the probability that among all points in $U$ the point $X_{i}$ is the one that is closest to $X_{n+1}$ is $1/n$, for all $i \in \{1,2,\dots, n\}$. Moreover, the expected distance from $X_{n+1}$ to its nearest neighbor in $U$ is $n^{-1/d}$.
\end{lemma}
\begin{proof}
Since the probability distribution is uniform, the probability that $X_{n+1}$ is closest to $X_i$ is the same for all $i \in \{1,2,\dots, n\}$, which implies that this probability is equal to $1/n$. The expected distance to the closest point in $U$ is an application of the order statistics of the  uniform distribution.
\qed
\end{proof}
An immediate consequence of this result is that each vertex of the RRT has unbounded degree, almost surely, as the number of samples approaches infinity.

One can also define a notion of infinite paths in the RRT, as follows. Let $\Lambda$ be the set of infinite sequences of natural numbers $\alpha = (\alpha_{1}, \alpha_{2}, \ldots)$. For any $i \in \mathbb{N}$, let $\prefix_{i}: \Sigma \to \mathbb{N}^{i}, (\alpha_{1}, \alpha_{2}, \ldots, \alpha_{i}, \ldots ) \mapsto 
(\alpha_{1}, \alpha_{2}, \ldots, \alpha_{i})$, be a function returning the prefix of length $i$ of an infinite sequence in $\Lambda$.  The lexicographic ordering of $\Lambda$ is such that, given $\alpha, \beta \in \Sigma$,  $\alpha \le \beta$ if and only if there exists $j \in \mathbb{N}$ such that  $\alpha_{i} = \beta_{i}$ for all $i \in \mathbb{N}$, $i \le  j-1$, and $\alpha_{j} \le \beta_{j}$. This is a total ordering of $\Lambda$, since $\mathbb{N}$ is a totally ordered set. Given $\alpha \in \Lambda$ and $i \in \mathbb{N}$, let  $\mathcal{L}_{\prefix_{i}(\alpha)}$ be the sum of the distances from the root vertex $x_\mathrm{init}$ to its $\alpha_{1}$-th child, from this vertex to its $\alpha_{2}$-th child, etc., for a total of $i$ terms. Because of Lemma 
\ref{lemma:rrt_optimality:connection_statistics}, this construction is well defined, almost surely, for a sufficiently large number of samples. 
For any infinite sequence $\alpha \in \Lambda$, let $\mathcal{L}_{\alpha} = \lim_{i \to +\infty} \mathcal{L}_{\prefix_{i}(\alpha)}$; the limit exists since $\mathcal{L}_{\prefix_{i}(\alpha)}$ is non-decreasing in $i$.

Consider infinite strings of the form $\mathbf{k}=(k,1,1,\ldots)$, $k \in \mathbb{N}$, and introduce the shorthand $\mathcal{L}_\mathbf{k}:=\mathcal{L}_{(k,1,1, \ldots)}$. The following lemma shows that, for any $k \in \mathbb{N}$, $\mathcal{L}_\mathbf{k}$ has  finite expectation, which immediately implies that $\mathcal{L}_\mathbf{k}$ takes only finite values with probability one. The lemma also provides a couple of other useful properties of $\mathcal{L}_\mathbf{k}$, which will be used later on.

\begin{lemma} \label{lemma:rrt_optimality:lengthfirst}
The expected value $\EE[\mathcal{L}_\mathbf{k}]$ is non-negative and finite, and monotonically non-increasing, in the sense that  $\EE[\mathcal{L}_\mathbf{k+1}] \le \EE[\mathcal{L}_\mathbf{k}]$, for any $k \in \mathbb{N}$. Moreover, $\lim_{k \to \infty} \EE[\mathcal{L}_\mathbf{k}] = 0$.
\end{lemma}

\begin{proof}
Under the simplifying assumptions that there are no obstacles in the unit cube and $\eta$ is large enough, the vertex set $V_n^\mathrm{RRT}$ of the graph maintained by the RRT algorithm is precisely the first $n$ samples and each new sample is connected to its nearest neighbor in $V_n^\mathrm{RRT}$. 

Define $Z_{i}$ as a random variable describing the contribution to  $\mathcal{L}_\mathbf{1}$
realized at iteration $i$; in other words, $Z_{i}$ is the distance of the $i$-th sample to its nearest neighbor among the first $i-1$ samples if the $i$-th sample is on the path 
used in computing $\mathcal{L}_\mathbf{1}$, and zero otherwise. Then, using Lemma \ref{lemma:rrt_optimality:connection_statistics}, 
$$\EE[ \mathcal{L}_\mathbf{1}] = \EE \left[\sum_{i=1}^{\infty} Z_{i}\right] = 
\sum_{i=1}^{\infty} \EE[Z_{i}] = \sum_{i=1}^{\infty}  i^{-1/d} \,\,i^{-1} = {\tt Zeta}(1+1/d),$$
where the second equality follows from the monotone convergence theorem and ${\tt Zeta}$ is the Riemann zeta function. Since ${\tt Zeta}(y)$ is finite for any $y>1$, $\EE[\mathcal{L}_\mathbf{1}]$ is a finite number for all $d \in \mathbb{N}$.

Let $N_{k}$ be the iteration at which the first sample contributing to $\mathcal{L}_{k}$ is generated.  Then, an argument similar to the one given above yields
$$
\EE[\mathcal{L}_\mathbf{k+1}] = \sum_{i = N_{k}+1}^\infty i^{- (1+ 1/d)} = \EE[\mathcal{L}_\mathbf{1}] - \sum_{i = 1}^{N_{k}} i^{-(1+1/d)}.
$$
Then, clearly, $\EE[\mathcal{L}_\mathbf{k+1}] < \EE[\mathcal{L}_\mathbf{k}]$ 
 for all $k \in \mathbb{N}$. Moreover, since $N_{k} \ge k$, it is the case that $\lim_{k \to \infty} \EE[{\cal L}_{\mathbf{k}}] = 0$.
\qed
\end{proof}

\subsection{Length of the longest path in a branch}

Given $k \in \mathbb{N}$, and the sequence $\mathbf{k}=(k,1,1,\ldots)$, the quantity $\sup_{\alpha \ge \mathbf{k}} {\cal L}_\alpha$ is an upper bound on the length of any path in the $k$-th branch of the RRT, or in any of the following branches.
The next result bounds the probability that this quantity is very large.
\begin{lemma} \label{lemma:rrt_optimality:lengthmax}
For any $\epsilon > 0$, 
$$
\PP \left( \left\{ \sup_{\alpha \ge \mathbf{k}} \mathcal{L}_\alpha > \epsilon \right\} \right) \,\, \le \, \, \frac{\EE[\mathcal{L}_\mathbf{k}]}{\epsilon}.
$$
\end{lemma}

First, we state and prove the following intermediate result.
\begin{lemma} \label{lemma:rrt_optimality:length_comparison}
$\EE[\mathcal{L}_{\alpha}] \le \EE[\mathcal{L}_\mathbf{k}]$, for all $\alpha \ge \mathbf{k}$.
\end{lemma}

\begin{proof}
The proof is by induction. Since $\alpha \ge \mathbf{k}$, then $\prefix_{1}(\alpha) \ge k$, and Lemma \ref{lemma:rrt_optimality:lengthfirst} implies that 
$\EE[\mathcal{L}_{(\prefix_{1}(\alpha), 1, 1, \ldots)}] \le \EE[\mathcal{L}_\mathbf{k}]$. Moreover, it is also the case that, for any $i \in \mathbb{N}$  (and some abuse of notation),
$\EE[\mathcal{L}_{(\prefix_{i+1}(\alpha), 1, 1, \ldots)}] \le \EE[\mathcal{L}_{(\prefix_{i}(\alpha), 1, 1, \ldots)}]$, by a similar argument considering a tree rooted at 
the last vertex reached by the finite path $\prefix_{i}(\alpha)$. Since $(\prefix_{i+1}(\alpha), 1, 1, \ldots) \ge (\prefix_{i}(\alpha), 1, 1, \ldots) \ge (k, 1,1, \ldots)$, the result follows. 
\qed
\end{proof}

\begin{proof}[Proof of Lemma~\ref{lemma:rrt_optimality:lengthmax}] 
Define the random variable $\bar\alpha := \inf \{ \alpha \ge \mathbf{k} \given {\cal L}_\alpha > \epsilon\}$, and set $\bar\alpha := \mathbf{k}$ if ${\cal L}_\alpha \le \epsilon$ for all $\alpha \ge \mathbf{k}$. Note that $\bar\alpha \ge \mathbf{k}$ holds surely. Hence,  by Lemma~\ref{lemma:rrt_optimality:length_comparison}, $\EE[{\cal L}_{\bar\alpha}] \le \EE[{\cal L}_{\mathbf{k}}]$. 
Let $I_\epsilon$ be the indicator random variable for the event $S_\epsilon := \{\sup_{\alpha \ge \mathbf{k}} {\cal L}_\alpha > \epsilon\}$. Then,
$$
\EE[{\cal L}_\mathbf{k}] \ge \EE[{\cal L}_{\bar\alpha}] = \EE[{\cal L}_{\bar\alpha} I_\epsilon] + \EE[{\cal L}_{\bar\alpha} (1- I_\epsilon)] \ge \epsilon \, \PP(S_\epsilon),
$$
where the last inequality follows from the fact that ${\cal L}_{\bar\alpha}$ is at least $\epsilon$ whenever the event $S_\epsilon$ occurs. \qed
\end{proof}

A useful corollary of Lemmas~\ref{lemma:rrt_optimality:lengthfirst} and \ref{lemma:rrt_optimality:lengthmax} is the following.
\begin{corollary} \label{corollary:rrt_optimality:long_path}
For any $\epsilon > 0$, $\lim_{k \to \infty} \PP(\{\sup_{\alpha \ge \mathbf{k}} {\cal L}_\alpha > \epsilon\}) = 0$.
\end{corollary}

\subsection{Violation of the necessary condition}

Recall from Lemma~\ref{lemma:rrt_optimality:necessary_condition} that a necessary condition for asymptotic optimality is that the $k$-th branch of the RRT contains vertices outside the $R$-ball centered at $x_\mathrm{init}$ for infinitely many $k$, where $0 < R < \inf_{y \in {\cal X}_\mathrm{goal}} \Vert y - x_\mathrm{init} \Vert$. Clearly, the latter event can occur only if longest path in the $k$-th branch of the RRT is longer than $R$ for infinitely many $k$. That is, 
$$
\PP\left(\left\{\lim_{n \to \infty} Y_n^\mathrm{RRT} = c^* \right\}\right) \le \PP\left(\limsup_{k \to \infty} \left\{ \sup\nolimits_{\alpha \ge \mathbf{k}} {\cal L}_\alpha > R\right\}\right).
$$
The event on the right hand side is monotonic in the sense that $\{\sup_{\alpha > \mathbf{k}+1} {\cal L}_\alpha > R\} \supseteq \{\sup_{\alpha \ge \mathbf{k}} {\cal L}_\alpha > R\}$ for all $k \in \naturals$. Hence, $\lim_{k \to \infty} \{\sup_{\alpha \ge \mathbf{k}} {\cal L}_\alpha > R\}$ exists. In particular, $\PP (\limsup_{k \to \infty} \{\sup_{\alpha \ge \mathbf{k}} {\cal L}_\alpha > R\}) = \PP (\lim_{k \to \infty} \{\sup_{\alpha \ge \mathbf{k}} {\cal L}_\alpha > R\})  = \lim_{k \to \infty} \PP (\{\sup_{\alpha \ge \mathbf{k}} {\cal L}_\alpha > R\}) $, where the last equality follows from the continuity of probability measures. Since $\lim_{k \to \infty} \PP \left( \left\{ \sup_{\alpha \ge \mathbf{k}} {\cal L}_\alpha > R \right\} \right) = 0$ for all $R > 0$ by Corollary~\ref{corollary:rrt_optimality:long_path}, $\PP (\{ \lim_{n \to \infty} Y_n^\mathrm{RRT} = c^* \}) = 0$.

\section{Proof of Theorem~\ref{theorem:optimality_prmstar} (Asymptotic optimality of PRM$^*$)}
\label{proof:optimality_prmstar}

An outline of the proof is given below, before the details are provided.

\subsection{Outline of the proof}

Let $\sigma^*$ denote a robustly optimal path. By definition, $\sigma^*$ has weak $\delta$-clearance. 
First, define a sequence $\{ \delta_\N \}_{\N \in \naturals}$ such that $\delta_\N > 0$ for all $\N \in \naturals$ and $\delta_\N$ approaches zero as $\N$ approaches infinity. Construct a sequence $\{ \sigma_\N\}_{\N \in \naturals}$ of paths such that $\sigma_\N$ has strong $\delta_\N$-clearance for all $\N \in \naturals$ and $\sigma_\N$ converges to $\sigma^*$ as $\N$ approaches infinity. 

Second, define a sequence $\{q_\N\}_{\N \in \naturals}$. For all $\N \in \naturals$, construct a set $B_\N = \{B_{\N,1}, B_{\N,2}, \dots, B_{\N,\MM_\N} \}$ of overlapping balls, each with radius $q_\N$, that collectively ``cover'' the path $\sigma_\N$. See Figures~\ref{figure:covering_balls} and \ref{figure:prmstar_balls}. Let $x_{\M} \in B_{\N, \M}$ and $x_{\M + 1} \in B_{\N, \M + 1}$ be any two points from two consecutive balls in $B_\N$. Construct $B_\N$ such that (i) $x_{\M}$ and $x_{\M + 1}$ have distance no more than the connection radius $r(\N)$ and (ii) the straight path connecting $x_{\M}$ and $x_{\M + 1}$ lies entirely within the obstacle free space. 
These requirements can be satisfied by setting $\delta_\N$ and $q_\N$ to certain constant fractions of $r(\N)$.

Let $A_\N$ denote the event that each ball in $B_\N$ contains at least one vertex of the graph returned by the PRM$^*$ algorithm, when the algorithm is run with $\N$ samples.
Third, show that $A_\N$ occurs for all large $\N$, with probability one. Clearly, in this case, the PRM$^*$ algorithm will connect the vertices in consecutive balls with an edge, and any path formed in this way will be collision-free. 

Finally, show that any sequence of paths generated in this way converges to the optimal path $\sigma^*$. Using the robustness of $\sigma^*$, show that the cost of the best path in the graph returned by the PRM$^*$ algorithm converges to $c(\sigma^*)$ almost surely.

\subsection{Construction of the sequence $\{\sigma_\N\}_{\N \in \naturals}$ of paths} \label{section:proof_prmstar:sigma_n}

The following lemma  establishes a  connection between the notions of strong and weak $\delta$-clearance. 

\begin{lemma} \label{lemma:weak_delta_clearance}
Let $\sigma^*$ be a path be a path that has strong $\delta$-clearance.
Let $\{\delta_\N\}_{\N \in \naturals}$ be a sequence of real numbers such that $\lim_{n \to \infty} \delta_\N = 0$ and $0 \le \delta_\N \le \delta$ for all $\N \in \naturals$. Then, there exists a sequence $\{ \sigma_\N \}_{\N \in \naturals}$ of paths such that $\lim_{\N \to \infty} \sigma_\N = \sigma^*$ and $\sigma_\N$ has strong $\delta_\N$-clearance for all $\N \in \naturals$.
\end{lemma}
\begin{proof}
First, define a sequence $\{ \X_\N \}_{\N \in \naturals}$ of subsets of $\X_\mathrm{free}$ such that $\X_\N$ is the closure of the $\delta_\N$-interior of $\X_\mathrm{free}$, i.e.,
$$
\X_\N := \mathrm{cl} (\mathrm{int}_{\delta_\N} (\X_\mathrm{free}))
$$
for all $\N \in \naturals$. Note that, by definition, (i) $\X_\N$ are closed subsets of $\X_\mathrm{free}$, and (ii) any point $\X_\N$ has distance at least $\delta_\N$ to any point in the obstacle set $\X_\mathrm{obs}$.

Then, construct the sequence $\{ \sigma_\N \}_{\N \in \naturals}$ of paths, where $\sigma_\N \in \Sigma_{\X_\N}$, as follows.
Let $\psi : [0,1] \to \Sigma_\mathrm{free}$ denote the homotopy with $\psi(0) = \sigma^*$; the existence of $\psi$ is guaranteed by weak $\delta$-clearance of $\sigma^*$. 
Define 
$$
\alpha_\N := \max_{\alpha \in [0,1]} \{ \alpha \,\vert\, \psi (\alpha) \in \Sigma_{\X_\N}\} \quad\mbox{ and }\quad\sigma_\N := \psi(\alpha_\N).
$$
Since $\Sigma_{\X_\N}$ is closed, the maximum in the definition of $\alpha_\N$ is attained. Moreover, since $\psi(1)$ has strong $\delta$-clearance and $\delta_\N \le \delta$,  $\sigma_\N \in \Sigma_{\X_\N}$, 
which implies the strong $\delta_\N$-clearance of $\sigma_\N$. 

Clearly, $\bigcup_{\N \in \naturals} \X_\N  = \X_\mathrm{free}$, since $\lim_{\N \to \infty} \delta_\N = 0$. 
Also, by weak $\delta$-clearance of $\sigma^*$, for any $\alpha \in (0,1]$, there exists some $\delta_\alpha \in (0, \delta]$ such that $\psi(\alpha)$ has strong $\delta_\alpha$-clearance.
Then, $\lim_{\N \to \infty} \alpha_\N = 0$, which implies $\lim_{\N \to \infty} \sigma_\N = \sigma^*$.\qed
\end{proof}

Recall that the connection radius of the PRM$^*$ algorithm was defined as 
$$
r_\N = \gamma_\mathrm{PRM}\left(\frac{\log \N}{\N}\right)^{1/d} \, > \, 2 (1 + 1/d)^{1/d} \left(\frac{\mu(X_\mathrm{free})}{\VolumeDBall{d}}\right)^{1/d} \left(\frac{\log \N}{\N}\right)^{1/d}
$$ 
(see Algorithm~\ref{algorithm:PRM*} and the definition of the ${\tt Near}$ procedure in Section~\ref{section:algorithms:primitive_procedures}). Let $\theta_1$ be a small positive constant; the precise value of $\theta_1$ will be provided shortly in the proof of Lemma~\ref{lemma:vertices_in_balls}. Define  
$$
\delta_\N := \min\left\{ \delta, \frac{1 + \theta_1}{2 + \theta_1} r_\N \right\},\quad\quad \mbox{ for all } \N \in \naturals.
$$

By definition, $0 \le \delta_\N \le \delta$ holds. Moreover, $\lim_{\N \to \infty} \delta_\N = 0$, since $\lim_{\N \to \infty} r_\N = 0$. Then, by Lemma~\ref{lemma:weak_delta_clearance}, there exists a sequence $\{\sigma_\N \}_{\N \in \naturals}$ of paths such that $\lim_{\N \to \infty} \sigma_\N = \sigma^*$ and $\sigma_\N$ has strong $\delta_\N$-clearance for all $\N \in \naturals$.

\subsection{Construction of the sequence $\{ B_\N \}_{\N \in \naturals}$ of sets of balls}
\label{section:proof_prmstar:b_n}

First, a construction of a finite set of balls that collectively ``cover'' a path $\sigma_\N$ is provided. The construction is illustrated in Figure~\ref{figure:covering_balls}. 
\begin{definition}[Covering balls] \label{definition:covering_balls}
Given a path $\sigma_\N : [0,1] \to \X$, and the real numbers $ q_\N, l_\N \in \reals_{>0}$, the set ${\tt CoveringBalls}(\sigma_\N,q_\N,l_\N)$ is defined as a set $\{B_{\N,1}, B_{\N,2}, \dots, B_{\N,\MM_\N}\}$ of $\MM_\N$ balls of radius $q_n$ such that $B_{\N,\M}$ is centered at $\sigma(\tau_\M)$, and
\begin{itemize}
\item the center of $B_{\N,1}$ is $\sigma(0)$, i.e., $\tau_1 = 0$,
\item the centers of two consecutive balls are exactly $l_\N$ apart, i.e., $\tau_\M := \min\{ \tau \in [\tau_{\M-1},1] \given \| \sigma(\tau) - \sigma(\tau_{\M-1})\| \ge l_n\}$ for all $\M \in \{2,3,\dots, \MM_\N \}$, 
\item and $\MM - 1$ is the largest number of balls that can be generated in this manner while the center of the last ball, $B_{\N, \MM_\N}$ is $\sigma(1)$, i.e., $\tau_{\MM_\N} = 1$.
\end{itemize}
\end{definition}

\begin{figure}
\centering
\includegraphics[height = 5cm]{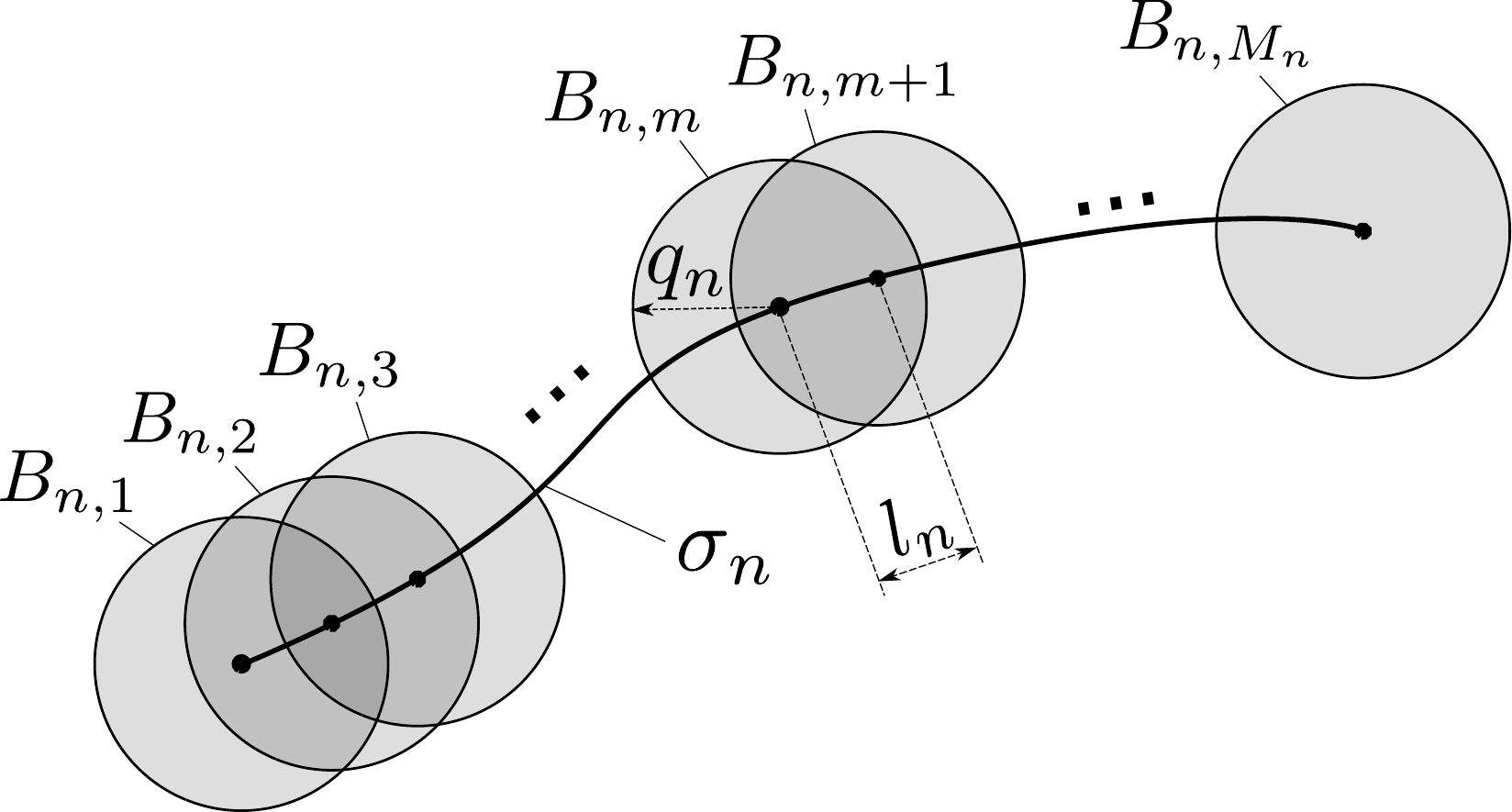}
\caption{An illustration of the ${\tt CoveringBalls}$ construction. A set of balls that collectively cover the trajectory $\sigma_\N$ is shown. All balls have the same radius, $q_\N$. The spacing between the centers of two consecutive balls is $l_\N$.}
\label{figure:covering_balls}
\end{figure}

For each $\N \in \naturals$, define 
$$
q_\N := \frac{\delta_\N}{1 + \theta_1}.
$$ 
Construct the set $B_\N = \{B_{\N,1}, B_{\N,2}, \dots, B_{\N,M_\N}\}$ of balls as $B_\N := {\tt CoveringBalls} (\sigma_\N, q_\N, \theta_1 q_\N)$ using Definition~\ref{definition:covering_balls} (see Figure~\ref{figure:covering_balls}). 
By construction, each ball in $B_\N$ has radius $q_\N$ and the centers of consecutive balls in $B_\N$ are $\theta_1 q_\N$ apart (see Figure~\ref{figure:prmstar_balls} for an illustration of covering balls with this set of parameters). The balls in $B_\N$ collectively cover the path $\sigma_\N$.

\begin{figure}
\centering
\includegraphics[height = 5cm]{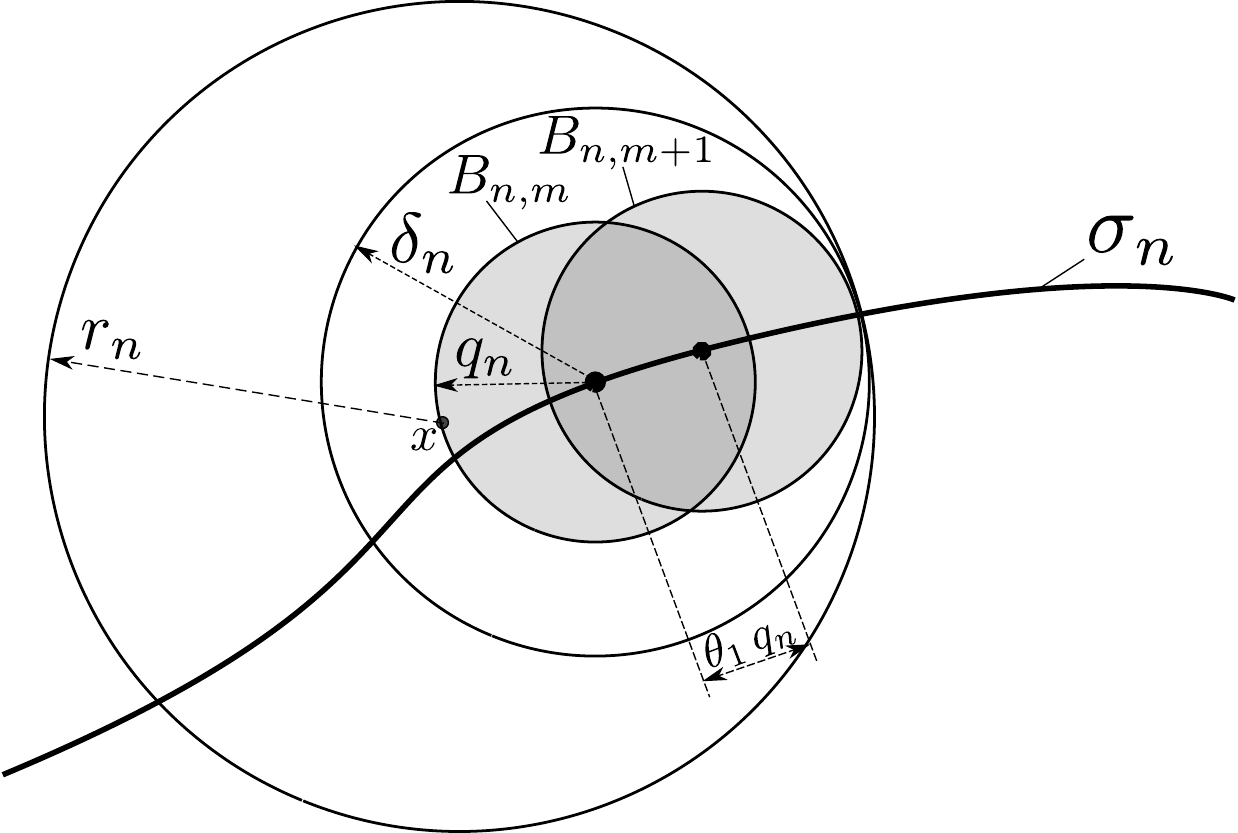}
\caption{An illustration of the covering balls for PRM$^*$ algorithm. The $\delta_\N$-ball is guaranteed to be inside the obstacle-free space. The connection radius $r_\N$ is also shown as the radius of the connection ball centered at a vertex $x \in B_{\N,m}$. The vertex $x$ is connected to all other vertices that lie within the connection ball.} 
\label{figure:prmstar_balls}
\end{figure}

\subsection{The probability that each ball in $B_\N$ contains at least one vertex}

Recall that $G^{\mathrm{PRM}^*}_\N = (V^{\mathrm{PRM}^*}_\N, E^{\mathrm{PRM}^*}_\N)$ denotes the graph returned by the PRM$^*$ algorithm, when the algorithm is run with $\N$ samples. 
Let $A_{\N,\M}$ denote the event that the ball $B_{\N,\M}$ contains at least one vertex of the graph generated by the PRM$^*$ algorithm, i.e., $A_{\N,\M} = \left\{B_{\N, \M} \cap V^{\mathrm{PRM}^*}_\N \neq \emptyset \right\}$. 
Let $A_\N$ denote the event that all balls in $B_\N$ contain at least one vertex of the PRM$^*$ graph, i.e., $A_\N = \bigcap_{\M = 1}^{\MM_\N } A_{\N,\M}$. 
\begin{lemma} \label{lemma:vertices_in_balls}
If $\gamma_\mathrm{PRM} > 2 \, (1 + 1/d)^{1/d} \, \left(\frac{\mu(X_\mathrm{free})}{\VolumeDBall{d}}\right)^{1/d}$, then there exists a constant $\theta_1 > 0$ such that the event that every ball in $B_\N$ contains at least one vertex of the PRM$^*$ graph occurs for all large enough $\N$ with probability one, i.e.,
$$
\PP \left(\liminf_{\N \to \infty} A_\N \right) = 1.
$$
\end{lemma}
\begin{proof}
The proof is based on a Borel-Cantelli argument which can be summarized as follows.
Recall that $A_\N^c$ denotes the complement of $A_\N$.
First, the sum $\sum_{\N = 1}^\infty \PP (A_\N^c)$ is shown to be bounded. By the Borel-Cantelli lemma~\citep{grimmett.stirzaker.book01}, this implies that the probability that $A_\N$ holds infinitely often as $\N$ approaches infinity is zero. Hence, the probability that $A_\N$ holds infinitely often is one. 
In the rest of the proof, an upper bound on $\PP (A_\N)$ is computed, and this upper bound is shown to be summable. 

First, compute a bound on the number of balls in $B_\N$ as follows. Let $s_\N$ denote the length of $\sigma_\N$, i.e., $s_\N := \mathrm{TV}(\sigma_\N)$. Recall that the balls in $B_\N$ were constructed such that the centers of two consecutive balls in $B_\N$ have distance $\theta_1\, q_\N$. The segment of $\sigma_\N$ that starts at the center of $B_{\N,\M}$ and ends at the center of $B_{\N,\M+1}$ has length at least $\theta_1 q_\N$, except for the last segment, which has length less than or equal to $\theta_1 q_\N$. Let $\N_0 \in \naturals$ be the number such that $\delta_\N < \delta$ for all $\N \ge \N_0$. Then, for all $\N \ge \N_0$, 
\begin{eqnarray*}
\card{ B_\N } = M_\N & \le & \frac{s_\N}{\theta_1 q_\N} = \frac{(1+\theta_1)s_\N}{\theta_1\delta_\N} =  \frac{(2 + \theta_1) \, s_\N}{\theta_1 \, r_\N}  \\
& = & \frac{(2 + \theta_1) \,  s_\N}{\theta_1 \, \gamma_\mathrm{PRM}} \left( \frac{\N}{\log \N} \right)^{1/d}.
\end{eqnarray*}

Second, compute the volume of a single ball in $B_i$ as follows.
Recall that $\mu(\cdot)$ denotes the usual Lebesgue measure, and $\VolumeDBall{d}$ denotes the volume of a unit ball in the $d$-dimensional Euclidean space. For all $\N \ge \N_0$, 
$$
\mu(B_{\N,\M}) = \VolumeDBall{d} \, q_\N^d = \VolumeDBall{d} \left(\frac{\delta_\N}{1 + \theta_1}\right)^d =  \VolumeDBall{d} \left(\frac{r_\N}{2 + \theta_1}\right)^d = \VolumeDBall{d} \, \left(\frac{\gamma_\mathrm{PRM}}{2 + \theta_1}\right)^d \, \frac{\log \N}{\N} 
$$

For all $\N \ge I$, the probability that a single ball, say $B_{\N,1}$, does not contain a vertex of the graph generated by the PRM$^*$ algorithm, when the algorithm is run with $\N$ samples, is
\begin{eqnarray*}
\PP \left( A_{\N,1}^c \right) & = & \left( 1 - \frac{\mu(B_{\N,1})}{\mu(X_\mathrm{free})}\right)^\N \\
& = & \left( 1 - \frac{\VolumeDBall{d}}{\mu(X_\mathrm{free})} \left(\frac{\gamma_\mathrm{PRM}}{2 +\theta_1}\right)^d \frac{\log \N}{\N}\right)^\N
\end{eqnarray*}
Using the inequality $(1- 1/f(\N))^r \le e^{-r / f(\N)}$, the right-hand side can be bounded as
$$
\PP (A_{\N,1}) \le e^{-\frac{\VolumeDBall{d}}{\mu(X_\mathrm{free})} \left(\frac{\gamma_\mathrm{PRM}}{2 +\theta_1}\right)^d \log \N} = \N^{- \frac{\VolumeDBall{d}}{\mu(X_\mathrm{free})} \left(\frac{\gamma_\mathrm{PRM}}{2 +\theta_1}\right)^d}.
$$

Hence,
\begin{eqnarray*}
\PP \left(A_\N^c\right) = \PP \left(\bigcup\nolimits_{\M = 1}^{\MM_\N} A_{\N,\M}^c\right) & \le & \sum_{\M = 1}^{\MM_\N} \PP \left( A_{\N,\M}^c\right) = \MM_\N \,\, \PP (A_{\N,1}^c)\\
& \le & \frac{(2 + \theta_1) s_\N}{\theta_1 \, \gamma_\mathrm{PRM}} \left( \frac{\N}{\log \N}\right)^{1/d} i^{-\frac{\VolumeDBall{d}}{\mu(X_\mathrm{free})} \left(\frac{\gamma_\mathrm{PRM}}{2 +\theta_1}\right)^d} \\
& = & \frac{(2 + \theta_1) s_\N}{\theta_1 \, \gamma_\mathrm{PRM}}  \, \frac{1}{(\log \N)^d} \,\, \N^{- \left(\frac{\VolumeDBall{d}}{\mu(X_\mathrm{free})} \left(\frac{\gamma_\mathrm{PRM}}{2 +\theta_1}\right)^d - \frac{1}{d} \right)}
\end{eqnarray*}
where the first inequality follows from the union bound. 

Finally, $\sum_{\N = 1}^{\infty} \PP(A_\N^c) < \infty$ holds, if $\frac{\VolumeDBall{d}}{\mu(X_\mathrm{free})} \left(\frac{\gamma_\mathrm{PRM}}{2 +\theta_1}\right)^d - \frac{1}{d} > 1$, which can be satisfied for any $\gamma_{PRM} > 2 (1 + 1/d)^{1/d} \left( \frac{\mu(X_\mathrm{free})}{\VolumeDBall{d}}\right)^{1/d}$ by appropriately choosing $\theta_1$. 
Then, by the Borel-Cantelli lemma~\citep{grimmett.stirzaker.book01}, $\PP (\limsup_{\N \to \infty} A_{\N}^c) = 0$, which implies $\PP(\liminf_{\N \to \infty} A_\N) = 1$.\qed
\end{proof}

\subsection{Connecting the vertices in subsequent balls in $B_\N$}

Let $Z_\N := \{x_1, x_2, \dots, x_{\MM_\N}\}$ be any set of points such that $x_\M \in B_{\N,\M}$ for each $\M \in \{1,2,\dots, \MM_\N\}$. 
The following lemma states that for all $\N \in \naturals$ and all $\M \in \{1,2, \dots, \MM_\N-1\}$, the distance between $x_{\M}$ and $x_{\M+1}$ is less than the connection radius, $r_\N$, which implies that the PRM$^*$ algorithm will attempt to connect the two points $x_{\M}$ and $x_{\M+1}$ if they are in the vertex set of the PRM$^*$ algorithm. 
\begin{lemma} \label{lemma:vertex_closeness}
If $x_{\N,\M} \in B_{\N,\M}$ and $x_{\N,\M+1} \in B_{\N, \M+1}$, then $\Vert x_{\N,\M+1} - x_{\N,\M}\Vert \le r_\N$, for all $\N \in \naturals$ and all $\M \in \{1,2,\dots, \MM_i-1\}$.
\end{lemma}
\begin{proof}
Recall that each ball in $B_\N$ has radius $q_\N = \frac{\delta_\N}{(1 + \theta_1)}$. Given any two points $x_\M \in B_{\N,\M}$ and $x_{\M+1} \in B_{\N,\M+1}$, all of the following hold: (i) $x_\M$ has distance $q_\N$ to the center of $B_{\N,\M}$, (ii) $x_{\M+1}$ has distance $q_\N$ to the center of $B_{\N,\M+1}$, and (iii) centers of $B_{\N,\M}$ and $B_{\N,\M+1}$ have distance $\theta_1 \, q_\N$ to each other. Then, 

$$
\Vert x_{\N,\M+1} - x_{\N,\M}\Vert \le (2 + \theta_1) \, q_\N = \frac{2 + \theta_1}{1 + \theta_1} \, \delta_\N \le r_\N,
$$ 
where the first inequality is obtained by an application of the triangle inequality and the last inequality follows from the definition of $\delta_\N = \min\{ \delta, \frac{1+\theta_1}{2+\theta_1} \, r_\N\}$. \qed
\end{proof}

By Lemma~\ref{lemma:vertex_closeness}, conclude that the PRM$^*$ algorithm will attempt to connect any two vertices in consecutive balls in $B_\N$. The next lemma shows that any such connection attempt will, in fact, be successful. That is, the path connecting $x_{\N,\M}$ and $x_{\N,\M+1}$ is collision-free for all $\M \in \{1,2,\dots, \MM_\N \}$.
\begin{lemma} \label{lemma:no_collision}
For all $\N \in \naturals$ and all $\M \in \{1,2, \dots, \MM_\N\}$, if $x_\M \in B_{\N,\M}$ and $x_{\M+1} \in B_{\N,\M+1}$, then the line segment connecting $x_{\N,\M}$ and $x_{\N,\M+1}$ lies in the obstacle-free space, i.e., 
$$
\alpha \, x_{\N,\M} + (1 - \alpha) \, x_{\N,\M+1} \in X_\mathrm{free},\quad\quad \mbox{ for all }\alpha \in [0,1].
$$
\end{lemma}
\begin{proof}
Recall that $\sigma_\N$ has strong $\delta_\N$-delta clearance and that the radius $q_\N$ of each ball in $B_\N$ was defined as $q_\N = \frac{\delta_\N}{1 + \theta_1}$, where $\theta_1 > 0$ is a constant. Hence, any point along the trajectory $\sigma_\N$ has distance at least $(1 + \theta_1) \, q_\N$ to any point in the obstacle set. 
Let $y_\M$ and $y_{\M+1}$ denote the centers of the balls $B_{\N,\M}$ and $B_{\N,\M+1}$, respectively.
Since $y_\M = \sigma(\tau_\M)$ and $y_{\M+1} = \sigma(\tau_{\M+1})$ for some $\tau_\M$ and $\tau_{\M + 1}$, $y_\M$ and $y_{\M+1}$ also have distance $(1 + \theta_1) q_\N$ to any point in the obstacle set. 

Clearly, $\Vert x_\M - y_\M \Vert \le q_\N$. Moreover, the following inequality holds: 
$$
\Vert x_{\M+1} - y_\M \Vert \le \Vert (x_\M - y_{\M+1}) + (y_{\M+1} - y_{\M}) \Vert \le \Vert x_{\M+1} - y_{\M+1} \Vert + \Vert y_{\M+1} - y_\M \Vert \le q_\N + \theta_1 \, q_\N = (1 + \theta_1) \, q_\N.
$$
where the second inequality follows from the triangle inequality and the third inequality follows from the construction of balls in $B_\N$.

For any convex combination $x_\alpha := \alpha \, x_\M + (1 - \alpha) \, x_{\M+1}$, where $\alpha \in [0,1]$, the distance between $x_\alpha$ and $y_\M$ can be bounded as follows:
\begin{eqnarray*}
\big\Vert \big(\alpha\, x_\M + (1 + \alpha) \, x_{\M+1}\big) - y_\M \big\Vert 
& = & \big\Vert \alpha\, (x_\M - y_\M) + (1 + \alpha) \, (x_{\M+1} - y_\M ) \big\Vert \\
& = & \alpha \, \Vert x_\M - y_{\M} \Vert + (1 + \alpha) \, \Vert x_{\M+1} - y_{\M} \Vert \\
& = & \alpha \, q_\N + (1 + \alpha)\, (1 + q_\N) \le (1 + \theta_1) \, q_\N,
\end{eqnarray*}
where the second equality follows from the linearity of the norm. Hence, any point along the line segment connecting $x_\M$ and $x_{\M+1}$ has distance at most $(1 + \theta_1) \, q_\N$ to $y_\M$. Since, $y_\M$ has distance at least $(1 + \theta_1) q_\N$ to any point in the obstacle set, the line segment connecting $x_{\M}$ and $x_{\M+1}$ is collision-free. \qed
\end{proof}

\subsection{Convergence to the optimal path}

Let $P_\N$ denote the set of all paths in the graph $G^\AlgPRMstar_\N = (V^\AlgPRMstar_\N, E^\AlgPRMstar_\N)$. Let $\sigma_\N'$ be the path that is closest to $\sigma_\N$ in terms of the bounded variation norm among all those paths in $P_\N$, i.e., 
$
\sigma_\N' := \min_{\sigma' \in P_\N} \Vert \sigma' - \sigma_\N \Vert.
$
Note that the sequence $\{\sigma_\N'\}_{\N \in \naturals}$ is a random sequence of paths, since the graph $G^\AlgPRMstar_\N$, hence the set $P_\N$ of paths is random.
The following lemma states that the bounded variation distance between $\sigma_\N'$ and $\sigma_\N$ approaches to zero, with probability one.

\begin{lemma} \label{lemma:prmstar:convergence_in_bvnorm}
The random variable $\Vert \sigma_\N' - \sigma_\N \Vert_\BVnorm$ converges to zero almost surely, i.e., 
$$
\PP \left( \left\{ \lim\nolimits_{\N \to \infty} \Vert \sigma_\N' - \sigma_\N \Vert_\BVnorm = 0 \right\}\right) = 1.
$$
\end{lemma}
\begin{proof}
The proof of this lemma is based on a Borel-Cantelli argument.
It is shown that $\sum_{\N \in \naturals} \PP(\Vert \sigma_\N' - \sigma_\N \Vert_\BVnorm > \epsilon)$ is finite for any $\epsilon > 0$, which implies that $\Vert \sigma_\N' - \sigma_\N \Vert$ converges to zero almost surely by the Borel-Cantelli lemma~\citep{grimmett.stirzaker.book01}.
This proof uses a Poissonization argument in one of the intermediate steps. That is, a particular result is shown to hold in the Poisson process described in Lemma~\ref{lemma:poissonization}. Subsequently, the result is de-Poissonized, i.e., shown to hold also for the original process. 

Fix some $\epsilon > 0$.
Let $\alpha ,\beta \in (0,1)$ be two constants, both independent of $\N$. 
Recall that $q_\N$ is the radius of each ball in the set $B_\N$ of balls covering the path $\sigma_\N$. Let $I_{\N,\M}$ denote the indicator variable for the event that the ball $B_{\N,\M}$ has no point that is within a distance $\beta \, q_\N$ from the center of $B_{\N,\M}$.
For a more precise definition, let $\beta \, B_{\N,\M}$ denote the ball that is centered at the center of $B_{\N,\M}$ and has radius $\beta \, r_\N$. Then, 
$$
I_{\N,\M} :=
\begin{cases}
1, & \mbox{if } (\beta \, B_{\N,\M} ) \cap V^\AlgPRMstar = \emptyset ,\\
0, & \mbox{otherwise.}
 \end{cases}
$$
Let $K_\N$ denote the number of balls in $B_\N$ that do not contain a vertex that is within a $\beta \, q_\N$ distance to the center of that particular ball, i.e., $K_\N := \sum_{\M = 1}^{\MM_\N} I_{\N,\M}$. 

Consider the event that $I_{\N,\M}$ holds for at most an $\alpha$ fraction of the balls in $B_\N$, i.e.,
$
\{ K_\N \le \alpha \, \MM_\N  \}.
$
This event is important for the following reason. 
Recall that the vertices in subsequent balls in $B_\N$ are connected by edges in $G^\AlgPRMstar_\N$ by Lemmas~\ref{lemma:vertex_closeness} and \ref{lemma:no_collision}. 
If only at most an $\alpha$ fraction of the balls do not have a vertex that is less than a distance of $\beta \, r_\N$ from their centers (hence, a $(1-\alpha)$ fraction have at least one vertex within a distance of $\beta\,r_\N$ from their centers), i.e., $\{ K_\N \le \alpha \, \MM_\N \}$ holds, then the bounded variation difference between $\sigma_\N'$ and $\sigma_\N$ is at most $(\sqrt{2} \, \alpha +  \beta (1 - \alpha)) L \le  \sqrt{2} \, (\alpha + \beta) L$, where $L$ is a finite bound on the length of all paths in $\{ \sigma_\N\}_{\N \in \naturals}$, i.e., $L := \sup_{\N \in \naturals} \TV(\sigma_\N)$. 
That is,
$$
\{ K_\N \le \alpha \, \MM_\N  \} \subseteq \left\{ \Vert \sigma_\N' - \sigma_\N \Vert_\BVnorm \le \sqrt{2}\, (\alpha + \beta) \, L\right\}
$$
Taking the complement of both sides and using the monotonicity of probability measures,
$$
\PP\left(\left\{ \Vert \sigma_\N' - \sigma_\N \Vert_\BVnorm > \sqrt{2}\, (\alpha + \beta) \, L\right\} \right) \le \PP \left(\{ K_\N \ge \alpha \, \MM_\N  \}\right).
$$
In the rest of the proof, it is shown that the right hand side of the inequality above is summable for all small $\alpha, \beta > 0$, which implies that $\PP\left( \{ \Vert \sigma_\N' - \sigma_\N \ \Vert  > \epsilon \} \right)$ is summable for all small $\epsilon>0$.

For this purpose, the process that provides independent uniform samples from $\X_\mathrm{free}$ is approximated by an equivalent Poisson process described in Section~\ref{section:rgg}. A more precise definition is given as follows. Let $\{\Z_1, \Z_2, \dots, \Z_\N\}$ denote the binomial point process corresponding to the ${\tt SampleFree}$ procedure. Let $\nu < 1$ be a constant independent of $\N$. Recall that $\Poisson(\nu\,\N)$ denotes the Poisson random variable with intensity $\nu\,\N$ (hence, mean value $\nu\,\N$). Then, the process $\PoissonApproximation_{\nu\,\N} := \{\Z_1, \Z_2, \dots, \Z_{\Poisson(\nu\,\N)} \}$ is a Poisson process restricted to $\mu(\X_\mathrm{free})$ with intensity $\nu\,\N / \mu(\X_\mathrm{free})$ (see Lemma~\ref{lemma:poissonization}). Thus, the expected number of points of this Poisson process is $\nu\,\N$.

Clearly, the set of points generated by one process is a subset of the those generated by the other. However, since $\nu < 1$, in most trials the Poisson point process $\PoissonApproximation_{\nu\,\N}$ is a subset of the binomial point process.

Define the random variable $\widetilde{K}_{\N}$ denote the number of balls of that fail to have one sample within a distance $\beta \, r_\N$ to their centers, when the underlying point process is $\PoissonApproximation_{\nu \, \N}$ (instead of the independent uniform samples provided by the ${\tt SampleFree}$ procedure). In other words, $\widetilde{K}_\N$ is the random variable that is defined similar to $K_\N$, except that the former is defined with respect to the points of $\PoissonApproximation_{\nu \, \N}$ whereas the latter is defined with respect to the $\N$ samples returned by ${\tt SampleFree}$ procedure. 

Since $\{\widetilde{K}_\N > \alpha\,\MM_\N\}$ is a decreasing event, i.e., the probability that it occurs increases if $\PoissonApproximation_{\nu \N}$ includes fewer samples, the following bound holds~\citep[see, e.g.,][]{penrose.book03}
$$
\PP \big( \left\{K_\N  \ge \alpha \, \MM_\N \right\} \big) \le \PP \big(\{\widetilde{K}_\N  \ge \alpha \, \MM_\N \} \big) + \PP(\{ \Poisson(\nu\, \N) \ge \N \}).
$$
Since a Poisson random variable has exponentially-decaying tails, the second term on the right hand side can be bounded as
\begin{eqnarray*}
\PP(\{\Poisson(\nu\,\N) \ge \N \}) \le e^{- c \N},
\end{eqnarray*}
where $c >0$ is a constant. 

The first term on the right hand side can be computed directly as follows. First, for all small $\beta$, the balls of radius $\beta \, r_\N$ are all disjoint (see Figure~\ref{figure:prmstar_balls_tilde}). Denote this set of balls by $\widetilde{B}_{\N,\M} = \{\widetilde{B}_{\N,1}, \widetilde{B}_{\N,2}, \dots, \widetilde{B}_{\N,\MM_\N}\}$. More precisely, $\widetilde{B}_{\N,\M}$ is the ball of radius $\beta\,q_\N$ centered at the center of $B_{\N,\M}$. 
Second, observe that the event $\{K_\N > \alpha \, \MM_\N\}$ is equivalent to the event that at least an $\alpha$ fraction of all the balls in $\widetilde{B}_{\N}$ include at least one point of the process $\PoissonApproximation_{\nu \, \N}$. Since, the point process $\PoissonApproximation_{\nu\, \N}$ is Poisson and the balls in $\widetilde{B}_\N$ are disjoint for all small enough $\beta$, the probability that a single ball in $\widetilde{B}_\N$ does not contain a sample is 
$
p_\N := \exp(- \VolumeDBall{d} \,(\beta q_\N)^d \, \nu \, \N / \mu(\X_\mathrm{free}) ) \le \exp(-c \, \beta\, \nu \, \log \N)
$
for some constant $c$.
Third, by the independence property of the Poisson point process, the number of balls in $\widetilde{B}_\N$ that do not include a point of the point process $\PoissonApproximation_{\nu \, \N}$ is a binomial random variable with parameters $\MM_\N$ and $p_\N$. Then, for all large $\N$,
$$
\PP \Big(\big\{ \widetilde{K}_\N \ge \alpha\, \MM_\N \big\}\Big) 
\le
\PP\left(\left\{  \Binomial(\MM_\N,p_\N) \ge \alpha \MM_\N  \right\}\right) \le \exp(- \MM_\N \, p_\N ).
$$

\begin{figure}
\begin{center}
\includegraphics[height = 6cm]{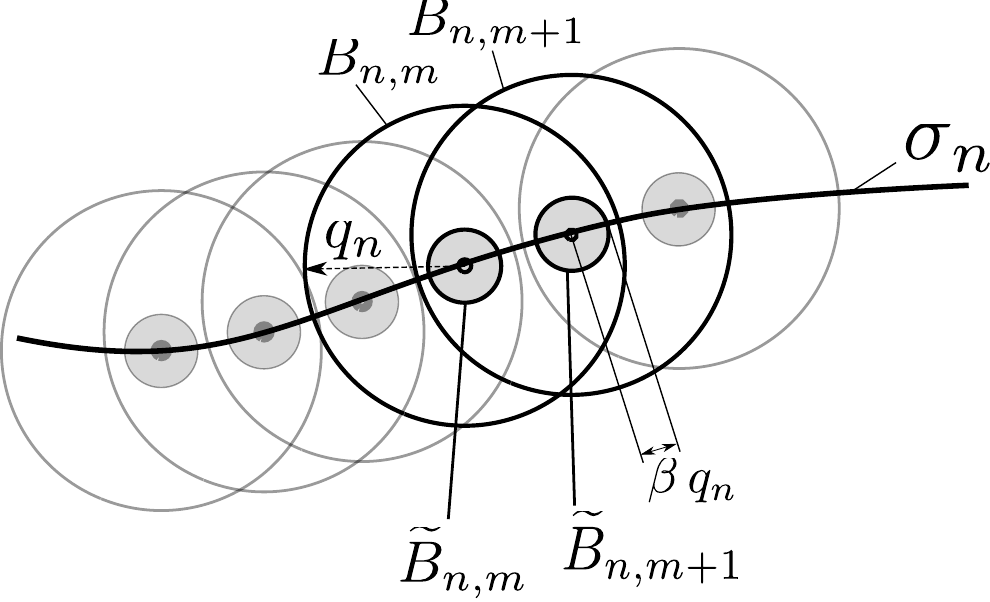}
\end{center}
\caption{The set $\widetilde{B}_{\N,\M}$ of non-intersection balls is illustrated.}
\label{figure:prmstar_balls_tilde}
\end{figure}

Combining the two inequalities above, the following bound is obtained for the original sampling process
$$
\PP \left(\left\{ K_\N \ge \alpha \, \MM_\N \right\} \right) \le e^{-c \, \N} + e^{-\MM_\N \, p_\N}.
$$
Summing up both sides, 
$$
\sum_{\N = 1}^\infty \PP \left(\left\{ K_\N \ge \alpha \, \N \right\} \right) < \infty.
$$
This argument holds for all $\alpha, \beta, \nu > 0$. Hence, for all $\epsilon > 0$,
$$
\sum_{\N = 1}^\infty \PP \left(\left\{ \Vert \sigma_\N' - \sigma_\N \Vert_\BVnorm > \epsilon \right\}\right) < \infty.
$$
Then, by the Borel-Cantelli lemma, $\PP\left(\left\{ \lim_{\N \to \infty} \Vert \sigma_\N' - \sigma_\N \Vert_\BVnorm  = 0 \right\} \right) = 1$.\qed
\end{proof}

Finally, the following lemma states that the cost of the minimum cost path in the graph returned by the $\AlgPRMstar$ algorithm converges to the optimal cost $c^*$ with probability one.
Recall that $Y^\AlgPRMstar_\N$ denotes the cost of the minimum-cost path in the graph returned by the $\AlgPRMstar$ algorithm, when the algorithm is run with $\N$ samples.
\begin{lemma} \label{lemma:prmstar:cost_convergence}
Under the assumptions of Theorem~\ref{theorem:optimality_prmstar}, the cost of the minimum-cost path present in the graph returned by the $\AlgPRMstar$ algorithm converges to the optimal cost $c^*$ as the number of samples approaches infinity, with probability one, i.e., 
$$
\PP \left( \left\{ \lim_{\N \to \infty} Y^\AlgPRMstar_\N = c^* \right\} \right) = 1.
$$
\end{lemma}
\begin{proof}
Recall that $\sigma^*$ denotes the optimal path, and that $\lim_{\N \to \infty} \sigma_\N = \sigma^*$ holds surely. By Lemma~\ref{lemma:prmstar:convergence_in_bvnorm}, $\lim_{\N \to \infty} \Vert \sigma_\N' -\sigma_\N \Vert_\BVnorm = 0$ holds with probability one. Thus, by repeated application of the triangle inequality, 
$
\lim_{\N \to \infty} \Vert \sigma_\N' -\sigma^* \Vert_\BVnorm = 0,
$
i.e., 
$$
\PP \left(\big\{ \lim_{\N \to \infty} \Vert \sigma_\N' - \sigma^* \Vert_\BVnorm = 0 \big\}\right) = 1.
$$
Then, by the robustness of the optimal path $\sigma^*$, it follows that 
$$
\PP \left(\big\{ \lim_{\N \to \infty} c(\sigma_\N') =c^* \big\}\right) = 1.
$$
That is the costs of the paths $\{ \sigma_\N' \}_{\N \in \naturals}$ converges to the optimal cost almost surely, as the number of samples approaches infinity.
\qed
\end{proof}

\section{Proof of Theorem~\ref{theorem:optimality_k_prmstar} (Asymptotic Optimality of $k$-nearest PRM$^*$)}
\label{proof:optimality_k_prmstar}

The proof of this theorem is similar to that of Theorem~\ref{theorem:optimality_prmstar}. For the reader's convenience, a complete proof is provided at the expense of repeating some of the arguments.

\subsection{Outline of the proof}

Let $\sigma^*$ be a robust optimal path with weak $\delta$-clearance. First, define the sequence $\{ \sigma_\N \}_{\N \in \naturals}$ of paths as in the proof of Theorem~\ref{theorem:optimality_prmstar}.

Second, define a sequence $\{q_\N\}_{\N \in \naturals}$ and tile $\sigma_\N$ with a set $B_\N = \{B_{\N, 1}, B_{\N, 2}, \dots, B_{\N, \MM}\}$ of overlapping balls of radius $q_\N$. See Figures~\ref{figure:covering_balls} and \ref{figure:k_prmstar_balls}. Let $x_{\M} \in B_{\N,\M}$ and $x_{\M + 1} \in B_{\N,\M+1}$ be any two points from subsequent balls in $B_\N$. Construct $B_\N$ such that the straight path connecting $x_{\M}$ and $x_{\M+1}$ lies entirely inside the obstacle free space.
Also, construct a set $B_\N'$ of balls such that (i) $B_{\N, \M}'$ and $B_{\N,\M}$ are centered at the same point and (ii) $B_{\N,\M}$ contains $B_{\N,\M}$, and $B_{\N,\M+1}$, for all $\M \in \{1,2,\dots, \MM_\N - 1\}$. 

Let $A_\N$ denote the event that each ball in $B_\N$ contains at least one vertex, and $A_\N'$ denote the event that each ball in $B_\N'$ contains at most $k(\N)$ vertices of the graph returned by the $k$-nearest PRM$^*$ algorithm.
Third, show that $A_\N$ and $A_\N'$ occur together for all large $\N$, with probability one. Clearly, this implies that the PRM$^*$ algorithm will connect vertices in subsequent ball in $B_\N$ with an edge, and any path formed by connecting such vertices will be collision-free.

Finally, show that any sequence of paths formed in this way converges to $\sigma^*$. Using the robustness of $\sigma^*$, show that the best path in the graph returned by the $k$-nearest PRM$^*$ algorithm converges to $c(\sigma^*)$ almost surely.

\subsection{Construction of the sequence $\{\sigma_\N\}_{\N \in \naturals}$ of paths}

Let $\theta_1, \theta_2 \in \reals_{>0}$ be two constants, the precise values of which will be provided shortly.
Define 
$$
\delta_\N := \min\left\{ \delta, \, (1 + \theta_1) \left( \frac{(1 + 1/d + \theta_2)\, \mu(X_\mathrm{free})}{\VolumeDBall{d}}\right)^{1/d} \left( \frac{\log \N}{\N} \right)^{1/d}\right\}.
$$ 

Since $\lim_{\N \to \infty} \delta_\N = 0$ and $0 \le \delta_\N \le \delta $ for all $\N \in \naturals$, by Lemma~\ref{lemma:weak_delta_clearance}, there exists a sequence $\{\sigma_\N\}_{\N \in \naturals}$ of paths such that $\lim_{\N \to \infty} \sigma_\N = \sigma^*$ and $\sigma_\N$ is strongly $\delta_\N$-clear for all $\N \in \naturals$.

\subsection{Construction of the sequence $\{B_\N\}_{\N \in \naturals}$ of sets of balls}

Define
$$
q_\N := \frac{\delta_\N}{1 + \theta_1}.
$$
For each $\N \in \naturals$, use Definition~\ref{definition:covering_balls} to construct a set $B_\N = \{B_{\N,1}, B_{\N,2}, \dots, B_{\N,\MM_\N}\}$ of overlapping balls that collectively cover $\sigma_\N$ as $B_\N := {\tt CoveringBalls}(\sigma_\N, q_\N, \theta_1 q_\N)$ (see Figures~\ref{figure:covering_balls} and \ref{figure:k_prmstar_balls} for an illustration).

\begin{figure}
\centering
\includegraphics[height = 5cm]{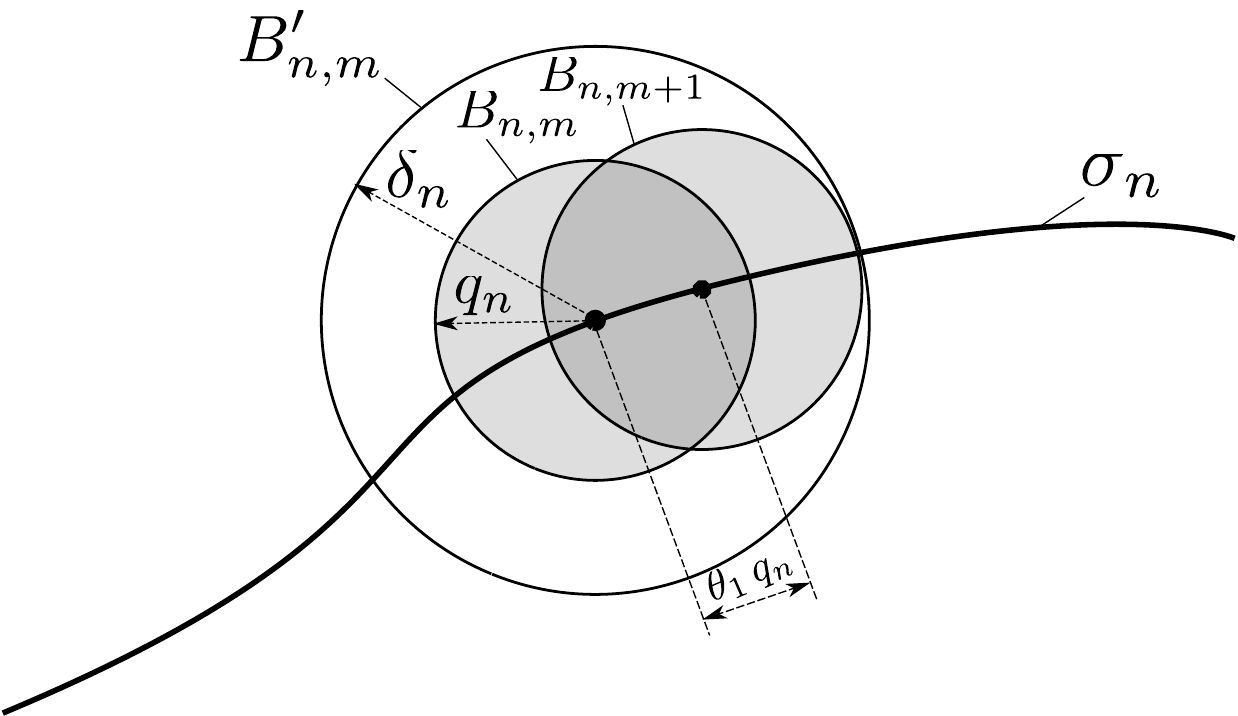}
\caption{An illustration of the covering balls for the $k$-nearest PRM$^*$ algorithm. The $\delta_\N$ ball is guaranteed to contain the balls $B_{\N,\M}$ and $B_{\N,\M + 1}$.}
\label{figure:k_prmstar_balls}
\end{figure}

\subsection{The probability that each ball in $B_\N$ contains at least one vertex}

Recall that $G^{k\mathrm{PRM}^*}_\N = (V^{k\mathrm{PRM}^*}_\N, E^{k\mathrm{PRM}^*}_\N)$ denotes the graph returned by the $k$-nearest PRM$^*$ algorithm, when the algorithm is run with $\N$ samples.
Let $A_{\N,\M}$ denote the event that the ball $B_{\N,\M}$ contains at least one vertex from $V^{k\mathrm{PRM}^*}_\N$, i.e., $A_{\N,\M} = \left\{ B_{\N,\M} \cap V^{k\mathrm{PRM}^*}_\N \neq \emptyset \right\}$. Let $A_\N$ denote the event that all balls in $B_{\N,\M}$ contains at least one vertex of $G^{k\mathrm{PRM}^*}_\N$, i.e., $A_\N = \bigcap_{\M = 1}^{\MM_\N} A_{\N,\M}$.

Recall that $A_\N^c$ denotes the complement of the event $A_\N$, $\mu(\cdot)$ denotes the Lebesgue measure, and $\VolumeDBall{d}$ is the volume of the unit ball in the $d$-dimensional Euclidean space. 
Let $s_\N$ denote the length of $\sigma_\N$.
\begin{lemma} \label{lemma:at_least_one_vertex}
For all $\theta_1, \theta_2 > 0$, 
$$
\PP(A_\N^c) 
\,\,\le\,\, 
\frac{s_\N}{\theta_1} \left(\frac{\VolumeDBall{d} }{\theta_1 \, (1 + 1/d + \theta_2) \, \mu(X_\mathrm{free})}\right)^{1/d} \, \frac{1}{(\log \N)^{1/d} \,\,\, \N^{1 + \theta_2}}.
$$
In particular, $\sum_{\N = 1}^\infty \PP(A_\N^c) < \infty$ for all $\theta_1, \theta_2 > 0$.
\end{lemma}
\begin{proof}

Let $\N_0 \in \naturals$ be a number for which $\delta_\N < \delta$ for all $\N > \N_0$.
A bound on the number of balls in $B_\N$ can computed as follows.
For all $\N > \N_0$, 
$$
M_\N = \vert B_\N \vert \le \frac{s_\N}{\theta_1\,q_\N} = \frac{s_\N}{\theta_1} \left( \frac{\VolumeDBall{d}}{(1+1/d+\theta_2) \mu(X_\mathrm{free})}\right)^{1/d} \left( \frac{\N}{\log \N} \right)^{1/d}.
$$

The volume of each ball $B_\N$ can be computed as 
$$
\mu(B_{\N,\M}) = \VolumeDBall{d} (q_\N)^d = (1+ 1/d + \theta_2) \, \mu(X_\mathrm{free}) \frac{\log \N}{\N}.
$$

The probability that the ball $B_{\N,\M}$ does not contain a vertex of the $k$-nearest PRM$^*$ algorithm can be bounded as 
\begin{eqnarray*}
\PP(A_{\N,\M}^c) = \left(1 - \frac{\mu(B_{\N,\M})}{\mu(X_\mathrm{free})}\right)^{\N} = \left(1 - (1 + 1/d + \theta_2) \frac{\log \N}{\N} \right)^{\N} \le \N^{- (1 + 1/d + \theta_2)}.
\end{eqnarray*}

Finally, the probability that at least one of the balls in $B_\N$ contains no vertex of the $k$-nearest PRM$^*$ can be bounded as
\begin{eqnarray*}
\PP(A_\N) & = &\PP \left(\bigcup\nolimits_{\M = 1}^{\MM_\N}A_{\N,\M}\right)  \le  \sum_{\M = 1}^{\MM_\N} \PP (A_{\N,\M}) =  \MM_\N \, \PP(A_{\N,1}) \\
& \le & \frac{s_\N}{\theta_1} \left( \frac{\VolumeDBall{d}}{(1+1/d+\theta_2) \,\mu(X_\mathrm{free})}\right)^{1/d} \left( \frac{\N}{\log \N} \right)^{1/d} \N^{- (1 + 1/d + \theta_2)} \\
& = & \frac{s_\N}{\theta_1} \left( \frac{\VolumeDBall{d}}{(1 + 1/d + \theta_2)\, \mu(X_\mathrm{free})} \right)^{1/d} \frac{1}{(\log \N)^{1/d} \,\,\, \N^{1 + \theta_2}}.
\end{eqnarray*}
Clearly, $\sum_{\N = 1}^\infty \PP(A_\N^c) < \infty$ for all $\theta_1, \theta_2 > 0$.
\qed
\end{proof}

\subsection{Construction of the sequence $\{B_\N'\}_{\N \in \naturals}$ of sets of balls}

Construct a set $B_\N' = \{B_{\N,1}, B_{\N,2}, \dots, B_{\N,\MM_\N}\}$ of balls as $B_\N' := {\tt CoveringBalls}(\sigma_\N, \delta_\N, \theta_1 q_\N)$ so that each ball in $B_\N'$ has radius $\delta_\N$ and the spacing between two balls is $\theta_1 q_\N$ (see Figure~\ref{figure:k_prmstar_balls}). 

Clearly, the centers of balls in $B_\N'$ coincide with the centers of the balls in $B_\N$, i.e., the center of $B_{\N,\M}'$ is the same as the center of $B_{\N,\M}$ for all $\M \in \{1,2,\dots,\MM_\N\}$ and all $\N \in \naturals$. However, the balls in $B_\N'$ have a larger radius than those in $B_\N$.

\subsection{The probability that each ball in $B_\N'$ contains at most $k(\N)$ vertices}

Recall that the $k$-nearest PRM algorithm connects each vertex in the graph with its $k(\N)$ nearest vertices when the algorithm is run with $\N$ samples, where $k(\N) = k_\mathrm{PRM} \log \N$. 
Let $A_\N'$ denote the event that all balls in $B_\N'$ contain at most $k(\N)$ vertices of $G^{k\mathrm{PRM}^*}_\N$. 

Recall that $A_{\N}'^c$ denotes the complement of the event $A_{\N}$.
\begin{lemma} \label{lemma:bound_b_i_prime}
If $k_\mathrm{PRM} > e\,(1 + 1/d)$, then there exists some $\theta_1, \theta_2 > 0$ such that
$$
\PP(A_\N'^c) \le \frac{s_\N}{\theta_1} \left( \frac{\VolumeDBall{d}}{(1+1/d+\theta_2) \mu(X_\mathrm{free})}\right)^{1/d} \frac{1}{(\log \N)^{1/d} \,\,\,\N^{-(1+\theta_1)^d (1+ 1/d + \theta_2)}}.
$$
In particular, $\sum_{\N = 1}^\infty \PP(A_\N'^c) < \infty$ for some $\theta_1, \theta_2 > 0$.
\end{lemma}
\begin{proof}

Let $\N_0 \in \naturals$ be a number for which $\delta_\N < \delta$ for all $\N > \N_0$.
As shown in the proof of Lemma~\ref{lemma:at_least_one_vertex}, the number of balls in $B_\N'$ satisfies 
$$
M_\N \,\,=\,\, \vert B_\N' \vert  \,\,\le\,\,  \frac{s_\N}{\theta_1 q_\N} \,\,=\,\, \frac{s_\N}{\theta_1} \left( \frac{\VolumeDBall{d}}{(1+1/d+\theta_2) \mu(X_\mathrm{free})}\right)^{1/d} \left( \frac{\N}{\log \N} \right)^{1/d}. 
$$

For all $\N > \N_0$, the volume of $B_{\N,\M}'$ can be computed as 
$$
\mu(B_{\N,\M}') = \VolumeDBall{d} \, (\delta_\N)^d =(1 + \theta_1)^d \,(1 + 1/d + \theta_2)\, \mu(X_\mathrm{free}) \, \frac{\log \N}{\N}.
$$ 

Let $I_{\N,\M,\I}$ denote the indicator random variable of the event that sample $\I$ falls into ball $B_{\N,\M}'$. The expected value of $I_{\N,\M,\I}$ can be computed as
$$
\EE[I_{\N,\M,\I}] = \frac{\mu(B_{\N,\M}')}{\mu(X_\mathrm{free})} = (1 + \theta_1)^d \,(1 + 1/d + \theta_2)\,\frac{\log \N}{\N}.
$$
Let $N_{\N,\M}$ denote the number of vertices that fall inside the ball $B_{\N,\M}'$, i.e., $N_{\N,\M} = \sum_{\I = 1}^{\N} I_{\N,\M,\I}$. Then, %
$$
\EE[N_{\N,\M}] = \sum_{\I = 1}^{\N} \EE[I_{\N,\M,\I}] = \N \,\, \EE[I_{\N,\M,1}] = (1 + \theta_1)^d (1+1/d+\theta_2) \log \N.
$$
Since $\{ I_{\N,\M,\I} \}_{\I = 1}^{\N}$ are independent identically distributed random variables, large deviations of their sum, $M_{\N,\M}$, can be bounded by the following Chernoff bound~\citep{dubhashi.panconesi.book09}:
$$
\PP \big(\,\big\{ \,N_{\N,\M} > (1+\epsilon) \, \EE[N_{\N,\M}] \,\big\}\,\big) 
\,\,\le\,\, 
\left( \frac{e^\epsilon}{(1 + \epsilon)^{(1+ \epsilon)}} \right)^{\EE[N_{\N,\M}]},
$$
for all $\epsilon > 0$. In particular, for $\epsilon = e -1$,
$$
\PP \big(\,\big\{ \,N_{\N,\M} > e\, \EE[N_{\N,\M}] \,\big\}\,\big) \,\,\le\,\, e^{-\EE[N_{\N,\M}]} 
\,\,=\,\, 
e^{-(1+ \theta_1)^d (1 + 1/d +\theta_2) \log \N} 
\,\,=\,\, 
\N^{-(1+\theta_1)^d (1+ 1/d + \theta_2)}.
$$

Since $k(\N) > e \, (1 + 1/d) \, \log \N$, there exists some $\theta_1, \theta_2 > 0$ independent of $\N$ such that
$
e\, \EE[N_{\N,k}] = e\, (1 + \theta_1) \, (1+ 1/d + \theta_2) \, \log \N \le k(\N). 
$
Then, for the same values of $\theta_1$ and $\theta_2$,
$$
\PP \big(\,\big\{ \,N_{\N,\M} > k(\N) \,\big\}\,\big) 
\,\,\le\,\, 
\PP \big(\,\big\{ \,N_{\N,\M} > e \, \EE[N_{\N,\M}] \,\big\}\,\big) 
\,\,\le\,\,  
\N^{-(1+\theta_1)^d (1+ 1/d + \theta_2)}.
$$

Finally, consider the probability of the event that at least one ball in $B_{\N}$ contains more than $k(\N)$ nodes. Using the union bound together with the inequality above
\begin{eqnarray*}
\PP \left(\bigcup\nolimits_{\M = 1}^{\MM_\N} \big\{N_{\N,\M} > k(\N) \big\} \right) \,\,\le\,\, \sum_{\M = 1}^{\MM_\N} \PP \big( \big\{ N_{\N,\M} > k(\N) \big\} \big) \,\, = \,\, \MM_\N \,\, \PP \big( \{ N_{\N,1} > k(\N) \} \big) 
\end{eqnarray*}

Hence,
$$
\PP (A_\N'^c) = \PP \left(\bigcup\nolimits_{\M = 1}^{\MM_\N} \big\{N_{\N,\M} > k(\N) \big\} \right) \,\,\le\,\, \frac{s_\N}{\theta_1} \left( \frac{\VolumeDBall{d}}{(1+1/d+\theta_2) \mu(X_\mathrm{free})}\right)^{1/d} \frac{1}{(\log \N)^{1/d} \,\,\,\N^{-(1+\theta_1)^d (1+ 1/d + \theta_2)}}.
$$
Clearly, $\sum_{\N = 1}^\infty \PP(A_\N'^c) < \infty$ for the same values of $\theta_1$ and $\theta_2$.
\qed
\end{proof}

\subsection{Connecting the vertices in the subsequent balls in $B_\N$}

First, note the following lemma.

\begin{lemma}\label{lemma:k_prmstar:attempt_connection}
If $k_\mathrm{PRM} > e \, (1 + 1/d)^{1/d}$, then there exists $\theta_1, \theta_2 > 0$ such that the event that each ball in $B_\N$ contains at least one vertex and each ball in $B_\N'$ contains at most $k(n)$ vertices occurs for all large $\N$, with probability one, i.e.,
$$
\PP \left( \liminf_{\N \to \infty} (A_\N \cap A_\N') \right) = 1.
$$
\end{lemma}
\begin{proof}
Consider the event $A_\N^c \cup A_\N'^c$, which is the complement of $A_\N \cap A_\N'$. Using the union bound, 
$$
\PP\left( A_\N^c \cup A_\N'^c \right) \le \PP(A_\N^c) + \PP(A_\N'^c).
$$
Summing both sides,
$$
\sum_{\N = 1}^\infty \PP (A_\N^c \cup A_N'^c) \le \sum_{\N = 1}^\infty \PP(A_\N^c) + \sum_{\N = 1}^\infty \PP(A_\N'^c) < \infty,
$$
where the last inequality follows from Lemmas~\ref{lemma:at_least_one_vertex} and \ref{lemma:bound_b_i_prime}. Then, by the Borel-Cantelli lemma, $\PP \left(\limsup_{\N \to \infty} (A_\N^c \cup A_\N'^c) \right) = \PP \left( \limsup_{\N \to \infty} (A_\N \cap A_\N')^c \right) = 0$, which implies $\PP \left(\liminf_{\N \to \infty} (A_\N \cap A_\N') \right) = 1$. \qed
\end{proof}

Note that for each $\M \in \{1,2, \dots, \MM_\N - 1\}$, both $B_{\N,\M}$  and $B_{\N,\M+1}$ lies entirely inside the ball $B_{\N,\M}'$ (see Figure~\ref{figure:k_prmstar_balls}). 
Hence, whenever the balls $B_{\N,\M}$ and $B_{\N,\M+1}$ contain at least one vertex each, and $B_{\N,\M}'$ contains at most $k(\N)$ vertices, the $k$-nearest PRM$^*$ algorithm attempts to connect all vertices in $B_{\N,\M}$ and $B_{\N,\M+1}$ with one another.

The following lemma guarantees that connecting any two points from two consecutive balls in $B_\N$ results in a collision-free trajectory. The proof of the lemma is essentially the same as that of Lemma~\ref{lemma:no_collision}.
\begin{lemma} \label{lemma:k_no_collision}
For all $\N \in \naturals$ and all $\M \in \{1,2, \dots, \MM_\N\}$, if $x_{\M} \in B_{\N,\M}$ and $x_{\M+1} \in B_{\N,\M+1}$, then the line segment connecting $x_{\M}$ and $x_{\M+1}$ lies in the obstacle-free space, i.e., 
$$
\alpha \, x_{\M} + (1 - \alpha) \, x_{\M+1} \in X_\mathrm{free}, \quad\quad \mbox{ for all } \alpha \in [0,1].
$$
\end{lemma}

\subsection{Convergence to the optimal path}

The proof of the following lemma is similar to that of Lemma~\ref{lemma:prmstar:convergence_in_bvnorm}, and is omitted here.

Let $P_\N$ denote the set of all paths in the graph returned by $\AlgkPRMstar$ algorithm at the end of $\N$ iterations.
Let $\sigma_\N'$ be the path that is closest to $\sigma_\N$ in terms of the bounded variation norm among all those paths in $P_\N$, i.e., 
$
\sigma_\N' := \min_{\sigma' \in P_\N} \Vert \sigma' - \sigma_\N \Vert.
$
\begin{lemma} 
The random variable $\Vert \sigma_\N' - \sigma_\N \Vert_\BVnorm$ converges to zero almost surely, i.e., 
$$
\PP \left( \left\{ \lim\nolimits_{\N \to \infty} \Vert \sigma_\N' - \sigma_\N \Vert_\BVnorm = 0 \right\}\right) = 1.
$$
\end{lemma}

A corollary of the lemma above is that $\lim_{\N \to \infty} \sigma_\N' = \sigma^*$ with probability one. Then, the result follows by the robustness of the optimal solution (see the proof of Lemma~\ref{lemma:prmstar:cost_convergence} for details).

\section{Proof of Theorem~\ref{theorem:optimality_rrg} (Asymptotic optimality of RRG)}
\label{proof:optimality_rrg}

\subsection{Outline of the proof}

The proof of this theorem is similar to that of Theorem~\ref{theorem:optimality_prmstar}. The main difference is the definition of $C_\N$ that denotes the event that the RRG algorithm has sufficiently explored the obstacle free space. More precisely, $C_\N$ is the event that for any point $x$ in the obstacle free space, the graph maintained by the RRG algorithm algorithm includes a vertex that can be connected to $x$.

Construct the sequence $\{ \sigma_\N \}_{\N \in \naturals}$ of paths and the sequence $\{ B_\N \}_{\N \in \naturals}$ of balls as in the proof of Theorem~\ref{theorem:optimality_prmstar}. 
Let $A_\N$ denote the event that each ball in $B_{\N}$ contains a vertex of the graph maintained by the RRG by the end of iteration $\N$. Compute $\N$ by conditioning on the event that $C_i$ holds for all $i \in \{ \lfloor \theta_3 \, \N \rfloor, \ldots, \N\}$, where $0 < \theta_3 < 1$ is a constant. Show that the probability that $C_i$ fails to occur for any such $i$ is small enough to guarantee that $A_\N$ occurs for all large $\N $ with probability one. Complete the proof as in the proof of Theorem~\ref{theorem:optimality_prmstar}.

\subsection{Definitions of $\{\sigma_\N\}_{\N \in \naturals}$ and $\{B_\N\}_{\N \in \naturals}$}

Let $\theta_1 > 0$ be a constant. 
Define $\delta_\N$, $\sigma_\N$, $q_\N$, and $B_\N$ as in the proof of Theorem~\ref{theorem:optimality_prmstar}.

\subsection{Probability that each ball in $B_\N$ contains at least one vertex}

Let $A_{\N,\M}$ be the event that the ball $B_{\N,\M}$ contains at least one vertex of the RRG at the end of $\N$ iterations. Let $A_\N$ be the event that all balls in $B_\N$ contain at least one vertex of the RRG at the end of iteration $\N$, i.e., $A_\N = \bigcap_{\M = 1}^{\MM_\N} A_{\N,\M}$, where $\MM_\N$ is the number of balls in $B_\N$. Recall that $\gamma_\mathrm{RRG}$ is the constant used in defining the connection radius of the RRG algorithm (see Algorithm~\ref{algorithm:RRG}).

\begin{lemma} \label{lemma:rrg:vertices_in_balls}
If $\gamma_\mathrm{RRG} > 2 (1 + 1/d)^{1/d} \left( \frac{\mu(X_\mathrm{free})}{\VolumeDBall{d}} \right)^{1/d}$ then there exists $\theta_1 > 0$ such that $A_\N$ occurs for all large $\N$ with probability one, i.e.,
$$
\PP \left(\liminf\nolimits_{\N \to \infty} A_\N \right) = 1.
$$
\end{lemma}
The proof of this lemma requires two intermediate results, which are provided next. 

Recall that $\eta$ is the parameter used in the ${\tt Steer}$ procedure (see the definition of {\tt Steer} procedure in Section~\ref{section:algorithms:primitive_procedures}).
Let $C_\N$ denote the event that for any point $x \in X_\mathrm{free}$, the graph returned by the RRG algorithm includes a vertex $v$ such that $\Vert x - v \Vert \le \eta$ and the line segment joining $v$ and $x$ is collision-free.
The following lemma establishes an bound on the probability that this event fails to occur at iteration $\N$.
\begin{lemma} \label{lemma:bounding_c_i}
There exists constants $a, b \in \reals_{>0}$ such that $P(C_\N^c) \le a \, e^{-b\,\N}$ for all $\N \in \naturals$.
\end{lemma}
\begin{proof}
Partition $X_\mathrm{free}$ into finitely many convex sets such that each partition is bounded by a ball a radius $\eta$. Such a finite partition exists by the boundedness of $X_\mathrm{free}$. Denote this partition by $X_1',X_2', \dots, X_\MM'$. Since the probability of failure decays to zero with an exponential rate, for any $\M \in \{1,2,\dots,\MM\}$, the probability that $X_\M'$ fails to contain a vertex of the RRG decays to zero with an exponential rate, i.e., 
$$
\PP\left(\left\{ \nexists x \in V_\N^\mathrm{RRG} \cap X_\M'\right\}\right) \le a_\M \, e^{-b_\M \, \N}
$$
The probability that at least one partition fails to contain one vertex of the RRG also decays to zero with an exponential rate. That is, there exists $a, b \in \reals_{> 0}$ such that 
$$
\PP\left(\bigcup\nolimits_{\M = 1}^\MM \left\{ \nexists x \in V_\N^\mathrm{RRG} \cap X_\M'\right\}\right) \le 
\sum_{\M = 1}^\MM  \PP\left(\left\{ \nexists x \in V_\N^\mathrm{RRG} \cap X_\M'\right\}\right) \le \sum_{\M = 1}^\MM a_\M \, e^{-b_\M \N} \le a \, e^{-b \, \N},
$$
where the first inequality follows from the union bound.
\qed
\end{proof}

Let $0 < \theta_3 < 1$ be a constant independent of $\N$. Consider the event that $C_\I$ occurs for all $\I$ that is greater than $\theta_3 \, \N$, i.e., $\bigcap_{ \I = \lfloor \theta_3 \, \N \rfloor}^{\N} C_\I$. The following lemma analyzes the probability of the event that $\bigcap_{ \I = \lfloor \theta_3 \, \N \rfloor}^{\N} C_\I$ fails to occur.
\begin{lemma} \label{lemma:finite_c}
For any $\theta_3 \in (0,1)$, 
$$
\sum_{\N  = 1}^{\infty} \PP\left( \left( \bigcap\nolimits_{\I = \lfloor \theta_3 \N \rfloor}^{\N} C_\I \right)^c \right) \,\,<\,\, \infty.
$$
\end{lemma}
\begin{proof}
The following inequalities hold:
$$
\sum_{\N = 1}^{\infty} \PP \left( \left( \bigcap\nolimits_{\I = \lfloor \theta_3 \, \N \rfloor}^{\N} C_\I \right)^c \right) 
\,\,= \,\,\sum_{\N = 1}^\infty \PP \left( \bigcup\nolimits_{\I = \lfloor \theta_3 \, \N \rfloor}^\N C_\I^c \right) 
\,\,\le\,\, \sum_{\N =1}^\infty \sum_{\I = \lfloor \theta_3 \,i \rfloor}^\N \PP (C_\I^c) 
\,\,\le\,\, \sum_{\N =1}^\infty \sum_{\I = \lfloor \theta_3 \,\N \rfloor}^\N a \, e^{-b \, \I},
$$
where the last inequality follows from Lemma~\ref{lemma:bounding_c_i}. The right-hand side is finite for all $a,b > 0$. \qed
\end{proof}

\begin{proof}[Proof of Lemma~\ref{lemma:rrg:vertices_in_balls}]
It is shown that $\sum_{\N = 1}^\infty \PP \left(A_\N^c \right) < \infty$, which, by the Borel-Cantelli Lemma~\citep{grimmett.stirzaker.book01}, implies that $A_\N^c$ occurs infinitely often with probability zero, i.e., $\PP (\limsup_{\N \to \infty} A_\N^c ) = 0$, which in turn implies $\PP (\liminf_{\N \to \infty} A_\N) =  1$.

Let $\N_0 \in \naturals$ be a number for which $\delta_\N < \delta$ for all $\N > \N_0$.
First, for all $\N > \N_0$, the number of balls in $B_\N$ can be bounded by (see the proof of Lemma~\ref{lemma:vertices_in_balls} for details)
$$
M_\N = \vert B_\N \vert \le \frac{(2 + \theta_1)\,s_\N}{\theta_1 \, \gamma_\mathrm{RRG}} \left( \frac{\N}{\log \N} \right)^{1/d}.
$$

Second, for all $\N > \N_0$, the volume of each ball in $B_\N$ can be calculated as (see the proof of Lemma~\ref{lemma:vertices_in_balls})
$$
\mu (B_{\N,\M}) \,\,=\,\,  \VolumeDBall{d} \, \left( \frac{\gamma_\mathrm{PRM}}{2 + \theta_1} \right)^d \frac{\log \N}{\N},
$$
where $\VolumeDBall{d}$ is the volume of the unit ball in the $d$-dimensional Euclidean space.

Third, conditioning on the event $\bigcap_{\I = \lfloor \theta_3 \,\N \rfloor}^\N C_\I$, each new sample will be added to the graph maintained by the RRG algorithm as a new vertex between iterations $\I = \lfloor \theta_3 \, \N \rfloor$ and $\I =\N$. 
Thus, 
\begin{eqnarray*}
\PP \left(A_{\N,\M}^c\,\Big\vert\, \bigcap\nolimits_{\I = \lfloor \theta_3 \, \N \rfloor}^\N C_\I \right) 
& \le & \left( 1 - \frac{\mu(B_{\N,\M})}{\mu(X_\mathrm{free})}  \right)^{\N - \lfloor \theta_3\,\N \rfloor} 
\,\, \le \,\, \left( 1 - \frac{\mu(B_{\N,\M})}{\mu(X_\mathrm{free})}  \right)^{(1 - \theta_3)\,\N} \\
& \le & \left( 1 - \frac{ \VolumeDBall{d} }{ \mu(X_\mathrm{free}) } \left( \frac{\gamma_\mathrm{RRG}}{2 + \theta_1} \right)^d \frac{\log \N}{\N}  \right)^{(1 - \theta_3)\,\N} \\
& \le & e^{- \frac{(1- \theta_3) \, \VolumeDBall{d}}{\mu(X_\mathrm{free})} \left( \frac{\gamma_\mathrm{RRG}}{2 + \theta_1} \right)^d \log \N } \le \N^{- \frac{(1- \theta_3) \, \VolumeDBall{d}}{\mu(X_\mathrm{free})} \left( \frac{\gamma_\mathrm{RRG}}{2 + \theta_1} \right)^d},
\end{eqnarray*}
where the fourth inequality follows from $(1 - 1/f(\N))^{g(\N)} \le e^{g(\N)/f(\N)}$.

Fourth, 
\begin{eqnarray*}
\PP \left(A_\N^c \,\Big\vert\, \bigcap\nolimits_{\I = \lfloor \theta_3 \, \N \rfloor}^\N C_\I\right) 
& \le &  \PP \left(\bigcup\nolimits_{\M = 1}^{\MM_\N} A_{\N,\M}^c \,\Big\vert\, \bigcap\nolimits_{\I = \lfloor \theta_3 \, \N \rfloor}^\N C_\I\right)  \\
& \le & \sum_{\M = 1}^{\MM_\N} \PP \left(A_{\N,\M}^c \,\big\vert\,\bigcap\nolimits_{\I = \lfloor \theta_3 \, \N \rfloor}^\N C_\I \right) \\ 
& = & \MM_\N \,\, \PP \left(A_{\N,1}^c \,\big\vert\,\bigcap\nolimits_{\I = \lfloor \theta_3 \, \N \rfloor}^\N C_\I \right) \\ 
& \le & \frac{(2 + \theta_1)\,s_\N}{\theta_1 \, \gamma_\mathrm{RRG}} \left( \frac{\N}{\log \N} \right)^{1/d} \,\, \N^{- \frac{(1- \theta_3) \, \VolumeDBall{d}}{\mu(X_\mathrm{free})} \left( \frac{\gamma_\mathrm{RRG}}{2 + \theta_1} \right)^d}.
\end{eqnarray*}
Hence, 
$$
\sum_{\N = 1}^\infty \PP \left(A_\N^c \,\Big\vert\, \bigcap\nolimits_{\I = \lfloor \theta_3 \, \N \rfloor}^\N C_\I\right) \,\, < \,\, \infty, 
$$
whenever $\frac{(1- \theta_3) \, \VolumeDBall{d}}{\mu(X_\mathrm{free})} \left( \frac{\gamma_\mathrm{RRG}}{2 + \theta_1} \right)^d - 1/d > 1$, i.e., $\gamma_\mathrm{RRG} > (2 + \theta_1) (1 + 1/d)^{1/d} \left(\frac{\mu(X_\mathrm{free})}{(1-\theta_3) \, \VolumeDBall{d}} \right)^{1/d}$, which is satisfied by appropriately choosing the constants $\theta_1$ and $\theta_3$, since $\gamma_\mathrm{RRG} > 2 \, (1 + 1/d)^{1/d} \left(\frac{\mu(X_\mathrm{free})}{\VolumeDBall{d}} \right)^{1/d}$.

Finally, 
\begin{eqnarray*}
\PP \left(A_\N^c \,\big\vert\, \bigcap\nolimits_{\I = \lfloor \theta_3 \, \N \rfloor}^{\N} C_\I \right) 
& = & \frac{\PP \left(A_\N^c \cap \left(\cap_{\I = \lfloor \theta_3 \, \N \rfloor}^\N C_\I \right)\right)}{\PP \left( \bigcap_{\I = \lfloor \theta_3 \, \N \rfloor}^\N C_\I \right)} \\
& \ge & \PP (A_\N^c \cap (\cap_{\I = \lfloor \theta_3 \, \N \rfloor}^\N C_\I)) \\
& = & 1 - \PP (A_\N \cup (\cap_{\I = \lfloor \theta_3 \, \N \rfloor}^\N C_\I)^c) \\
& \ge & 1 - \PP (A_\N) - \PP \big((\cap_{\I = \lfloor \theta_3 \, \N \rfloor}^\N C_\I)^c\big) \\
& = & \PP(A_\N^c) - \PP \big((\cap_{\I = \lfloor \theta_3 \, \N \rfloor}^\N C_\I)^c\big).
\end{eqnarray*}
Taking the infinite sum of both sides yields
$$
\sum_{\N = 1}^\infty \PP (A_\N^c) \le 
\sum_{\N = 1}^{\infty} \PP\left(A_\N^c \,\big\vert\, \bigcap\nolimits_{\I = \lfloor \theta_3 \, \N \rfloor}^\N C_\N\right) 
+ \sum_{\N = 1}^\infty \PP \left(\left( \bigcap\nolimits_{\I = \lfloor \theta_3 \, \N \rfloor}^\N C_\N \right)^c \right).
$$
The first term on the right hand side is shown to be finite above. The second term is finite by Lemma~\ref{lemma:finite_c}. Hence, $\sum_{\N = 1}^\infty \PP (A_\N) < \infty$. Then, by the Borel Cantelli lemma, $A_\N^c$ occurs infinitely often with probability zero, which implies that its complement $A_\N$ occurs for all large $\N$, with probability one. \qed
\end{proof}

\subsection{Convergence to the optimal path}

The proof of the following lemma is similar to that of Lemma~\ref{lemma:prmstar:convergence_in_bvnorm}, and is omitted here.

Let $P_\N$ denote the set of all paths in the graph returned by $\AlgRRG$ algorithm at the end of $\N$ iterations.
Let $\sigma_\N'$ be the path that is closest to $\sigma_\N$ in terms of the bounded variation norm among all those paths in $P_\N$, i.e., 
$
\sigma_\N' := \min_{\sigma' \in P_\N} \Vert \sigma' - \sigma_\N \Vert.
$
\begin{lemma} 
The random variable $\Vert \sigma_\N' - \sigma_\N \Vert_\BVnorm$ converges to zero almost surely, i.e., 
$$
\PP \left( \left\{ \lim\nolimits_{\N \to \infty} \Vert \sigma_\N' - \sigma_\N \Vert_\BVnorm = 0 \right\}\right) = 1.
$$
\end{lemma}

A corollary of the lemma above is that $\lim_{\N \to \infty} \sigma_\N' = \sigma^*$ with probability one. Then, the result follows by the robustness of the optimal solution (see the proof of Lemma~\ref{lemma:prmstar:cost_convergence} for details).

\section{Proof of Theorem~\ref{theorem:optimality_k_rrg} (asymptotic optimality of $k$-nearest RRG)}
\label{proof:optimality_k_rrg}

\subsection{Outline of the proof}

The proof of this theorem is a combination of that of Theorem~\ref{theorem:optimality_k_prmstar} and \ref{theorem:optimality_rrg}. 

Define the sequences $\{\sigma_\N\}_{\N \in \naturals}$, $\{B_\N\}_{\N \in \naturals}$, and $\{ B_\N' \}_{\N \in \naturals}$ as in the proof of Theorem~\ref{theorem:optimality_k_prmstar}. Define the event $C_\N$ as in the proof of Theorem~\ref{theorem:optimality_rrg}.
Let $A_\N$ denote the event that each ball in $B_\N$ contains at least one vertex, and $A_\N'$ denote the event that each ball in $B_\N'$ contains at most $k(\N)$ vertices of the graph maintained by the RRG algorithm, by the end of iteration $\N$. Compute $A_\N$ and $A_\N'$ by conditioning on the event that $C_i$ holds for all $i = \theta_3 \, \N$ to $\N$. Show that this is enough to guarantee that $A_\N$ and $A_\N'$ hold together for all large $\N$, with probability one.

\subsection{Definitions of $\{\sigma_\N\}_{\N \in \naturals}$, $\{B_\N\}_{\N \in \naturals}$, and $\{B_\N'\}_{\N \in \naturals}$}

Let $\theta_1, \theta_2 > 0$ be two constants. Define $\delta_\N$, $\sigma_\N$, $q_\N$, $B_\N$, and $B_\N'$ as in the proof of Theorem~\ref{theorem:optimality_k_prmstar}.

\subsection{The probability that each ball in $B_\N$ contains at least one vertex}
Let $A_{\N,\M}$ denote the event that the ball $B_{\N,\M}$ contains at least one vertex of the graph maintained by the $k$-nearest RRG algorithm by the end of iteration $\N$. Let $A_\N$ denote the event that all balls in $B_{\N,\M}$ contain at least one vertex of the same graph, i.e., $A_\N = \bigcup_{\M = 1}^{\MM_\N} A_{\N,\M}$.
Let $s_\N$ denote the length of $\sigma_\N$, i.e., $TV(\sigma_\N)$. 
Recall $\eta$ is the parameter in the ${\tt Steer}$ procedure. Let $C_\N$ denote the event that for any point $x\in X_\mathrm{free}$, the $k$-nearest RRG algorithm includes a vertex $v$ such that $\Vert x - v\Vert \le \eta$. 
\begin{lemma} \label{lemma:rrg:at_least_one_vertex}
For any $\theta_1, \theta_2 > 0$ and any $\theta_3 \in (0,1)$, 
$$
\PP\left(A_\N^c \,\Big\vert\, \bigcap\nolimits_{\I = \lfloor \theta_3 \,\N \rfloor}^\N C_\I \right) 
\,\,\le\,\,
\frac{s_\N}{\theta_1} \left(\frac{\VolumeDBall{d} }{\theta_1 \, (1 + 1/d + \theta_2) \, \mu(X_\mathrm{free})}\right)^{1/d} \, \frac{1}{(\log \N)^{1/d} \,\,\, \N^{(1- \theta_3)(1 + 1/d + \theta_2) - 1/d}}.
$$
In particular, $\sum_{\N = 1}^\infty \PP(A_\N^c \,\vert\, \bigcap\nolimits_{\I = \lfloor \theta_3 \,\N \rfloor}^\N C_\I) < \infty$ for any $\theta_1, \theta_2 > 0$ and some $\theta_3 \in (0,1)$.
\end{lemma}

\begin{proof}
Let $\N_0 \in \naturals$ be a number for which $\delta_\N < \delta$ for all $\N > \N_0$.
Then, for all $\N > \N_0$, 
$$
M_\N = \vert B_\N \vert \le \frac{s_\N}{\theta_1\,q_\N} = \frac{s_\N}{\theta_1} \left( \frac{\VolumeDBall{d}}{(1+1/d+\theta_2) \mu(X_\mathrm{free})}\right)^{1/d} \left( \frac{\N}{\log \N} \right)^{1/d}.
$$

The volume of each ball $B_\N$ can be computed as 
$$
\mu(B_{\N,\M}) = \VolumeDBall{d} (q_\N)^d = (1+ 1/d + \theta_2) \, \mu(X_\mathrm{free}) \frac{\log \N}{\N}.
$$

Given $\bigcap_{i = \lceil \theta_3\,\N \rceil}^{\N} C_i$, the probability that the ball $B_{\N,\M}$ does not contain a vertex of the $k$-nearest PRM$^*$ algorithm can be bounded as 
\begin{eqnarray*}
\PP\Big(A_{\N,\M}^c \,\big\vert\, \bigcap\nolimits_{i = \lceil \theta_3\,\N \rceil}^{\N} C_i \Big) 
\,\,=\,\, 
\left(1 - \frac{\mu(B_{\N,\M})}{\mu(X_\mathrm{free})}\right)^{(1- \theta_3)\,\N} = \left(1 - (1 + 1/d + \theta_2) \frac{\log \N}{\N} \right)^{(1 - \theta_3)\,\N} \le \N^{- (1 - \theta_3)\,(1 + 1/d + \theta_2)}.
\end{eqnarray*}

Finally, the probability that at least one of the balls in $B_\N$ contains no vertex of the $k$-nearest PRM$^*$ can be bounded as
\begin{eqnarray*}
\PP(A_\N) & = &\PP \left(\bigcup\nolimits_{\M = 1}^{\MM_\N}A_{\N,\M}\right) \,\,\le\,\,  \sum_{\M = 1}^{\MM_\N} \PP (A_{\N,\M}) =  \MM_\N \, \PP(A_{\N,1}) \\
& \le & \frac{s_\N}{\theta_1} \left( \frac{\VolumeDBall{d}}{(1 + 1/d + \theta_2)\, \mu(X_\mathrm{free})} \right)^{1/d} \frac{1}{(\log \N)^{1/d} \,\,\, \N^{(1 - \theta_3)\,(1 + 1/d+ \theta_2) - 1/d}}.
\end{eqnarray*}
Clearly, for all $\theta_1, \theta_2 > 0$, there exists some $\theta_3 \in (0,1)$ such that $\sum_{\N = 1}^\infty \PP(A_\N^c) < \infty$.
\qed
\end{proof}

\subsection{The probability that each ball in $B_\N'$ contains at most $k(\N)$ vertices}

Let $A_\N'$ denote the event that all balls in $B_\N'$ contain at most $k(n)$ vertices of the graph maintained by the RRG algorithm, by end of iteration $\N$.

\begin{lemma} \label{lemma:k_rrg:a_n_prime_bound}
If $k_\mathrm{PRM} > e\,(1 + 1/d)$, then there exists $\theta_1, \theta_2, \theta_3 > 0$ such that 
$$
\PP\left(A_\N'^c \,\Big\vert\, \bigcap\nolimits_{\I = \lfloor\theta_3\,\N \rfloor}^\N C_\I \right) 
\,\,\le\,\,
\frac{s_\N}{\theta_1} \left( \frac{\VolumeDBall{d}}{(1+1/d+\theta_2) \mu(X_\mathrm{free})}\right)^{1/d} \frac{1}{(\log \N)^{1/d} \,\,\,\N^{-(1-\theta_3)(1+\theta_1)^d (1+ 1/d + \theta_2)}}.
$$
In particular, $\sum_{\N = 1}^\infty \PP(A_\N^c \,\vert\, \bigcap\nolimits_{\I = \lfloor \theta_3 \,\N \rfloor}^\N C_\I) < \infty$ for some $\theta_1, \theta_2 > 0$ and some $\theta_3 > 0$.
\end{lemma}
\begin{proof}

Let $\N_0 \in \naturals$ be a number for which $\lambda_\N < \delta$ for all $\N > \N_0$.
Then, the number of balls in $B_\N'$ and the volume of each ball can be computed as
$$
M_\N \,\,=\,\, \vert B_\N' \vert  \,\,\le\,\,  \frac{s_\N}{\theta_1 q_\N} \,\,=\,\, \frac{s_\N}{\theta_1} \left( \frac{\VolumeDBall{d}}{(1+1/d+\theta_2) \mu(X_\mathrm{free})}\right)^{1/d} \left( \frac{\N}{\log \N} \right)^{1/d}. 
$$
$$
\mu(B_{\N,\M}') = \VolumeDBall{d} \, (\lambda_\N)^d =(1 + \theta_1)^d \,(1 + 1/d + \theta_2)\, \mu(X_\mathrm{free}) \, \frac{\log \N}{\N}.
$$ 

Let $I_{\N,\M,\I}$ denote the indicator random variable of the event that sample $\I$ falls into ball $B_{\N,\M}'$. The expected value of $I_{\N,\M,\I}$ can be computed as
$$
\EE[I_{\N,\M,\I}] = \frac{\mu(B_{\N,\M}')}{\mu(X_\mathrm{free})} = (1 + \theta_1)^d \,(1 + 1/d + \theta_2)\,\frac{\log \N}{\N}.
$$
Let $N_{\N,\M}$ denote the number of vertices that fall inside the ball $B_{\N,\M}'$ between iterations $\lfloor \theta_3 \, \N \rfloor$ and $\N$, i.e., $N_{\N,\M} = \sum_{\I = \lfloor \theta_3 \, \N \rfloor}^{\N} I_{\N,\M,\I}$. Then, 
$$
\EE[N_{\N,\M}] = \sum_{\I = \lfloor \theta_3 \, \N \rfloor}^{\N} \EE[I_{\N,\M,\I}] = (1- \theta_3)\, \N \,\, \EE[I_{\N,\M,1}] = (1 - \theta_3)\, (1 + \theta_1)^d \,(1+1/d+\theta_2) \log \N.
$$
Since $\{ I_{\N,\M,\I} \}_{\I = 1}^{\N}$ are independent identically distributed random variables, large deviations of their sum, $M_{\N,\M}$, can be bounded by the following Chernoff bound~\citep{dubhashi.panconesi.book09}:
$$
\PP \big(\,\big\{ \,N_{\N,\M} > (1+\epsilon) \, \EE[N_{\N,\M}] \,\big\}\,\big) 
\,\,\le\,\, 
\left( \frac{e^\epsilon}{(1 + \epsilon)^{(1+ \epsilon)}} \right)^{\EE[N_{\N,\M}]},
$$
for all $\epsilon > 0$. In particular, for $\epsilon = e -1$,
$$
\PP \big(\,\big\{ \,N_{\N,\M} > e\, \EE[N_{\N,\M}] \,\big\}\,\big) \,\,\le\,\, e^{-\EE[N_{\N,\M}]} 
\,\,=\,\, 
\N^{-(1-\theta_3)\,(1+\theta_1)^d \, (1+ 1/d + \theta_2)}.
$$

Since $k(\N) > e \, (1 + 1/d) \, \log \N$, there exists some $\theta_1, \theta_2 > 0$ and $\theta_3 \in (0,1)$, independent of $\N$, such that
$
e\, \EE[N_{\N,k}] = e\, (1- \theta_3) \, (1 + \theta_1) \, (1+ 1/d + \theta_2) \, \log \N \le k(\N). 
$
Then, for the same values of $\theta_1$ and $\theta_2$,
$$
\PP \big(\,\big\{ \,N_{\N,\M} > k(\N) \,\big\}\,\big) 
\,\,\le\,\, 
\PP \big(\,\big\{ \,N_{\N,\M} > e \, \EE[N_{\N,\M}] \,\big\}\,\big) 
\,\,\le\,\,  
\N^{-(1 - \theta_3)\,(1+\theta_1)^d \,(1+ 1/d + \theta_2)}.
$$

Finally, consider the probability of the event that at least one ball in $B_{\N}$ contains more than $k(\N)$ nodes. Using the union bound together with the inequality above
\begin{eqnarray*}
\PP \left(\bigcup\nolimits_{\M = 1}^{\MM_\N} \big\{N_{\N,\M} > k(\N) \big\} \right) \,\,\le\,\, \sum_{\M = 1}^{\MM_\N} \PP \big( \big\{ N_{\N,\M} > k(\N) \big\} \big) \,\, = \,\, \MM_\N \,\, \PP \big( \{ N_{\N,1} > k(\N) \} \big) 
\end{eqnarray*}

Hence,
\begin{eqnarray*}
\PP \Big(A_\N'^c \,\big\vert\, \bigcap\nolimits_{i = \lfloor \theta_3\, \N \rfloor}^\N C_i\Big) 
& = & \PP \left(\bigcup\nolimits_{\M = 1}^{\MM_\N} \big\{N_{\N,\M} > k(\N) \big\} \right) \\
& \le & \frac{s_\N}{\theta_1} \left( \frac{\VolumeDBall{d}}{(1+1/d+\theta_2) \mu(X_\mathrm{free})}\right)^{1/d} \frac{1}{(\log \N)^{1/d} \,\,\,\N^{-(1- \theta_3)\,(1+\theta_1)^d (1+ 1/d + \theta_2)}}.
\end{eqnarray*}
Clearly, $\sum_{\N = 1}^\infty \PP \big(A_\N'^c \,\vert\, \cap_{i = \lfloor \theta_3\, \N \rfloor}^\N C_i\big) < \infty$ for the same values of $\theta_1$, $\theta_2$, and $\theta_3$.
\qed.
\end{proof}

\subsection{Connecting the vertices in subsequent balls in $B_\N$}

\begin{lemma}\label{lemma:k_rrg:attempt_connection}
If $k_\mathrm{PRM} > e \, (1 + 1/d)^{1/d}$, then there exists $\theta_1, \theta_2 > 0$ such that the event that each ball in $B_\N$ contains at least one vertex and each ball in $B_\N'$ contains at most $k(n)$ vertices occurs for all large $\N$, with probability one, i.e.,
$$
\PP \left( \liminf_{\N \to \infty} (A_\N \cap A_\N') \right) = 1.
$$
\end{lemma}

First note the following lemma.
\begin{lemma} \label{lemma:k_rrg:c_n_bound}
For any $\theta_3 \in (0,1)$,
$$
\sum_{\N  = 1}^{\infty} \PP\left( \left( \bigcap\nolimits_{\I = \lfloor \theta_3 \N \rfloor}^{\N} C_\N \right)^c \right) \,\,<\,\, \infty.
$$
\end{lemma}
\begin{proof}
Since the RRG algorithm and the $k$-nearest RRG algorithm have the same vertex sets, i.e., $V^{\mathrm{RRG}}_\N = V^{k\mathrm{RRG}}_\N$ surely for all $\N \in \naturals$, the lemma follows from Lemma~\ref{lemma:finite_c}. \qed
\end{proof}

\begin{proof}[Proof of Lemma~\ref{lemma:k_rrg:attempt_connection}]
Note that 
\begin{eqnarray*}
\PP \left( (A_\N^c \cup A_\N'^c) \,\big\vert\, \bigcap\nolimits_{\I = \lfloor \theta_3 \, \N \rfloor}^{\N} C_\I \right) 
& = & \frac{\PP \left(A_\N^c \cap \left(\cap_{\I = \lfloor \theta_3 \, \N \rfloor}^\N C_\I \right)\right)}{\PP \left( \bigcap_{\I = \lfloor \theta_3 \, \N \rfloor}^\N C_\I \right)} \\
& \ge & \PP \left( (A_\N^c \cup A_\N'^c) \cap \big(\cap_{\I = \lfloor \theta_3 \, \N \rfloor}^\N C_\I\big)\right) \\
& \ge & \PP(A_\N^c \cup A_\N'^c) - \PP \big((\cap_{\I = \lfloor \theta_3 \, \N \rfloor}^\N C_\I)^c\big),
\end{eqnarray*}
where the last inequality follows from the union bound. Rearranging and using the union bound, 
$$
\PP(A_\N^c \cup A_\N'^c)  
\,\,\le\,\, 
\PP \big(A_\N^c \,\big\vert\, \cap_{\I = \lfloor \theta_3 \, \N \rfloor}^{\N} C_\I \big) + \PP \big(A_\N^c \,\big\vert\, \cap_{\I = \lfloor \theta_3 \, \N \rfloor}^{\N} C_\I \big) +  \PP \big(\big(\cap_{\I = \lfloor \theta_3 \, \N \rfloor}^\N C_\I\big)^c\big).
$$
Summing both sides,
$$
\sum_{\N = 1}^\infty \PP(A_\N^c \cup A_\N'^c) 
\,\,\le\,\,
\sum_{\N = 1}^\infty \PP \left( A_\N^c  \,\big\vert\, \bigcap\nolimits_{\I = \lfloor \theta_3 \, \N \rfloor}^{\N} C_\I \right)
+
\sum_{\N = 1}^\infty \PP \left( A_\N'^c \,\big\vert\, \bigcap\nolimits_{\I = \lfloor \theta_3 \, \N \rfloor}^{\N} C_\I \right) 
+
\sum_{\N = 1}^\infty \PP \left(\left(\bigcap\nolimits_{\I = \lfloor \theta_3 \, \N \rfloor}^\N C_\I\right)^c\right),
$$
where the right hand side is finite by Lemmas~\ref{lemma:rrg:at_least_one_vertex}, \ref{lemma:k_rrg:a_n_prime_bound}, and \ref{lemma:k_rrg:c_n_bound}, by picking $\theta_3$ close to one. Hence, $\sum_{\N = 1}^\infty \PP(A_\N^c \cup A_\N'^c) < \infty$. Then, by the Borel-Cantelli lemma, $\PP (\limsup_{\N \to \infty} (A_\N^c \cup A_\N'^c)) = 0$, or equivalently $\PP (\liminf_{\N \to \infty} (A_\N \cap A_\N')) = 1$.
\qed
\end{proof}
\subsection{Convergence to the optimal path}

The proof of the following two lemmas are essentially the same as that of Lemma~\ref{lemma:prmstar:convergence_in_bvnorm}, and is omitted here.
Let $P_\N$ denote the set of all paths in the graph returned by $\AlgkRRG$ algorithm at the end of $\N$ iterations.
Let $\sigma_\N'$ be the path that is closest to $\sigma_\N$ in terms of the bounded variation norm among all those paths in $P_\N$, i.e., 
$
\sigma_\N' := \min_{\sigma' \in P_\N} \Vert \sigma' - \sigma_\N \Vert.
$
\begin{lemma} 
The random variable $\Vert \sigma_\N' - \sigma_\N \Vert_\BVnorm$ converges to zero almost surely, i.e., 
$$
\PP \left( \left\{ \lim\nolimits_{\N \to \infty} \Vert \sigma_\N' - \sigma_\N \Vert_\BVnorm = 0 \right\}\right) = 1.
$$
\end{lemma}

A corollary of the lemma above is that $\lim_{\N \to \infty} \sigma_\N' = \sigma^*$ with probability one. Then, the result follows by the robustness of the optimal solution (see the proof of Lemma~\ref{lemma:prmstar:cost_convergence} for details).

\section{Proof of Theorem~\ref{theorem:optimality_rrtstar}
(Asymptotic optimality of RRT$^*$)} \label{proof:optimality_rrtstar}

For simplicity, the proof will assume the steering parameter $\eta$ to be large enough, i.e., 
$\eta \ge \mathrm{diam}(\X)$, although the results hold for any $\eta>0$.

\subsection{Marked point process} \label{section:proof:rrtstar:marked_process}

Consider the following marked point process. Let $\{\Z_1, \Z_2, \dots, \Z_\N\}$ be a independent uniformly distributed points drawn from $X_\mathrm{free}$ and let $\{Y_1, Y_2, \dots, Y_\N\}$ be independent uniform random variables with support $[0,1]$. Each point $\Z_\I$ is associated with a mark $Y_\I$ that describes the order of $\Z_\I$ in the process. More precisely, a point $X_\I$ is assumed to be drawn after another point $\Z_{\I'}$ if $Y_{\I'} < Y_{\I}$. We will also assume that the point process includes the point $x_\mathrm{init}$ with mark $Y = 0$.

Consider the graph formed by adding an edge $(\Z_{\I'}, \Z_\I)$, whenever (i) $Y_{\I'} < Y_{\I}$ and (ii) $\Vert \Z_\I - \Z_{\I'} \Vert \le r_\N$ both hold. Notice that, formed in this way, $G_\N$ includes no directed cycles. Denote this graph by $G_\N = (V_\N, E_\N)$. Also, consider a subgraph $G'_\N$ of $G_\N$ formed as follows. Let $c(\Z_\I)$ denote the cost of best path starting from $x_\mathrm{init}$ and reaching $\Z_\I$. In $G_\N'$, each vertex $\Z_{\I}$ has a single parent $\Z_{\I}$ with the smallest cost $c(\Z_\I)$. Since the graph is built incrementally, the cost of the best path reaching $\Z_\I$ will be the same as the one reaching $\Z_{\I'}$ in both $G_\N$ and $G_\N'$. Clearly, $G_\N'$ is equivalent to the graph returned by the RRT$^*$ algorithm at the end of $\N$ iterations, if the steering parameter $\eta$ is large enough. 

Let $Y_\N$ and the $Y_\N'$ denote the costs of the best paths starting from $x_\mathrm{init}$ and reaching the goal region in $G_\N$ and $G_\N'$, respectively. Then, $\limsup_{\N \to \infty} Y_\N = \limsup_{\N \to \infty} Y_\N'$ surely. In the rest of the proof, it is  shown that $\PP(\{\limsup_{\N \to \infty} Y_\N\}) = 1$, which implies that $\PP (\{\limsup_{\N \to \infty} Y_\N'\}) = 1$, which in turn implies the result.

\subsection{Definitions of $\{\sigma_\N\}_{\N \in \naturals}$ and $\{B_\N\}_{\N \in \naturals}$}
Let $\sigma^*$ denote an optimal path. Define 
$$
\delta_\N := \min\{ \delta, 4 \, r_\N \},
$$
where $r_\N$ is the connection radius of the RRT$^*$ algorithm.
Let $\{\sigma_\N\}_{\N \in \naturals}$ be the sequence paths, the existence of which is guaranteed by Lemma~\ref{lemma:weak_delta_clearance}.

For each $\N \in \naturals$, construct a sequence $\{B_\N\}_{\N \in \naturals}$ of balls that cover $\sigma_\N$ as $B_\N = \{B_{\N,1}, B_{\N,2}, \dots, B_{\N,\MM_\N} \} := {\tt CoveringBalls}(\sigma_\N, r_\N, 2\, r_\N)$ (see Definition~\ref{definition:covering_balls}), where $r_\N$ is the connection radius of the RRT$^*$ algorithm, i.e., $r_\N = \gamma_{\mathrm{RRT}^*} \left(\frac{\log n}{n}\right)^{1/d}$. Clearly, the balls in $B_{\N}$ are openly disjoint, since the spacing between any two consecutive balls is $2 \,r_\N$.

\subsection{Connecting the vertices in subsequent balls in $B_\N$}

For all $\M \in \{1,2,\dots, \MM_\N\}$, let $A_{\N,\M}$ denote the event that there exists two vertices $\Z_\I, \Z_{\I'} \in V^\AlgRRTstar_\N$  such that $\Z_\I \in B_{\N, \M}, \, \Z_{\I'} \in B_{\N,\M+1}$ and $Y_{\I'} \le Y_{\I}$, where $Y_\I$ and $Y_{\I'}$ are the marks associated with points $\Z_\I$ and $\Z_{\I'}$, respectively. Notice that, in this case, $\Z_\I$ and $\Z_{\I'}$ will be connected with an edge in $G_\N$. Let $A_\N$ denote the event that $A_{\N,\M}$ holds for all $\M \in \{1,2,\dots, \MM\}$, i.e., $A_\N = \bigcap_{\M =1}^{\MM} A_{\N,\M}$.

\begin{lemma} \label{lemma:rrtstar:balls_in_vertices}
If $\gamma_\AlgRRTstar > 4 \, \left(\frac{\mu(\X_\mathrm{free})}{\VolumeDBall{d}}\right)^{1/d}$, then $A_\N$ occurs for all large $\N$, with probability one, i.e., 
$$
\PP \left(\liminf_{\N \to \infty} A_\N\right) = 1.
$$
\end{lemma}

\begin{proof}
The proof of this result is based on a Poissonization argument. Let $\Poisson(\lambda)$ be a Poisson random variable with parameter $\lambda = \theta \, \N$, where $\theta \in (0,1)$ is a constant independent of $\N$. Consider the point process that consists of exactly $\Poisson(\theta \, \N)$ points, i.e., $\{\Z_1, \Z_2, \dots, \Z_{\Poisson(\theta\,\N)}\}$. This point process is a Poisson point process with intensity $\theta \, \N\, / \mu(X_\mathrm{free})$ by Lemma~\ref{lemma:poissonization}. 

Let $\tilde{A}_{\N,\M}$ denote the event that there exists two vertices $\Z_\I$ and $\Z_{\I'}$ in the vertex set of the RRT$^*$ algorithm such that $\Z_{\I}$ and $\Z_{\I'}$ are connected with an edge in $\tilde{G}_\N$, where $\tilde{G}_\N$ is the graph returned by the RRT$^*$ when the algorithm is run for $\Poisson(\theta \, \N)$ many iterations, i.e., $\Poisson(\theta \, \N)$ samples are drawn from $\X_\mathrm{free}$. 

Clearly, 
$
\PP (A_{\N,\M}^c) = \PP (\tilde{A}_{\N,\M}^c \given \{\Poisson(\theta \, \N) = \N\}).
$
Moreover,
$$
\PP (A_{\N,\M}^c) 
\,\,\le\,\, 
\PP (\tilde{A}_{\N,\M}^c) + \PP(\{\Poisson(\theta\, \N) > \N \}).
$$
since $\PP(A_{\N,\M}^c)$ is non-increasing with $\N$~\cite[see, e.g.,][]{penrose.book03}. 
Since $\theta < 1$,
$
\PP(\{\Poisson(\theta\, \N) > \N \}) 
\,\,\le\,\, 
e^{- a \, \N},
$
where $a > 0$ is a constant independent of $\N$.

To compute $\PP (\widetilde{A}_{\N,\M}^c)$, a number of definitions are provided. Let $N_{\N,\M}$ denote the number of vertices that lie in the interior of $B_{\N,\M}$. Clearly, $\EE[N_{\N,\M}] = \frac{\VolumeDBall{d} \,\gamma_{\AlgRRTstar}^d}{\mu(X_\mathrm{free})} \, \log \N$, for all $\M \in \{1,2,\dots, \MM_\N\}$. For notational simplicity, define $\alpha := \frac{\VolumeDBall{d} \,\gamma_{\AlgRRTstar}^d}{\mu(X_\mathrm{free})}$. Let $\epsilon \in (0,1)$ be a constant independent of $\N$. Define the event 
\begin{eqnarray*}
C_{\N,\M,\epsilon} & := & \left\{ N_{\N,\M} \ge (1- \epsilon) \,\EE[N_{\N,\M}]  \right\}
= \left\{ N_{\N,\M} \ge (1- \epsilon) \,\alpha \,\log \N\right\}
\end{eqnarray*}
Since $N_{\N,\M,\epsilon}$ is binomially distributed, its large deviations from its mean can be bounded as follows~\citep{penrose.book03},
$$
\PP\left( C_{\N, \M, \epsilon}^c \right) = \PP(\{N_{\N,\M,\epsilon} \le (1- \epsilon) \,\EE[N_{\N,\M}] \})
\le 
e^{-\alpha \,H(\epsilon)\, \log \N} = \N^{-\alpha H(\epsilon)},
$$
where $H(\epsilon) = \epsilon + (1-\epsilon) \log (1-\epsilon)$. Notice that $H(\epsilon)$ is a continuous function of $\epsilon$ with $H(0) = 0$ and $H(1) = 1$. Hence, $H(\epsilon)$ can be made arbitrary close to one by taking $\epsilon$ close to one.

Then, 
\begin{eqnarray*}
\PP (\tilde{A}_{\N,\M}^c) 
& = & 
\PP(\tilde{A}_{\N,\M}^c \,\vert\, C_{\N,\M, \epsilon}\cap C_{\N,\M+1,\epsilon}) \,\PP (C_{\N,\M,\epsilon} \cap C_{\N, \M + 1,\epsilon}) \\
& &+ \PP(\tilde{A}_{\N,\M}^c \,\vert\, (C_{\N,\M, \epsilon} \cap C_{\N, \M+1,\epsilon})^c)\, \PP ((C_{\N,\M,\epsilon}\cap C_{\N, \M+1,\epsilon})^c) \\
& \le & \PP(\tilde{A}_{\N,\M}^c \,\vert\, C_{\N,\M, \epsilon}\cap C_{\N,\M+1,\epsilon}) \,\PP (C_{\N,\M,\epsilon} \cap C_{\N, \M + 1,\epsilon}) + \PP (C_{\N,\M,\epsilon}^c) + \PP(C_{\N, \M+1,\epsilon}^c),
\end{eqnarray*}
where the last inequality follows from the union bound.

First, using the spatial independence of the underlying point process, 
\begin{eqnarray*}
\PP \left(C_{\N, \M,\epsilon} \cap C_{\N,\M+1,\epsilon} \right) = \PP \left( C_{\N,\M,\epsilon} \right) \, \PP \left( C_{\N,\M+1,\epsilon} \right) \le n^{-2\,\alpha\,H(\epsilon)}.
\end{eqnarray*}

Second, observe that $\PP (A_{\N,\M}^c \,\vert\, N_{\N,\M} = k, N_{\N,\M+1} = k')$ is a non-increasing function of both $k$ and $k'$, since the probability of the event $\tilde{A}_{\N,\M}$ can not increase with the increasing number of points in both balls, $B_{\N,\M}$ and $B_{\N,\M+1}$. 
Then, 
\begin{eqnarray*}
\PP(\tilde{A}_{\N,\M}^c \,\vert\, C_{\N,\M, \epsilon}\cap C_{\N,\M+1,\epsilon}) 
& = &
\PP(\tilde{A}_{\N,\M}^c \,\vert\, \{ N_{\N,\M} \ge (1 - \epsilon) \, \alpha \, \log N_{\N,\M} , N_{\N,\M+1} \ge (1 - \epsilon) \, \alpha \, \log N_{\N,\M+1}  \})  \\
& \le &
\PP(\tilde{A}_{\N,\M}^c \,\vert\, \{ N_{\N,\M} = (1 - \epsilon) \, \alpha \, \log N_{\N,\M} , N_{\N,\M+1} = (1 - \epsilon) \, \alpha \, \log N_{\N,\M+1}  \})
\end{eqnarray*}

The term on the right hand side is one minus the probability that the maximum of $\alpha \, \log\N$ number of uniform samples drawn from $[0,1]$ is smaller than the minimum of $\alpha \, \log\N$ number of samples again drawn from $[0,1]$, where all the samples are drawn independently. This probability can be calculated as follows. From the order statistics of uniform distribution, the minimum of $\alpha \, \log\N$ points sampled independently and uniformly from $[0,1]$ has the following probability distribution function:
$$
f_\mathrm{min} (x) = \frac{(1 - x)^{\alpha\, \log \N - 1}}{\BetaFunction{1}{\alpha \, \log(n)}},
$$
where $\BetaFunction{\cdot}{\cdot}$ is the Beta function (also called the Euler integral)~\citep{abramowitz.stegun.book64}. The maximum of the same number of independent uniformly distributed random variables with support $[0,1]$ has the following cumulative distribution function:
$$
F_\mathrm{max} (x) = x^{\alpha \log \N}
$$
Then, 
\begin{eqnarray*}
\PP(\tilde{A}_{\N,\M}^c \,\vert\, C_{\N,\M, \epsilon}\cap C_{\N,\M+1,\epsilon}) 
& \le &
\int_{0}^{1} F_\mathrm{max} (x) \,f_\mathrm{min} (x) \, d x \\
& = & 
 \frac{\GammaFunction{(1-\epsilon) \, \alpha\, \log \N} \, \GammaFunction{(1-\epsilon) \,\epsilon \, \log\N}}{2 \, \GammaFunction{2 (1-\epsilon)\, \alpha\, \log (\N)}} \\
& \le &  
 \frac{((1-\epsilon) \, \alpha\,\log \N)! \, ((1-\epsilon) \, \alpha\,\log \N)!}{2 \, (2 \,(1-\epsilon) \, \alpha\, \log \N)!}  \\
& = & 
\frac{((1-\epsilon) \, \alpha\, \log \N)!}{2 (2 (1-\epsilon) \, \alpha\, \log \N) (2 (1-\epsilon) \, \alpha\,\log \N - 1) \cdots 1} \\
& \le &  
 \frac{1}{2^{(1-\epsilon) \, \alpha\,\log \N}} = \N^{-\, \log (2) \, (1-\epsilon) \, \alpha\,},
\end{eqnarray*}
where $\GammaFunction{\cdot}$ is the gamma function~\citep{abramowitz.stegun.book64}.

Then, 
\begin{eqnarray*}
\PP (\tilde{A}_{\N,\M}^c) \le n^{- \alpha \big(2\,H(\epsilon) + \log(2)\,(1 - \epsilon) \big)} + 2\, n^{- \alpha \, H(\epsilon)}.
\end{eqnarray*}
Since $2\,H(\epsilon) + \log(2)\,(1 - \epsilon)$ and $H(\epsilon)$ are both continuous and increasing in the interval $(0.5,1)$, the former is equal to $2 - \log(4) > 0.5$ and the latter is equal to $1$ as $\epsilon$ approaches one from below, there exists some $\bar{\epsilon} \in (0.5,1)$ such that both $2\,H(\bar\epsilon) + \log(2)\,(1 - \bar\epsilon) > 0.5$ and $H(\bar\epsilon) > 0.5$. Thus, 
\begin{eqnarray*}
\PP (\tilde{A}_{\N,\M}^c) \le n^{- \alpha/2} + 2\, n^{- \alpha/2} = 3 \, n^{- \alpha/2}.
\end{eqnarray*}

Hence, 
\begin{eqnarray*}
\PP (A_{\N,\M}^c) 
& \le & 
\PP(\tilde{A}_{\N, \M}^c) + \PP(\Poisson(\theta \, \N) > \N) \\
& \le & 
3\, n^{-\alpha/2} + e^{-a\,\N}
\end{eqnarray*}

Recall that $A_\N$ denotes the event that $A_{\N,\M}$ holds for all $\M \in \{1,2,\dots, \MM_\N\}$. Then, 
$$
\PP(A_\N^c) 
\,\,=\,\, 
\PP\left(\left(\bigcap\nolimits_{\M = 1}^{\MM_\N} A_{\N,\M}\right)^c \right) 
\,\,=\,\,
\PP\left(\bigcup\nolimits_{\M = 1}^{\MM_\N} A_{\N,\M}^c \right) 
\,\,\le\,\, 
\sum_{\M = 1}^{\MM_\N} \, \PP\left(A_{\N,\M}^c \right) 
\,\,=\,\,
\MM_\N \, \PP(A_{\N,1}^c),
$$
where the last inequality follows from the union bound. The number of balls in $B_\N$ can be bounded as
$$
| B_\N | 
\,\,=\,\, 
\MM_\N 
\,\,\le\,\,
\beta \, \left(\frac{\N}{\log \N}\right)^{1/d},
$$
where $\beta$ is a constant. Combining this with the inequality above, 
$$
\PP(A_\N^c) 
\,\,\le\,\, 
\beta \left( \frac{n}{\log \N} \right)^{1/d}  \, \left( 3\,n^{-\alpha/2} + e^{-a\,\N}\right),
$$
which is summable for $\alpha > 2 \, (1 + 1/d)$. Thus, by the Borel-Cantelli lemma, the probability that $A_\N^c$ occurs infinitely often is zero, i.e., $\PP (\limsup_{\N \to \infty} A_{\N}^c) = 0$, which implies that $A_\N$ occurs for all large $\N$ with probability one, i.e., $\PP (\liminf_{\N \to \infty} A_{\N}) = 1$.
\qed
\end{proof} 

\subsection{Convergence to the optimal path}

The proof of the following lemma is similar to that of Lemma~\ref{lemma:prmstar:convergence_in_bvnorm}, and is omitted here.

Let $P_\N$ denote the set of all paths in the graph returned by $\AlgRRTstar$ algorithm at the end of $\N$ iterations.
Let $\sigma_\N'$ be the path that is closest to $\sigma_\N$ in terms of the bounded variation norm among all those paths in $P_\N$, i.e., 
$
\sigma_\N' := \min_{\sigma' \in P_\N} \Vert \sigma' - \sigma_\N \Vert.
$
\begin{lemma} 
The random variable $\Vert \sigma_\N' - \sigma_\N \Vert_\BVnorm$ converges to zero almost surely, i.e., 
$$
\PP \left( \left\{ \lim\nolimits_{\N \to \infty} \Vert \sigma_\N' - \sigma_\N \Vert_\BVnorm = 0 \right\}\right) = 1.
$$
\end{lemma}

A corollary of the lemma above is that $\lim_{\N \to \infty} \sigma_\N' = \sigma^*$ with probability one. Then, the result follows by the robustness of the optimal solution (see the proof of Lemma~\ref{lemma:prmstar:cost_convergence} for details).

\end{document}